%% file: main.tex
\documentclass{article}

\PassOptionsToPackage{numbers, sort&compress}{natbib}

\usepackage[main, final]{neurips_2026}

\usepackage[utf8]{inputenc} 
\usepackage[T1]{fontenc}    
\usepackage{hyperref}       
\usepackage{url}            
\usepackage{booktabs}       
\usepackage{amsfonts}       
\usepackage{nicefrac}       
\usepackage{microtype}      
\usepackage{xcolor}         

\usepackage{graphicx}
\usepackage{subcaption}
\usepackage{wrapfig}

\usepackage{soul}
\usepackage{mathtools}
\usepackage{multirow}
\usepackage{float}
\usepackage{bm}
\usepackage{bbm}
\usepackage{enumitem}

\usepackage{hyperref}

\usepackage{amsmath}
\usepackage{amssymb}
\usepackage{mathtools}
\usepackage{amsthm}

\usepackage{flafter}
\usepackage{etoc}

\input{math_commands.tex}

\input{notations}

\usepackage{thmtools, thm-restate}

\usepackage[capitalize,noabbrev]{cleveref}

\usepackage{pifont}

\hypersetup{
	colorlinks=true,
	linkcolor=myred,
	citecolor=myred,
	urlcolor=mygray,
}

\definecolor{myblue}{RGB}{0,0,150}
\definecolor{myred}{RGB}{120,50,50}
\definecolor{mygray}{RGB}{90,90,110}

\crefname{definition}{def.}{defs.}
\Crefname{definition}{Def.}{Defs.}
\crefname{proposition}{prop.}{props.}
\Crefname{proposition}{Prop.}{Props.}

\theoremstyle{plain}
\newtheorem{theorem}{Theorem}[section]
\newtheorem{proposition}[theorem]{Proposition}

\theoremstyle{definition}
\newtheorem{definition}[theorem]{Definition}

\theoremstyle{remark}

\crefmultiformat{figure}{Figures~#2#1#3}{, #2#1#3}{, #2#1#3}{, #2#1#3}
\Crefmultiformat{figure}{Figures~#2#1#3}{, #2#1#3}{, #2#1#3}{, #2#1#3}

\crefmultiformat{table}{Tables~#2#1#3}{, #2#1#3}{, #2#1#3}{, #2#1#3}
\Crefmultiformat{table}{Tables~#2#1#3}{, #2#1#3}{, #2#1#3}{, #2#1#3}

\usepackage{comment}
\usepackage[textsize=tiny]{todonotes}

\title{Mechanistic Interpretability with Sparse Autoencoder Neural Operators}

\author{%
  \vspace{4pt}
  Bahareh Tolooshams$^{1,2,3}$\thanks{Equal contribution.}
  \quad
  Ailsa Shen$^{4*}$
  \quad
  Anima Anandkumar$^4$\\
  $^1$University of Alberta
  \quad
  $^2$Alberta Machine Intelligence Institute (Amii)\\
  $^3$Canada CIFAR AI Chair
  \quad 
  $^4$California Institute of Technology (Caltech)
}

\begin{document}

\maketitle

\begin{abstract}
We introduce sparse autoencoder neural operators (SAE-NOs), a new class of sparse autoencoders that operate in function spaces rather than fixed-dimensional Euclidean representations. We formalize the functional representation hypothesis, where data are explained through sparse compositions of structured functions. Unlike standard SAEs that represent concepts with scalar activations, SAE-NOs parameterize concepts as functions, enabling representations that capture not only a concept's presence, but also how and where it is expressed across the input domain. We achieve this through joint sparsity: concept sparsity selects active concepts, while domain sparsity governs where they are expressed. We instantiate this framework using Fourier neural operators (SAE-FNOs), parameterizing concepts as integral operators in the Fourier domain. This functional and spectral parameterization is particularly advantageous when data exhibit spatial structure across scales or when concepts are frequency-structured. We characterize SAE-FNO on vision data and demonstrate that it learns localized patterns, uses concepts more efficiently, and exhibits stable concept characteristics across sparsity levels. We further show that SAE-FNO adapts to changes in domain size and generalizes across discretizations, operating at resolutions beyond those seen during training, where standard SAEs fail. We also introduce lifting into SAEs and show theoretically and empirically that it acts as a preconditioner that accelerates optimization. Overall, our results show that moving from vector-valued to functional parameterizations, with concept and domain sparsity, extends SAEs from representing concept presence to modeling structured concept expression, highlighting the importance of parameterization.
\end{abstract}

\section{Introduction}\label{sec:intro}
\vspace{-1mm}
Sparse representations provide a principled framework for uncovering parsimonious latent structures in machine learning~\citep{bengio2013representation}. Building on this foundation, sparse autoencoders (SAEs)~\citep{makhzani2014ksae} enable the learning of structured and interpretable representations~\citep{huben2023sparse} and provide a framework for understanding internal model capabilities~\citep{doshi2017towards, kim2018interpretability}. They are widely used for mechanistic interpretability~\citep{elhage2022toy, bricken2023towards, nanda2023progress, rajamanoharan2024jumping, park2024linear, templeton2024scaling, lieberum2024gemma, gao2025scaling, fel2025archetypal, marks2025sparse, karvonen2025saebench, costa2025flat} to explain model behavior by identifying the internal concepts that drive computation in large neural networks. To extract these concepts, most existing SAEs rely on the linear representation hypothesis (LRH)~\citep{elhage2022toy, olah2023distributed, zheng2026model}, which models data as sparse combinations of linearly accessible and approximately orthogonal concepts. While this perspective has been supported empirically~\citep{fel2023holistic, merullo2025linear}, recent work suggests that neural representations may exhibit richer organization, including hierarchical concepts~\citep{costa2025flat} and additive mixtures of manifolds~\citep{2026manifold}.

Despite these developments, the dominant SAE formulation (which we refer to as SAE-MLP) remains fundamentally tied to shallow MLP encoders and linear decoders~\citep{bricken2023towards, bussmann2024batchtopk, gao2025scaling, rajamanoharan2024jumping}. Consequently, most advances have focused on encoder nonlinearities~\citep{hindupur2025projecting} or training procedures~\citep{bussmann2025learning}.

\begin{figure}[t]
    \centering 
    \includegraphics[width=1.00\linewidth]{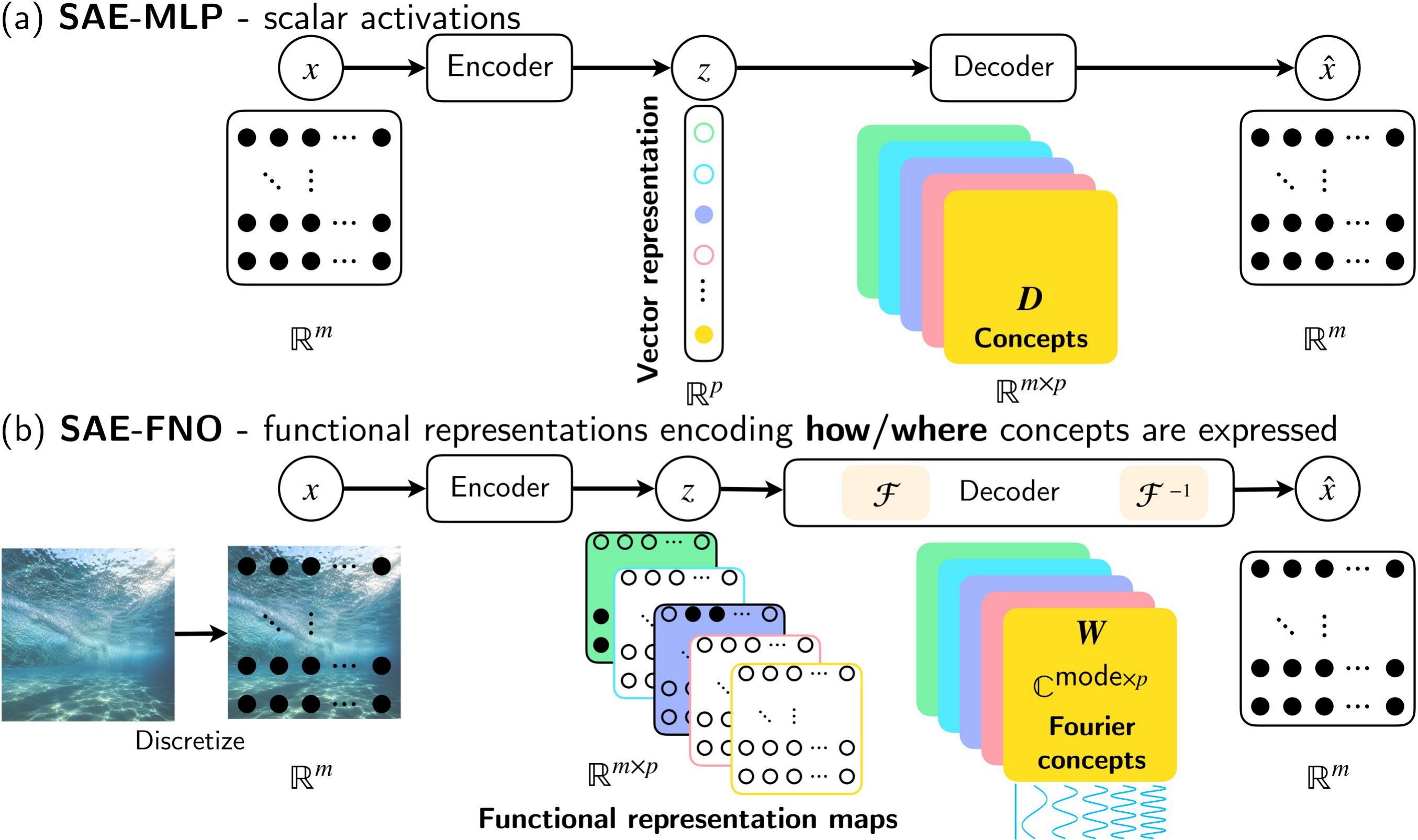}
    \vspace{-2mm}
    \caption{\textbf{Sparse Autoencoder Neural Operator (SAE-NO).} (a) Conventional SAEs (SAE-MLP) parameterize concepts $\D$ in the Euclidean domain, assigning each concept a scalar activation indicating presence and strength. (b) SAE-FNOs use functional concept parameterization via Fourier integral operators, encoding \emph{concept sparsity} and \emph{domain sparsity} to capture how and where concepts are expressed. SAE-FNO treats inputs as samples from an underlying continuous function and supports querying of representations, reconstructions, and concepts at varying resolutions.}
    \label{fig:fig1}
    \vspace{-7.5mm}
\end{figure}
While SAE-MLPs have proven effective for sparse model recovery~\citep{olshausen1996emergence}, identifying monosemantic features~\citep{bricken2023towards}, and capturing geometry of neural representations~\citep{fel2026into, karkada2026symmetry}, their design imposes a restrictive inductive bias on what a concept can represent: i) concept utilization is summarized by a scalar activation, encoding only presence and strength, ii) concepts are parameterized as vectors in the same domain as the inputs, and iii) learning and inference occur in Euclidean spaces at a fixed discretization. Consequently, codes encode whether concepts are present, but not \emph{how} or \emph{where} they are expressed across the domain. For spatially structured data, this forces concepts into globally uniform contributions, limiting their ability to represent localized or domain-dependent structure.

These limitations are particularly critical in scientific modeling~\citep{kovachki2023neural}, computer vision, and the analysis of physical signals, where data are structured (e.g., images, fields, spatiotemporal processes)~\citep{mallat1999wavelet, krizhevsky2012imagenet}. While convolutional SAEs (SAE-CNNs) introduce localized concepts, they remain tied to a fixed resolution with pre-defined receptive fields, inheriting inductive biases imposed by parameterization in the input domain. Furthermore, while recent works~\citep{bhalla2026temporal, lubana2026priors} extend SAEs to sequential settings via temporal smoothness~\citep{klindt2021towards} or hybrid dense–sparse representations~\citep{tasissa2024discriminative}, these approaches model dynamics. Our focus, instead, lies in capturing continuous structures within the representation itself.

To overcome the restrictions of Euclidean, fixed-discretization formulations of SAEs, we turn to neural operators (NOs)~\citep{lu2019deeponet,li2020neural,li2021fourier,kovachki2023neural}, which learn mappings between infinite-dimensional function spaces. Neural operators are central to data-driven scientific modeling~\citep{azizzadenesheli2024neural}, particularly for solving PDEs, due to their ability to generalize across discretizations. They have also shown promise in neuroscience for modeling brain activity~\citep{ghafourpour2025noble}. Yet, while their predictive capabilities~\citep{li2024geometry, jatyani2025unifying} and recent extension to function-space autoencoders~\citep{bunker2024autoencoders} are well established, their representational properties and potential for concept learning and mechanistic interpretability remain largely unexplored.

%
\textbf{Our contributions}\quad We introduce sparse autoencoder neural operators (SAE-NOs), a new class of sparse autoencoders where concepts are parameterized as functions, offering resolution-invariant representations. We formalize the functional representation hypothesis (FRH) and instantiate it as SAE Fourier neural operators (SAE-FNOs, \url{https://github.com/Anima-Lab/SAE-FNO}).

\begin{figure}[t]
    \centering
    \begin{subfigure}{0.98\textwidth}
        \centering
        \includegraphics[width=\linewidth]{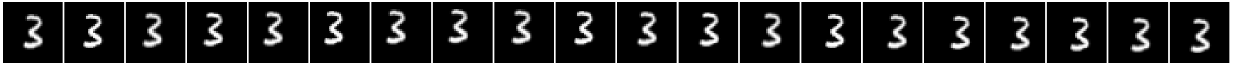}
        \vspace{-4.8mm}
        \caption{Sequence of image frames in time.}
    \end{subfigure}
    \begin{subfigure}{0.98\textwidth}
        \centering
        \includegraphics[width=\linewidth]{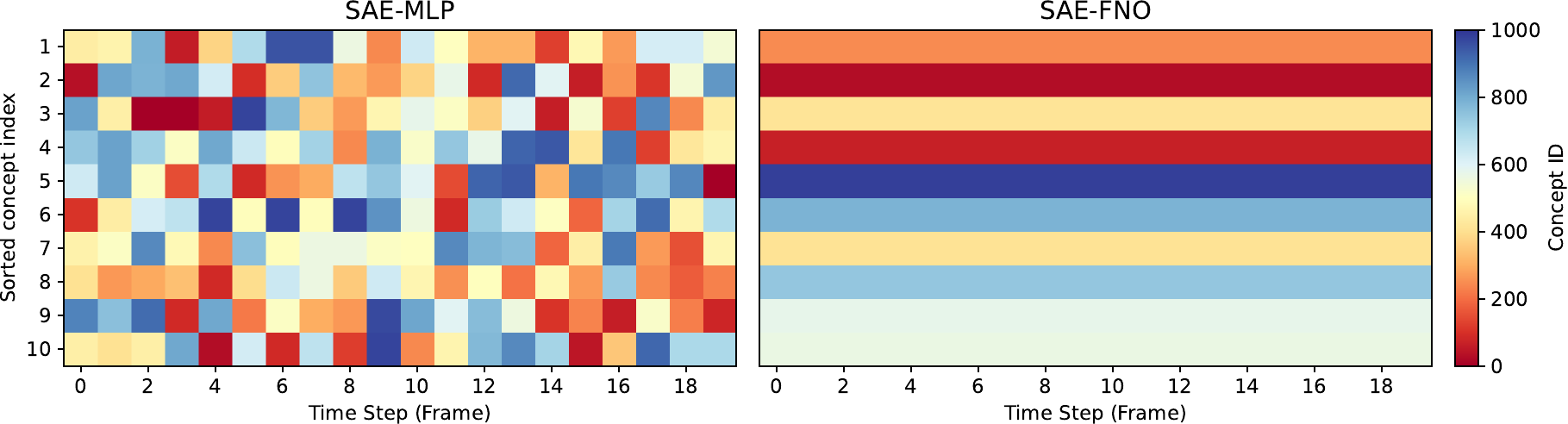}
        \caption{Top concept activations over time.}
    \end{subfigure}
    \begin{subfigure}{0.55\textwidth}
        \centering
        \includegraphics[width=0.68\linewidth]{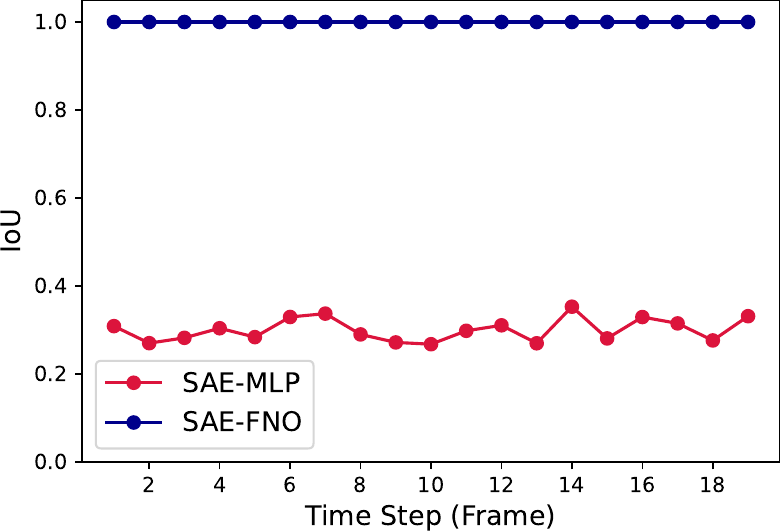}
        \caption{IoU of the active concept sets between consecutive frames.}
    \end{subfigure}
    \hfill
    \begin{subfigure}{0.35\textwidth}
        \centering
        \includegraphics[width=0.97\linewidth]{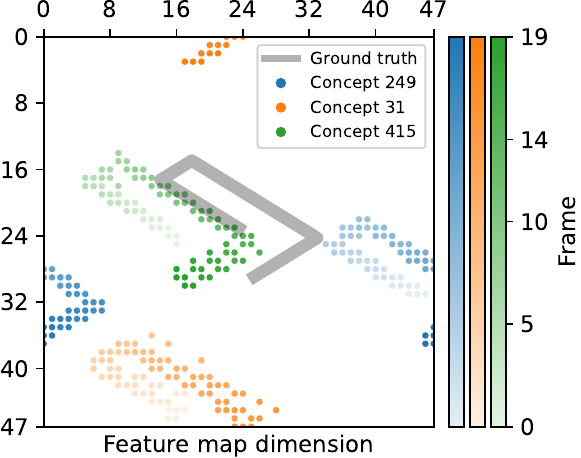}
        \caption{Code trajectories in domain.}
    \end{subfigure}
        \begin{subfigure}{1.00\linewidth}
        \centering
        \includegraphics[width=0.32\linewidth]{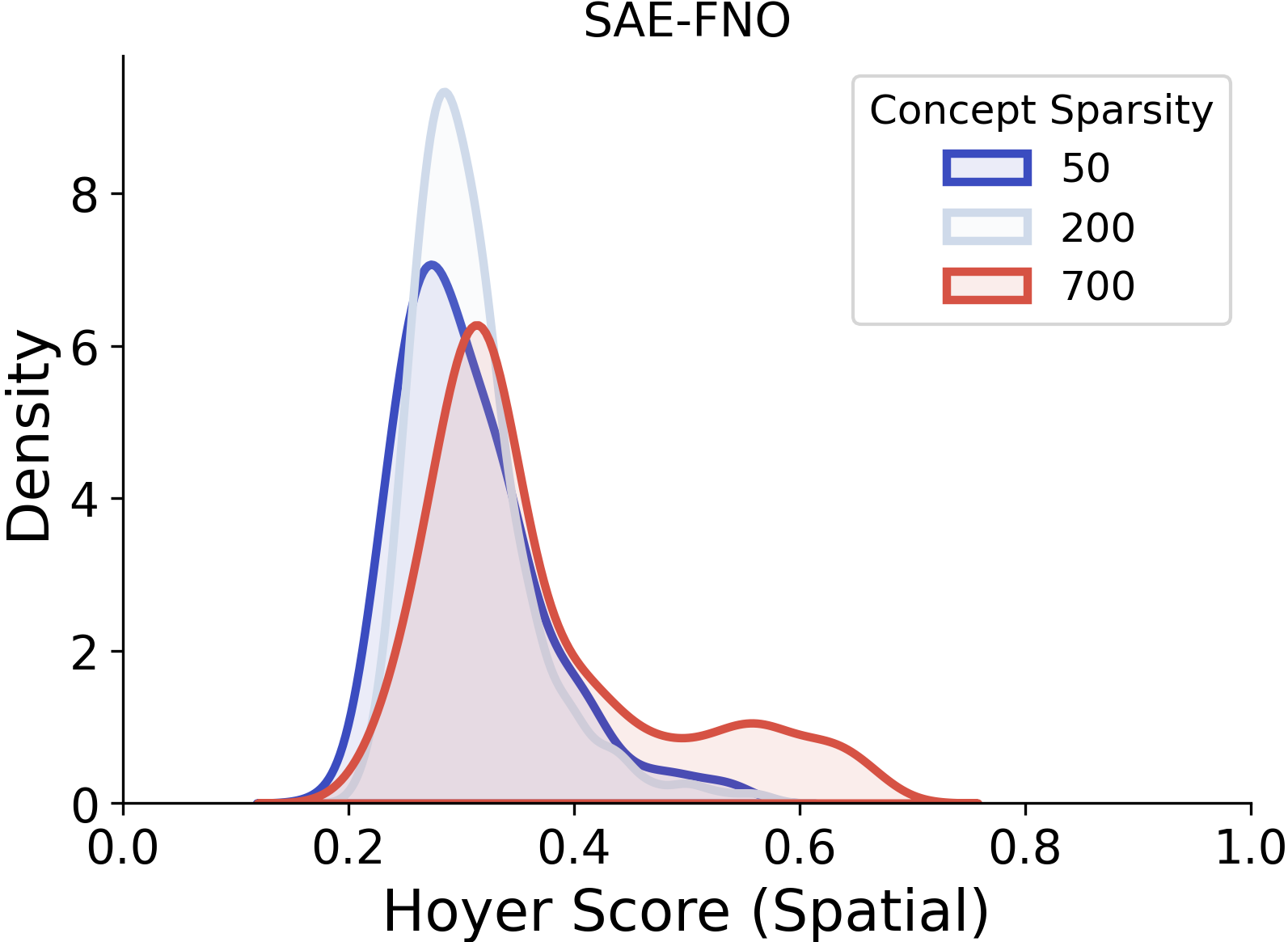}
        \includegraphics[width=0.32\linewidth]{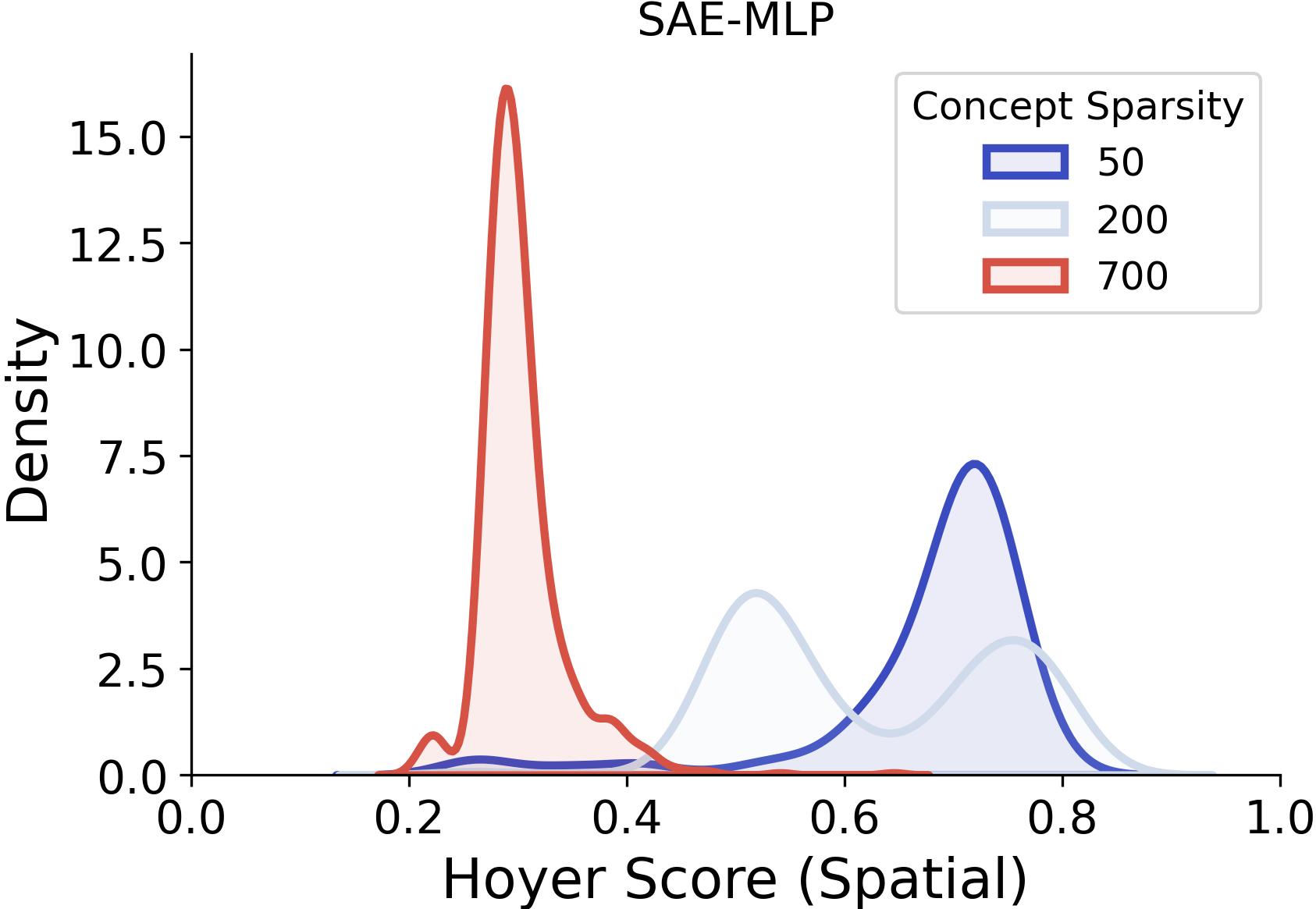}
        \includegraphics[width=0.32\linewidth]{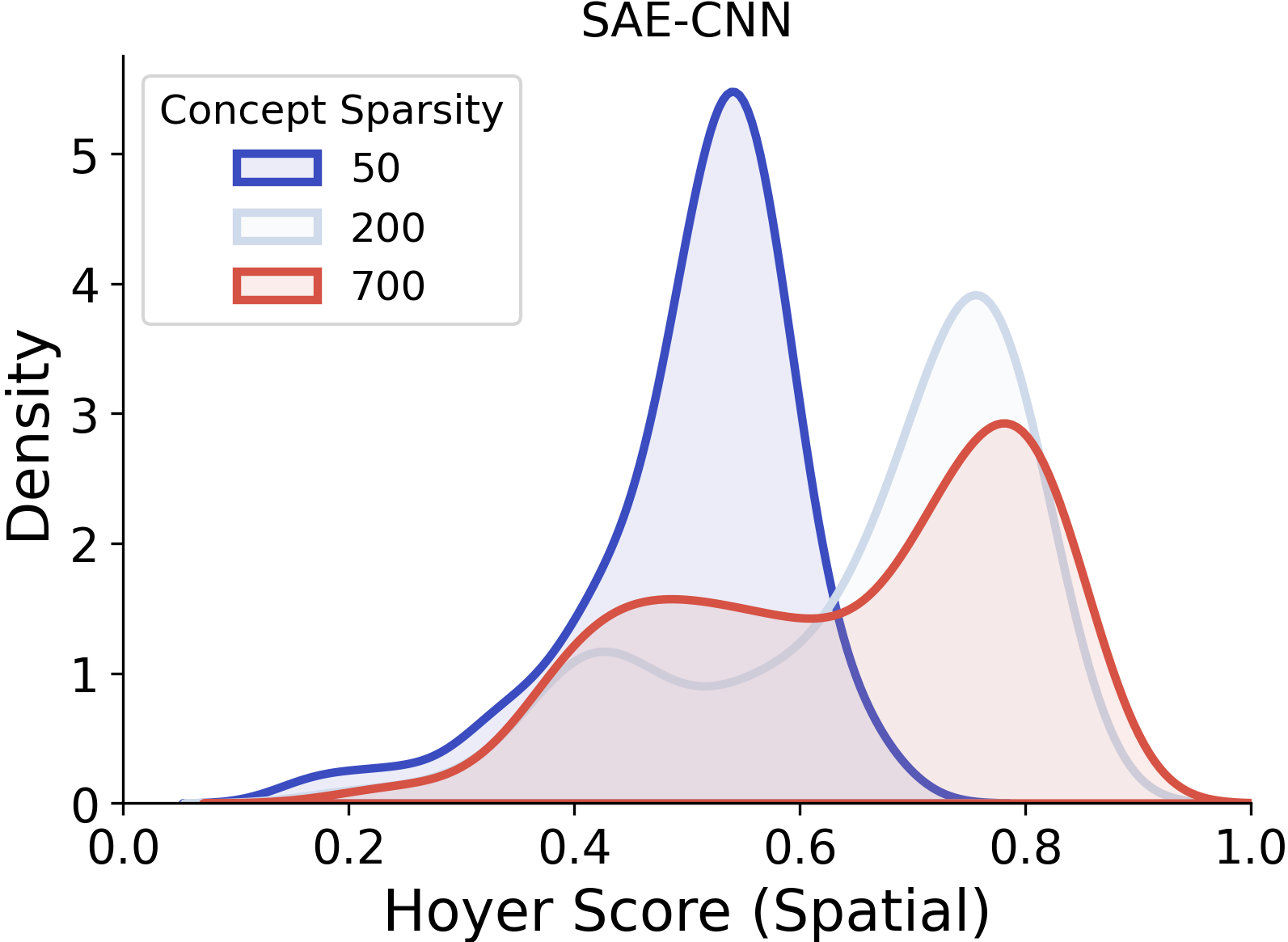}
        \caption{Locality and sparsity structure characterization of learned concepts using Hoyer score.}
    \end{subfigure}
    \caption{\textbf{Encoding where and how concepts are expressed.} (a) A digit translates across the domain over frames. (b) Top-10 concepts over time: SAE-MLP switches concepts as the digit moves, whereas SAE-FNO maintains stable concepts whose spatial activations shift. (c) Higher Intersection over Union (IoU) indicates stronger concept stability and consistency across time. SAE-FNO (blue) maintains a highly stable set of active concepts, whereas SAE-MLP (red) rapidly swaps its active concepts between frames. (d) Spatial feature-map trajectories of the top-3 SAE-FNO concepts follow the ground-truth motion; color denotes concept and brightness denotes time. (e) SAE-FNO exhibits stable concept characteristics across different sparsity levels. In contrast, SAE-MLP and SAE-CNN show significant changes in concept structure as sparsity varies, indicating undesirable sensitivity.}
    \label{fig:video}
    \vspace{-8.0mm}
\end{figure}
\textit{Functional sparse coding.} We replace scalar concept activations with feature-map representations, introducing \emph{concept} sparsity and \emph{domain} sparsity to encode not only which concepts are active, but also \emph{how} and \emph{where} they are expressed (\Cref{fig:fig1}). More precisely, concept sparsity acts as group sparsity (selecting active concepts), while domain sparsity is intra-group spatial sparsity.

\textit{Emergent temporal concept consistency.} Through domain sparsity and Fourier-based functional representations, SAE-FNO naturally reuses concepts across translated temporally varying patterns by shifting spatial activations over the domain. This emergent temporal consistency improves concept reuse and representation efficiency. In contrast, SAE-MLP lacks a mechanism to encode how concepts are expressed across the domain, and therefore switches between multiple distinct concepts to capture temporal transitions. Importantly, SAE-FNO learns stable concepts across spatial translations over time even when frames are treated independently during training (\Cref{fig:video}), without temporal memory, temporal smoothness, or sequence-level training constraints.

\textit{Structured and stable concepts.} We focus on extracting concepts from vision data, where local and spatial structure is inherent and where sparse autoencoders should capture such structure across both spatial locations and scales. We show that SAE-FNO learns localized and frequency-structured concepts whose characteristics remain stable across concept sparsity levels, unlike SAE-MLP and SAE-CNN. This decouples structure from expressivity, enabling controlled changes in concept structure through domain sparsity while increasing expressivity through concept sparsity.

\textit{Efficient concept usage and expressivity.} SAE-FNO uses fewer effective concepts to represent the data population and achieves higher expressivity per concept, especially as the input domain grows. 

\textit{Correlated concept representations.} SAE-FNO learns a broader range of concept correlations than SAE-MLP, enabled by feature-map representations where correlated concepts can differ in how and where they are expressed across the domain, resulting in reduced interference.

\textit{Generalization across discretizations.} Unlike standard SAEs, SAE-FNO operates across resolutions beyond those seen during training and is more stable under frequency-structured perturbations.

\textit{Lifting.} We introduce lifting into SAEs for concept extraction, mapping inputs to a higher-dimensional space where concept learning is performed. As a core component of neural operators~\citep{li2021fourier,li2023fourier,liu2025difffno}, lifting improves learning and optimization. We show that lifting acts as a preconditioner that accelerates recovery and improves concept learning.

Our experiments show that these properties consistently hold across different SAEs (TopK~\cite{makhzani2014ksae}, BatchTopK~\cite{bussmann2024batchtopk}, and Matryoshka~\cite{bussmann2025learning}), and datasets (MNIST~\citep{lecun2002gradient}, CIFAR~\cite{krizhevsky2009learning}, and ImageNet~\cite{deng2009imagenet}).

\vspace{-3mm}
\section{Preliminaries}\label{sec:prelim}
\vspace{-1mm}
\textbf{Sparse Generative Models}\quad A data sample $\x \in \R^{m}$ is said to follow a \textit{sparse linear generative model} if there exists a sparse code $\z \in \R^{p}$ (with $|\text{supp}(\z)| \leq k \ll p$) and an overcomplete dictionary $\D^{*} \in \R^{m \times p}$ (with $p \gg m$) that model the data as $\x = \D^{*} \z$. The sparse generative model, also known as sparse coding~\citep{olshausen1996emergence,olshausen1997sparse}, has a rich history in statistics and compressed sensing~\citep{donoho2006compressed, candes2006robust}. We focus on \textit{sparse model recovery}, where $\D^{*}$ is unknown (a.k.a. \textit{dictionary learning})~\citep{agarwal2016learning}, and can be learned by solving an outer objective~\citep{tolooshams2023phd, hindupur2025projecting}: $\min_{\D \in \mathcal{D}}\ \tfrac{1}{2}\ \| \x - \D \z \|_2^2 + \lambda \mathcal{R}(\z)$ subject to an inner one $\z = \argmin_{\smlv \in \R^{p}} \mathcal{L}(\x, \smlv, \theta)$, where $\mathcal{R}(\z)$ is a sparse regularizer. In mechanistic interpretability, sparse generative models take the form of the linear representation hypothesis (LRH) which, beyond linear reconstruction, also imposes a linear accessibility, hence a shallow encoder~\citep{costa2025flat}.

\textbf{Sparse Autoencoders (SAEs)}\quad The connection between the above optimization and SAEs is well established~\citep{malezieux2022understanding, tolooshams2022stable} with inner/outer levels mirroring the encoder–decoder structure of a neural network. Specifically, an SAE infers a sparse code $\z \in \R^{p}$ (i.e., $\z = \text{Enc}_{\theta}(\x)$), representing the data $\x \in \R^{m}$, with a linear combination of a dictionary of concepts (i.e., $\hat \x = \D \z$). Most widely used SAEs for mechanistic interpretability are shallow autoencoders~\citep{huben2023sparse, makhzani2014ksae, rajamanoharan2024jumping, bussmann2024batchtopk, hindupur2025projecting}.

\textbf{Neural Operators (NOs)}\quad Neural operators~\citep{lu2019deeponet,li2020neural,li2021fourier,kovachki2023neural} consist of three components: lifting, kernel integration, and projection (see~\Cref{app:background}). Lifting is a linear operator that maps data to a higher-dimensional channel space where learning occurs, followed by a linear projection. Fourier Neural Operator (FNO)~\citep{li2021fourier} models the kernel integral operator (\Cref{def:kernelopt}) with a convolution parameterized in Fourier space (\Cref{def:fno,def:fourier}). FNO is used with truncated frequency modes in practice, shown to improve performance and lower sensitivity to noise and discretization. Fourier modes denote the number of frequencies used to parameterize each kernel.

\vspace{-2mm}
\section{Sparse Autoencoder Neural Operators: Encoding How and Where Concepts Are Expressed}\label{sec:methods}
\vspace{-2mm}
We now introduce sparse autoencoder neural operators (SAE-NOs), extending sparse autoencoders from representing concept presence to modeling structured concept expression over the input domain.

\textbf{Functional Representation Hypothesis (FRH)} \quad Consider a linear operator $\mathcal{G}_{\D^{*}}: \mathcal{Z} \rightarrow \mathcal{X}$ in function spaces, where $\mathcal{Z}$ and $\mathcal{X}$ are Banach spaces of functions defined on bounded domains $D_z \subset \R^{p}$ and $D_x \subset \R^{m}$, respectively. The data sample $\x \in \mathcal{X}$ is said to follow the functional representation hypothesis (FRH), if there exists a sparse representation $\z$ drawn from the support $\mathcal{Z}$ of a probability measure $\mu$, where $\x = \mathcal{G}_{\D^{*}}(\z)$. Under this model, the goal of sparse concept recovery in function spaces is to estimate the underlying operator $\mathcal{G}_{\D^{*}}$ by solving the following:
\begin{equation}
\begin{aligned}
    \min_{\D \in \mathcal{D}}\ \ &\mathbb{E}_{\z \sim \mu} \| \mathcal{G}_{\D^{*}}(\z) - \mathcal{G}_{\D}(\hat \z) \|_{\mathcal{X}}^2\ \quad
    \text{s.t.}\quad \hat \z = \argmin_{\smlv}\ \| \smlv - \mathcal{H}_{\theta}(\mathcal{G}_{\D^{*}}(\z))\|_{\mathcal{Z}}^2 + \mathcal{R}(\smlv),
\end{aligned}
\end{equation}
where $\smlv$ is the coding variable, and the operator $\mathcal{H}_{\theta}: \mathcal{X} \rightarrow \mathcal{Z}$ aims to approximate the inverse of $\mathcal{G}_{\D^{*}}$. At the functional level, $\mathcal{R}$ encourages representations that are sparse both across concept functions and within their domain of expression. After discretization, these correspond to the concept and domain sparsity mechanisms introduced below. Model recovery is successful when the learned $\D$ in $\mathcal{G}_{\D}$ closely approximates $\D^{*}$ in $\mathcal{G}_{\D^{*}}$.

\textbf{Sparse Autoencoder Neural Operators (SAE-NOs)}\quad We formalize an SAE-NO with an encoder approximating the inverse generative operator and a decoder that captures the underlying operator and its learnable concepts from FRH. Under the FRH, the goal is to infer a sparse representation $\z$ as a sample from the function space $\mathcal{Z}$, representing the data sample from $\mathcal{X}$ via $\mathcal{G}_{\D}$:
\begin{equation}
\text{(encoder)}\quad \z = \text{Enc}_{\theta}(\x),\quad \text{(decoder)}\quad \hat \x = \text{Dec}_{\D}(\z),
\end{equation}
where $\text{Enc}_{\theta}$ is the operator associated with $\mathcal{H}_{\theta}$, and model recovery estimates $\D$ through $\text{Dec}_{\D}$. We note that real data does not need to strictly satisfy the FRH for SAE-NO to be effective. FRH acts as an \emph{inductive bias}, encouraging extraction of structures aligned with this hypothesis~\citep{hindupur2025projecting}.

\textbf{Fourier Parameterization}\quad We instantiate SAE-NO using a Fourier parameterization, defining Sparse Autoencoder Fourier Neural Operators (SAE-FNOs) as
$\mathcal{G}_{\D}(\z) \coloneqq \mathcal{F}^{-1} (\sum\nolimits_{c=1}^p \W_c \cdot (\mathcal{F} \z_{c}))$, where $\W_c = \mathcal{F} \D_c$ is the Fourier transform of the convolution kernel $\D_c$. In the data domain, this corresponds to a sum of convolutions and hence a linear decoder. This formulation induces an inductive bias toward sparsely structured convolutional representations in function spaces. Owing to the duality between convolution and multiplication, inference in SAE-FNOs can be viewed as recovering concepts under a sparse convolutional generative model (\Cref{def:sgmconv}), where reconstruction is expressed as a sum of convolutions between $\z$ and $\D$.

\textbf{Encoding How and Where Concepts Are Expressed}\quad The primary advantage of SAE-FNO is in capturing how and where concepts are expressed, beyond concept presence. We introduce two distinct notions of sparsity: \emph{concept} sparsity and \emph{domain} sparsity. Concept sparsity extends the standard notion of sparsity in SAEs by enforcing that only a small number of concepts are used to explain each data sample. Domain sparsity constrains the activation pattern within the feature map associated with each concept. Mathematically, concept sparsity induces a form of \emph{group sparsity} (where all entries of a concept's feature map constitute one group), and domain sparsity enforces element-wise sparsity within that feature map. Intuitively, concept sparsity decides \emph{which} concepts are active, and domain sparsity determines \emph{where} each active concept is expressed across the input domain.

As demonstrated in~\Cref{sec:results}, this joint sparsity structure enables SAE-FNO to learn concepts that are robust to variations in concept sparsity and shared across data examples, while being expressed differently through example-specific domain sparsity patterns. For example, SAE-FNO can represent translated instances of the same pattern using a stable set of concepts whose spatial activations shift across the domain, whereas SAE-MLPs require switching between multiple distinct concepts to capture such variations (\Cref{fig:video}, see also~\Cref{fig:video-concepts-over-time,fig:video-iou}). This ability to reuse concepts through domain sparsity leads to more efficient representations, requiring fewer concepts to explain the data.

Importantly, this notion of concept stability across time emerges naturally from the Fourier parameterization and domain sparsity of SAE-FNO. In \Cref{fig:video}, frames are treated as i.i.d.\ training examples, without imposing temporal smoothness constraints, sequence-level objectives, or temporal priors during training. This contrasts with recent works that explicitly model temporal consistency or dense/sparse structure across sequences through coordinated training over time~\citep{bhalla2026temporal,lubana2026priors}. 

Denoting the sparsifying projections for concept and domain by $\Pi_c$ and $\Pi_d$, respectively, SAE-FNO implements the following (see Appendix for implementation details; SAEs may include bias vectors):
\begin{equation}\label{eq:relusaefno}
    \begin{aligned}
        \min_{\W \in \mathcal{W}}\quad & \tfrac{1}{2}\| \x - \mathcal{F}^{-1} (\W \cdot \mathcal{F}\z) \|_2^2 + \lambda \mathcal{R}(\z) \\
        \text{s.t}\quad & \bm{\tilde v} = \mathcal{F}^{-1}(\bm{E} \cdot (\mathcal{F}\x))\\
        \text{(concept sparsity)}\quad &  \bm{v}  = \mathbf{1}_{\bm{a}} \cdot \bm{\tilde v},\qquad \text{where}\quad \bm{a} = \Pi_c \left( [\| \bm{\tilde v}_1\|_2, \|\bm{\tilde v}_2\|_2, \ldots, \| \bm{\tilde v}_p\|_2 ]^{\top} \right)  \\
        \text{(domain-sparsity)}\quad & \z_c = \Pi_d\left(\bm{v}_c\right),\quad \text{for}\ c=1,\ldots,p\quad \text{, where}\ \|\bm{v}_c\|_2 > 0,
    \end{aligned}
\end{equation}
where $\bm E$ and $\bm W$ denote the learned Fourier-domain encoder and decoder weights, respectively, $\mathcal{R}(\z)$ is a sparse regularizer, $\tilde{\smlv}$ is the dense pre-activation map, and $\mathbf{1}_{\bm{a}}$ is a sparse indicator vector that zeros out entire concepts with low $\ell_2$-norm (energy). For hard-sparsity variants used in our experiments (TopK, BatchTopK, and Matryoshka), sparsity is enforced directly through the sparsification operators rather than an explicit $\mathcal R(z)$ penalty. This joint SAE-FNO sparsification decouples expressivity (concept sparsity $k$) from structure (domain sparsity), unlike conventional SAEs where changing sparsity levels alters concept structure. This formulation makes SAE-FNO a general functional framework rather than an architecture: the FNO parameterization is orthogonal to the sparsification mechanism, allowing existing SAEs to be extended by replacing linear encoder/decoder mappings with operator-based ones while keeping their original activation within $\Pi_c$ and $\Pi_d$.

\textbf{Lifting}\quad We further extend SAEs to operate in lifted spaces via explicit linear \emph{lifting} and \emph{projection} modules. Distinct from the high-dimensional latent representations learned within SAEs~\citep{elhage2022toy}, lifting applies a linear transformation that maps inputs to a higher-dimensional space before entering the SAE. While widely used in neural operator architectures~\citep{li2021fourier,liu2025difffno, kovachki2023neural, azizzadenesheli2024neural, li2024geometry}, its role in SAE dynamics and recovery remains unexplored. We introduce lifting to provide a more favorable optimization parameterization while preserving the underlying data structure without increasing expressivity.

The lifted SAE-NO (L-SAE-NO) performs model recovery in function space by inferring a sparse code $\z$ as a sample from the function space $\mathcal{Z}$, representing the data from $\mathcal{X}$ via the operator $\text{Dec}_{\D_{\lift}}$:
\begin{equation}\label{eq:lifted-opsae}
\begin{aligned}
    \text{(lift)}\quad \y = \lift \x, \quad
    \text{(enc)}\quad \z = \text{Enc}_{\theta}(\y),
    \quad
    \text{(dec)}\quad & \hat{\y} = \text{Dec}_{\D_{\lift}}(\z),
    \quad
    \text{(proj)}\quad \hat{\x} = \proj \hat{\y},
\end{aligned}
\end{equation}
where $\text{Enc}_{\theta}$ encodes the lifted data $\y$, and concept recovery is achieved through the learned lifted concepts $\D_{\lift}$. Introducing these linear operators ($\lift$ and $\proj$) raises key questions: i) can concepts be faithfully recovered in the original domain? ii) how does lifting affect training dynamics and concept learning? and iii) under what conditions is a lifted SAE equivalent to its unlifted counterpart?

First, \Cref{fig:conv-sae-lift-b,fig:dense-lift-acc,fig:cifar-lifted-concepts} confirm that learned lifted concepts can be faithfully projected and interpreted back in the original data domain. Second, \Cref{prop:lift_fno_train} shows that the gradient dynamics of Fourier concepts $\W$ in L-SAE-FNO correspond to dynamics in the original space over $\D$, where lifting acts as a \emph{preconditioner}. This induces more isotropic updates and accelerates optimization~\citep{arora2019implicit} (analogous results for SAE-MLPs/CNNs in \Cref{prop:lift_train_equiv,prop:conv_lift_train_equiv}). Third, \Cref{prop:operator-lift-arch-equiv} establishes that L-SAE-FNO reduces to SAE-FNO when lifting and projection are tied ($\proj = \lift^\top$), orthogonal ($\lift^\top \lift = \eye$), and $\text{col}(\D_L)\subseteq\text{range}(\lift)$. (see also \Cref{prop:lift-arch-equiv,prop:conv-lift-arch-equiv}). 
\begin{restatable}[Training Dynamics of Lifted SAE-FNO]{proposition}{liftfnotrain}\label{prop:lift_fno_train}
    The training dynamics of the Lifted-SAE-FNO (L-SAE-FNO)'s frequency domain weights $\W_{L, c}^{(k + 1)} = \W_{L, c}^{(k)} + \nicefrac{\eta_L}{\sqrt{M}} \mathcal{F} \left( \lift \left( \x - \hat{\x}^{(k)} \right)\right) \odot \overline{\mathcal{F}(\z_{c, k})}$, with lifting $\lift$ and projection $\proj$, has an effective update in the original space, expressed as: $\D_c^{(k + 1)} = \D_c^{(k)} + \nicefrac{\eta_L}{M} \left( (\lift^\top\lift) \left( \x - \hat{\x}^{(k)} \right)\right) \star \z_{c, t}$, where $\lift^\top \lift$ can act as a preconditioner to rescale update directions to improve optimization ($\ast$: convolution, $\star$: correlations). If the architectural inference equivalence condition is satisfied (\Cref{prop:operator-lift-arch-equiv}), then the dynamics are equivalent to that of an SAE-FNO: $\D_c'^{(k + 1)} = \D_c'^{(k)} + \nicefrac{\eta}{M}\left(\x - \nicefrac{1}{\sqrt{M}} \mathcal{F}^{-1} ( \sum\limits_{i = 1}^p \mathcal{F}(\D_i^{(k)}) \odot \mathcal{F}(\z_{i, k}) ) \right) \star \z_{c, k}$.
\end{restatable}

\textbf{Identifiability}\quad Under the FRH, if the ground-truth kernels are band-limited and sampling meets Nyquist, the discretized SAE-FNO decoder is a convolutional dictionary, without discretization error. Identifiability theory for convolutional sparse coding therefore transfers directly: uniqueness is governed by \textit{stripe sparsity}, the number of active coefficients within the receptive field a kernel sees at once, via $||z||_{0,\infty}<0.5(1+1/\mu)$, where $\mu$ is the coherence of the shifted-kernel dictionary~\cite{donoho2003optimally,papyan2017working,ba2020deeply}. Our concept/domain sparsity factorizes stripe sparsity: $k$ bounds the number of active kernels, $s$ their within-kernel density, and their product $ks$ the local density constrained by the condition. We verify this identifiability empirically by measuring the coherence of a ground-truth shifted-kernel dictionary both across concepts and within concepts, and examine recovery against the expected activations per receptive field $r = ksh/m$ (\Cref{tab:identifiability_stripe_sparsity}, see~\Cref{app:rec_identifiability}). A self-contained theorem for the continuous operator setting remains open and can benefit from theoretical analysis in~\cite{kovachki2023neural, reinhardt2026statistical, kovachki2021onuniversal}.

\vspace{-3mm}
\section{Results}\label{sec:results}
\vspace{-1mm}
Following recent work that studies SAEs on interpretable outputs rather than internal activations~\citep{bohacek2026uncovering, jiang2026interpretable}, we train SAE-FNOs directly on data with spatial structure. This enables direct comparison between SAE-FNO and SAE-MLPs in the data domain. For fair comparison, we analyze SAE-FNO concepts after mapping them from their Fourier parameterization back into the data domain.

\begin{figure}[t]
    \centering 
    \begin{subfigure}{0.325\linewidth}
        \centering
        \includegraphics[width=0.98\linewidth]{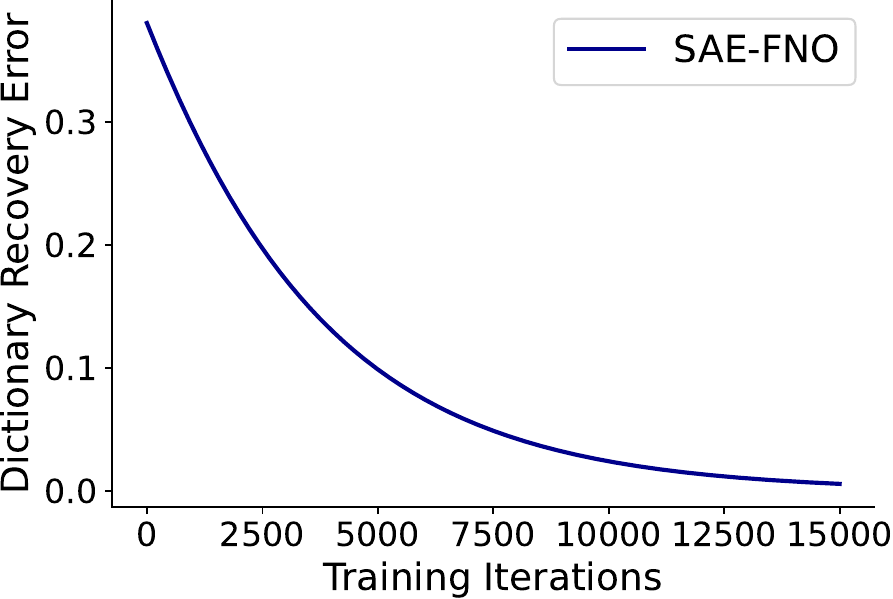}
        \caption{}
        \label{fig:synthetic_a}
    \end{subfigure}
    \begin{subfigure}{0.325\linewidth}
    \centering
    \includegraphics[width=0.98\linewidth]{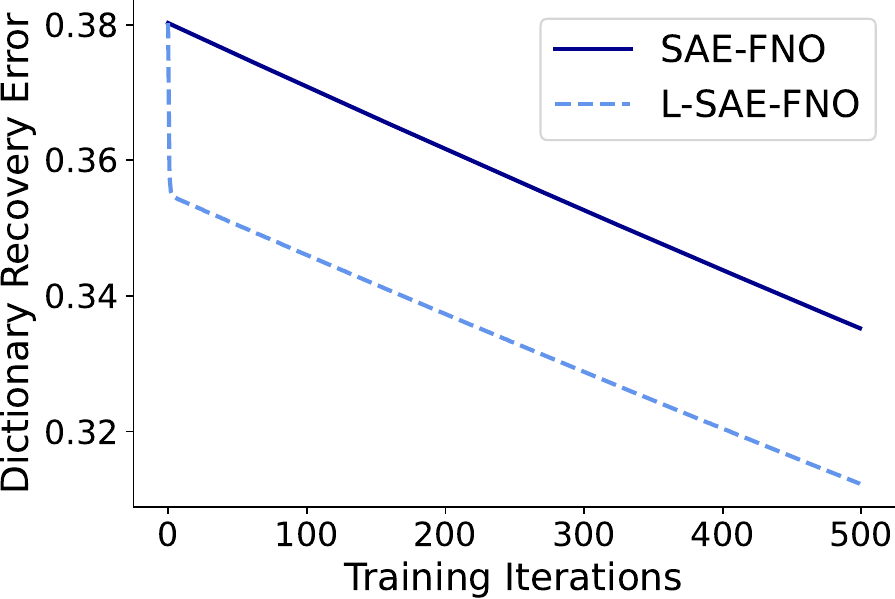}
    \caption{}
     \label{fig:synthetic_b}
    \end{subfigure}
    \begin{subfigure}{0.325\linewidth}
    \centering
    \includegraphics[width=0.98\linewidth]{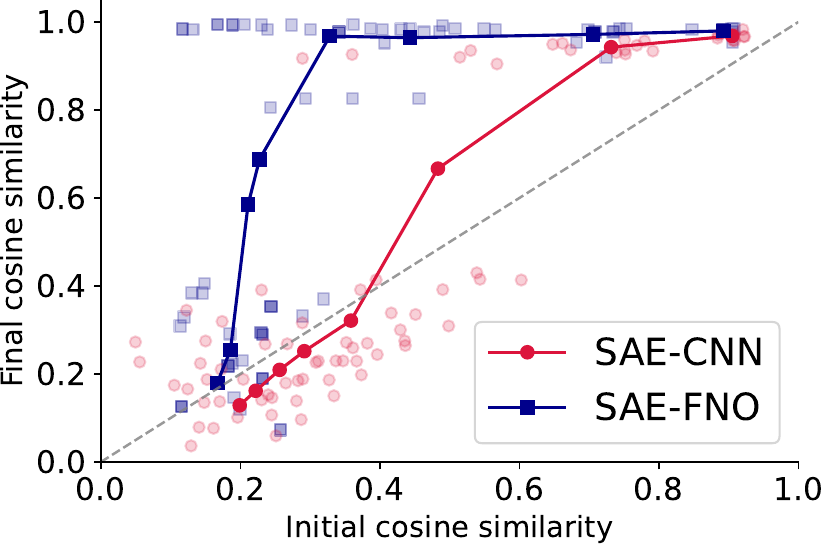}
    \caption{}
    \label{fig:synthetic_c}
    \end{subfigure}
    \vspace{-2mm}
    \caption{\textbf{Concept recovery under known generative model.} (a) SAE-FNO recovers the underlying concepts, with dictionary recovery error converging to zero. (b) Lifting to a higher-dimensional space acts as a preconditioner, accelerating recovery. (c) SAE-FNO outperforms SAE-CNN in recovering smooth concepts. Each point corresponds to a training run.}
    \label{fig:synthetic}
    \vspace{-3mm}
\end{figure}
\begin{figure}[t]
    \centering
        \begin{subfigure}{0.329\textwidth}
        \centering
    \includegraphics[width=1.00\linewidth]{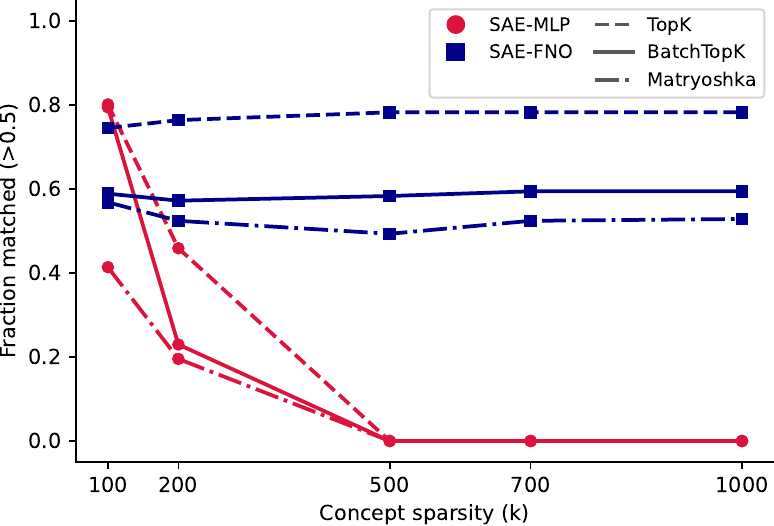}
    \caption{}
    \label{fig:identity-consistency}
    \end{subfigure}
    \begin{subfigure}{0.329\textwidth}
        \centering
        \includegraphics[width=1.00\linewidth]{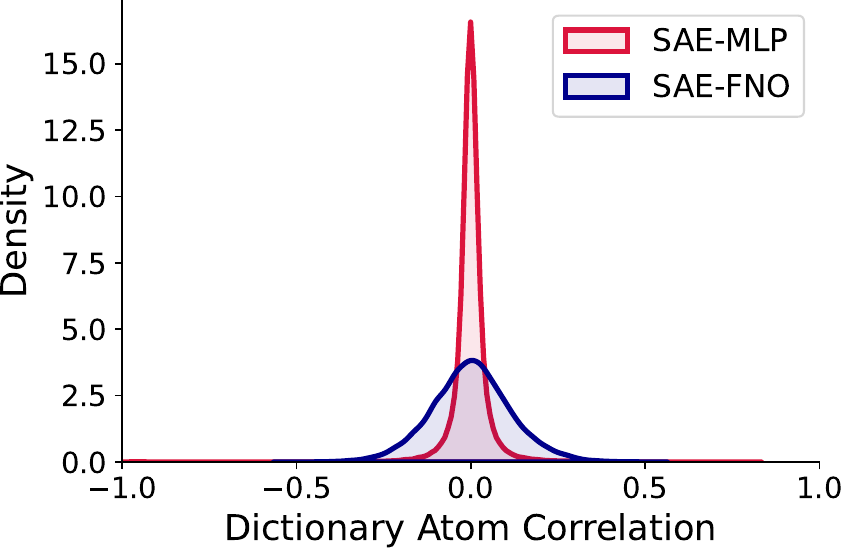}
    \caption{}
    \label{fig:mnist-atom-corr}
    \end{subfigure}
    \begin{subfigure}{0.329\textwidth}
        \centering
        \includegraphics[width=1.00\linewidth]{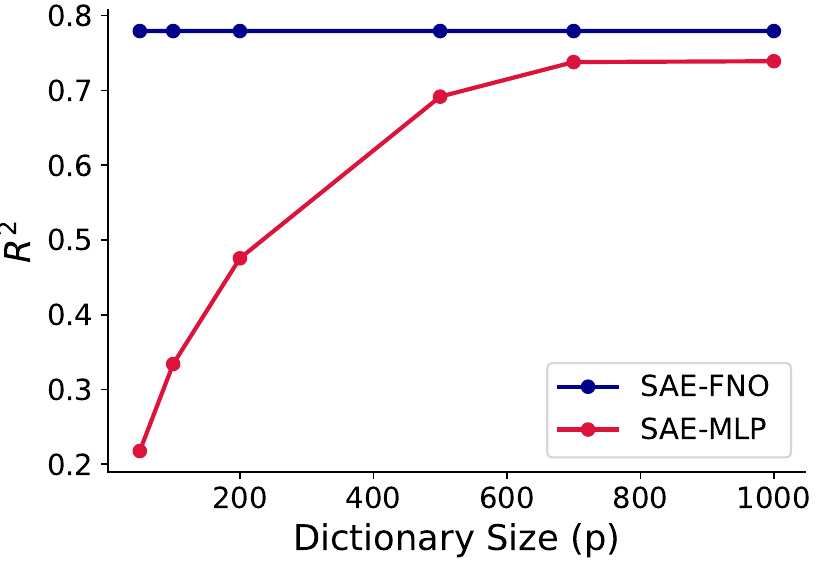}
    \caption{}
    \label{fig:imagenet-r2}
    \end{subfigure}
    \vspace{-5mm}
    \caption{\textbf{Structural stability, correlation of concepts, and expressivity.} (a) Fraction of matched features across runs (identity consistency) of increasing concept sparsity shows that SAE-FNO learns concepts with stable structural characteristics while SAE-MLP's concept structure changes as $k$ increases. (b) SAE-FNO overcomes orthogonality bias and recovers more correlated concepts. We compare the distribution of pairwise cosine similarities between learned dictionary atoms across varying k. (c) Expressivity of SAE, $R^2$, as a function of dictionary size $p$.}
    \label{fig:threesub}
    \vspace{-5mm}
\end{figure}
Crucially, our experiments isolate the impact of concept parameterization. By keeping all other aspects of the SAE fixed, including the sparsity mechanism (e.g., BatchTopK) and loss formulation (e.g., Matryoshka), and evaluating matched configurations across models (e.g., SAE-MLP-TopK vs. SAE-FNO-TopK), we ensure that any observed differences arise purely from the functional sparse coding formulation. For simplicity, we omit these details in naming and refer to the models as SAE-MLP and SAE-FNO, specifying the exact configuration in respective figures and tables.

Given the standard amortized inference setting used in shallow SAEs, we emphasize that SAE-FNO is not designed to resolve amortized inference pathologies such as feature splitting or absorption~\citep{sun2026price}. Rather, our results isolate the effect of functional parameterization.

%
\textbf{Concept Recovery Under Known Generative Model}\quad We characterize SAE-FNO under a known sparse convolutional generative model. We instantiate the model with an unrolled encoder~\citep{tolooshams2022stable} (see~\Cref{app:background}). As shown in \Cref{fig:synthetic_a}, SAE-FNO is successful in identifying the underlying model and learning the ground-truth concepts. We further analyze concept recovery in the lifted space, and find: i) underlying concepts can be recovered in the lifted space; ii) lifting can act as a preconditioner (\Cref{prop:lift_fno_train}), accelerating recovery (\Cref{fig:synthetic_b}) (emerging naturally through learning the lifting matrix) by reducing dictionary atom correlation, promoting a more isotropic loss landscape (\Cref{fig:dense-lift-acc}); iii) when lifting and projection are tied and orthogonal, the architectural inference (\Cref{prop:operator-lift-arch-equiv}) and training dynamics (\Cref{prop:lift_fno_train}) of L-SAE-NO reduce to SAE-NO. 

In real-data experiments, lifting further improves concept learning and reduces reconstruction error, consistently across Fourier modes (\Cref{fig:lifting-acc}). Beyond optimization benefits, we speculate that lifting shifts stochasticity and noise into the lifted space, while projection back to the original domain induces a smoothing effect that facilitates extraction of higher signal-to-noise concepts.

We further extend our synthetic experiments to the regime of large initialization error, comparing SAE-FNO and SAE-CNN in recovering smooth concepts. \Cref{fig:synthetic_c} shows that SAE-FNO outperforms SAE-CNN, demonstrating greater robustness as initialized cosine similarity worsens. Notably, all existing sparse coding theory assumes a ``close'' initialization, and, to our knowledge, provides no guarantees for recovery from random initialization~\citep{agarwal2014learning, arora2015simple, chatterji2017alternating, nguyen2019dynamics}.

\begin{figure}[t]
    \centering
    \begin{subfigure}{0.96\textwidth}
        \includegraphics[width=0.30\linewidth]{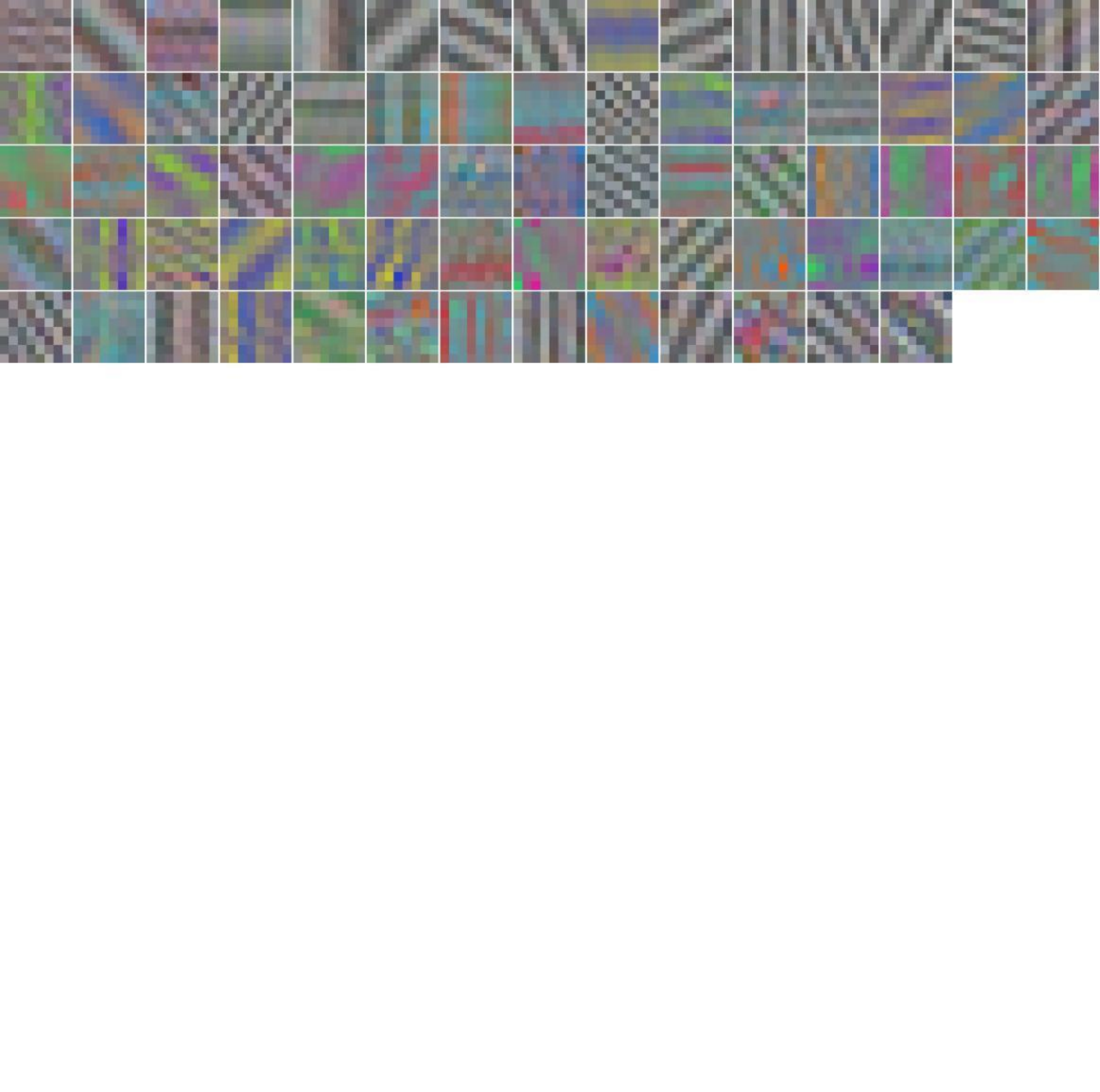}
        \hfill
        \includegraphics[width=0.30\linewidth]{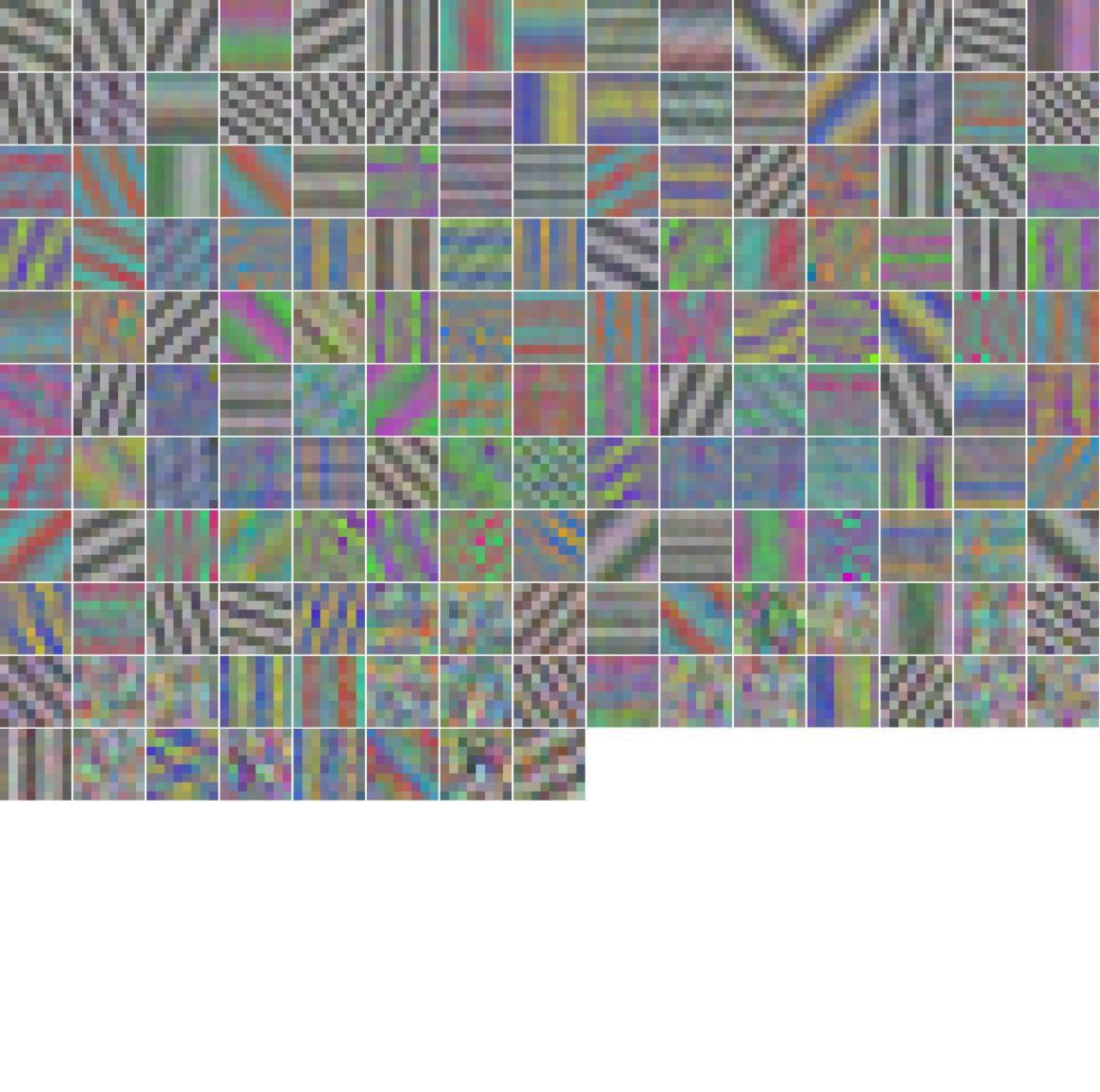}
        \hfill
        \includegraphics[width=0.30\linewidth]{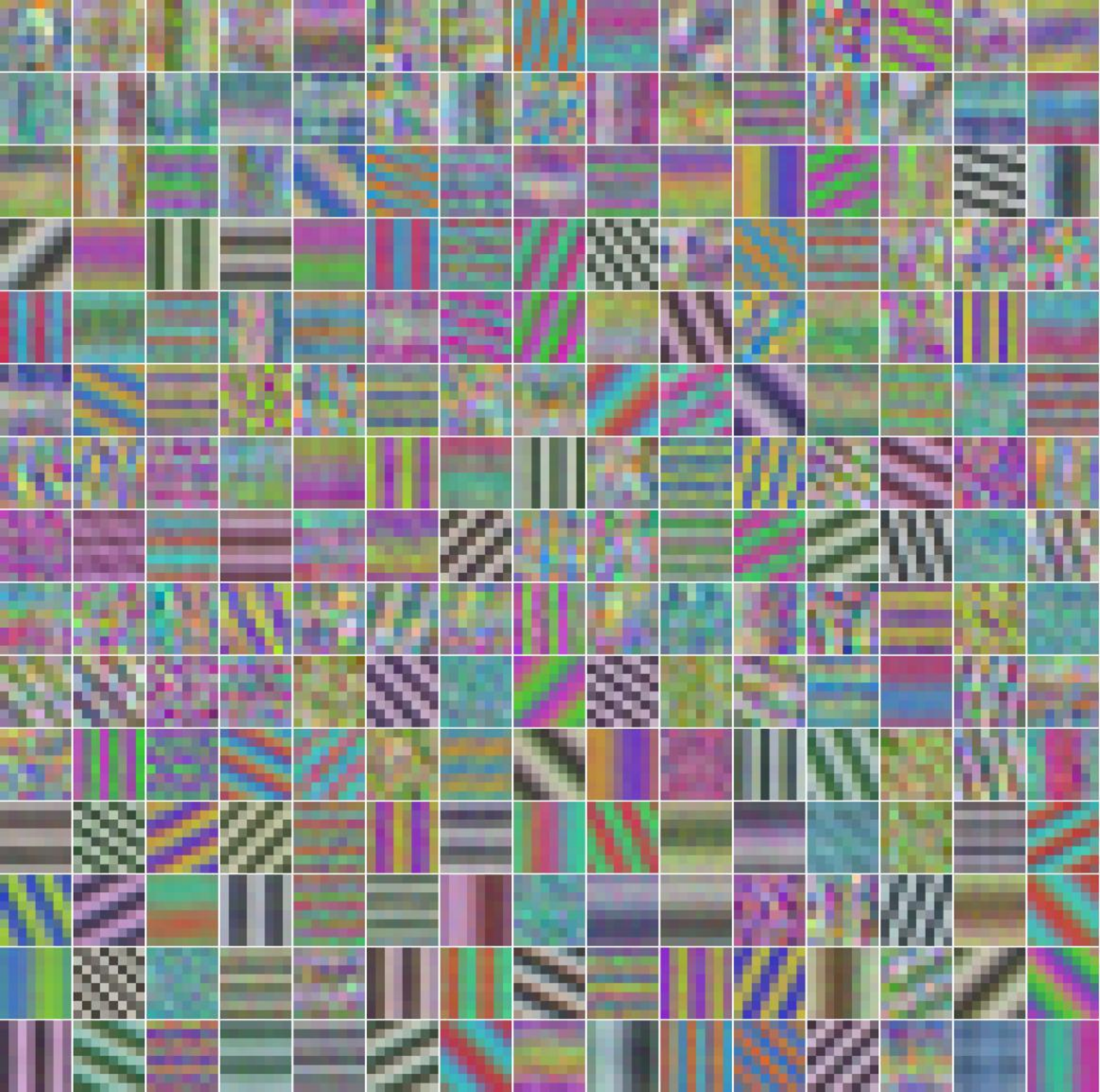}
        \caption{SAE-FNO as concept sparsity level changes (left) $k=25$ (middle) $k=50$ (right) $k=200$.}
    \end{subfigure}
    \begin{subfigure}{0.96\textwidth}
        \includegraphics[width=0.30\linewidth]{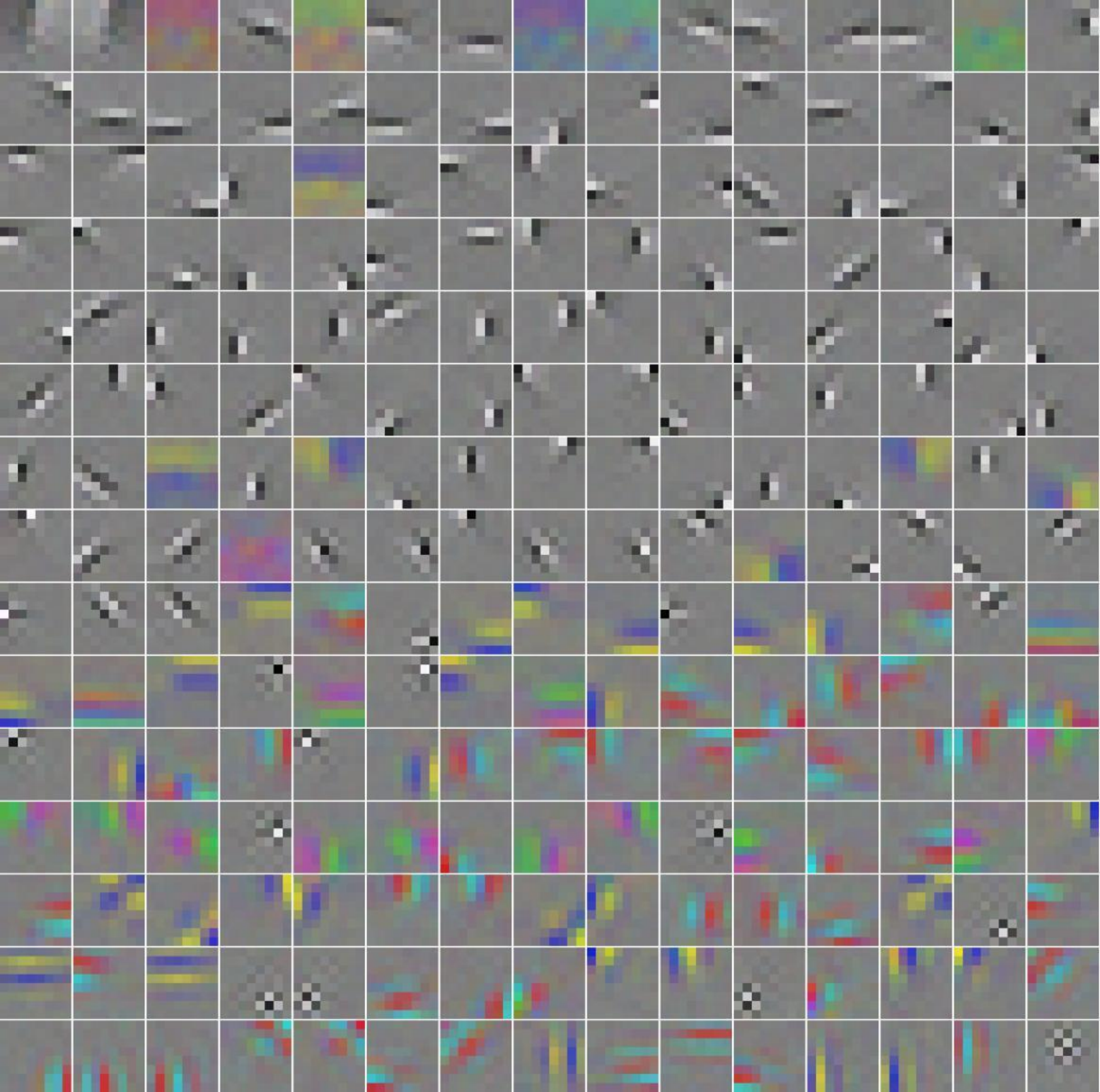}
        \hfill
        \includegraphics[width=0.30\linewidth]{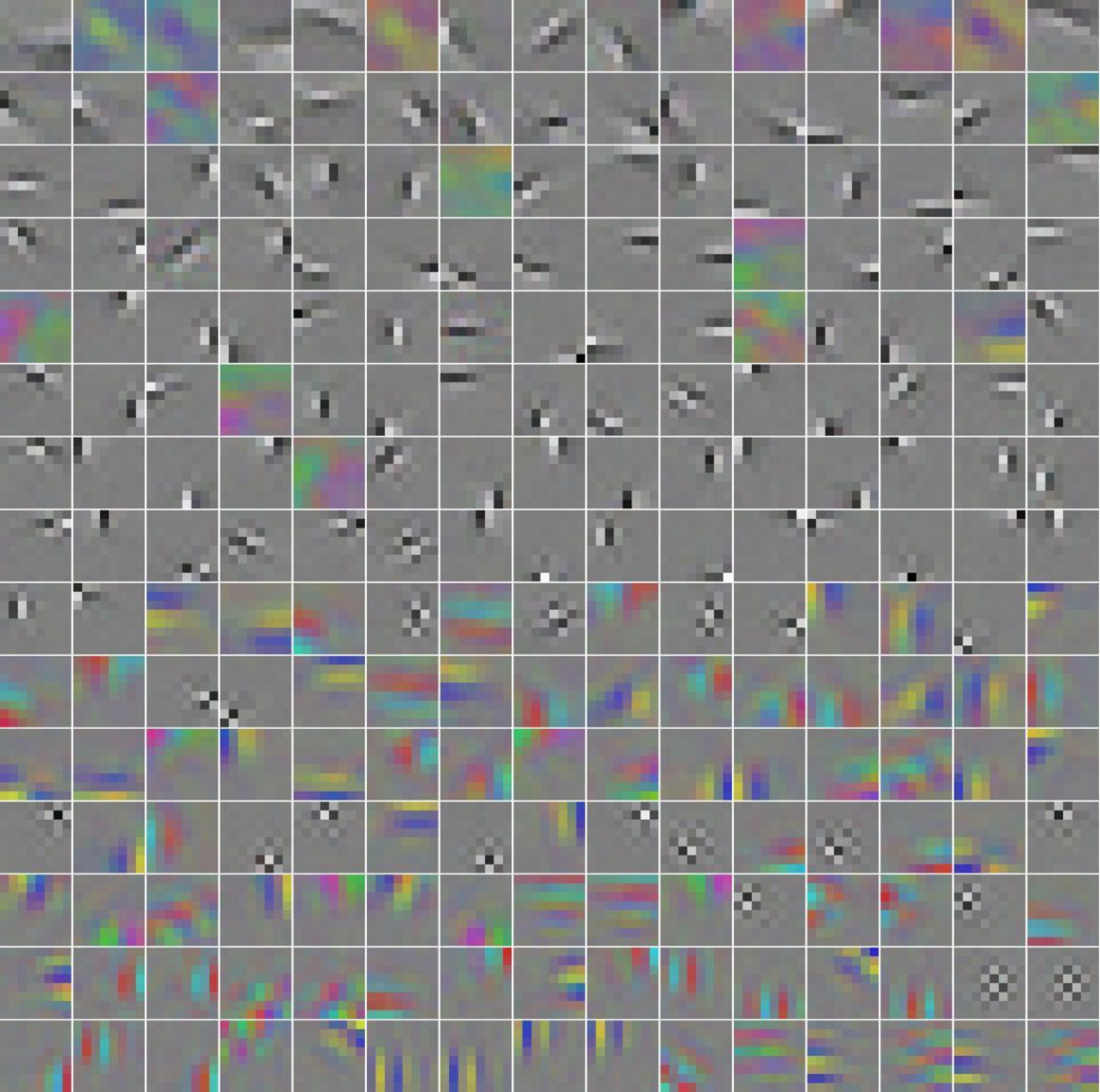}
        \hfill
        \includegraphics[width=0.30\linewidth]{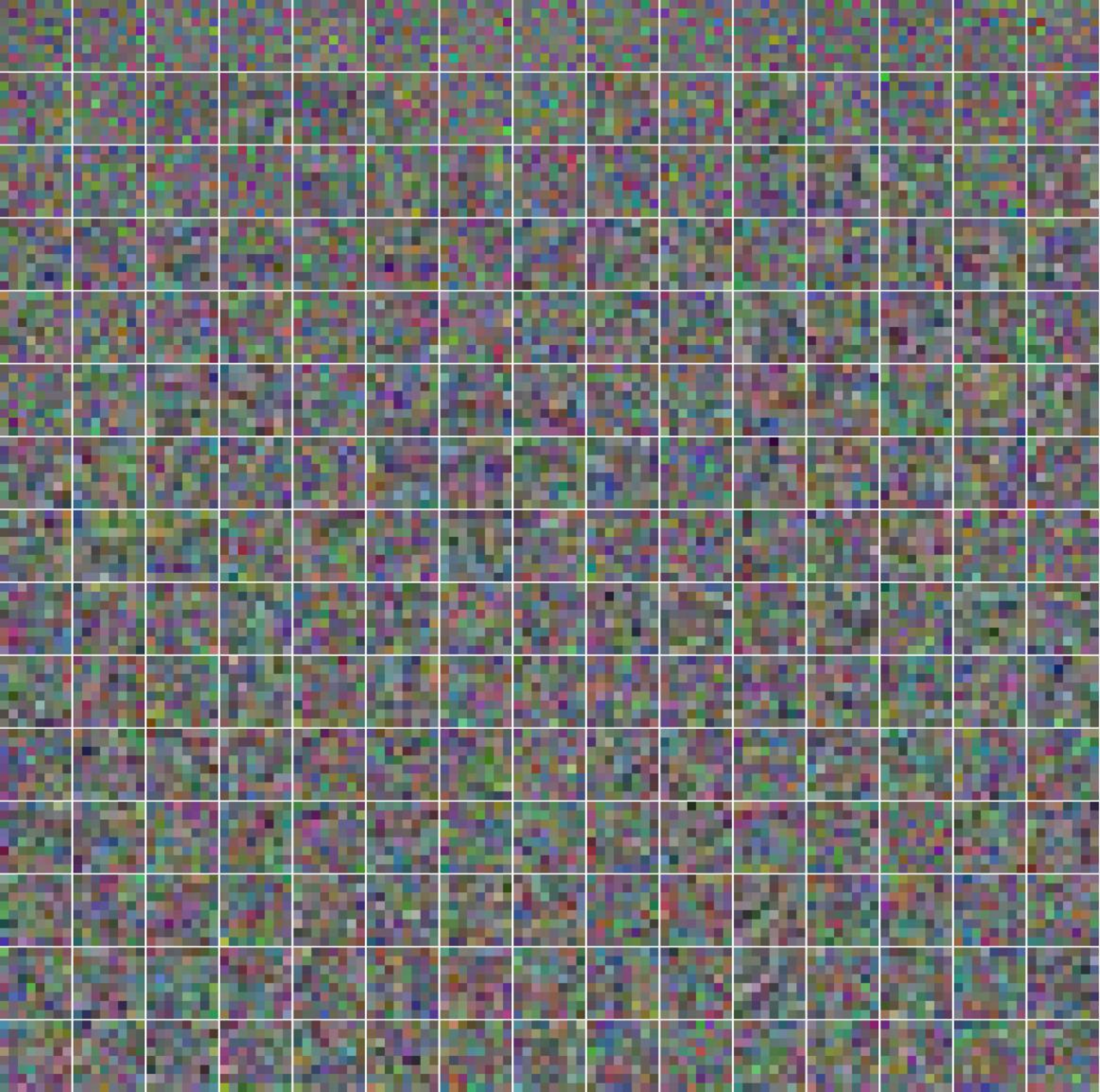}
        \caption{SAE-MLP as concept sparsity level changes (left) $k=25$ (middle) $k=50$ (right) $k=200$.}
    \end{subfigure}
    \vspace{-1mm}
    \caption{\textbf{Learned concepts across the dataset}. Each panel visualizes the most utilized concepts (each square is one concept) for a given concept sparsity level $k$. Fewer displayed concepts indicate more efficient concept reuse across the data population. SAEs are trained on CIFAR-10 images ($8 \times 8$) with TopK and $p=1000$. (a) SAE-FNO exhibits stable concept characteristics across sparsity levels due to domain sparsity. As sparsity decreases, more concepts become active while preserving similar structure. (b) SAE-MLP concepts vary significantly with sparsity, losing structure and failing to capture locality. See also~\Cref{fig:concept-visualization-32} for $32 \times 32$ inputs.}
    \label{fig:concept-visualization}
    \vspace{-6mm}
\end{figure}
%
%
\textbf{SAE-NO Concept Structural Stability}\quad A defining characteristic of SAE-NOs is their joint concept/domain sparsity and Fourier-based concept parameterization. To characterize the geometric structure of learned concepts, we use the Hoyer score~\citep{fel2026into} (\Cref{app:metrics}), where values near $1$ indicate spatially localized structured concepts and values near $0$ indicate dense ones. As shown in \Cref{fig:video}e, SAE-FNO learns concepts whose sparse structure remains stable across varying sparsity levels, whereas SAE-MLP structure does not as $k$ increases (see also~\Cref{tab:cnn-stability} for SAE-CNN). This structural stability is robust across Fourier parameterizations, datasets, and domain sparsity levels (\Cref{fig:mnist-fno-hoyer,fig:cifar-fno-hoyer-modes,fig:cifar-fno-hoyer-ss}), and avoids the sensitivity to architectural hyperparameters such as kernel size seen in SAE-CNNs (\Cref{fig:cnn-hoyer}). Moreover, SAE-FNO structural properties can be explicitly controlled through domain sparsity (\Cref{fig:change-strcuture-domain-cifar}, see also~\Cref{fig:change-modes-cifar} for mode changes).

This analysis does not claim that concepts remain identical as $k$ varies, but rather that their \emph{structural characteristics} remain stable. As $k$ increases, SAE-FNO uses more concepts while preserving the nature of the learned features, improving expressivity without altering concept structure. We verify this via an ``identity consistency'' metric (\Cref{fig:identity-consistency,fig:identity-consistency-0p7})~\citep{leask2025sparse} that matches concepts from an anchor model ($k=50$) to those at larger $k$ using cosine similarity (see~\Cref{app:metrics}). SAE-FNO maintains a consistent fraction of matched concepts, whereas SAE-MLP rapidly loses them, suggesting that SAE-FNO preserves a stable subset of concepts across sparsity levels while adding new ones.

\textbf{SAE-NO Concept Correlations}\quad SAE-FNOs permit a broader range of concept correlations. This contrasts with shallow SAEs, which tend to produce approximately orthogonal features~\citep{elhage2022toy,costa2025flat}. \Cref{fig:mnist-atom-corr,fig:atom-corr-supp} show that SAE-FNO exhibits a broader distribution of pairwise concept correlations, whereas SAE-MLPs are sharply concentrated near zero. This flexibility arises from the functional parameterization of SAE-FNO, where concepts are represented as feature maps with concept and domain sparsity. Consequently, concepts can remain moderately correlated while being expressed differently across the domain, reducing effective interference.

\begin{figure}[t]
    \centering
    \begin{subfigure}{0.245\textwidth}
        \includegraphics[width=0.95\linewidth]{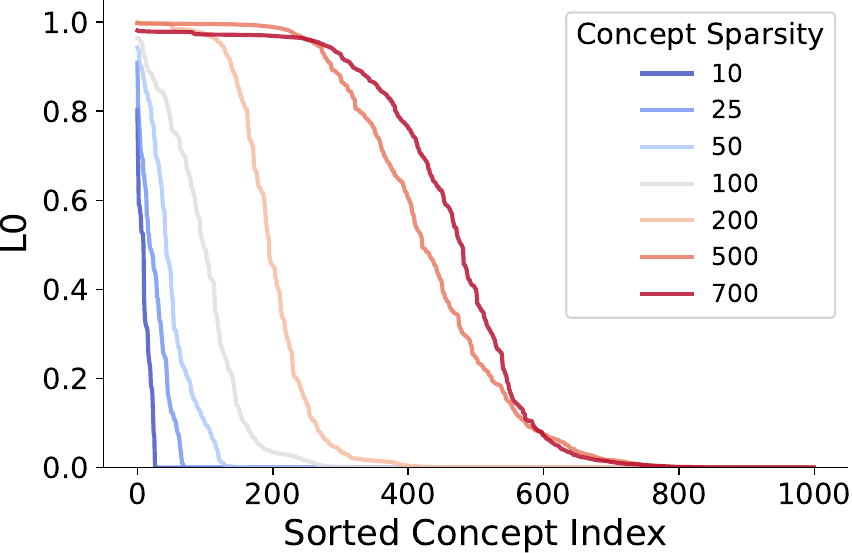}
        \caption{SAE-FNO ($8 \times 8)$.}
    \end{subfigure}
    \begin{subfigure}{0.245\textwidth}
        \includegraphics[width=0.98\linewidth]{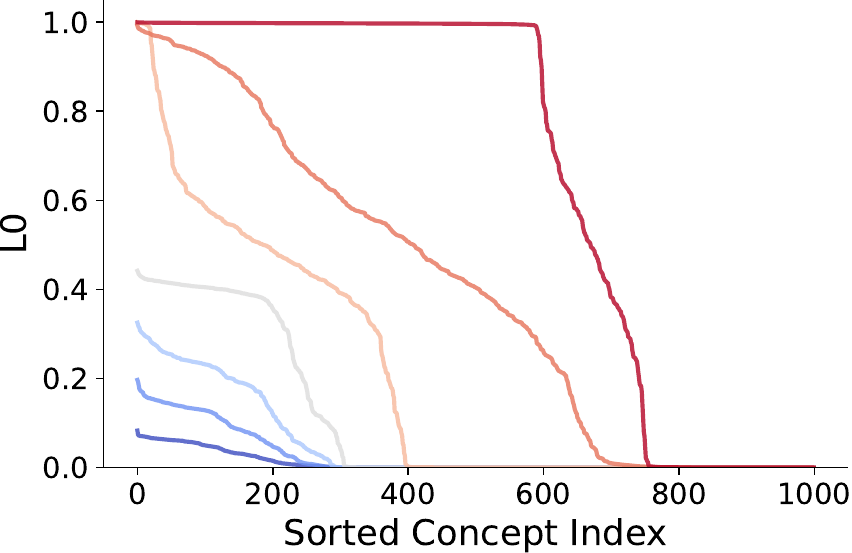}
        \caption{SAE-MLP.}
    \end{subfigure}
    \begin{subfigure}{0.245\textwidth}
            \includegraphics[width=0.98\linewidth]{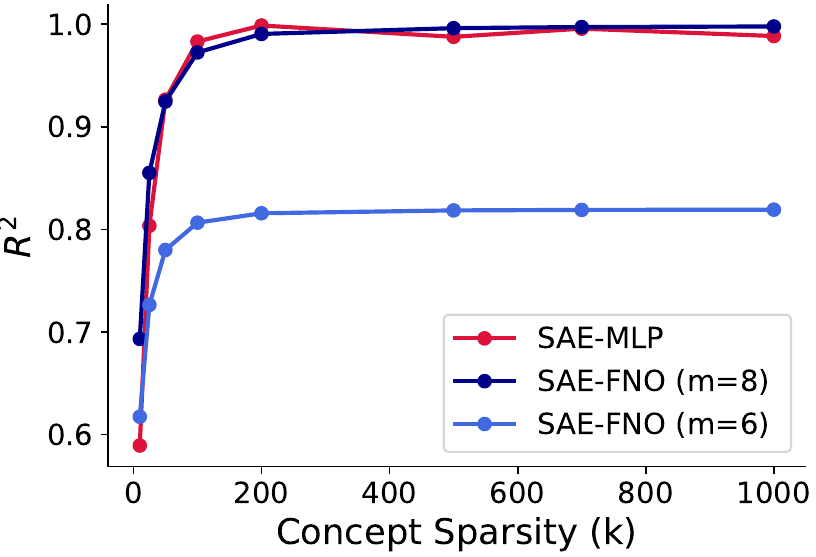}
        \caption{$\text{R}^2$ vs. $k$.}
    \end{subfigure}
    \begin{subfigure}{0.245\textwidth}
    \includegraphics[width=0.98\linewidth]{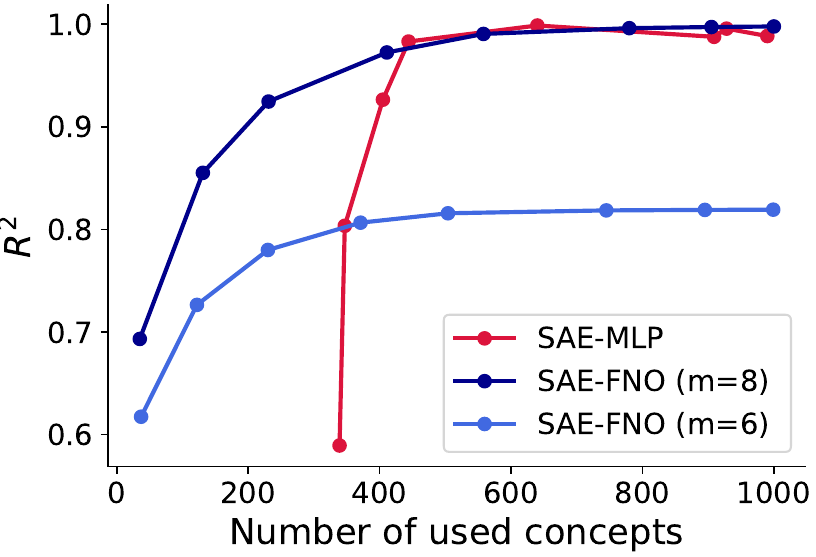}
    \caption{$\text{R}^2$ vs. used concepts.}
    \end{subfigure}
    \newline
    \begin{subfigure}{0.245\textwidth}
        \includegraphics[width=0.98\linewidth]{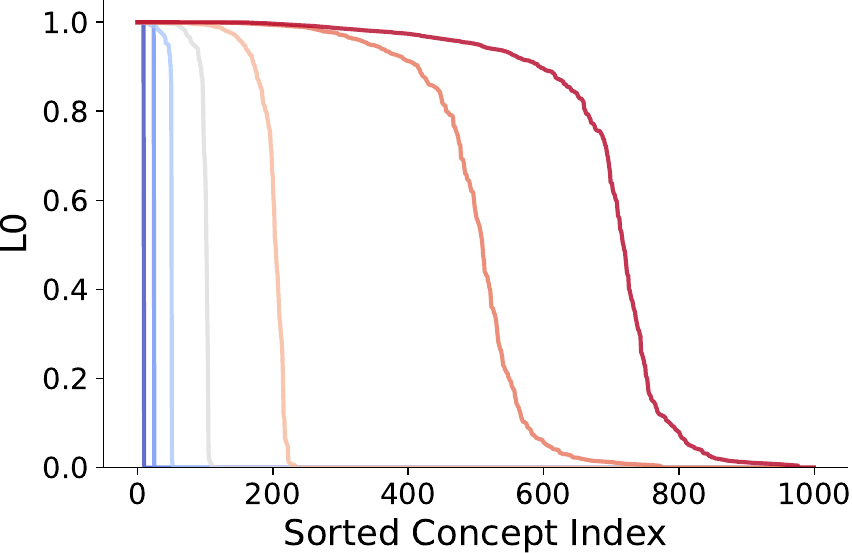}
        \caption{SAE-FNO ($32 \times 32)$.}
    \end{subfigure}
    \begin{subfigure}{0.245\textwidth}
        \includegraphics[width=0.98\linewidth]{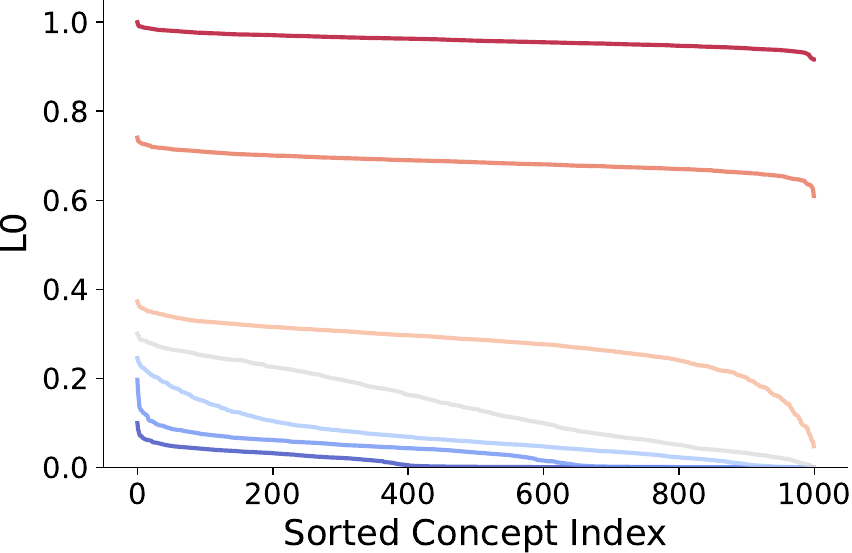}
        \caption{SAE-MLP.}
    \end{subfigure}
    \begin{subfigure}{0.245\textwidth}
            \includegraphics[width=0.98\linewidth]{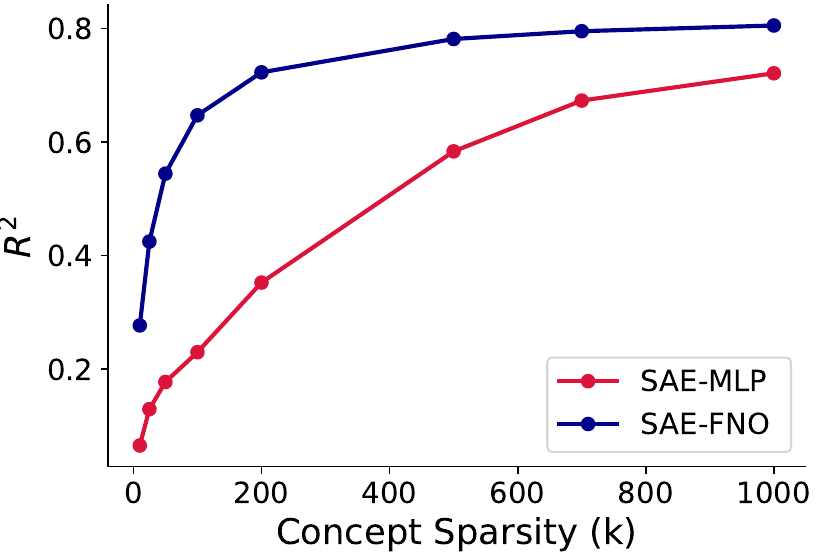}
        \caption{$\text{R}^2$ vs. $k$.}
    \end{subfigure}
    \begin{subfigure}{0.245\textwidth}
        \includegraphics[width=0.98\linewidth]{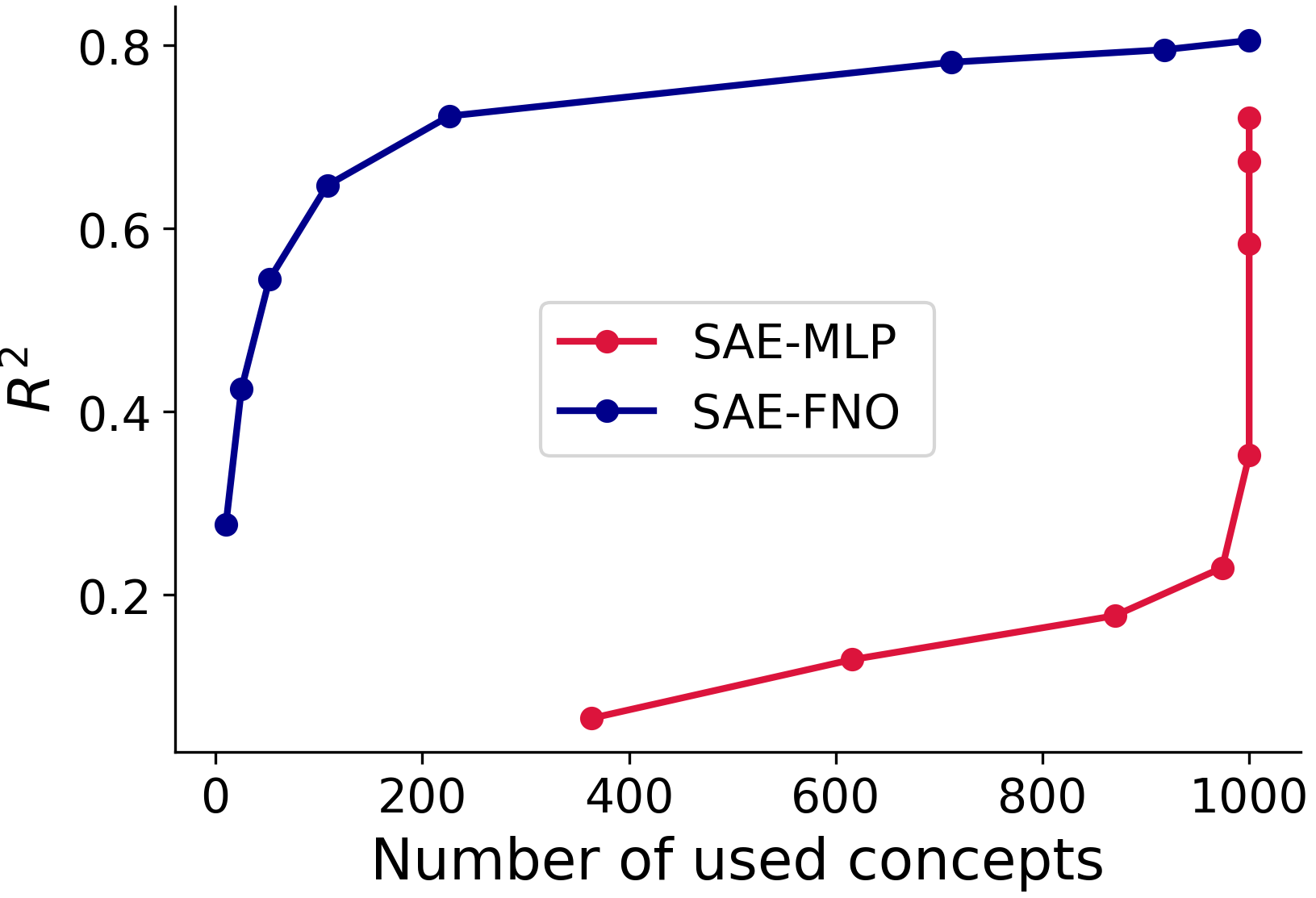}
        \caption{$\text{R}^2$ vs. used concepts.}
    \end{subfigure}
    \vspace{-4mm}
    \caption{\textbf{Concept utilization and expressivity}. First row: $8 \times 8$, second row: $32 \times 32$ for CIFAR10 and TopK. (a,b,e,f) $\ell_0$ code statistics (y-axis) across the data population, plotted against concept indices sorted by decreasing utilization (x-axis). SAE-FNO shows a sharp decay, indicating efficient concept usage, while SAE-MLP exhibits a flatter distribution, requiring more concepts, an effect that worsens for larger patches. (c,g) $R^2$ vs. $k$. (d,h) $R^2$ vs. the number of used concepts; SAE-FNO achieves higher expressivity with fewer concepts, whereas SAE-MLP fails to scale to larger domains.}
    \label{fig:utilization}
\vspace{-7mm}
\end{figure}
\textbf{Concept Visualization and Utilization}\quad We visualize learned concepts on natural images and examine their utilization across the data population (see~\Cref{app:expsetup} for detailed setup). Notably, conventional SAE-MLPs struggle to efficiently extract spatially local structure in the underlying data. In contrast, SAE-FNO learns sinusoidal filters for local patches and localized frequency patterns for larger images (\Cref{fig:concept-visualization,fig:concept-visualization-32}). Although parameterized via Fourier operators, the learned concepts are mapped back to the data domain for visualization, revealing localized structures similar to classical sparse coding and early visual system representations~\citep{olshausen1996emergence}.

SAE-FNO concept characteristics remain stable across varying sparsity levels, whereas SAE-MLP concepts change drastically and lose structure as $k$ increases. Remarkably, even without concept sparsity ($k=p$), domain sparsity alone drives SAE-FNO to learn concepts with similar structural properties (\Cref{fig:cifar-concept-stability-mode-ss}), while SAE-MLP degrades into noise-like patterns. This also holds across domain sparsity levels (\Cref{fig:change-strcuture-domain-cifar}). These structural properties, alongside gains in concept utilization and efficiency (\Cref{fig:imagenet-concept-utilization}), scale to ImageNet and persist across different sparsification and loss formulations, including BatchTopK and Matryoshka (\Cref{fig:imagenet-conceptstable-mlp,fig:imagenet-conceptstable-fno}). The active sparse codes and reconstructions of these variants are visualized in \Cref{fig:imagenet-rec-seq}. Finally, since we applied SAE-FNO directly on image data, learned concepts are low-level structures. SAE-FNO could naturally be applied to deep internal representations to extract high-level semantics~\citep{olshausen1997sparse}. 

This qualitative difference between SAE-MLP and SAE-FNO explains the efficiency observed in \Cref{fig:utilization}: SAE-FNO requires fewer active concepts to represent the data population. Specifically, SAE-FNO exhibits a much sharper decay in concept usage, indicating that a smaller set of structured concepts is sufficient to characterize the data, whereas SAE-MLP shows a flatter distribution and requires substantially more concepts. Consequently, SAE-FNO achieves higher expressivity at lower $p$ than SAE-MLP (\Cref{fig:imagenet-r2}). Moreover, the expressivity and smoothness of SAE-FNO can be controlled through the number of Fourier modes (\Cref{fig:fouriermode_ablation}), while domain sparsity mainly matters for expressivity when $k$ is small and has little effect once $k$ is large (\Cref{fig:expressivity-spatial-sparsity}). We note that concept sparsity in SAE-FNO is not directly equivalent to scalar sparsity in SAE-MLP. Our comparison therefore evaluates efficiency at the level of reusable structured concepts.

\textbf{SAE Generalization Under Discretization and Frequency-Structured Perturbations}\quad Finally, we show that SAE-FNO is more stable under spectral shifts, particularly discretization changes and frequency-structured perturbations. During inference, we perturb data with additive high-frequency noise and observe that SAE-FNO exhibits substantially smaller degradation than SAE-MLP. Specifically, standard SAEs show a large drop in reconstruction quality measured by PSNR (peak signal-to-noise ratio, higher is better), whereas SAE-FNO exhibits a much smaller relative decrease, indicating greater stability under high-frequency perturbations (\Cref{fig:imagenet-highfreq-noise}). This robustness naturally arises from the truncated Fourier parameterization of SAE-FNO.

Moreover, SAE-FNO is advantageous when spatial structure varies in scale or resolution at inference. Conventional SAEs operate on fixed-grid Euclidean representations, with concepts tied to a specific discretization. In contrast, the functional parameterization of SAE-FNO inherits key properties of neural operators, enabling inference across discretization levels. As shown in \Cref{fig:multires}, an SAE-FNO with $p=1000$ concepts trained at $28\times28$ resolution continues to represent data accurately at higher resolutions, whereas the reconstruction error of standard SAEs increases sharply beyond the training discretization. This discretization generalization persists across different Fourier mode truncations. Additional analysis under a known generative model is provided in \Cref{fig:simulated-fno-discretization-invariance}. Degradation under decimation is likewise a sampling effect rather than a parameterization limit. In \Cref{tab:discretization_decimation}, we show that recovery is preserved up to $4\times$ decimation and degrades only once aliasing becomes strong. Once aliasing is removed with a low-pass filter, recovery does not degrade at all. 
%
\vspace{-3mm}
\section{Discussion and Conclusion}\label{sec:discussion}
\vspace{-2mm}
\textbf{Ablating Functional Sparse Coding and Domain Sparsity}\quad The utility of SAE-FNO stems from four properties: i) operator-based parameterization provides discretization convergence, ii) Fourier parameterization introduces a spectral bias toward smooth global structures, iii) functional sparse coding represents each concept as a function over the domain, and iv) domain sparsity encodes where and how a concept is expressed. We argue that SAE-FNO's distinct advantages arise from the latter two: ablating concept and domain sparsity prevents the model from learning the structured and stable concepts observed in SAE-FNO (\Cref{fig:aefno-saefno-no-ss}).

\textbf{When is SAE-FNO Advantageous?}\quad SAE-FNO excels when data exhibit spatial structure that varies in scale or resolution. Its functional parameterization adapts the receptive field, an attribute absent in fixed-vector SAE-MLPs. This enables stable concept characteristics under domain size changes, in addition to the resolution shift and cross-discretization generalization.

Furthermore, SAE-FNO is highly effective when the underlying concepts are sparsely structured in the frequency domain. We demonstrate this on a synthetic time-series dataset comprising 10 single-frequency concepts, generated with a concept sparsity of 3 and a domain sparsity of 2 (i.e., each concept appears locally twice). Under additive Gaussian noise, SAE-FNO outperforms SAE-MLP in recovering the underlying concepts (\Cref{tab:synthetic-recovery}). Restricting SAE-FNO's frequency modes further improves performance. SAE-FNO's spectral inductive bias naturally filters noise and favors smooth, low-frequency concepts, offering a powerful practical advantage when target concepts are known not to be concentrated in high-frequencies (\Cref{tab:synthetic-recovery} last column).

\textbf{When is SAE-FNO Less Useful?}\quad SAE-FNO is designed for data with spatial or temporal structure, where neighboring dimensions are related or vary smoothly. Its Fourier parameterization introduces a spectral bias toward frequency-structured patterns, while domain sparsity assumes localized rather than globally dense concepts. When these assumptions fail (e.g., tabular, combinatorial, or highly irregular/discontinuous data), vector-based SAEs may be preferable. We construct one such case by applying a fixed random permutation to CIFAR pixels, which preserves every marginal statistic but removes smoothness and translation structure, hence violating the FRH (see~\Cref{app:perm} for details).

\textbf{Limitations}\quad The advantages of SAE-FNO incur higher computational and memory costs than SAE-MLPs due to Fourier operations and feature-map representations (\Cref{tab:throughput}, see \Cref{app:compute_cost}). For example, computational cost grows gradually and predictably with the number of modes, while domain sparsity is close to free: varying it over two orders of magnitude changes neither throughput nor memory materially (\Cref{tab:throughput_modes}). This trade-off is analogous to adopting CNNs over MLPs for improved structural inductive biases and representational efficiency. While at a scale representative of current SAE deployments, this gap persists and grows with dictionary size $p$ (\Cref{tab:throughput_deployment_scale}), comparing the two architectures at matched $p$ understates SAE-FNO's practical efficiency. A more practical comparison should evaluate the methods at matched reconstruction quality rather than matched $p$, since SAE-FNO reaches a given $R^2$ with an order of magnitude fewer concepts (\Cref{fig:imagenet-r2}).

SAE-FNO introduces two additional hyperparameters: number of modes (controlling smoothness) and domain sparsity (controlling spatial structure). Our ablations (\Cref{fig:cifar-fno-hoyer-modes,fig:fouriermode_ablation,fig:cifar-fno-hoyer-ss,fig:change-strcuture-domain-cifar}) demonstrate consistent behavior across modes, with expressivity increasing alongside more modes.

%
%
\textbf{Conclusion}\quad We introduced sparse autoencoder neural operators (SAE-NOs), a functional framework for sparse representation learning in structured domains. By replacing scalar codes with Fourier-based representations, we introduced two notions of sparsity: concept and domain sparsity. The resulting SAE, SAE-FNO, can capture which concepts are active and how and where they are expressed. Across multiple SAE variants and datasets, SAE-FNO learns more structured, stable, and efficient concepts, while naturally supporting spatial and temporal consistency, discretization generalization, and controllable expressivity. Our results highlight the central role of functional parameterization for learning representations, and open the door to applications of SAE-NOs in mechanistic interpretability of large neural operator models and scientific machine learning.

\newpage
\begin{ack}
B.T. acknowledges support from the Canada CIFAR AI Chairs Program and Alberta Machine Intelligence Institute (Amii). A.S. acknowledges support from Caltech's Summer Undergraduate Research Fellowship (SURF) and the Mellon Mays Undergraduate Fellowship. A.A. acknowledges support from the Bren Endowed Chair and the AI2050 Senior Fellow program. Authors also thank Digital Research Alliance of Canada and Amii for computing resources.
\end{ack}

\section*{Broader Impact}

This work advances the field of machine learning by introducing a new framework for representation learning and interpretability based on sparse autoencoder neural operators (SAE-NO). By extending sparse autoencoders to functional and multi-resolution settings, our approach enables more stable, generalizable concept representations, with potential benefits for scientific modeling and mechanistic interpretability in general. The methods proposed here are intended for understanding and characterizing learned representations rather than for deployment in high-stakes decision-making systems. Hence, we do not anticipate immediate negative societal impacts. However, as with any interpretability or representation learning tool, insights derived from these models could influence how complex systems are analyzed or trusted. Care should therefore be taken when applying SAEs in sensitive domains, and conclusions should be interpreted in conjunction with domain expertise. Overall, this work aims to contribute foundational insights into representation learning and mechanistic interpretability. We do not identify any foreseeable ethical concerns specific to this contribution.

\bibliographystyle{unsrt}
\bibliography{references}

\newpage
\appendix
\input{appendix}


\end{document}

%% file: math_commands.tex
\usepackage{amsmath,amsfonts,bm}

\def\eqref#1{equation~\ref{#1}}

\def\1{\bm{1}}

\DeclareMathAlphabet{\mathsfit}{\encodingdefault}{\sfdefault}{m}{sl}
\SetMathAlphabet{\mathsfit}{bold}{\encodingdefault}{\sfdefault}{bx}{n}

\newcommand{\R}{\mathbb{R}}

\DeclareMathOperator*{\argmin}{arg\,min}

%% file: notations.tex
\newcommand{\W}{{\bm W}}

\newcommand{\D}{{\bm D}}
\newcommand{\lift}{{\bm L}}
\newcommand{\proj}{{\bm P}}

\newcommand{\eye}{{\bm I}}

\newcommand{\smlb}{{\bm b}}
\newcommand{\smlv}{{\bm v}}
\newcommand{\smlw}{{\bm w}}
\newcommand{\y}{{\bm y}}
\newcommand{\x}{{\bm x}}
\newcommand{\z}{{\bm z}}


%% file: appendix.tex
\section{Appendix - Additional Background}\label{app:background}

\begin{table}[h]
\centering
\caption{Notation across continuous, discretized, Fourier-domain, and lifted spaces.}
\label{tab:notation}
\begin{tabular}{lll}
\toprule
\textbf{Setting} & \textbf{Symbol} & \textbf{Meaning} \\
\midrule
Continuous (FRH) & $\mathcal{Z}, \mathcal{X}$ & Function (Banach) spaces of codes / data \\
 & $\mathcal{G}_{D^*}: \mathcal{Z}\to\mathcal{X}$ & Ground-truth generative operator \\
 & $\mathcal{H}_\theta: \mathcal{X}\to\mathcal{Z}$ & Approximate inverse (encoder) operator \\
 & $D_c \in \mathbb{R}^h$ & Spatial-domain kernel for concept $c$ \\
\midrule
Discretized & $x \in \mathbb{R}^m$ & Discretized data sample \\
 & $z_c \in \mathbb{R}^d$ & Discretized code / feature map for concept $c$ \\
 & $D \in \mathbb{R}^{m\times p}$ & Discretized dictionary \\
 & $p$ & Dictionary size \\
 & $k$ & Concept sparsity \\
 & $s$ & Domain sparsity \\
\midrule
Fourier domain & $\mathcal{F}, \mathcal{F}^{-1}$ & Forward / inverse Fourier transform \\
 & $W_c = \mathcal{F}D_c$ & Frequency-domain weights for concept $c$ \\
 & $M$ & Number of retained Fourier modes \\
\midrule
Lifted space & $L$ & Lifting operator ($x \mapsto Lx$) \\
 & $P$ & Projection operator ($\hat{y}\mapsto P\hat y$) \\
 & $D_L$ & Lifted spatial-domain dictionary \\
 & $ W_L$ & Lifted Fourier-domain dictionary \\
\bottomrule
\end{tabular}
\end{table}

\textbf{Sparse model recovery}\quad The bilevel optimization takes a general form. When the inner level objective $\mathcal{L}(\x, \smlv, \theta) =  \frac{1}{2} \| \x - \D \smlv \|_2^2 + \mathcal{R}(\smlv)$, this problem is reduced to the classical alternating-minimization dictionary learning~\citep{agarwal2016learning, chatterji2017alternating}. In this classical form, sparse coding is widely applied in science and engineering; e.g., in signal processing~\citep{elad2010sparse}, image denoising~\citep{aharon2006ksvd}, and extracting interpretable gene modules~\citep{cleary2017trans, cleary2021compressed}. When the model parameters $\D$ are known, the problem is reduced to recovering the sparse representation by solving $\min_{\z \in \R^{p}}\ \tfrac{1}{2}\| \x  - \D \z \|_2^2 + \lambda \mathcal{R}(\z)$, where $\mathcal{R}(\z)$ is a sparse regularizer; setting $\mathcal{R}(\z) = \| \z \|_1$, the optimization problem is reduced to the lasso~\citep{tibshirani1996lasso}, or referred to as basis pursuit~\citep{chen2001atomic}, solving via proximal gradient descent~\citep{parikh2014proximal} such as iterative shrinkage-thresholding algorithm (ISTA)~\citep{daubehies2004ista} and its fast momentum-based version~\citep{beck2009fast}. 

\textbf{Sparse ReLU autoencoders}\quad For shallow ReLU autoencoders~\citep{bricken2023towards}, the formulation reduces to
\begin{equation}\label{eq:relusae}
    \begin{aligned}
        \min_{\D \in \mathcal{D}}\quad & \tfrac{1}{2} \| \x - (\D \z + \smlb_{\text{pre}}) \|_2^2 + \lambda \| \z \|_1\\
        \text{s.t}\quad & \z = \text{ReLU}(\W^\top (\x - \smlb_{\text{pre}}) + \smlb_{\text{enc}})
    \end{aligned}
\end{equation}
This connection between bilevel optimization and neural networks, highlighted in~\citep{tolooshams2023phd, hindupur2025projecting}, has been explored extensively in the context of sparse autoencoders. Prior works have studied the gradient dynamics of variants of ReLU networks when the encoder is shallow~\citep{nguyen2019dynamics, arora2015simple, rangamani2018sparseae} or deep, as in learned ISTA architectures~\citep{tolooshams2022stable}. Others have analyzed model recovery when the encoder is deep and iterative with hard-thresholding nonlinearity (JumpReLU~\citep{rajamanoharan2024jumping})~\citep{rambhatla2018noodl}.

Overall, optimizing or recovering sparse generative linear models has attracted sustained attention from two main deep learning communities over the past years; a) \textit{Unrolling learning}: using optimization formulations to design and theoretically study neural architectures~\citep{ablin2019stepsize, malezieux2022understanding, tolooshams2022stable, nguyen2019dynamics, arora2015simple, chen2018unfoldista, rangamani2018sparseae, tolooshams2020icml, rambhatla2018noodl}, and b) \textit{Sparse interpretability} or mechanistic interpretability: leveraging sparse models to interpret and analyze the internal representations of larger complex networks~\citep{elhage2022toy, huben2023sparse, bricken2023towards, rajamanoharan2024jumping, park2024linear, templeton2024scaling, lieberum2024gemma, gao2025scaling, fel2025archetypal, marks2025sparse, karvonen2025saebench}. We briefly expand on both works.

\textbf{Unrolled learning}\quad The early connections between sparse coding and deep learning trace back to sparse energy-based deep models~\citep{ranzato2007sparsenergy, ranzato2008sparsedbn} and the pioneering work of LISTA~\citep{gregor2010lista} in constructing sparsifying recurrent neural networks. This line of work, known as unfolding~\citep{hershey2014deep} or unrolling~\citep{monga2021algorithm}, designs neural network layers based on iterations of an algorithm that solves an optimization problem. This connection has enabled researchers to use optimization models as a proxy to theoretically study accelerated convergence~\citep{chen2018unfoldista, ablin2019stepsize}, gradient dynamics~\citep{rangamani2018sparseae, nguyen2019dynamics, tolooshams2020icml}, and model recovery in both shallow~\citep{arora2015simple} and deep neural networks~\citep{rambhatla2018noodl, tolooshams2022stable}. Furthermore, several works have explored the theoretical connections between convolutional neural networks and convolutional sparse coding~\citep{papyan2017convolutional, papyan2017working}, or more generally, between deeply sparse signal representations and deep neural networks~\citep{ba2020deeply}. An example of a deep unrolled JumpReLU neural network for sparse model recovery can be formulated as the following bilevel optimization problem.
\begin{equation}\label{eq:jumpdeepbilevel}
    \begin{aligned}
        \min_{\D \in \mathcal{D}}\quad \| \x - \D \z_T \|_2^2\quad
        \text{s.t.}\quad \z_t = \text{JumpReLU}_{\lambda}(\z_{t-1} - \alpha \D^\top(\D \z_{t-1} - \x))
    \end{aligned}
\end{equation}
for $t = 1, \ldots, T$, where $\z_0 = \bm{0}$, $\alpha$ is the step size, and $\text{JumpReLU}_{\lambda}(\z) = \z \cdot \mathbbm{1}_{\z > \lambda} $, where $\mathbbm{1}$ is the indicator function, and $\lambda$ controls the sparsity level of the representation. The inner mapping is now a deep/iterative encoder, which for simplicity we denote by $\z_T = f_{\theta}(\x)$. The outer objective enforces the structure of the decoder and the loss function. The bilevel optimization in~\Cref{eq:jumpdeepbilevel} can be mapped to the following neural network autoencoder architecture, which uses a recurrent and residual encoder. The synthetic experiments in this paper use this encoder in its architecture for sparse model recovery.
\begin{equation}\label{eq:jumpdeepsae}
\begin{aligned}
    \text{(encoder)}\quad & \z_t = \text{JumpReLU}_{\lambda}(\z_{t-1} - \alpha \D^\top(\D \z_{t-1} - \x)),\quad  \text{for}\quad t=1,\ldots,T\\
    \text{(decoder)}\quad & \hat \x = \D \z_T.
\end{aligned}
\end{equation}
For experiments recovering the sparse convolutional generative model, we use the architecture that takes convolution blocks, as shown below.
\begin{equation}\label{eq:jumpdeepsaeconv}
\begin{aligned}
    \text{(encoder)}\quad & \z_{c,t} = \text{JumpReLU}_{\theta}(\z_{c,t-1} - \alpha \D_c \star (\sum_{i=1}^p \D_i * \z_{i,t-1} - \x)),\quad  \text{for}\quad t=1,\ldots,T\\
    \text{(decoder)}\quad & \ \ \ \hat \x = \D * \z_T
\end{aligned}
\end{equation}
for $t = 1, \ldots, T$ and $c=1,\ldots,p$, where $\z_0 = \bm{0}$, $\alpha$ is the step size, $*$ is convolution operator, and $\star$ is a correlation operator. Recent work has extended unrolling ideas to neural operators~\citep{he2024selfcomposing}.

\textbf{Convolutional Setting}\quad Sparse generative models can be extended to \textit{sparse convolutional generative models} (\Cref{def:sgmconv}), where the concepts are locally and sparsely appearing in the data. 

\begin{definition}[Sparse Convolutional Generative Models]\label{def:sgmconv}
A data example $\x \in \R^{m}$ is said to follow a sparse convolutional generative model if there exists a sparse latent representation $\z = \{\z_c \in \R^{d}\}_{c=1}^p$ (with $\text{supp}(\z_c) \leq k \ll m$) and a set of $p$ localized dictionary $\{\D^{*}_c \in \R^{h}\}_{c=1}^p$ (with $h \ll m$) such that $\x = {\textstyle \sum_{c=1}^p} \D_c ^{*} \ast \z_c$
in the noiseless setting, where $\ast$ denotes the convolution operator.
\end{definition}

\textbf{Neural operators}\quad Neural operators consist of the following three modules:
\begin{itemize}[leftmargin=5mm]
    \setlength\itemsep{0em}
    \item {\bf Lifting}: This is a fully local operator which we model by a matrix $\lift: \R^{m} \rightarrow \R^{d_{v_0}}$. It maps the input $\{\x: D_x \rightarrow \R^{m}\}$ into a latent representation $\{\smlv_0: D_0 \rightarrow \R^{d_{v_0}}\}$, where $h > m$.
    \item {\bf Kernel Integration}: For $t=0,\ldots, T-1$, this is a non-local integral kernel operator that maps representation at one layer $\{\smlv_t: D_t \rightarrow \R^{d_{v_t}}\}$ to the next $\{\smlv_{t+1}: D_{t+1} \rightarrow \R^{d_{v_{t+1}}}\}$ (see~\Cref{def:kernelopt}).
    \item {\bf Projection}: This is a fully local operator, similar to the lifting operator. It maps the filtered lifted data to the output function, i.e., $\proj: \R^{d_{v_T}} \rightarrow \R^{m}$, where $d_{v_T} > m$.
\end{itemize}
where the kernel integral operator is used to map $\x$ to an estimate of the representation $\z$, the input/output domains would be $D_x$ and $D_z$, respectively. Moreover, when the kernel integral operator is used to refine the latent representation $\z$ from one layer to another, the input/output domain of the operator would be both $D_z$. 

\begin{definition}[Kernel integral operator $\mathcal{K}$ (restated from~\citep{kovachki2023neural})]\label{def:kernelopt}
Define the kernel integral operator by
\begin{equation}
(\mathcal{K}_t(\smlv_t))(x) \coloneqq \int_{D_t} \kappa^{(t)}(x,y) \smlv_t(y) d\smlv_t(y),\quad \forall x \in D_{t+1}
\end{equation}
where $\kappa^{(t)}: C(D_{t+1} \times D_t; \R^{dv_{t+1} \times dv_{t}})$ are the parameters of the kernels, modeled by a neural network, and $v_t$ is a Borel measure on $D_t$, where $C$ denotes the space of continuous functions.
\end{definition}
\begin{definition}[Fourier integral operator $\mathcal{K}$ (restated from~\citep{kovachki2023neural})]\label{def:fno}
Define the Fourier integral operator by
\begin{equation}
(\mathcal{K}(\phi)\smlv_t)(\x) = \mathcal{F}^{-1} \big(R_{\phi} \cdot (\mathcal{F}\smlv_t)\big)(\x),\qquad \forall \x \in D
\end{equation}
where $R_{\phi} \coloneqq \mathcal{F}(\kappa)$ is the Fourier transform of a periodic function $\kappa: D \rightarrow \mathbb{C}^{d_w \times d_v}$ parameterized by $\phi$.
\end{definition}

Fourier neural operator (FNO)~\citep{li2021fourier} models the kernel integral operator with a convolution operator parameterized in Fourier space (\Cref{def:fno}, see also~\Cref{def:fourier} for Fourier transform). The convolution operator parameterizes the kernel $\kappa(x,y) = \kappa(x - y)$ as a complex function $\kappa: D \rightarrow \mathbb{C}^{d_w \times d_v}$. FNO is used with truncated frequency modes in practice; this has been shown to improve performance and lower sensitivity to change (decreasing) the discretization sampling~\citep{li2021fourier, kovachki2023neural}.
\begin{definition}[Fourier transform]\label{def:fourier}
Let $\mathcal{F}$ denote the Fourier transform of the function $\smlv: D \rightarrow \R^{d_v}$, whose inverse is denoted by $\mathcal{F}^{-1}$ on the function $\smlw: D \rightarrow \mathbb{C}^{d_w}$. We have
\begin{equation}
\begin{aligned}
(\mathcal{F}\smlv)_j(k) &= \int_{D} \smlv_j(x) e^{-2i\pi\langle x,k\rangle}dx,\qquad j=1,\ldots, d_v\\
(\mathcal{F}^{-1}\smlw)_j(x) &= \int_{D} \smlw_j(k) e^{2i\pi\langle x,k\rangle}dk,\qquad j=1,\ldots, d_w
\end{aligned}
\end{equation}
where $i = \sqrt{-1}$ is the imaginary unit.
\end{definition}

\textbf{MACO}\quad Given the Fourier parameterization in our framework, it may appear related to MACO~\citep{fel2023unlocking}. However, the two approaches differ fundamentally in both goal and formulation. MACO uses Fourier parameterization for feature visualization, imposing phase and magnitude constraints to generate natural images. In contrast, our work focuses on representation learning, where Fourier operators parameterize concepts within a functional sparse coding framework, without imposing explicit constraints on Fourier modes.

\section{Appendix - Theoretical Analysis}

This section provides the mathematical derivations for the training dynamics of the architectures discussed in the main text. We derive the gradient descent update rules for a single data sample, which are used to learn the dictionary parameters. These derivations consider the \emph{analytic gradient}, i.e., the dictionary gradient is computed given the codes~\citep{malezieux2022understanding, tolooshams2022stable}.

\begin{proposition}
    [Architectural Inference Equivalence of Lifting in SAE]
    \label{prop:lift-arch-equiv}
    The architectural inference of an SAE is equivalent to a lifted-SAE if the projection is tied to the lifting operator ($\proj = \lift^\top$), $\operatorname{col}(\D_L)\subseteq\operatorname{range}(\lift)$, and the lifting operator is orthogonal ($\lift^\top\lift = \eye$), i.e., under the same learned model $\D = \proj \D_L$, the representation dynamics and output of the networks are equivalent.
\end{proposition}

\begin{proof}
We provide the proof for the sparse generative model case when the encoder implements a proximal gradient descent type of operation in an unrolled learning setting (e.g.,~\Cref{eq:jumpdeepsae}). For the SAE, the representation refinement at every encoder layer follows:
\begin{equation}
    \z_{t + 1} = \mathcal{R}(\z_t  - \alpha \D^\top (\D\z_t - \x))
\end{equation}
By construction, we incorporate the underlying model from the lifted-SAE to the standard SAE: $\D = \proj \D_L$, and re-write the encoder as:
\begin{equation}
    \z_{t + 1} = \mathcal{R}(\z_t  - \alpha (\proj\D_L)^\top (\proj\D_L\z_t - \x))= \mathcal{R}(\z_t  - \alpha \D_L^\top\proj^\top (\proj\D_L\z_t - \x))
\end{equation}
Using the assumption on the lifting-projection,
\begin{equation}
    \z_{t + 1} = \mathcal{R}\left(\z_t  - \alpha (\D_L)^\top(\D_L\z_t - \lift \x)\right)
\end{equation}
This final expression is the iterative encoder update for $\z$ in the lifted space, where the dictionary is $\D_L$ and the input is lifted by $\lift$. This completes the proof of architectural inference equivalence on the standard and lifted architectures.
\end{proof}


%
\begin{restatable}[Training Dynamics of Lifting]{proposition}{lifttrainequiv}\label{prop:lift_train_equiv}
\vspace{-2mm}
    The training dynamics of the lifted-SAE (L-SAE) $\D_L^{(k+1)}=\D_L^{(k)} + \eta_L \proj^T (\x - \proj\D_L^{(k)}\z)\z^{T}$, with lifting $\lift$ and projection $\proj$, has the effective update in the original space, expressed as: $\D^{(k+1)} = \D^{(k)} + \eta_L (\lift^\top \lift)(\x - \D^{(k)}\z)\z^\top$, where $\lift^\top \lift$ acts as a preconditioner, potentially accelerating learning by inducing a more isotropic update. If they satisfy the equivalent architectural inference (\Cref{prop:lift-arch-equiv}), then the dynamics are equivalent to those of an SAE: $\D^{(k+1)} = \D^{(k)} + \eta (\x - \D^{(k)}\z)\z^T$.  
\end{restatable}

\begin{proof}
For the SAE, we learn a dictionary $\D$ by minimizing the reconstruction loss for the data sample $\x$. The loss function is:
\begin{equation}
    \mathcal{L}(\D) = \dfrac{1}{2} \left\lVert \x - \D\z \right\lVert_2^2
\end{equation}
The gradient of the loss for $\D$ is:
\begin{equation}
    \dfrac{\partial\mathcal{L}}{\partial \D} = -(\x - \D\z)\z^\top
\end{equation}
The analytic gradient update, given $\z$, for $\D$ at iteration $k$ with a learning rate $\eta$ is therefore:
\begin{equation}
    \D^{(k + 1)} = \D^{(k)} - \eta  \dfrac{\partial\mathcal{L}}{\partial \D^{(k)}} = \D^{(k)} + \eta (\x - \D^{(k)}\z)\z^\top
\end{equation}
For lifted-SAE, we learn a dictionary $\D_L \in \R^{h \times p}$ in a higher-dimensional lifted space. The input $\x$ is lifted by $\lift \in \R^{h \times m}$ and the reconstruction is projected back by $\proj \in \R^{m \times h}$.

The loss function is defined in the original data space:
\begin{equation}
    \mathcal{L}(\D) = \dfrac{1}{2} \left\lVert \x - \proj\D_L\z \right\lVert_2^2
\end{equation}

The analytic gradient given $\z$ of the loss with respect to $\D_L$ is:
\begin{equation}
     \dfrac{\partial\mathcal{L}}{\partial \D_L} = -\proj^\top(\x - \proj\D_L\z)\z^\top
\end{equation}

The gradient update rule for the lifted dictionary $\D_L$ is:
\begin{equation}
    \D_L^{(k + 1)} = \D_L^{(k)} - \eta_L  \dfrac{\partial\mathcal{L}}{\partial \D_L^{(k)}} = \D_L^{(k)} + \eta_L \proj^\top (\x - \proj\D_L\z)(\z)^\top
\end{equation}

To understand the learning dynamics in the original space, we first incorporate the underlying model from L-SAEs into SAE using the relation $\D' = \proj \D_L$. Then, we derive the update rule for the effective dictionary, $\D'$. Left-multiplying the update for $\D_L^{(k + 1)}$ by $\proj$, we have:
\begin{equation}
    \D'^{(k + 1)} = \D'^{(k)} + \eta_L (\proj\proj^\top)(\x - \D'^{(k)}\z)(\z)^\top
\end{equation}

In our specific architecture, we constrain the projection matrix to be the transpose of the lifting matrix such that $\proj = \lift^\top$. Substituting this into the update rules yields:
\begin{equation}
    \D_L^{(k + 1)} = \D_L^{(k)} + \eta_L \lift (\x - \lift^\top\D_L\z)(\z)^\top
\end{equation}
\begin{equation}
    \D'^{(k + 1)} = \D'^{(k)} + \eta_L (\lift^\top\lift)(\x - \D'^{(k)}\z)(\z)^\top
\end{equation}
Hence, when $\lift^\top\lift = \eye$ and $\eta_L = \eta$, the lifted-SAE shows equivalent learning dynamics as the standard SAE for model recovery.
\end{proof}

\begin{proposition}
    [Architectural Inference Equivalence of Lifting in SAE-CNN]
    \label{prop:conv-lift-arch-equiv}
    The architectural inference of an SAE-CNN is equivalent to a lifted SAE-CNN (L-SAE-CNN) if the projection is tied to the lifting operator ($\proj = \lift^\top$), $\operatorname{col}(\D_L)\subseteq\operatorname{range}(\lift)$, and the lifting operator is orthogonal ($\lift^\top\lift = \eye$), i.e., under the same learned model $\D_c = \proj \D_{L, c}$, the representation dynamics and output of the networks are equivalent.
\end{proposition}

\begin{proof}
We provide the proof for the sparse convolutional generative model case when the encoder implements a proximal gradient descent type of operation in an unrolled learning setting (e.g.,~\Cref{eq:jumpdeepsaeconv}). For the SAE-CNN, the representation refinement for the $c$-th feature map at every encoder layer follows:
\begin{equation}
    \z_{c, t + 1} = \mathcal{R} \left(\z_{c, t} - \alpha \left( \sum\limits_{i = 1}^p \D_i * \z_{i, t} - \x\right) \star \D_c \right)
\end{equation}

By construction, we incorporate the underlying model from the lifted-SAE-CNN to the standard SAE-CNN: $\D_c^{(k)} = \proj\D_{L, c}^{(k)}$, and re-write the encoder using the adjoint property:
\begin{equation}
    \z_{c, t + 1} = \mathcal{R} \left(\z_{c, t} - \alpha \left( \proj^\top \left[ \left( \sum\limits_{i = 1}^p \proj\D_{L, i} * \z_{i, t} \right) - \x\right] \star \D_{L, c}  \right) \right)
\end{equation}
\begin{equation}
    \z_{c, t + 1} = \mathcal{R} \left(\z_{c, t} - \alpha \left( \sum\limits_{i = 1}^p (\proj^\top\proj)\D_{L, i} * \z_{i, t} - \proj^\top\x \right)  \star \D_{L, c}  \right) 
\end{equation}

Using the assumption on the lifting-projection,
\begin{equation}
    \z_{c, t + 1} = \mathcal{R} \left(\z_{c, t} - \alpha \left( \sum\limits_{i = 1}^p \D_{L, i} * \z_{i, t} - \lift\x \right)  \star \D_{L, c} \right) 
\end{equation}

This final expression is the iterative encoder update for $\z_c$ in the lifted space, where the dictionary is $\D_L$ and the input is lifted by $\lift$. This completes the proof of architectural inference equivalence for SAE-CNN and L-SAE-CNN.
\end{proof}

\begin{proposition}
    [Training Dynamics of Lifted SAE-CNN]\label{prop:conv_lift_train_equiv}
     The training dynamics of the lifted-SAE-CNN (L-SAE-CNN) $\D_{L, c}^{(k + 1)} = \D_{L, c}^{(k)} + \eta_L \lift\left(\x - \lift^\top \left(\sum\limits_{i = 1}^p \D_{L, i}^{(k)} * \z_{i, t} \right) \right) \star \z_{c, t}$, with a lifting operator $\lift$ and projection $\proj$, has the effective update in the original space, expressed as: $\D_c^{(k + 1)} = \D_c^{(k)} + \eta_L (\lift^\top\lift)\left(\x - \sum\limits_{i = 1}^p \D_i^{(k)} * \z_{i, t} \right)\star \z_{c, t}$, where $\lift^\top \lift$ acts as a preconditioner, potentially accelerating learning by inducing a more isotropic update. If the architectures satisfy the architectural inference equivalence condition (\Cref{prop:conv-lift-arch-equiv}), then the dynamics become equivalent to that of an SAE-CNN: $\D_c^{(k + 1)} = \D_c^{(k)} + \eta \left(\x - \sum\limits_{i = 1}^p \D_i^{(k)} * \z_{i, t}\right) \star \z_{c, t}$.
\end{proposition}

\begin{proof}
For the SAE-CNN, we learn a dictionary $\D$ by minimizing the reconstruction loss for the data sample $\x$. The loss function is:
\begin{equation}
    \mathcal{L}(\D) = \dfrac{1}{2} \left\lVert \x - \sum\limits_{c = 1}^p \D_c * \z_c \right\lVert_2^2
\end{equation}

The gradient of the loss with respect to the $c$-th kernel $\D_c$ is:
\begin{equation}
    \dfrac{\partial\mathcal{L}}{\partial \D_c} = -\left(\x - \sum\limits_{i = 1}^p \D_i * \z_i \right) \star \z_c
\end{equation}
where $\star$ denotes cross-correlation.

The analytic gradient update, given $\z$, for $\D_c$ at iteration $k$ with a learning rate $\eta$ is:
\begin{equation}
    \D_c^{(k + 1)} = \D_c^{(k)} + \eta \left(\x - \sum\limits_{i = 1}^p \D_i^{(k)} * \z_{i, t}\right) \star \z_{c, t}
\end{equation}

For L-SAE-CNN, we learn a dictionary $\D_L = \{\D_{L, c}\}_{c = 1}^p$ that operates in a higher-dimensional lifted space. The input signal $\x$ is first lifted by $\lift$ (acts on the channel dimension) to $\tilde{x}$ and the reconstruction $\hat{\x}$ is obtained by projecting the lifted reconstruction $\hat{\tilde{\x}} = \sum_c \D_{L, c} * \z_c$ back by $\proj$.

The loss function is defined in the original data space: 
\begin{equation}
    \mathcal{L}(\D) = \dfrac{1}{2} \left\lVert \x - \proj \left(\sum\limits_{c = 1}^p \D_{L, c} * \z_c \right) \right\lVert_2^2
\end{equation}

The analytic gradient of the loss, given $\z$,  with respect to the $c$-th lifted kernel $\D_{L, c}$ is:
\begin{equation}
     \dfrac{\partial\mathcal{L}}{\partial \D_{L, c}} = - \proj^\top\left(\x - \proj \left(\sum\limits_{i = 1}^p \D_{L, i} * \z_i \right) \right) \star \z_c
\end{equation}

The gradient update rule for $\D_{L, c}$ is:
\begin{equation}
    \D_{L, c}^{(k + 1)} = \D_{L, c}^{(k)} + \eta_L \proj^\top\left(\x - \proj \left(\sum\limits_{i = 1}^p \D_{L, i}^{(k)} * \z_{i, t} \right) \right) \star \z_{c, t}
\end{equation}

To understand the learning dynamics in the original space, we derive the update rule for the effective dictionary $\D'_c = \proj \D_{L, c}$, i.e., the case where both frameworks contain the same underlying model but in a different form. Left-multiplying the update for $\D_{L, c}^{(k + 1)}$ update rule by $\proj$, we get:
\begin{equation}
    \D_c'^{(k + 1)} = \D_c'^{(k)} + \eta_L (\proj\proj^\top)\left(\x - \sum\limits_{i = 1}^p \D_i'^{(k)} * \z_{i, t} \right)\star \z_{c, t}
\end{equation}

In our specific architecture, we constrain the projection matrix to be the transpose of the lifting matrix such that $\proj = \lift^\top$. Substituting this into the update rules yields:
\begin{equation}
    \D_{L, c}^{(k + 1)} = \D_{L, c}^{(k)} + \eta_L \lift\left(\x - \lift^\top \left(\sum\limits_{i = 1}^p \D_{L, i}^{(k)} * \z_{i, t} \right) \right) \star \z_{c, t}
\end{equation}
\begin{equation}
    \D_c'^{(k + 1)} = \D_c'^{(k)} + \eta_L (\lift^\top\lift)\left(\x - \sum\limits_{i = 1}^p \D_i'^{(k)} * \z_{i, t} \right)\star \z_{c, t}
\end{equation}

Hence, when $\lift^\top\lift = \eye$, the lifted-SAE-CNN shows equivalent learning dynamics as the standard SAE-CNN for model recovery.
\end{proof}

\liftfnotrain*

\begin{proof}
For the SAE-FNO, we learn frequency-domain weights $\W_c = \mathcal{F}\D_c$ by minimizing the reconstruction loss for the data sample $\x$. Taking normalization into account, the reconstruction operator is $\hat{\x} = \mathcal{G}_{\W}(\z) = \frac{1}{\sqrt{M}} \mathcal{F}^{-1} \big(\sum_{c=1}^p \W_c \cdot (\mathcal{F} \z_{c})\big)$, where $\W_c = \mathcal{F} \D_c$, where $M$ is the number of modes. The loss function is:
\begin{equation}
    \mathcal{L}(\W) = \dfrac{1}{2} \mathbb{E}_{\z \sim \mu} \left[\left\lVert \x - \mathcal{G}_{\W}(\z) \right\lVert_{\mathcal{X}}^2 \right]
\end{equation}
where $\x = \mathcal{G}_{\D^{*}}(\z)$ is the ground-truth signal.

The gradient of the loss with respect to the $c$-th  frequency-domain kernel $\W_c$ is:
\begin{equation}
    \dfrac{\partial\mathcal{L}}{\partial \W_c} = - \dfrac{1}{\sqrt{M}} \mathcal{F}\left(\x - \hat{\x} \right) \odot \overline{\mathcal{F}(\z_c)}
\end{equation}

The analytic gradient update, given $\z$, for $\W_c$ at iteration $k$ with a learning rate $\eta$ is:
\begin{equation}
    \W_c^{(k + 1)} = \W_c^{(k)} + \dfrac{\eta}{\sqrt{M}}\mathcal{F}\left(\x - \hat{\x}^{(k)}\right) \odot \overline{\mathcal{F}(\z_{c, k})}
\end{equation}

Applying the inverse Fourier transform and the cross-correlation theorem, we obtain the effective update rule for the spatial dictionary $\D_c'$:
\begin{equation}
    \D_c'^{(k + 1)} = \D_c'^{(k)} + \dfrac{\eta}{M}\left(\x - \hat{\x}^{(k)} \right) \star \z_{c, t} = \D_c'^{(k)} + \dfrac{\eta}{M}\left(\x - \dfrac{1}{\sqrt{M}} \mathcal{F}^{-1}  \sum\limits_{i = 1}^p \mathcal{F}\D_i^{(k)} \odot \mathcal{F}\z_{i, k} \right) \star \z_{c, t}
\end{equation}

In the lifted scenario, we learn a dictionary $\W_{L} = \{\W_{L, c}\}^p_{c = 1}$ that operates in a higher-dimensional lifted frequency domain. The input signal $\x$ is lifted by $\lift$, and the reconstruction is projected back by $\proj$. The reconstruction operator is $\hat{\x} = \proj \left(\mathcal{G}_{\W_L}(\z) \right) = \proj \left( \frac{1}{\sqrt{M}} \mathcal{F}^{-1} \left( \sum_{c = 1}^p \W_{L, c} \odot \mathcal{F}(\z_c) \right)\right)$.

The loss function is defined in the original spatial space:
\begin{equation}
    \mathcal{L}(\W) = \dfrac{1}{2} \mathbb{E}_{\z \sim \mu} \left[\left\lVert \x - \proj \left(\mathcal{G}_{\W_L}(\z) \right)\right\lVert_{\mathcal{X}}^2 \right]
\end{equation}

The gradient of the loss with respect to $\W_{L, c}$, given $\z$, is:
\begin{equation}
    \dfrac{\partial\mathcal{L}}{\partial \W_{L, c}} = - \dfrac{1}{\sqrt{M}} \mathcal{F} \left( \proj^\top \left( \x - \hat{\x}^{(k)}\right)\right) \odot \overline{\mathcal{F}(\z_c)}
\end{equation}

The analytic gradient update for $\W_{L, c}$ at iteration $k$ with a learning rate $\eta_L$ is:
\begin{equation}
    \W_{L, c}^{(k + 1)} = \W_{L, c}^{(k)} + \dfrac{\eta_L}{\sqrt{M}} \mathcal{F} \left( \proj^\top \left( \x - \hat{\x}^{(k)} \right)\right) \odot \overline{\mathcal{F}(\z_{c, t})}
\end{equation}

To understand the learning dynamics in the original spatial-domain, we derive the update rule for the effective lifted spatial dictionary $\D'_{L, c}$.  Taking the inverse Fourier, we get:
\begin{equation}
    \D_{L, c}'^{(k + 1)} = \D_{L, c}'^{(k)} + \dfrac{\eta_L}{M} \left( \proj^\top \left( \x - \hat{\x}^{(k)} \right)\right) \star \z_{c, t}
\end{equation}

We also derive the update rule for the effective spatial dictionary in the original space $\D_c'$. Left-multiplying the update for $\D_{L, c}'$ by $\proj$, we have:
\begin{equation}
    \D_c'^{(k + 1)} = \D_c'^{(k)} + \dfrac{\eta_L}{M} \left( (\proj\proj^\top) \left( \x - \hat{\x}^{(k)} \right)\right) \star \z_{c, t}
\end{equation}

In our specific architecture, we constrain the projection matrix to be the transpose of the lifting matrix such that $\proj = \lift^\top$. Substituting this into the update rules yields:
\begin{equation}
    \W_{L, c}^{(k + 1)} = \W_{L, c}^{(k)} + \dfrac{\eta_L}{\sqrt{M}} \mathcal{F} \left( \lift \left( \x - \hat{\x}^{(k)} \right)\right) \odot \overline{\mathcal{F}(\z_{c, t})}
\end{equation}
\begin{equation}
    \D_{L, c}'^{(k + 1)} = \D_{L, c}'^{(k)} + \dfrac{\eta_L}{M} \left( \lift \left( \x - \hat{\x}^{(k)} \right)\right) \star \z_{c, t}
\end{equation}
\begin{equation}
    \D_c'^{(k + 1)} = \D_c'^{(k)} + \dfrac{\eta_L}{M} \left( (\lift^\top\lift) \left( \x - \hat{\x}^{(k)} \right)\right) \star \z_{c, t}
\end{equation}

Substituting for $\hat{\x}^{(k)}$, the full effective spatial-domain update is:
\begin{equation}
    \D_c'^{(k + 1)} = \D_c'^{(k)} + \dfrac{\eta_L}{M} \left( (\lift^\top\lift) \left( \x - \dfrac{1}{\sqrt{M}} \mathcal{F}^{-1} \left( \sum\limits_{i = 1}^p \mathcal{F}(\D_i'^{(k)}) \odot \mathcal{F}(\z_{i, t}) \right) \right)\right) \star \z_{c, t}
\end{equation}

Hence, when $\lift^\top\lift = \eye$, the lifted-SAE-FNO shows equivalent learning dynamics as the standard SAE-FNO for model recovery.
\end{proof}

\begin{proposition}
    [Architectural Inference Equivalence of SAE-CNN and SAE-FNO]
    \label{prop:conv-operator-sae-arch-equiv}
    The architectural inference of an SAE-CNN is equivalent to a sparse autoencoder neural operator (SAE-NO) if the integral operator of SAE-FNO and parameters of SAEs are defined as a convolution where the SAE-FNO's frequency-domain weights $\W_c$ are the Fourier transform of SAE-CNN's spatial domain kernels: $\W_c = \mathcal{F}\D_c$.
\end{proposition}

\begin{proof}
We provide the proof for an encoder that implements a proximal gradient descent type of operation in an unrolled learning setting. For the SAE-CNN, the representation refinement for the $c$-th feature map at every encoder layer follows:

\begin{equation}
    \label{eq:l-convsae-code-update}
    \z_{c, t + 1} = \mathcal{R}\left(\z_{c, t}  - \alpha \left(\sum\limits_{i = 1}^p\D_i * \z_{i, t} - \x\right) \star \D_c \right)
\end{equation}

By construction, we incorporate the underlying model from the SAE-FNO to the SAE-CNN: $\W_c = \mathcal{F}\D_c$, and re-write the gradient of the encoder using correlation and convolution Theorems:
\begin{equation}
    \mathcal{F}^{-1} \left(\mathcal{F} \left(\sum\limits_{i = 1}^p\D_i * \z_{i, t} - \x \right) \odot \mathcal{F}\D_c^H \right)= \mathcal{F}^{-1} \left(\sum\limits_{i = 1}^p\W_i \odot \mathcal{F}\z_{i, t} - \mathcal{F}\x\right) \odot (\W_c)^H 
\end{equation}

Substituting this back into \Cref{eq:l-convsae-code-update} gives:
\begin{equation}
     \z_{c, t + 1} = \mathcal{R}\left(\z_{c, t}  - \alpha \mathcal{F}^{-1} \left(\left(\sum\limits_{i = 1}^p\W_i \odot \mathcal{F}\z_{i, t} \right) - \mathcal{F}\x\right) \odot (\W_c)^H\right)
\end{equation}

This final expression is the SAE-FNO's iterative encoder update for $\z_c$ in the spectral domain, where the dictionary is $\W$. This completes the proof of architectural inference equivalence between SAE-CNN and SAE-FNO.
\end{proof}

\begin{proposition}
    [Training Dynamics of SAE-CNN and SAE-FNO]
    \label{prop:conv-operator-equiv}
    The training dynamics of the SAE-FNO's frequency-domain weights $\W_{c}^{(k + 1)} = \W_{c}^{(k)} + \dfrac{\eta}{\sqrt{M}} \mathcal{F} \left( \x - \hat{\x}^{(k)} \right) \odot \overline{\mathcal{F}(\z_{c, t})}$ has an effective update in the original space, expressed as: $\D_c'^{(k + 1)} = \D_c'^{(k)} + \dfrac{\eta}{M}\left(\x - \hat{\x}^{(k)} \right) \star \z_{c, t}$. If the architectures satisfy the architectural inference equivalence condition (\Cref{prop:conv-operator-sae-arch-equiv}), then the dynamics is equivalent to that of an SAE-CNN: $\D_c^{(k + 1)} = \D_c^{(k)} + \eta \left(\x - \sum\limits_{i = 1}^p \D_i^{(k)} * \z_{i, t}\right) \star \z_{c, t}$. 
\end{proposition}

\begin{proof}
For a standard SAE-CNN, we learn a dictionary $\D$ by minimizing the reconstruction loss for the data sample $\x$. The loss function is:
\begin{equation}
    \mathcal{L}(\D) = \dfrac{1}{2} \left\lVert \x - \sum\limits_{c = 1}^p \D_c * \z_c \right\lVert_2^2
\end{equation}

The gradient of the loss with respect to the $c$-th kernel $\D_c$ is:
\begin{equation}
    \dfrac{\partial\mathcal{L}}{\partial \D_c} = -\left(\x - \sum\limits_{i = 1}^p \D_i * \z_i \right) \star \z_c
\end{equation}
where $\star$ denotes cross-correlation.

The analytic gradient update, given $\z$, for $\D_c$ at iteration $k$ with a learning rate $\eta$ is:
\begin{equation}
    \D_c^{(k + 1)} = \D_c^{(k)} + \eta \left(\x - \sum\limits_{i = 1}^p \D_i^{(k)} * \z_{i, t}\right) \star \z_{c, t}
\end{equation}

For a standard SAE-FNO, we learn frequency-domain weights $\W_c = \mathcal{F}\D_c$ by minimizing the reconstruction loss for the data sample $\x$. Taking normalization into account, the reconstruction operator is $\hat{\x} = \mathcal{G}_{\W}(\z) = \frac{1}{\sqrt{M}} \mathcal{F}^{-1} \big(\sum_{c=1}^p \W_c \cdot (\mathcal{F} \z_{c})\big)$, where $\W_c = \mathcal{F} \D_c$, where $M$ is the number of modes. The loss function is:
\begin{equation}
    \mathcal{L}(\W) = \dfrac{1}{2} \mathbb{E}_{\z \sim \mu} \left[\left\lVert \x - \mathcal{G}_{\W}(\z) \right\lVert_{\mathcal{X}}^2 \right]
\end{equation}
where $\x = \mathcal{G}_{\D^{*}}(\z)$ is the ground-truth signal. 

The gradient of the loss with respect to the $c$-th frequency-domain kernel $\W_c$ is:
\begin{equation}
    \dfrac{\partial\mathcal{L}}{\partial \W_c} = - \dfrac{1}{\sqrt{M}} \mathcal{F}\left(\x - \hat{\x} \right) \odot \mathcal{F}(\z_c)
\end{equation}

The analytic gradient update, given $\z$, for $\W_c$ at iteration $k$ with a learning rate $\eta$ is:
\begin{equation}
    \W_c^{(k + 1)} = \W_c^{(k)} + \dfrac{\eta}{\sqrt{M}}\mathcal{F}\left(\x - \hat{\x}^{(k)}\right) \odot \mathcal{F}(\z_{c, k})
\end{equation}

Applying the inverse Fourier transform and the cross-correlation theorem, we obtain the effective update rule for the spatial dictionary $\D_c'$:
\begin{equation}
    \D_c'^{(k + 1)} = \D_c'^{(k)} + \dfrac{\eta}{M}\left(\x - \hat{\x}^{(k)} \right) \star \z_{c, t}
\end{equation}

This is equivalent to the gradient update of SAE-CNN with a scaling factor of $\frac{1}{M}$ on the learning rate, which arises from the inverse Fourier normalization conventions. 
\end{proof}

\begin{proposition}
    [Architectural Inference Equivalence of L-SAE-CNN and L-SAE-FNO]
    \label{prop:l-conv-operator-sae-arch-equiv}
    The architectural inference of a L-SAE-CNN is equivalent to a L-SAE-FNO if the integral operator of L-SAE-FNO and parameters of SAEs are defined as a convolution where the L-SAE-FNO's lifted frequency-domain weights $\W_{L, c}$ are the Fourier transform of L-SAE-CNN's lifted spatial domain kernels: $\W_{L, c} = \mathcal{F}\D_{L, c}$.
\end{proposition}

\begin{proof}
We provide a proof for an encoder that implements a proximal gradient descent type of operation in an unrolled learning setting. For the L-SAE-CNN, the representation refinement for the $c$-th feature map at every encoder layer follows:
\begin{equation}
    \label{eq:l-operator-code-update}
    \z_{c, t + 1} = \mathcal{R} \left( \z_{c, t} - \alpha \left(\sum\limits_{i = 1}^p \D_{L, i} * \z_{i, t} - \lift\x \right)  \star \D_{L, c} \right)
\end{equation}

By construction, we incorporate the underlying model from the L-SAE-FNO to the L-SAE-CNN: $\W_{L, c} = \mathcal{F}\D_{L, c}$, and re-write the gradient of the encoder using Correlation and Convolution Theorems:
\begin{equation}
    \mathcal{F}^{-1} \left(\mathcal{F} \left(\sum\limits_{i = 1}^p\D_{L, i}^{(k)} * \z_{i, t} - \lift \x \right) \odot \mathcal{F}\D_{L, c}^H \right)= \mathcal{F}^{-1} \left(\sum\limits_{i = 1}^p\W_{L, i}^{(k)} \odot \mathcal{F}\z_{i, t} - \mathcal{F}(\lift\x)\right) \odot (\W_{L, c})^H 
\end{equation}

Substituting this back into \Cref{eq:l-operator-code-update} gives:
\begin{equation}
    \z_{c, t + 1} = \mathcal{R}\left(\z_{c, t}  - \alpha \mathcal{F}^{-1} \left[\left(\left(\sum\limits_{i = 1}^p\W_{L, i} \odot \mathcal{F}(\z_{i, t}) \right) - \mathcal{F}(\lift\x)\right) \odot (\W_{L, c})^H\right] \right)
\end{equation}

This final expression is the L-SAE-FNO's iterative update for $\z_c$ in the lifted spectral domain, where the dictionary is $\W_{L, c}$. This completes the proof of architectural inference equivalence between L-SAE-CNN and L-SAE-FNO.
\end{proof}

\begin{restatable}[The Architectural Inference of L-SAE-FNO reduced to SAE-FNO]{proposition}{liftarchequi}
    \label{prop:operator-lift-arch-equiv}
    The architectural inference of an SAE-FNO is equivalent to a L-SAE-FNO if the projection is tied to the lifting operator ($\proj = \lift^\top$), $\operatorname{col}(\D_L)\subseteq\operatorname{range}(\lift)$, and the lifting operator is orthogonal ($\lift^\top\lift = \eye$), i.e., under the same learned model $\W_c = \proj \W_{L, c}$, the representation dynamics and output of the networks are equivalent.
\end{restatable}

\begin{proof}
We prove this equivalence by establishing a chain of architectural inference equivalences that connects the SAE-FNO to the Lifted SAE-FNO, leveraging the propositions previously established:
\begin{enumerate}
    \item From \Cref{prop:conv-operator-sae-arch-equiv}, the architectural inference of an SAE-FNO is equivalent to an SAE-CNN.
    \item From \Cref{prop:conv-lift-arch-equiv}, the architectural inference of an SAE-CNN is equivalent to L-SAE-CNN.
    \item From \Cref{prop:l-conv-operator-sae-arch-equiv}, the architectural inference of a L-SAE-CNN is equivalent to a L-SAE-FNO.
\end{enumerate}
By the transitive property of these equivalences, we can establish a direct architectural inference equivalence between SAE-FNO and L-SAE-FNO under the same lifting-projection conditions. This completes the proof.
\end{proof}

\begin{restatable}[Training Dynamics of SAE-FNO]{proposition}{snoequiv}\label{prop:sno_equiv}
Consider an SAE-FNO whose integral operator is implemented as a convolution, with frequency-domain weights $\W_c$ defined as the Fourier transform of the SAE-CNN's spatial kernels, i.e., $\W_c = \mathcal{F}\D_c$. The training dynamics of the SAE-FNO are then expressed as: $\D_c^{(k + 1)} = \D_c^{(k)} + \tfrac{\eta}{M}  (\x - \tfrac{1}{\sqrt{M}} \mathcal{F}^{-1} [{\textstyle \sum_{i = 1}^p} \mathcal{F} \D_i^{(k)} \odot \mathcal{F} \z_{i}]) \star \z_{c}$ where $M$ is the full number of modes, $\odot$ denotes element-wise multiplication, and $\star$ denotes correlation. If the SAE-FNO's filters span the full frequency range and are constrained to have the same spatial support as in the SAE-CNN, then its dynamics reduce to those of the SAE-CNN (see \Cref{prop:conv-operator-equiv}).
\end{restatable}

\section{Appendix - Experiment Setup}\label{app:expsetup}

\subsection{Synthetic Experiments with Known Ground Truth}

\textbf{Sparse Generative Model} \quad The results shown in \Cref{fig:dense-lift-acc} are from a synthetic dataset that follows a sparse generative model. The atoms of the ground-truth dictionary $\D^* \in \R^{1000 \times 1500}$ are drawn from a standard normal distribution and then $\ell_2$-normalized. The dataset consists of 50,000 samples $\x \in \R^{1000}$. Each latent code $\z \in \R^{1500}$ has a total sparsity of 20. The amplitudes of the non-zero elements are drawn from a sub-Gaussian distribution with mean of 15 and standard deviation of 1.0.

\textbf{SAE} \quad For both SAE (SAE-MLP) and L-SAE (L-SAE-MLP) models, the encoder implements $f_\theta(\x)$ following sparse model recovery using a deep unrolled JumpReLU network (eq. \ref{eq:jumpdeepbilevel},\ref{eq:jumpdeepsae}) with $T = 50$ layers, as described in the sparse model recovery framework. The non-linearity is Hard-Thresholding with a threshold of $\lambda = 0.5$, and the algorithm's internal step-size is $\alpha = 0.2$. The decoder reconstructs data as $\hat{\x} = \D\z$. The dictionary weights for both SAE and L-SAE are initialized with a noisy ($\sigma = 0.02)$ version of the ground-truth dictionary. For the L-SAE, the 1000-dimensional input is lifted to a 1200-dimensional space. Both models are trained to minimize the Mean Squared Error (MSE) reconstruction loss using the SGD optimizer with a learning rate of $\eta = 10^{-3}$. 

\textbf{Sparse Convolutional Generative Model} \quad The results shown in \Cref{fig:synthetic} are from a synthetic dataset that follows a sparse convolutional generative model (\Cref{def:sgmconv}). The ground-truth dictionary $\D^*$ consists of $p = 5$ kernels with spatial support $h = 99$. Each kernel is a multi-channel signal with 64 input channels. The values for each kernel are drawn from a standard normal distribution and then normalized. The dataset consists of $50{,}000$ samples $\x \in \R^{64 \times 1000}$. Each sample is generated by convolving a single (i.e., sparsity of 1) randomly chosen kernel from the dictionary with a sparse feature map, whose non-zero amplitudes are drawn from a sub-Gaussian distribution. For~\Cref{fig:synthetic}c, the data is generated using a smooth (low-pass filtered) version of the ground-truth dictionary described above.

\textbf{SAE-CNN} \quad For SAE-CNN and L-SAE-CNN models, the encoder implements the convolutional version of a deep unrolled network (\Cref{eq:jumpdeepsaeconv}) with $T = 50$ layers, as described in the sparse model recovery framework. The non-linearity is Hard-Thresholding with a threshold of $\lambda = 10$, and the algorithm's internal step-size is $\alpha = 0.01$. The dictionary kernels for both SAE-CNN and L-SAE-CNN are initialized with a noisy ($\sigma = 0.05$) version of the ground-truth kernels. For L-SAE-CNN, the 64-channel input is lifted to a 128-dimensional space. Both models are trained to minimize the Mean Squared Error (MSE) reconstruction loss using the SGD optimizer with a learning rate of $\eta = 0.04$. 

\textbf{SAE-FNO}\quad For SAE-FNO and L-SAE-FNO models, the encoder implements a deep unrolled network with $T = 50$ layers, as described in the sparse functional model recovery framework. The non-linearity is Hard-Thresholding with a threshold of $\lambda = 10$, and the algorithm's internal step-size is $\alpha = 0.01$. The dictionary kernels are initialized with a noisy ($\sigma = 0.05)$ version of the ground-truth. For L-SAE-FNO, the 64-channel input is lifted to a 128-dimensional space. Both SAE-FNO and L-SAE-FNO are trained to minimize the mean squared error (MSE) reconstruction loss using the SGD optimizer with a learning rate of $\eta_L = 20.04$. 

\textbf{1D Discretization Generalization (Upsampling)} \quad We evaluate models trained at base resolution on higher-resolution inputs ($2 \times$ and $4 \times$ upsampling) during inference (\Cref{fig:simulated-fno-discretization-invariance}). We use discrete-time interpolation: expanding the sequence to the target resolution by inserting zeros between the original data points, and passing the expanded signal through a 10th-order Butterworth low-pass filter to eliminate high-frequency aliasing artifacts. The filter cutoff is set to half of the original sampling rate, and the resulting signal amplitude is scaled by the upsampling factor to preserve signal energy. 

\textbf{1D Time-Series} \quad We train and evaluate SAE-MLP and SAE-FNO on a synthetic 1D time-series dataset of $50,000$ samples (length $m = 1000$, $1$ channel), constructed via 1D transposed convolution of sparse codes with ground-truth dictionary $D^*$. The dictionary contains $p=10$ normalized cosine waves (kernel size 99) with frequencies linearly spaced from 1 to 10 cycles per window. Each sample is generated with a concept sparsity of $k = 3$ and a domain sparsity of 2, meaning each active frequency appears at exactly two random spatial locations. Activation amplitudes are drawn from a sub-Gaussian distribution (mean $5.0$, standard deviation $1.0$) with randomly flipped polarities. We evaluate the models on $\sigma_{\text{noise}} = 0$ (clean baseline) and $\sigma_{\text{noise}} = 0.5$ (high-noise) (see~\Cref{tab:synthetic-recovery}).


\subsection{Image Experiments}

\textbf{Datasets and Preprocessing} \quad We evaluate our SAE-FNO and SAE-MLP (SAE-CNN) across MNIST, CIFAR-10, and ImageNet datasets. For MNIST, the images are inherently $28 \times 28$ grayscale and are passed directly into the SAEs without additional preprocessing. For CIFAR-10 and ImageNet, we first extract random patches to match the model's expected receptive field (e.g., $8 \times 8$ for CIFAR, $16 \times 16$ or $32 \times 32$ for ImageNet images initially resized to $256 \times 256$). Additionally, Global Contrast Normalization (GCN) and ZCA Whitening ($\epsilon = 0.1$) are computed and applied after patch extraction. 

\textbf{Architecture and Training} \quad Unless specified otherwise, all models are trained with a dictionary size of $p = 1000$ concepts. For MNIST and CIFAR-10 experiments, we enforce concept and domain sparsity via the standard Top-$K$ mechanism. For ImageNet, we evaluate Batch Top-$K$ and Matryoshka SAE variants in addition to Top-$K$ SAE. 

All models use the AdamW optimizer with a base learning rate of $10^{-3}$. The primary objective is the Mean Squared Error (MSE) reconstruction loss. Throughout training, we apply unit normalization along the channel dimension. After taking the $L_2$ norm over the spatial dimensions, we scale the weights by $1/\sqrt{\text{ch}}$ (where ch is the number of channels or lift dimension). Additionally, for SAE-FNOs, we scale the encoder and decoder weights of the SAE-FNO by a factor of $\sqrt{\text{modes}}$ for training stability. The learning rate is also scaled accordingly.

\textbf{High-frequency Noise Robustness} \quad Models trained on ImageNet $16 \times 16$ patches are subjected to additive high-frequency noise at inference. We sample white Gaussian noise, project it into the Fourier domain via a 2D FFT, and apply a high-pass filter that zeros out all frequencies within a radial cutoff of $r \leq 0.25 \min(H, W)$. The filtered spectrum is inverted back to the spatial domain, rescaled, and added to the raw input images. We then compute the Peak Signal-to-Noise Ratio (PSNR) between the model's reconstruction and the ground-truth image, as shown in \Cref{fig:imagenet-highfreq-noise}.

\textbf{2D Discretization Generalization (Upsampling)} \quad For spatial image data, we apply bilinear interpolation directly to the test images to scale them to the target resolutions. To evaluate SAE-MLP on these higher-dimensional inputs, its fixed-size dense parameter matrices must be proportionally scaled. We reshape the encoder and decoder weight matrices (as well as the learned spatial encoder bias tensor when applicable) from flat vectors into their original spatial configurations (e.g., mapping from $\mathbb{R}^{d \times 784}$ to $\mathbb{R}^{d \times 28 \times 28}$). These tensors are then bilinearly interpolated to the new target resolution (e.g., $\mathbb{R}^{d \times 56 \times 56}$) and flattened back into dense matrices.

\textbf{Ablation of Concept and Domain Sparsity} \quad To isolate the specific contributions of our sparsity formulation, we compare the full SAE-FNO against two restricted variants (\Cref{fig:aefno-saefno-no-ss}): 1) SAE-FNO without domain sparsity. This network retains the concept-level sparsity to select active dictionary atoms, but leaves their corresponding map activations dense across the spatial domain (applying only spatial ReLU); and 2) AE-FNO (no sparsity). This is a baseline model where both concept and domain sparsity mechanisms are removed. The projection onto the $p$ dictionary atoms is followed by a ReLU nonlinearity, allowing all concepts to activate simultaneously across the entire domain.

\textbf{Translated MNIST} \quad To evaluate the stability of learned concepts under continuous translation, we construct a custom moving MNIST dataset. Standard MNIST digits are placed onto an expanded $48 \times 48$ spatial canvas and linearly translated across the domain over a fixed temporal window ($T = 20$). Crucially, both the SAE-MLP and SAE-FNO process each frame independently: the Top-K sparsity mechanism is computed and applied independently to each individual frame without any constraints enforcing temporal consistency. We then encode the translating digit frame-by-frame and identify the top active concepts at each time step, ranked by their $\ell_2$ activation magnitude (\Cref{fig:video}).

\textbf{Failure Case: Loss of Spatial Structure under Pixel Permutation} \quad To construct a controlled violation of the Functional Representation Hypothesis, we apply a single fixed random permutation to the pixel positions of CIFAR-10 images before training (same permutation across all runs and seeds within an experiment). This preserves every marginal statistic (pixel values, distribution, dimensionality) but destroys spatial coherence, removing both properties the FRH relies on: the smoothness the Fourier parametrization exploits, and the translation structure that makes a concept reusable across spatial locations. We train SAE-FNO with both truncated and full modes (to separate the cost of losing smoothness from the cost of losing translation structure), and at two domain sizes, $8 \times 8$ patches and full $32 \times 32$ images (to trace the translation-structure cost to the number of positions available for concept reuse). All settings use ZCA-whitened CIFAR-10
images, $p=1000$, TopK sparsification. Results are reported in \Cref{tab:perm-exp1,tab:perm-exp2-full,tab:perm-exp2-patch}.

\section{Appendix - Metrics}\label{app:metrics}

\textbf{Hoyer Score}\quad For a flattened representation $\smlv_c\in \mathbb{R}^d$ of the $c$-th concept, the Hoyer score is: 
\begin{equation}
    \text{Hoyer}(\smlv_c) = \frac{ \sqrt{d} - \frac{\|\smlv_c\|_1}{\|\smlv_c\|_2}}{\sqrt{d} - 1}
\end{equation}

\textbf{Identity Consistency} \quad To quantify whether SAEs learn a stable and consistent set of features regardless of the specific sparsity constraint, we introduce the Identity Consistency metric. This measures the fraction of learned concepts in an anchor model ($k_{\text{ref}} = 50$ in \Cref{fig:identity-consistency}) that are consistently recovered in models trained at more relaxed sparsity levels ($k > 50$).

To ensure the metric evaluates meaningful representations, we restrict the matching process to active concepts. A concept is considered active if it generates a non-zero activation for at least 1\% of the evaluation images. Let $\mathcal{D}_{\text{ref}}$ be the set of active, unit-normalized spatial dictionary atoms from the anchor model, and $\mathcal{D}_{(k)}$ be the active, unit-normalized atoms from the comparison model. 

To accurately reflect the capacity of each model, we tailor the similarity metric to their respective inductive biases. Because SAE-MLPs use scalar activations, their dictionary atoms are fixed to their specific position. A structural feature learned at one spatial location cannot be shifted and applied to another. Therefore, we evaluate SAE-MLP using cosine similarity: 
\begin{equation}
    \text{sim}_{\text{MLP}}(d_i, d_j) = d_i^\top d_j
\end{equation} 
Conversely, SAE-FNO concepts are parameterized functionally in the Fourier domain, allowing the model to represent the same structural feature translated across the spatial domain. This leads us to compute similarity via maximum spatial cross-correlation: 
\begin{equation}
    \text{sim}_{\text{FNO}}(d_i, d_j) = \max_{\Delta x, \Delta y} (\mathcal{F}^{-1}(\mathcal{F}(d_i) \odot \overline{\mathcal{F}(d_j)}))
\end{equation}
where $\odot$ denotes element-wise multiplication.

For every concept in the anchor model, we independently find the best-matching concept in the comparison model. For each $d_i \in \mathcal{D}_{\text{ref}}$, the maximum similarity score is defined as 
\begin{equation}
    s(d_i) = \max_{d_j \in \mathcal{D}_{(k)}} \text{sim} (d_i, d_j)
\end{equation}
From this, we derive the identity consistency metric, defined as the proportion of anchor concepts successfully recovered above a similarity threshold $\tau$ (typically $\tau \in \{0.5, 0.7\}$)~\citep{leask2025sparse}:
\begin{equation}
    \text{Identity Consistency} = \frac{1}{|\mathcal{D}_{\text{ref}}|} \sum_i \mathbf{1}\big(s(d_i) > \tau\big)
\end{equation}

\textbf{Concept Utilization and Expressivity} \quad To evaluate representation efficiency, we quantify how effectively the architectures use their dictionary capacity and how accurately they reconstruct the test population distribution (\Cref{fig:utilization} and \ref{fig:imagenet-concept-utilization}).

Let $z_p^{(i)}$ be the feature map activations of concept $p$ for image $i$ after enforcing sparsity. For a feature map with spatial size $S$ (where $S = 1$ for scalar SAE-MLPs and $S = H \times W$ for SAE-FNOs), the normalized concept utilization $L_0(p)$ across a validation set of $N$ images is defined as the total count of non-zero spatial activations relative to the total possible activations:
\begin{equation}
    L_0(p) = \frac{1}{N \cdot S}\sum_{i=1}^N \|z_p^{(i)}\|_0    
\end{equation}
This normalization naturally bounds $L_0(p) \in [0, 1]$. Additionally, we measure the proportion of data variance explained by the model's reconstructions via goodness-of-fit ($R^2$), defined as:
\begin{equation}
    R^2 = 1 - \frac{\text{MSE}}{\text{Var}(X)}   
\end{equation}
where $X$ represents the ground-truth images. We analyze $R^2$ against concept sparsity level and the number of globally active concepts (\Cref{fig:utilization}). A concept $p$ is considered as "globally active" if it produces at least one non-zero spatial activation ($\|z_p^{(i)}\|_0 > 0$) on $\ge 1\%$ of the validation set. Models that achieve a higher $R^2$ using fewer globally active concepts show better representational efficiency.

\textbf{Dictionary Atom Correlation} \quad To evaluate the orthogonality and structural redundancy of the learned dictionaries, we compute the distribution of pairwise cosine similarities between all atoms within a model. This evaluation measures spatial overlap. For SAE-FNO, the $p$ dictionary atoms are projected into the spatial domain and flattened into vectors $d \in \mathbb{R}^{CHW}$, before computing the cosine similarity. We plot the density of these similarities for all unique atom pairs (excluding self-correlations) (\Cref{fig:mnist-atom-corr-supp}).

\section{Appendix - Identifiability}\label{app:rec_identifiability}

We verify the identifiability empirically by measuring the coherence of a ground-truth shifted-kernel dictionary of $p=10$ single-frequency kernels ($\mu = 0.31$ across concepts, $\mu = 0.98$ within concepts) gives a threshold of $B = 0.5(1+1/\mu) \approx 2.1$ across and $\approx 1.01$ within concepts. Training over 5 seeds and reporting recovery against the expected activations per receptive field $r = ksh/m$ (\Cref{tab:identifiability_stripe_sparsity}), recovery holds at $0.82$--$0.94$ below this threshold and decays to $0.56$ at $r \approx 5$. This explains a residual asymmetry: at equal $r$, increasing domain density degrades recovery more than concept count, since shifted copies of one smooth concept have higher mutual coherence ($\mu = 0.98$) than distinct concepts ($\mu = 0.31$), so confusable pairs arise from within-concept overlap and are penalized heavily by domain sparsity.

\begin{table}[H]
    \centering
    \caption{\textbf{Stripe-sparsity identifiability.} Recovery of a ground-truth dictionary of $p=10$ single-frequency kernels ($h=99$, $m=1000$), reported as final cosine similarity to the ground-truth dictionary (mean $\pm$ std, 5 seeds; initial cosine similarity $0.33 \pm 0.02$ throughout), as a function of the expected number of active coefficients per receptive field $r = ksh/m$, for concept sparsity $k$ and domain sparsity $s$.}
    \label{tab:identifiability_stripe_sparsity}
    \begin{tabular}{ccc}
        \toprule
        $r$ & $(k, s)$ & Final cosine similarity \\ \midrule
        0.20 & (1, 2)  & $0.89 \pm 0.08$ \\
        0.30 & (3, 1)  & $0.94 \pm 0.03$ \\
        0.40 & (2, 2)  & $0.89 \pm 0.07$ \\
        0.59 & (3, 2)  & $0.89 \pm 0.12$ \\
        0.99 & (5, 2)  & $0.82 \pm 0.14$ \\
        1.19 & (3, 4)  & $0.75 \pm 0.13$ \\
        1.39 & (7, 2)  & $0.82 \pm 0.15$ \\
        2.38 & (3, 8)  & $0.57 \pm 0.03$ \\
        4.75 & (3, 16) & $0.56 \pm 0.02$ \\
        \bottomrule
    \end{tabular}
\end{table}

\section{Appendix - Temporal Concept Consistency Experiment}\label{app:temp}

Experiments in~\Cref{fig:video} have $p=1000$ and TopK with a constructed moving MNIST for a,b,c,d and static MNIST for e.

\begin{figure}[H]
    \centering
    \begin{subfigure}{\textwidth}
        \includegraphics[width=\linewidth]{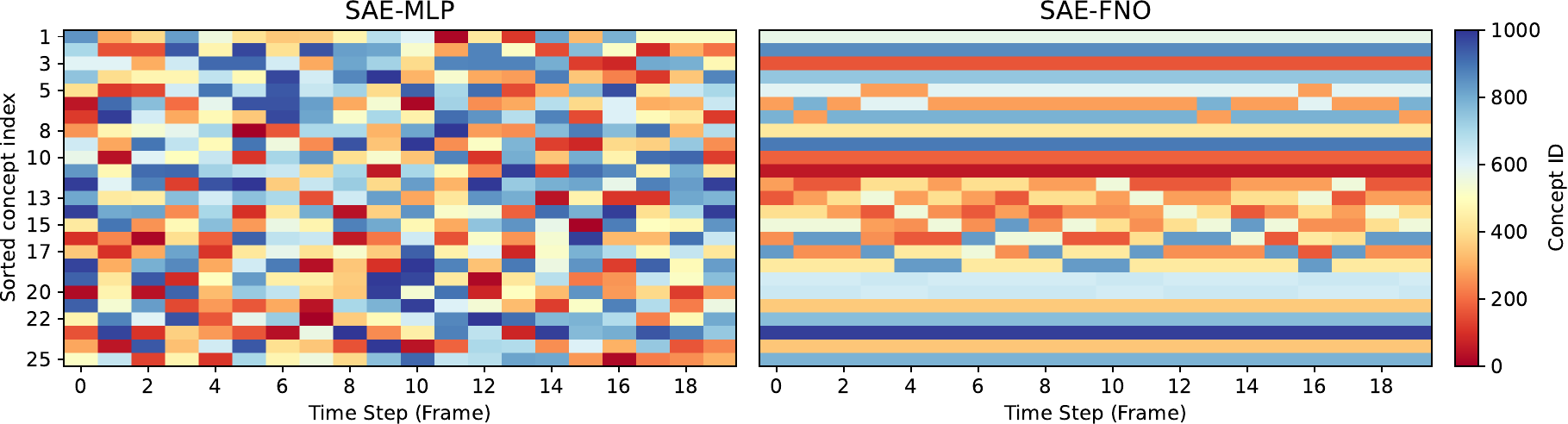}
        \caption{$k = 25$.}
    \end{subfigure}
    \begin{subfigure}{\textwidth}
        \includegraphics[width=\linewidth]{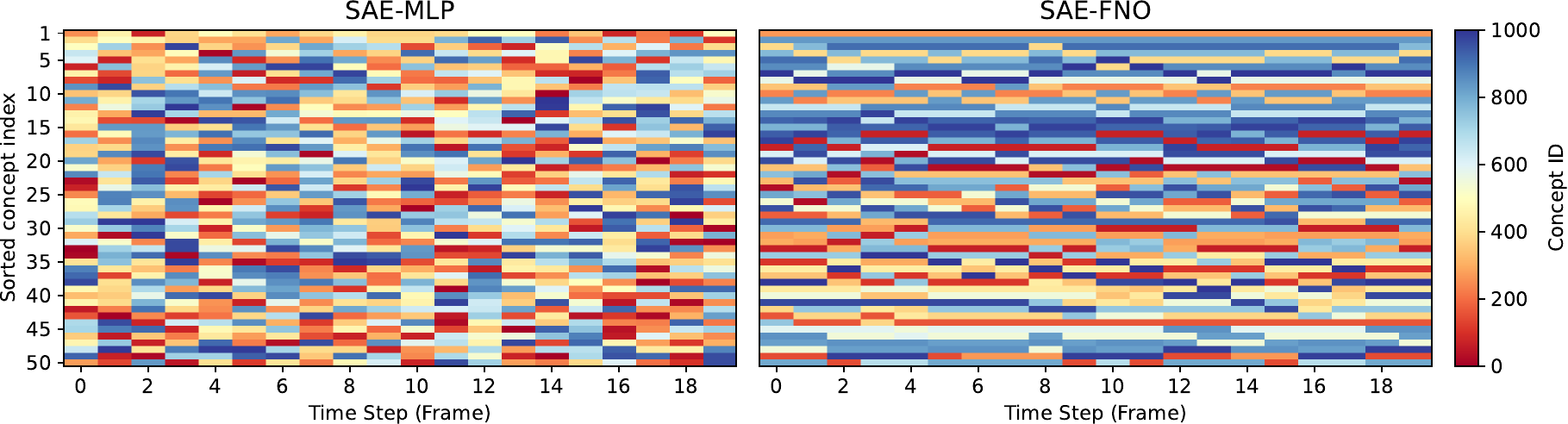}
        \caption{$k = 50$.}
    \end{subfigure}
    \caption{\textbf{Top concept activations over time for higher concept sparsity levels.} Extended visualizations of \Cref{fig:video}b for (a) $k=25$ and (b) $k=50$. Each cell's color represents the concept ID active at a given time step. Consistent with lower sparsity levels, SAE-FNO maintains stable concept identities over time (horizontal bands), using domain sparsity to handle spatial translations. In contrast, SAE-MLP constantly activates new, distinct concepts to represent the moving digit.}
    \label{fig:video-concepts-over-time}
\end{figure}

\begin{figure}[H]
    \centering
    \begin{subfigure}{0.44\textwidth}
        \includegraphics[width=\linewidth]{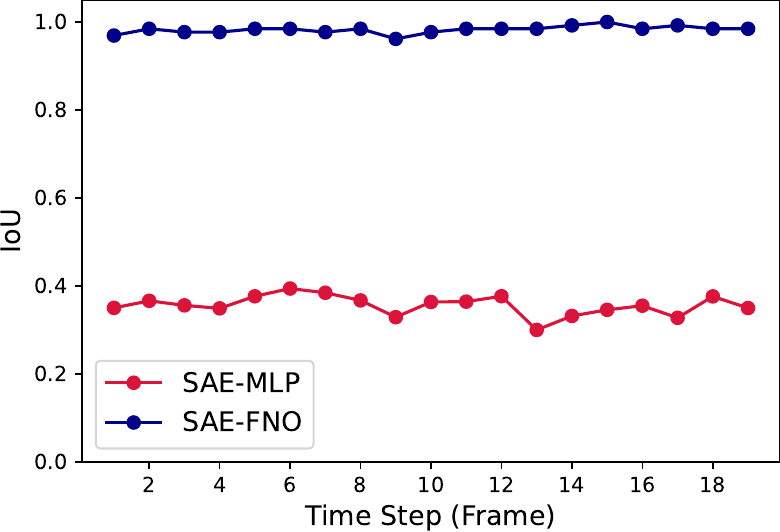}
        \caption{$k = 25$}
    \end{subfigure}
    \begin{subfigure}{0.44\textwidth}
        \includegraphics[width=\linewidth]{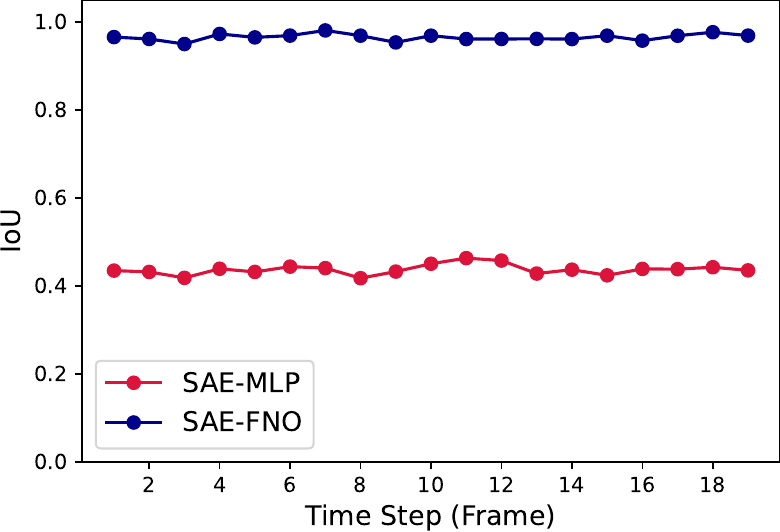}
        \caption{$k = 50$}
    \end{subfigure}
    \caption{\textbf{SAE-FNO concept stability across sparsity levels.} Intersection over Union (IoU) of the active concept sets between consecutive frames in translating MNIST digits (see \Cref{fig:video}(a)-(d)). At each time step $t$, we compute the IoU between the top $k$ concepts (ranked by their spatial $L_2$ norm) activated in frame $t$ and those activated in frame $t+1$, averaged across batch. Higher IoU indicates stronger concept stability and consistency, meaning the model tracks the moving digit using a consistent subset of concepts. While SAE-MLP (red) rapidly swaps its active concepts between frames, SAE-FNO (blue) maintains a highly stable set of active concepts ($\text{IoU} \approx 1$).}
    \label{fig:video-iou}
\end{figure}

\section{Appendix - Lifting Experiments}\label{app:lifting}
\begin{figure}[H]
\centering
\includegraphics[width=0.5\linewidth]{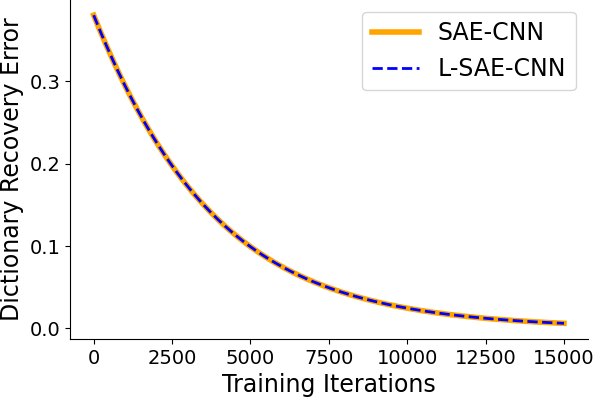}
21\caption{\textbf{Lifting.} Lifted SAE is equivalent to SAE when learning with $\lift^{\text{T}}\lift = \eye$.}
    \label{fig:conv-sae-lift-b}
\end{figure}

\begin{figure}[H]
    \centering 
    \begin{subfigure}{0.45\textwidth}
        \centering
        \includegraphics[width=\linewidth]{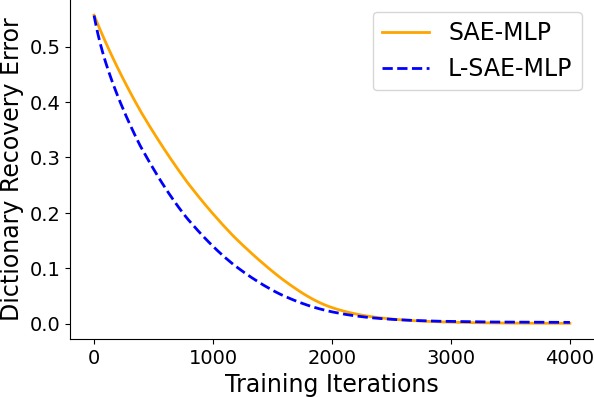}
        \caption{Dictionary error}
        \label{fig:dense-lift-acc-dict-err}
    \end{subfigure}
    \hfill
    \begin{subfigure}{0.45\textwidth}
        \centering
        \includegraphics[width=\linewidth]{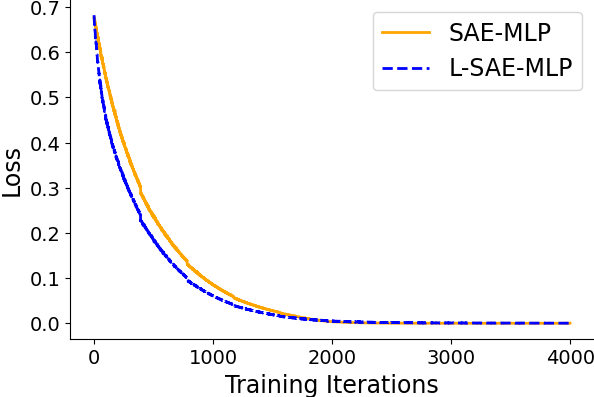}
        \caption{Reconstruction loss}
        \label{fig:dense-lift-acc-rec-loss}
    \end{subfigure}
    \hfill
    \begin{subfigure}{0.45\textwidth}
        \centering
        \includegraphics[width=\linewidth]{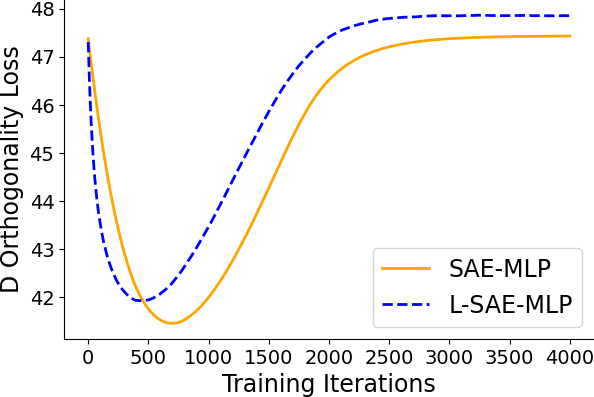}
        \caption{D Orthogonality Loss}
        \label{fig:dense-lift-acc-orthd}
    \end{subfigure}
    \hfill
    \begin{subfigure}{0.45\textwidth}
        \centering
        \includegraphics[width=\linewidth]{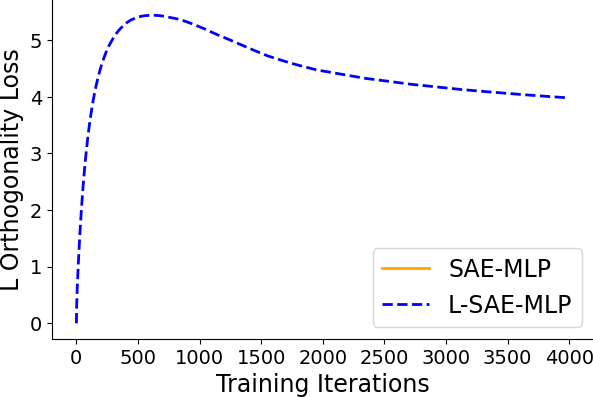}
        \caption{L Orthogonality Loss}
        \label{fig:dense-lift-acc-orthl}
    \end{subfigure}%
    \caption{\textbf{Learning the lifting operator $\lift$ accelerates model recovery.} Comparing SAE-MLP with L-SAE-MLP.}
    \label{fig:dense-lift-acc}
\end{figure}

\begin{figure}[H]
    \centering
    \begin{subfigure}{0.24\textwidth}
        \includegraphics[width=\linewidth]{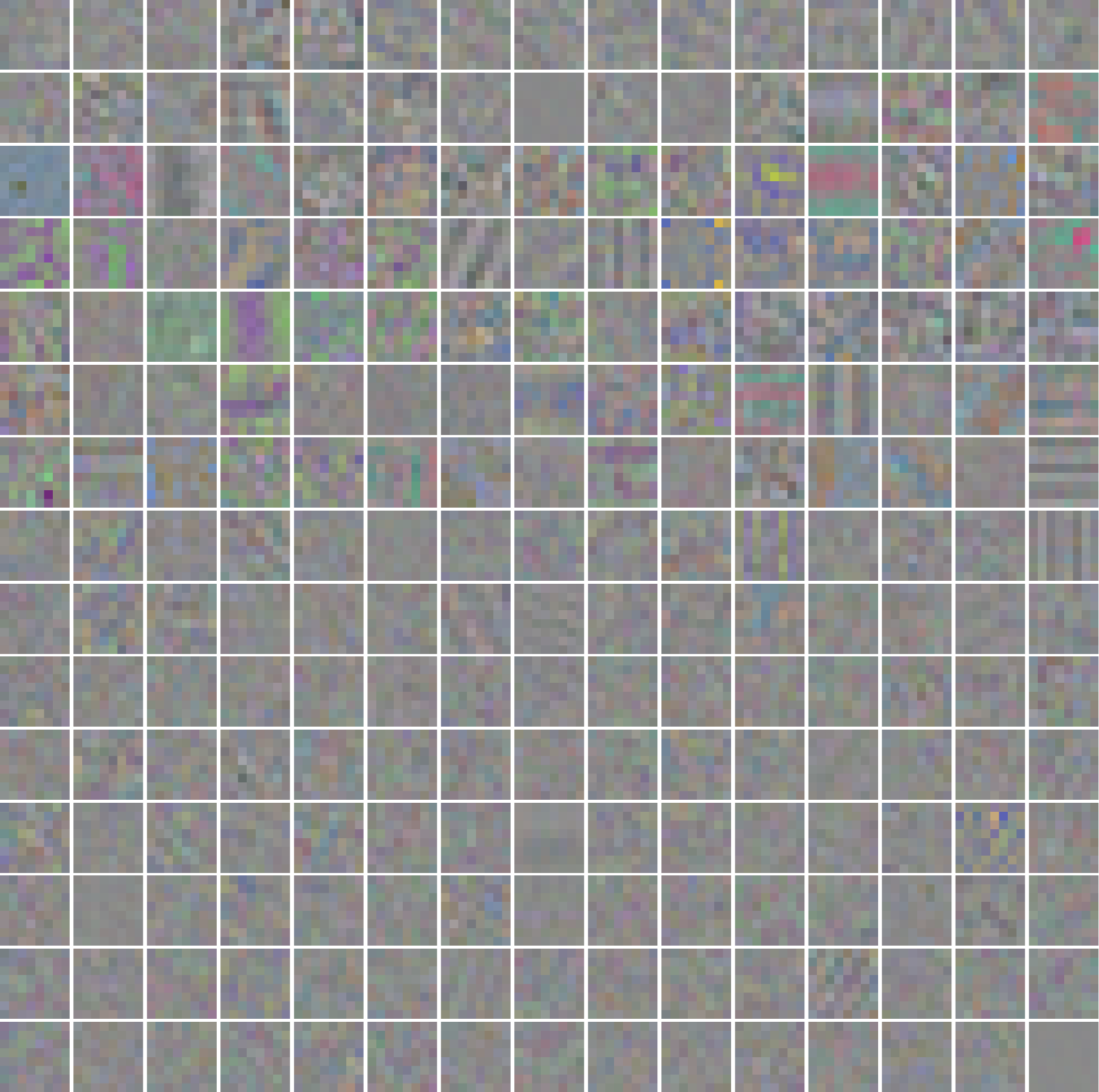}
        \caption{$k = 25$.}
    \end{subfigure}
    \hfill
    \begin{subfigure}{0.24\textwidth}
        \includegraphics[width=\linewidth]{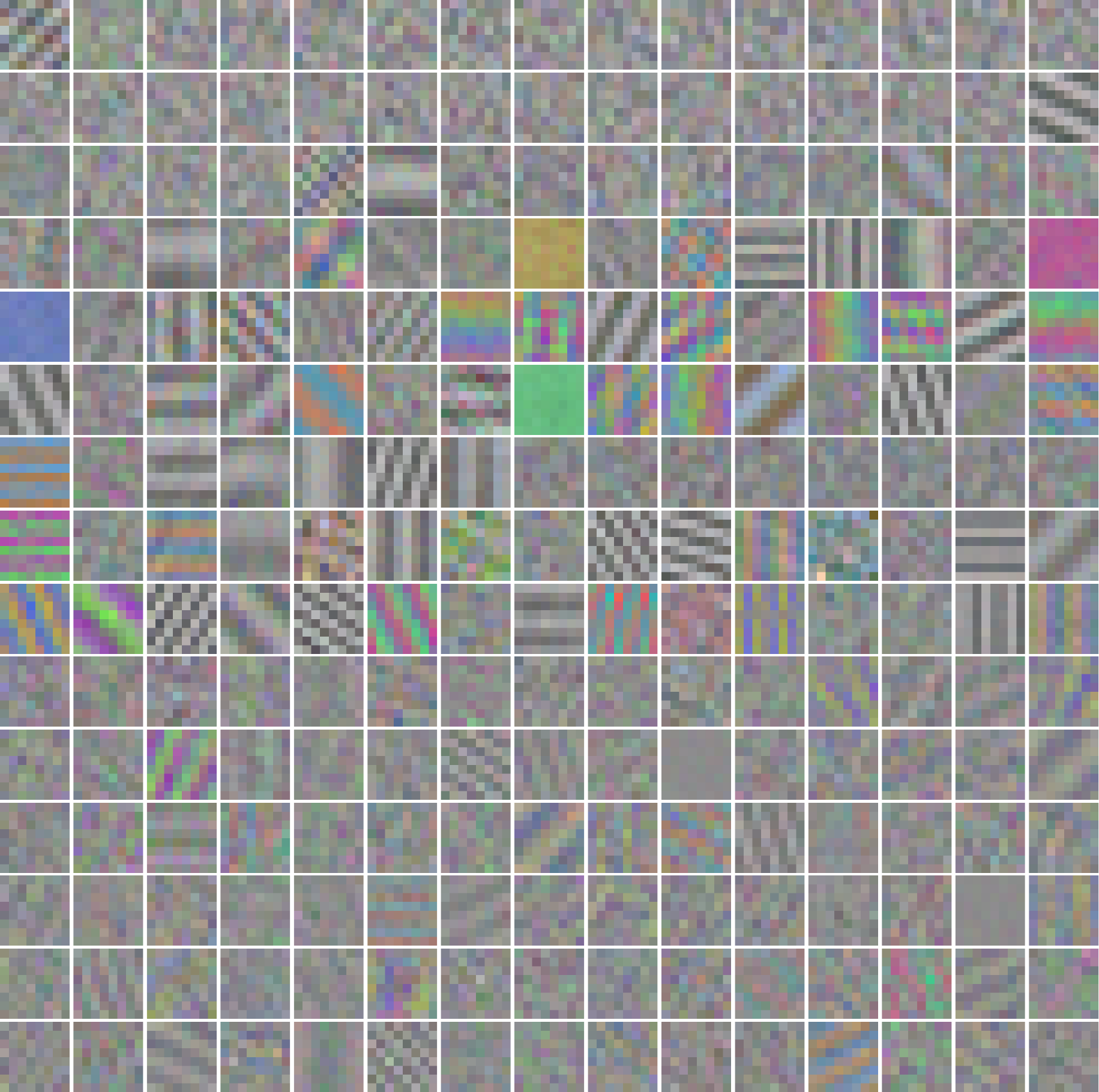}
        \caption{$k = 50$.}
    \end{subfigure}
    \hfill
    \begin{subfigure}{0.24\textwidth}
        \includegraphics[width=\linewidth]{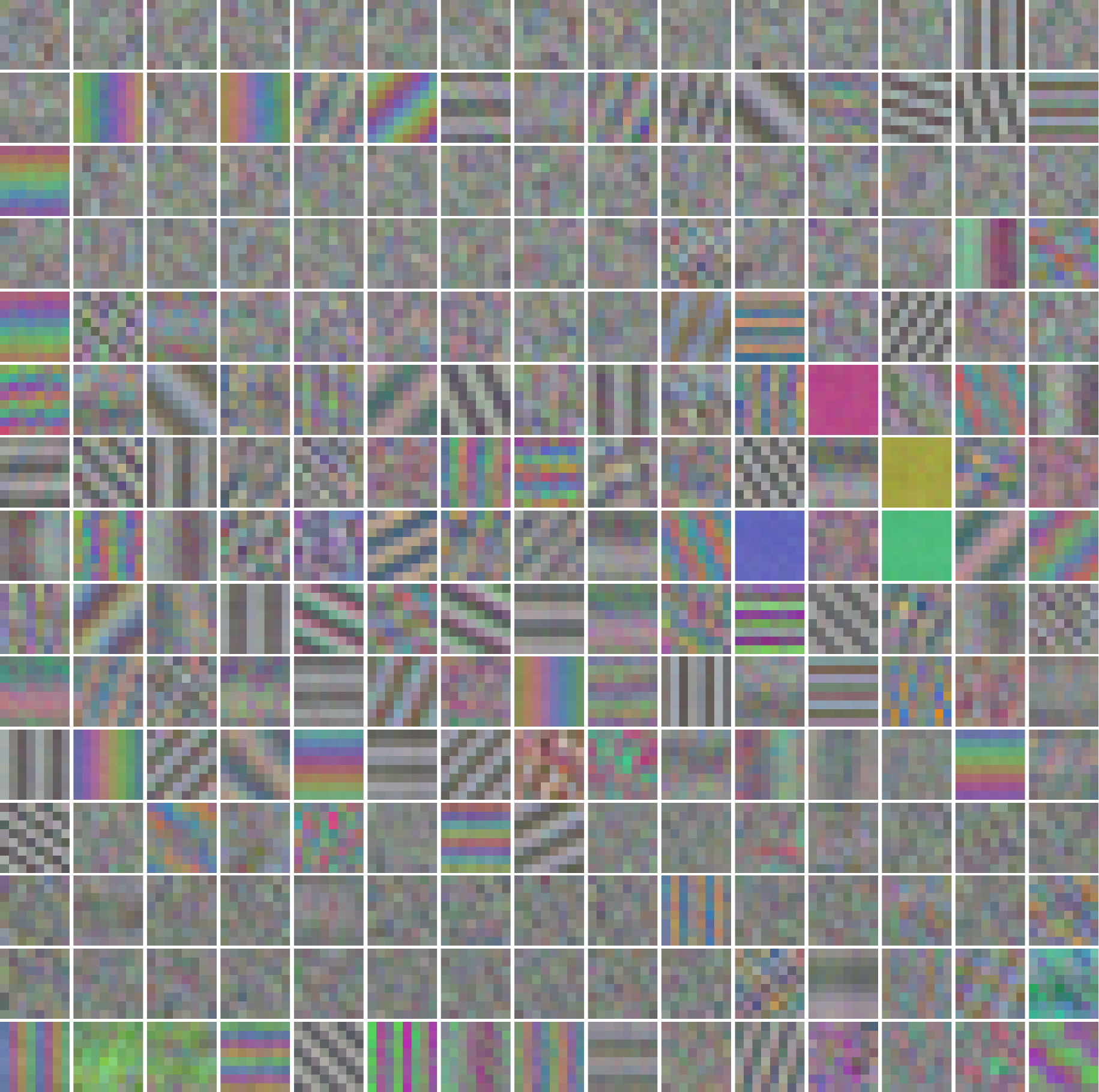}
        \caption{$k = 100$.}
    \end{subfigure}
    \hfill
    \begin{subfigure}{0.24\textwidth}
        \includegraphics[width=\linewidth]{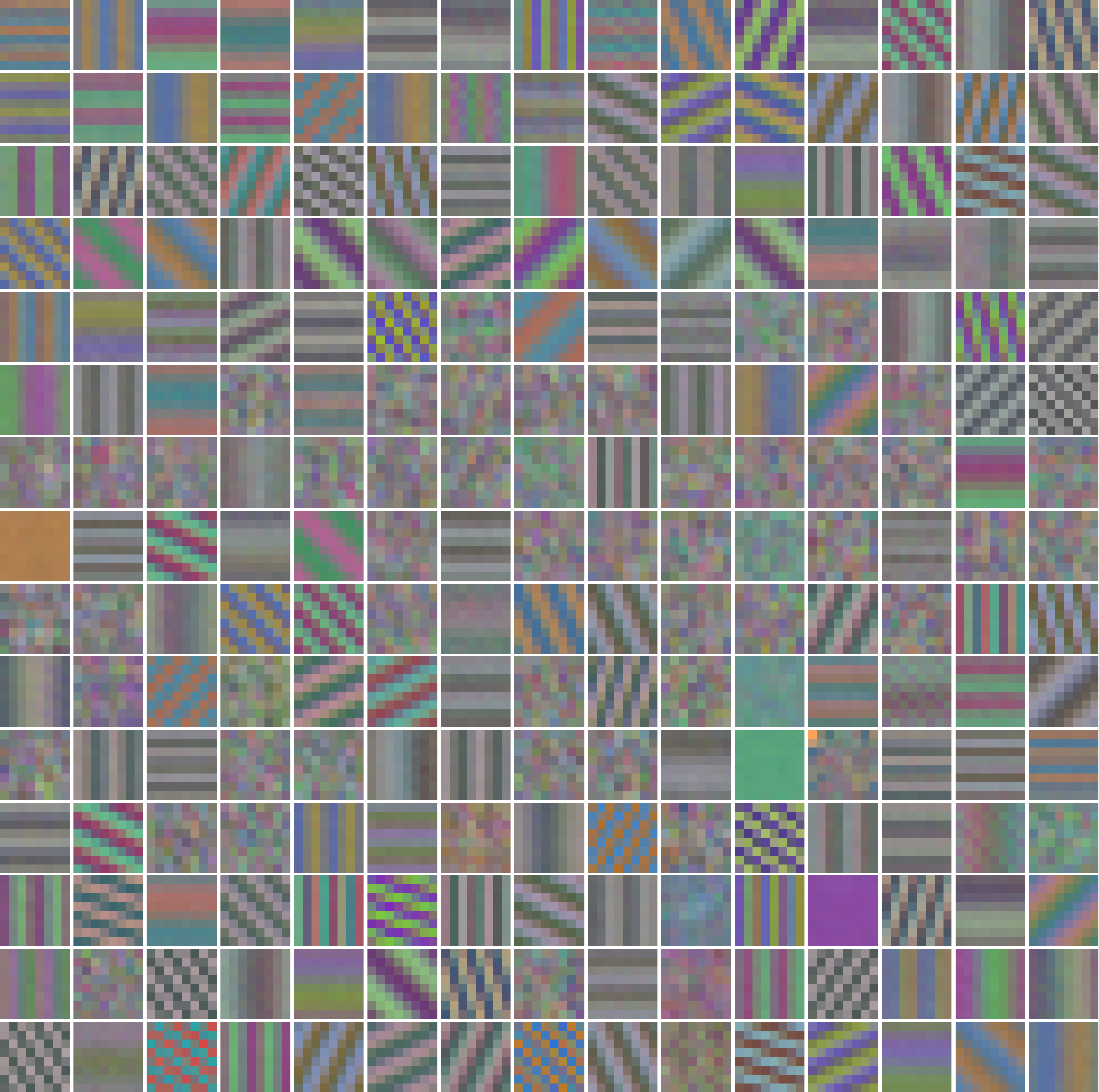}
        \caption{$k = 200$.}
    \end{subfigure}
    \caption{\textbf{Lifted concepts visualized in the input domain.} CIFAR TopK-SAE-FNO (modes = 6, spatially 3-sparse).}
    \label{fig:cifar-lifted-concepts}
\end{figure}

\begin{figure}[H]
    \centering 
    \begin{subfigure}{1.00\linewidth}
        \centering
        \includegraphics[width=0.49\linewidth]{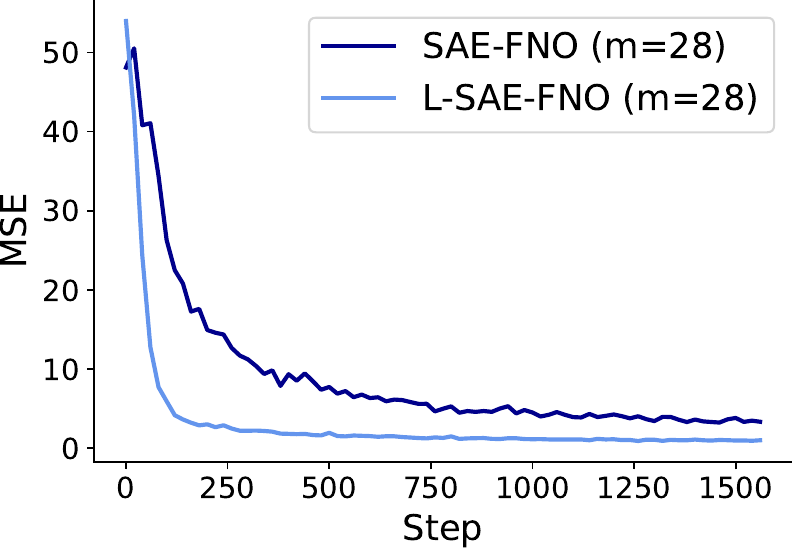}
        \includegraphics[width=0.49\linewidth]{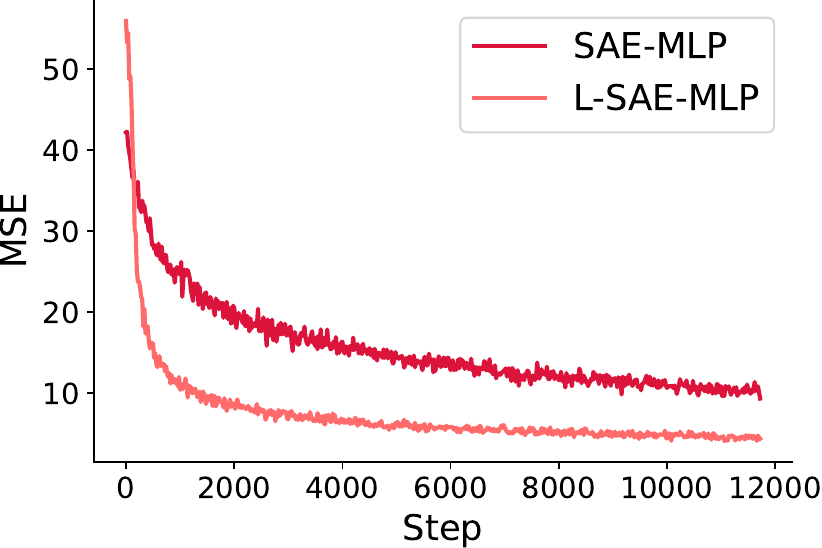}
            \caption{Experiments on MNIST.}
    \label{fig:mnist-lifting-acc}
    \end{subfigure}
    \begin{subfigure}{1.00\linewidth}
        \centering
    \includegraphics[width=0.49\linewidth]{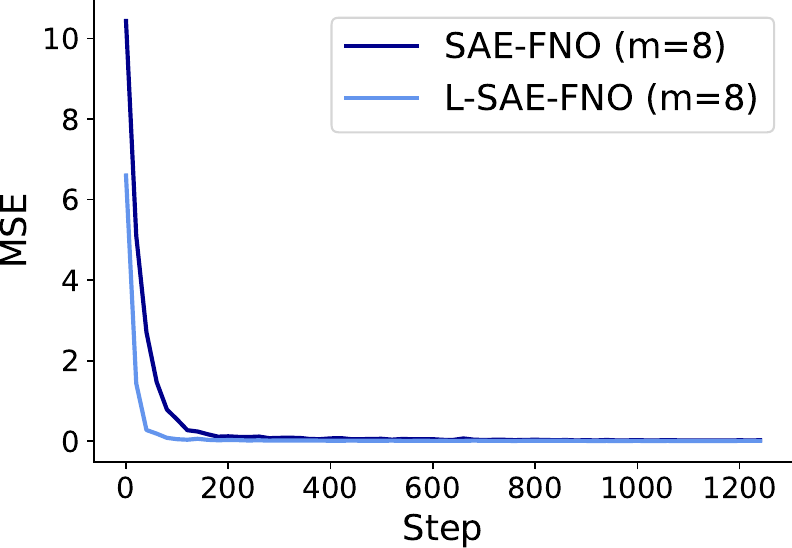}
    \includegraphics[width=0.49\linewidth]{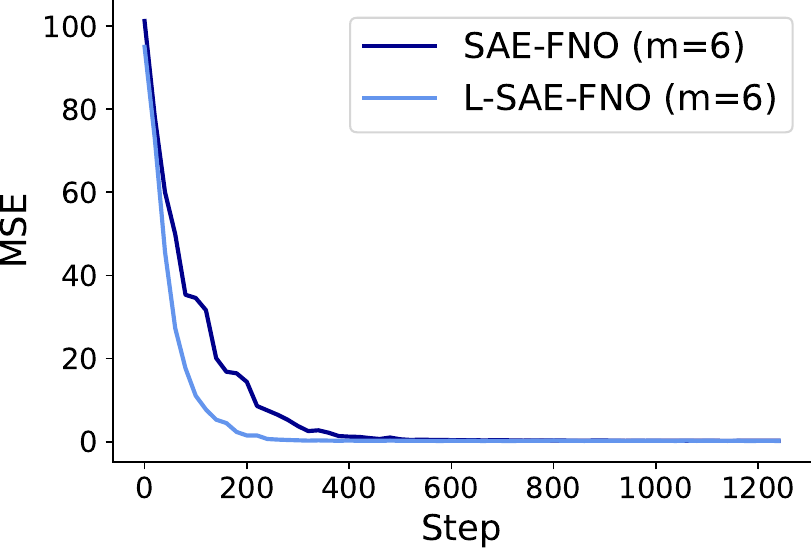}
    \hfill
    \includegraphics[width=0.49\linewidth]{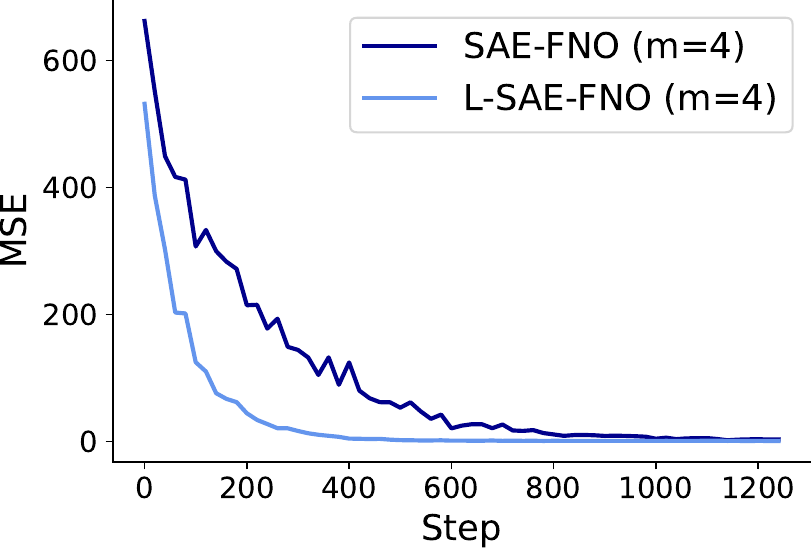}
    \includegraphics[width=0.49\linewidth]{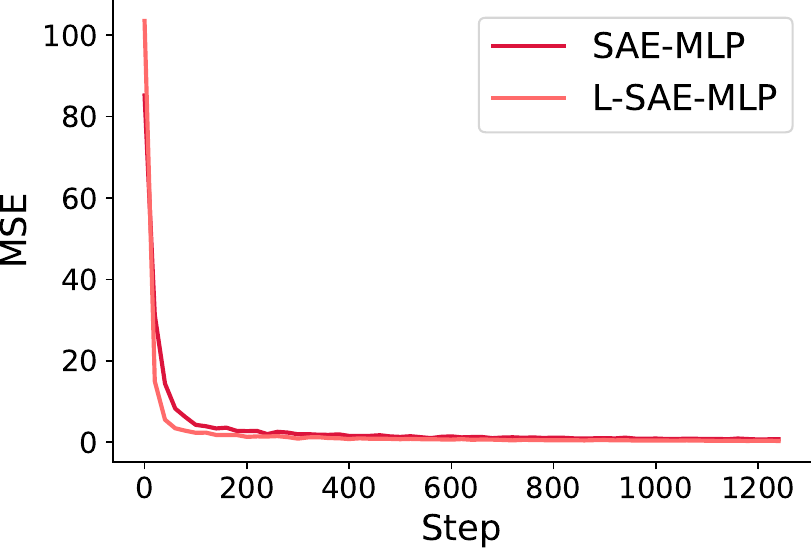}
    \caption{Experiments on CIFAR.}
    \label{fig:cifar-lifting-acc}
    \end{subfigure}
    \caption{\textbf{Lifting accelerates learning.} Lifting acts as a preconditioner introducing an implicit acceleration to learning. In real-world settings, this can result in a better fit to the data by reducing the impact of noise on learned concepts. The impact of lifting applies to both SAE-FNO and SAE-MLP.}
    \label{fig:lifting-acc}
\end{figure}

\newpage
\section{Appendix - Concept Structural Stability and Correlation Experiments}\label{app:conceptstrstab}

\begin{figure}[H]
    \centering 
    \begin{subfigure}{0.74\textwidth}
        \centering
    \includegraphics[width=0.325\linewidth]{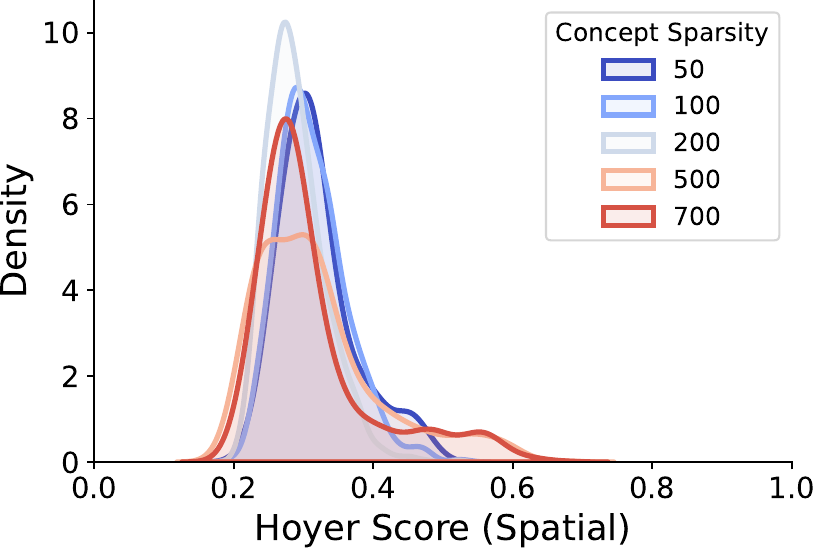}
        \includegraphics[width=0.325\linewidth]{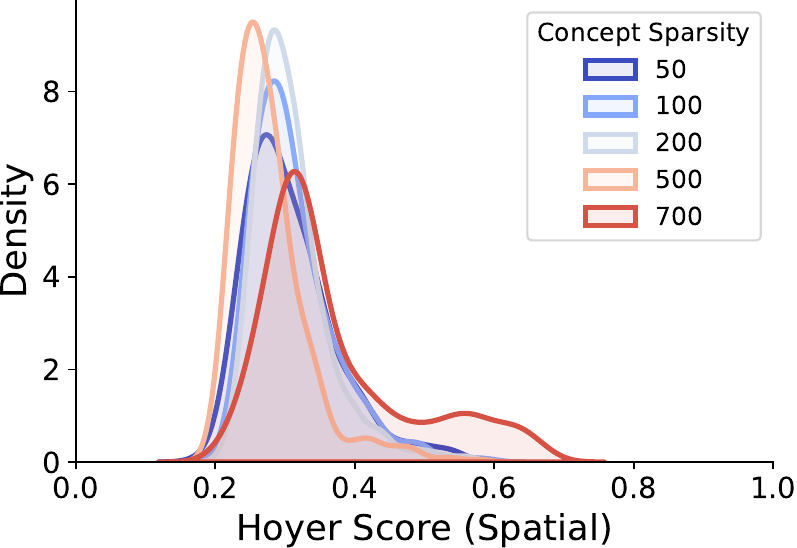}
        \includegraphics[width=0.325\linewidth]{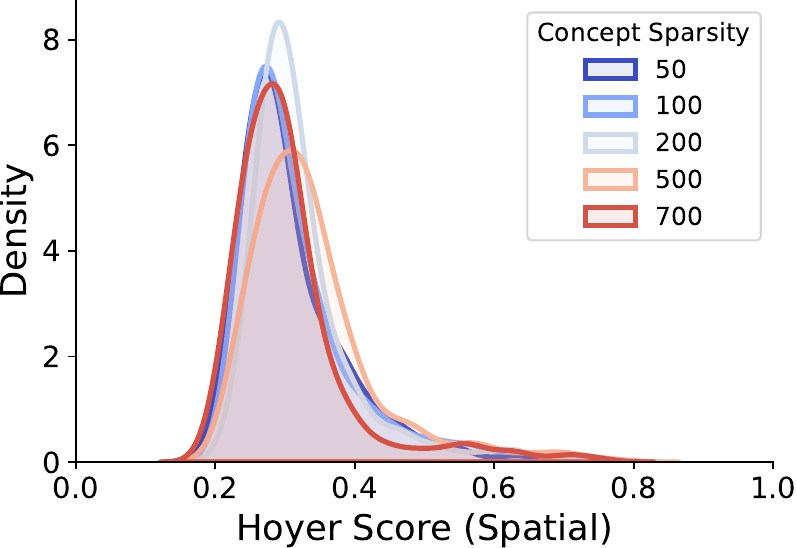}
        \caption{SAE-FNO with 15, 20, and 28 Fourier modes (left to right).}
    \end{subfigure}
    \hfill
    \begin{subfigure}{0.245\textwidth}
        \centering
        \includegraphics[width=1.0\linewidth]{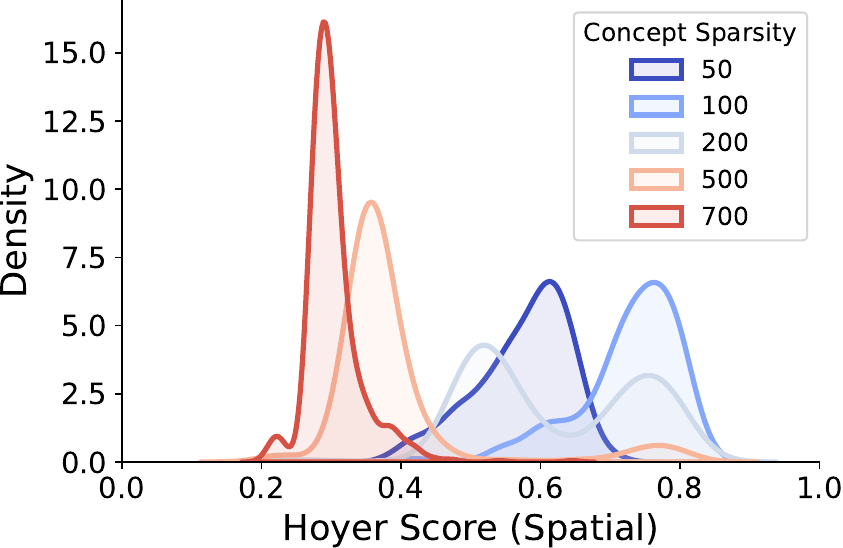}
        \caption{SAE-MLP.}
    \end{subfigure}%

    \caption{\textbf{Stability of concept characteristics across concept sparsity levels (MNIST).} Distribution of Hoyer scores of learned SAE-MLP and SAE-FNO (spatial sparsity 0.007) concepts across concept sparsity levels $k$ (densities normalized to unit area). (a) SAE-FNO's distribution is stable across $k$ for full and truncated modes. (b) SAE-MLP's distribution shifts drastically with $k$.}
    \label{fig:mnist-fno-hoyer}
\end{figure}

\begin{figure}[H]
    \centering
    \begin{subfigure}{1\textwidth}
        \centering
        \includegraphics[width=0.325\linewidth]{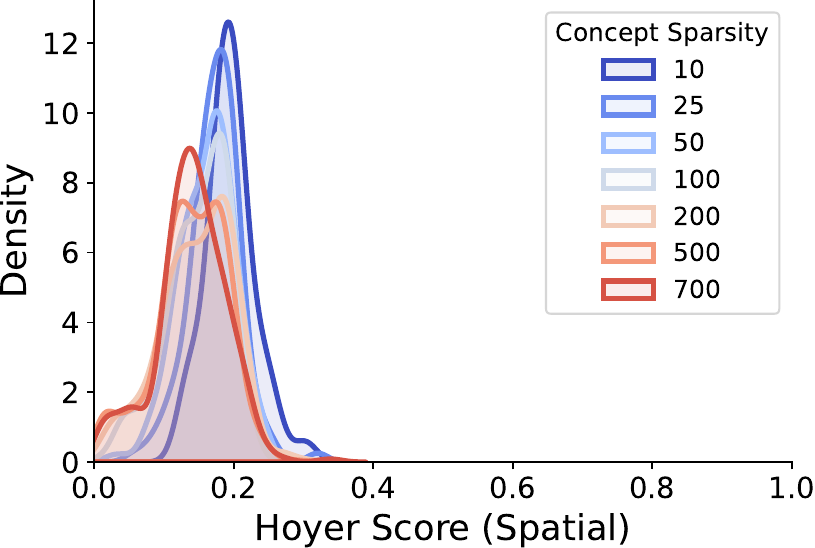}
        \includegraphics[width=0.325\linewidth]{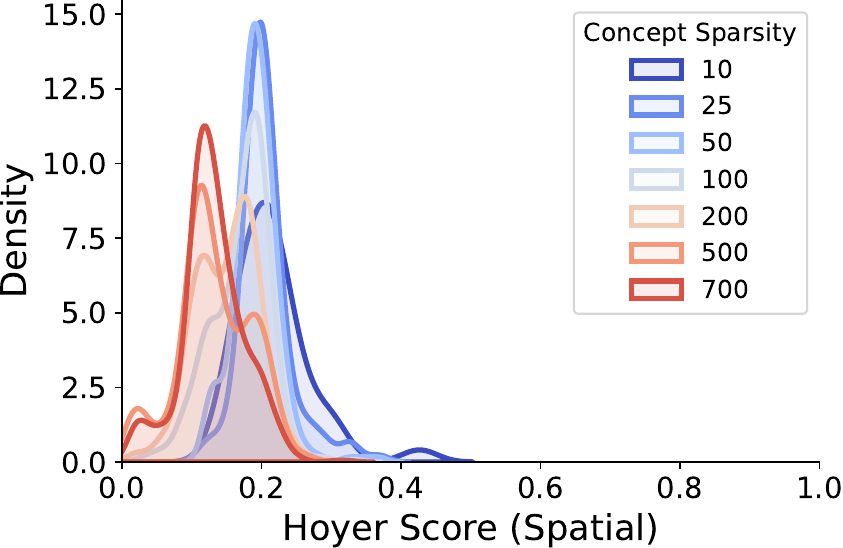}
        \includegraphics[width=0.325\linewidth]{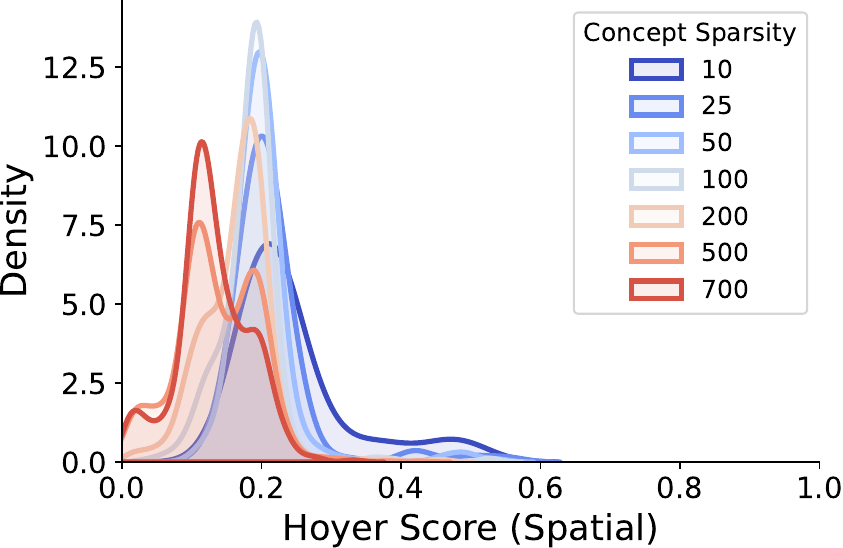}
        \caption{SAE-FNO with 4, 6, and 8 Fourier modes (left to right).}
        \label{fig:cifar-fno-hoyer-modes}
    \end{subfigure}%
    
    \begin{subfigure}{0.74\textwidth}
        \centering
        \includegraphics[width=0.325\linewidth]{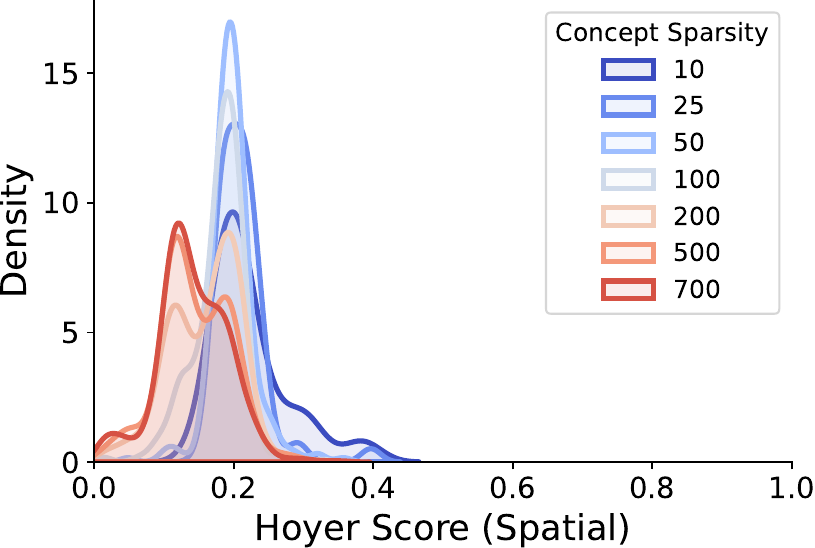}
        \includegraphics[width=0.325\linewidth]{figures/cifar_hoyer/hoyer_score_overlay_fno_p1000_m6_ss0p05_spatial.pdf}
        \includegraphics[width=0.325\linewidth]{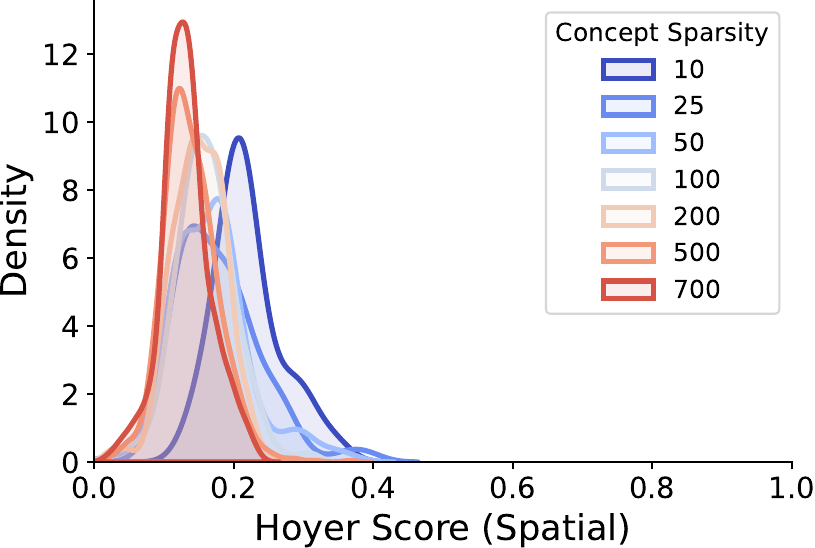}
        \caption{SAE-FNO with spatially 6-, 3-, and 1-sparse activations (left to right). }
        \label{fig:cifar-fno-hoyer-ss}
    \end{subfigure}%
    \begin{subfigure}{0.245\textwidth}
        \centering
        \includegraphics[width=1.0\linewidth]{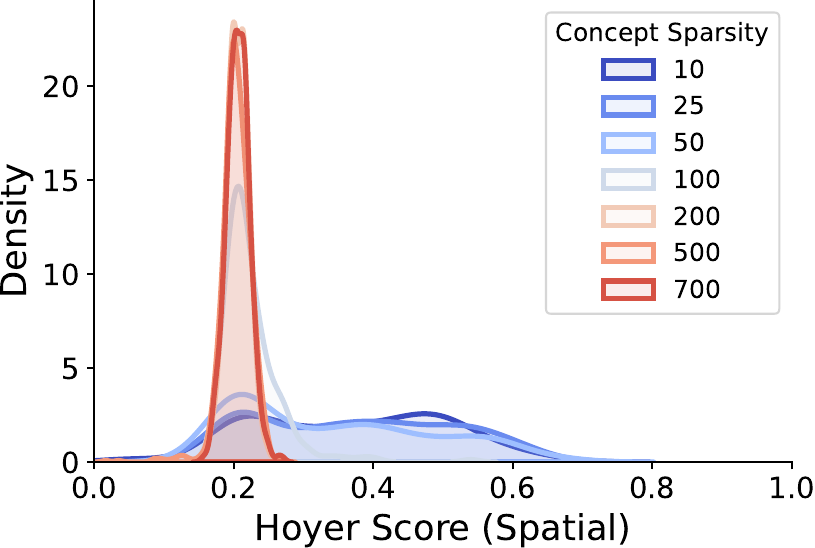}
        \caption{SAE-MLP.}
    \end{subfigure}%
    
    \caption{\textbf{Stability of concept characteristics across Fourier modes and domain sparsity levels (CIFAR-10).} Distribution of spatial Hoyer scores across concept sparsity levels $k$. (a) SAE-FNO with varying Fourier modes and fixed domain sparsity (3-sparse). (b) SAE-FNO with varying domain sparsity and 6 Fourier modes. (c) SAE-MLP for comparison. SAE-FNO maintains stable concept characteristics across $k$, whereas SAE-MLP shows larger changes in its Hoyer score distributions. }
    \label{fig:cifar-fno-hoyer}
\end{figure}

\begin{figure}[H]
    \centering
    \begin{subfigure}{1.00\textwidth}
        \centering
        \includegraphics[width=0.325\linewidth]{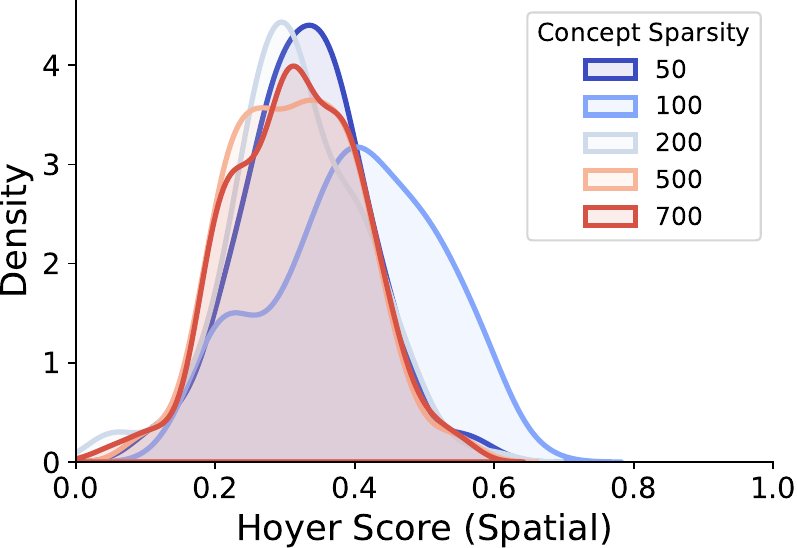}
        \includegraphics[width=0.325\linewidth]{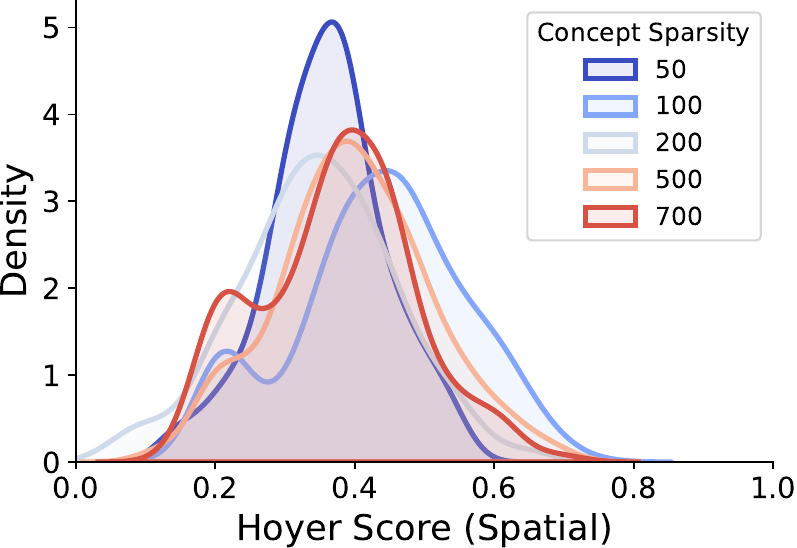}
        \includegraphics[width=0.325\linewidth]{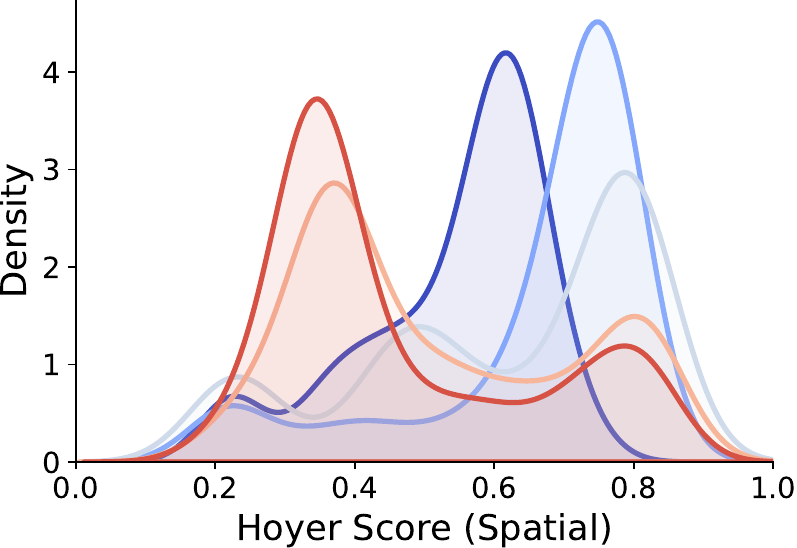}
        \caption{MNIST. From left to right: kernel size $8 \times 8$, $10 \times 10$, $20 \times 20$.}
    \label{fig:mnist-cnn-hoyer}
    \end{subfigure}%
    \\
    \begin{subfigure}{1.00\textwidth}
        \centering
        \includegraphics[width=0.325\linewidth]{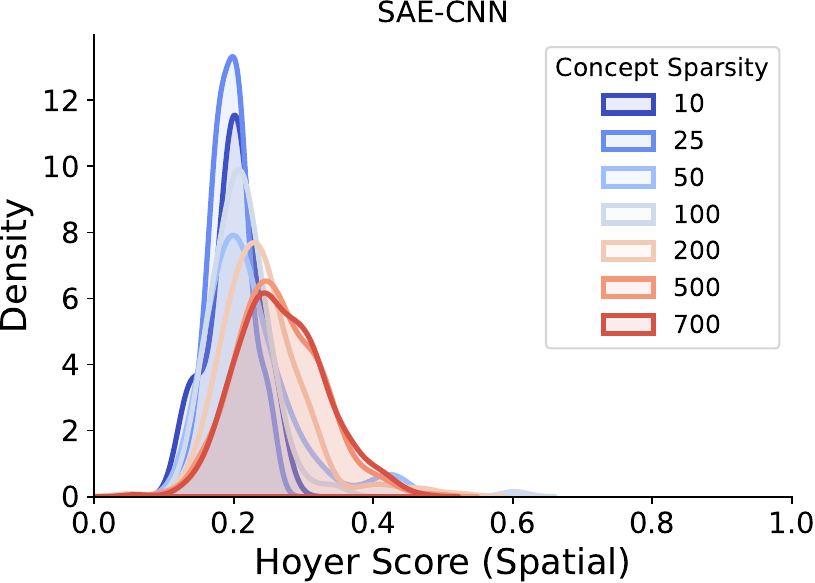}
        \includegraphics[width=0.325\linewidth]{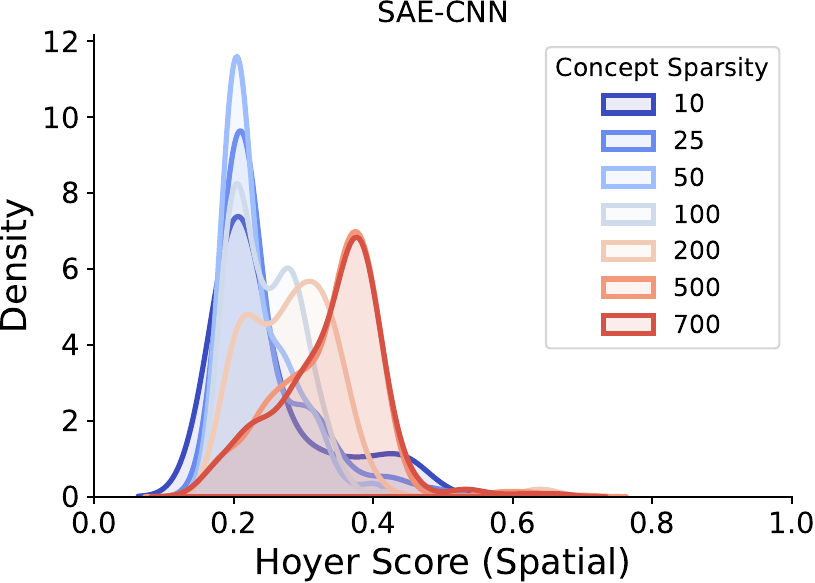}
        \includegraphics[width=0.325\linewidth]{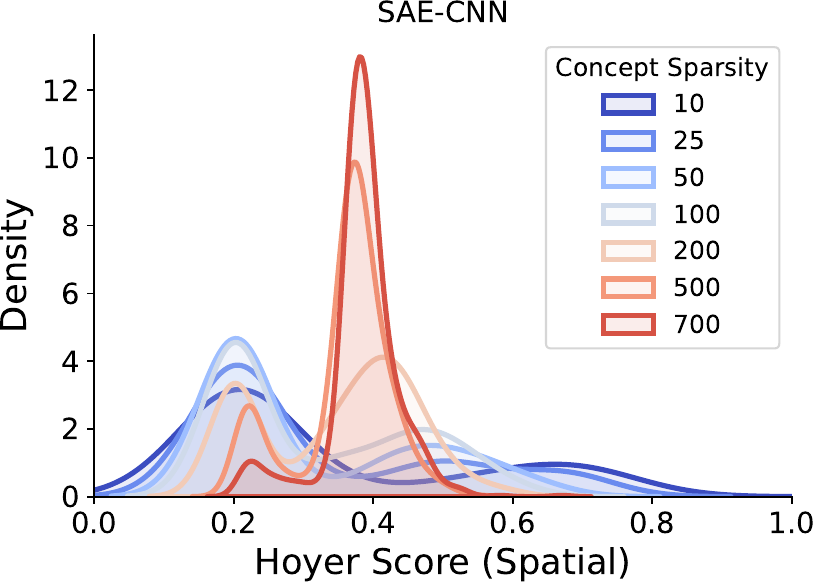}
        \caption{CIFAR. Domain sparsity: 5-sparse. From left to right: kernel size $5\times5$, $10 \times 10$, $20 \times 20$.}
        \label{fig:cifar-cnn-hoyer-ker}
    \end{subfigure}%
    \caption{\textbf{Stability of concept characteristics of SAE-CNN depends on kernel size.} The distribution of SAE-CNN atoms becomes more stable as kernel size decreases, consistent with the stronger locality bias imposed by smaller receptive fields.}
    \label{fig:cnn-hoyer}
\end{figure}

\begin{table}[H]
\centering
\caption{\textbf{SAE-FNO has higher structural stability than SAE-CNN as $k$ changes.} Variance of the mean, median, and mode of the Hoyer-score distribution ($\times 10^{-4}$), together with the best $R^2$ across $k$, on CIFAR $32 \times 32$ (domain sparsity of 5 pixels). Lower variance indicates concept structure that is more stable as $k$ varies.}
\label{tab:cnn-stability}
\begin{tabular}{lcccc}
\toprule
Model & Var(mean) & Var(median) & Var(mode) & $R^2$ (max) \\
\midrule
SAE-FNO (24 modes)      & 6.40 & 1.04 & 0.43 & 0.795 \\
SAE-FNO (full, 32 modes) & 1.61 & 0.03 & 0.03 & 0.945 \\
SAE-CNN ($5\times5$)    & 7.11 & 6.66 & 4.04 & 0.602 \\
SAE-CNN ($10 \times 10$)  & 18.9 & 35.5 & 58.8 & 0.678 \\
SAE-CNN ($15\times15$)  & 26.3 & 69.7 & 91.5 & 0.710 \\
SAE-CNN ($20 \times 20$)  & 6.55 & 71.0 & 64.6 & 0.845 \\
\bottomrule
\end{tabular}
\end{table}

\begin{figure}[H]
    \centering
    \begin{subfigure}{0.5\textwidth}
        \includegraphics[width=\linewidth]{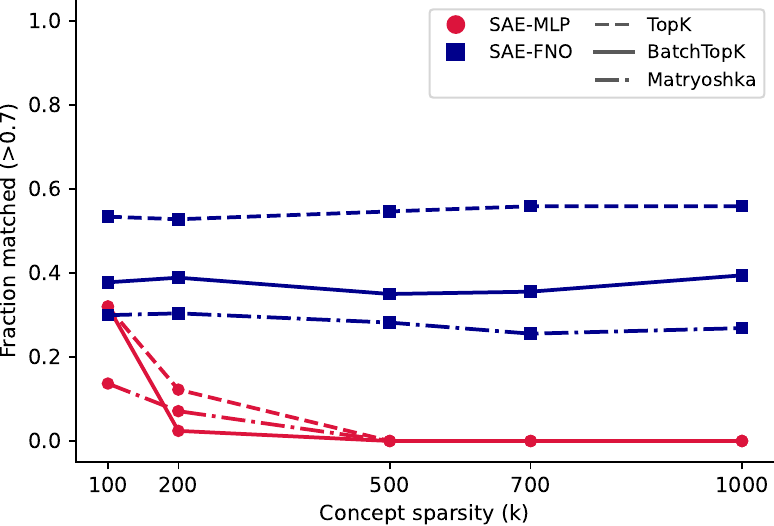}
    \end{subfigure}
    \caption{\textbf{Identity consistency across concept sparsity levels.} Fraction of active concepts from an anchor model ($k=50$) that can be matched to concepts learned at larger $k$, using a similarity threshold of 0.7. Results are shown for TopK, BatchTopK, and Matryoshka SAE variants. SAE-FNO preserves a larger fraction of concepts across sparsity levels than SAE-MLP.}
    \label{fig:identity-consistency-0p7}
\end{figure}

\begin{figure}[H]
    \centering
    \begin{subfigure}{1.00\textwidth}
        \centering
      \includegraphics[width=0.24\linewidth]{figures/mnist_atom_correlation/corr_overlay_arch_compare_p1000_cnn28x28_fno20_signed_k50.pdf}
        \includegraphics[width=0.24\linewidth]{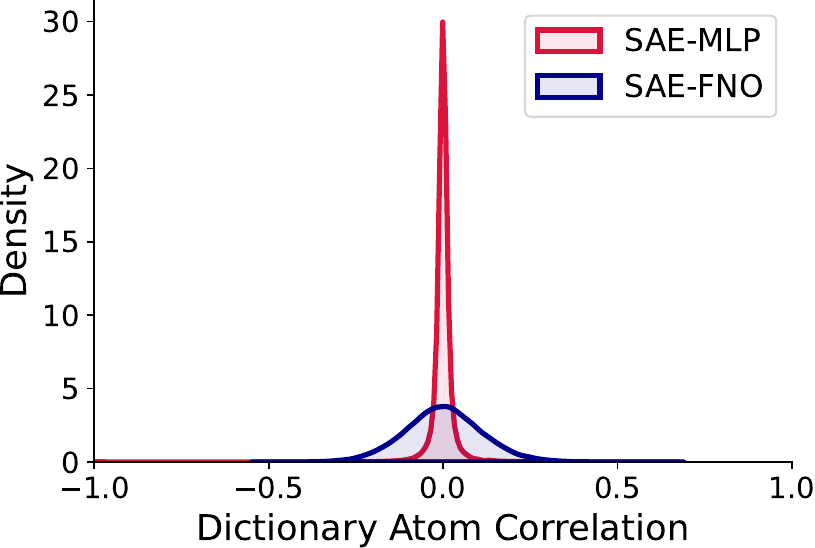}
        \includegraphics[width=0.24\linewidth]{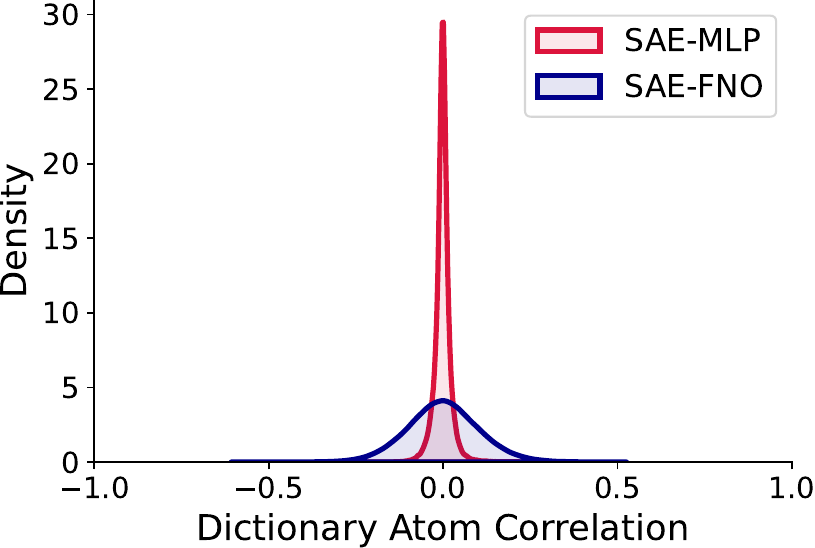}
        \includegraphics[width=0.24\linewidth]{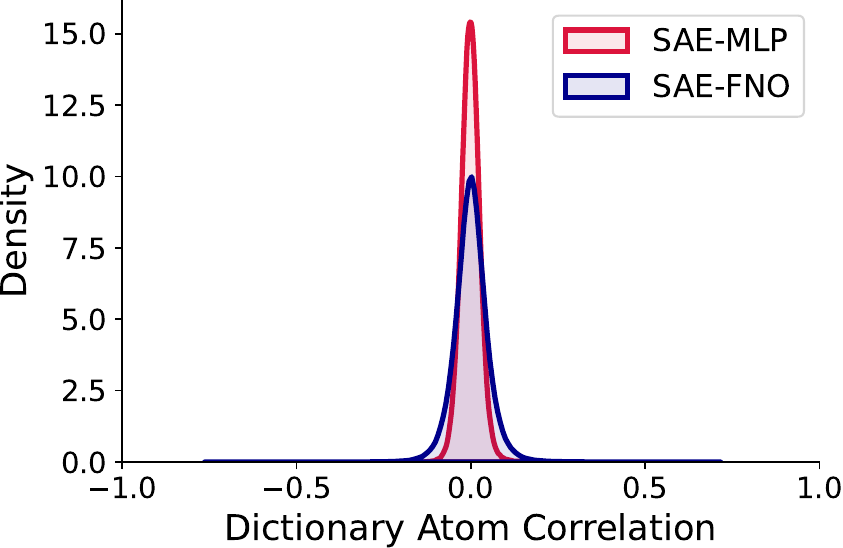}
        \caption{MNIST. From left to right: $k=50, 100, 200, 700$.}
    \label{fig:mnist-atom-corr-supp}
    \end{subfigure}
    \begin{subfigure}{1.00\textwidth}
        \centering
        \includegraphics[width=0.24\linewidth]{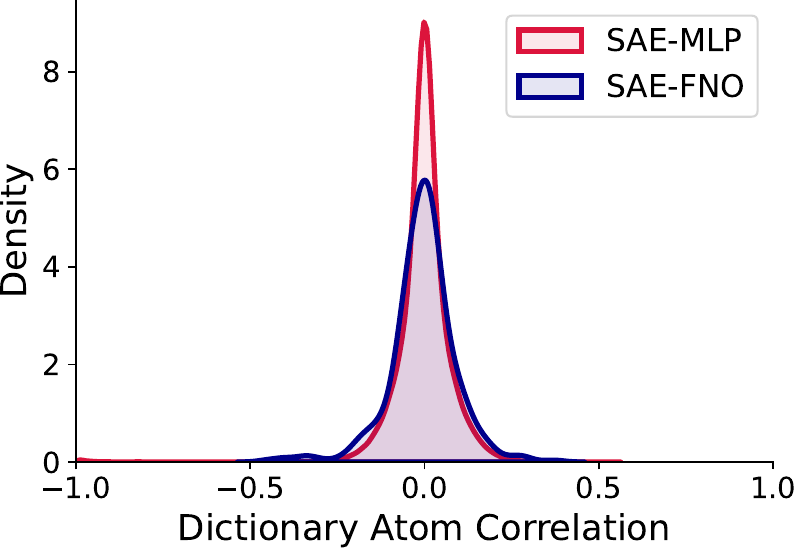}
        \includegraphics[width=0.24\linewidth]{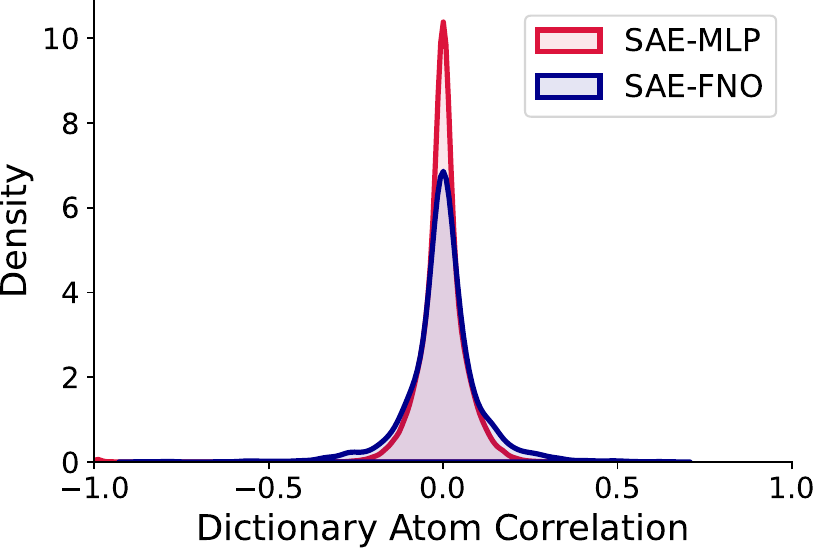}
        \includegraphics[width=0.24\linewidth]{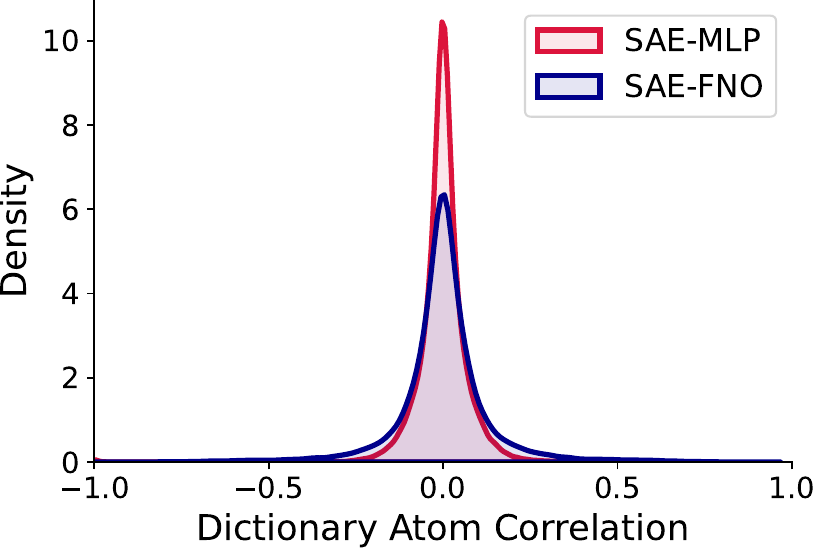}
        \includegraphics[width=0.24\linewidth]{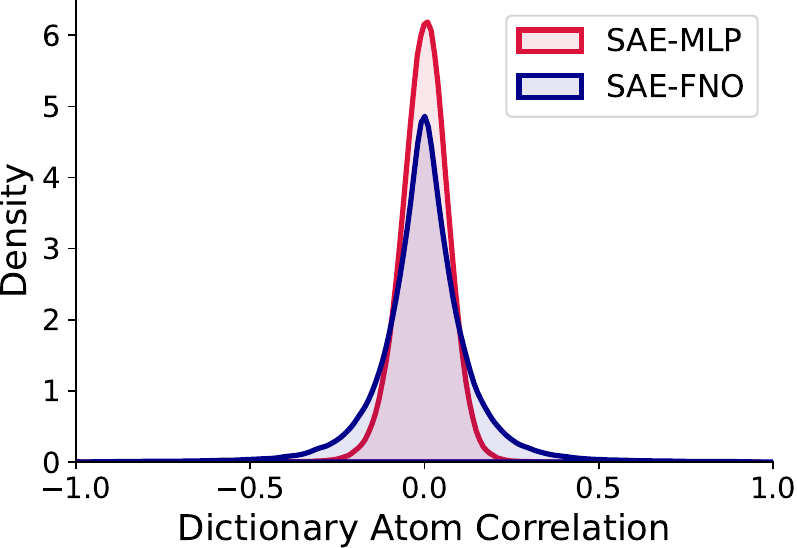}
            \caption{CIFAR-10. SAE-FNO uses 6 Fourier modes and spatial sparsity of 0.02. From left to right: $k=10, 25, 100, 700$.}
    \label{fig:cifar-atom-corr}
    \end{subfigure}
    \label{fig:atom-corr-supp}
    \caption{\textbf{SAE-FNO overcomes orthogonality bias and recovers a broader range of correlated concepts.} Distribution of pairwise cosine similarities between learned SAE-FNO (blue) and SAE-MLP (red) dictionary atoms across concept sparsity levels. SAE-FNO displays a broader correlation distribution with heavier tails compared to SAE-MLP.}
\end{figure}

\section{Appendix - Concept Visualization Experiments}\label{app:conceptvisualization}

\begin{figure}[H]
    \centering
    \begin{subfigure}{1.0\textwidth}
        \includegraphics[width=0.32\linewidth]{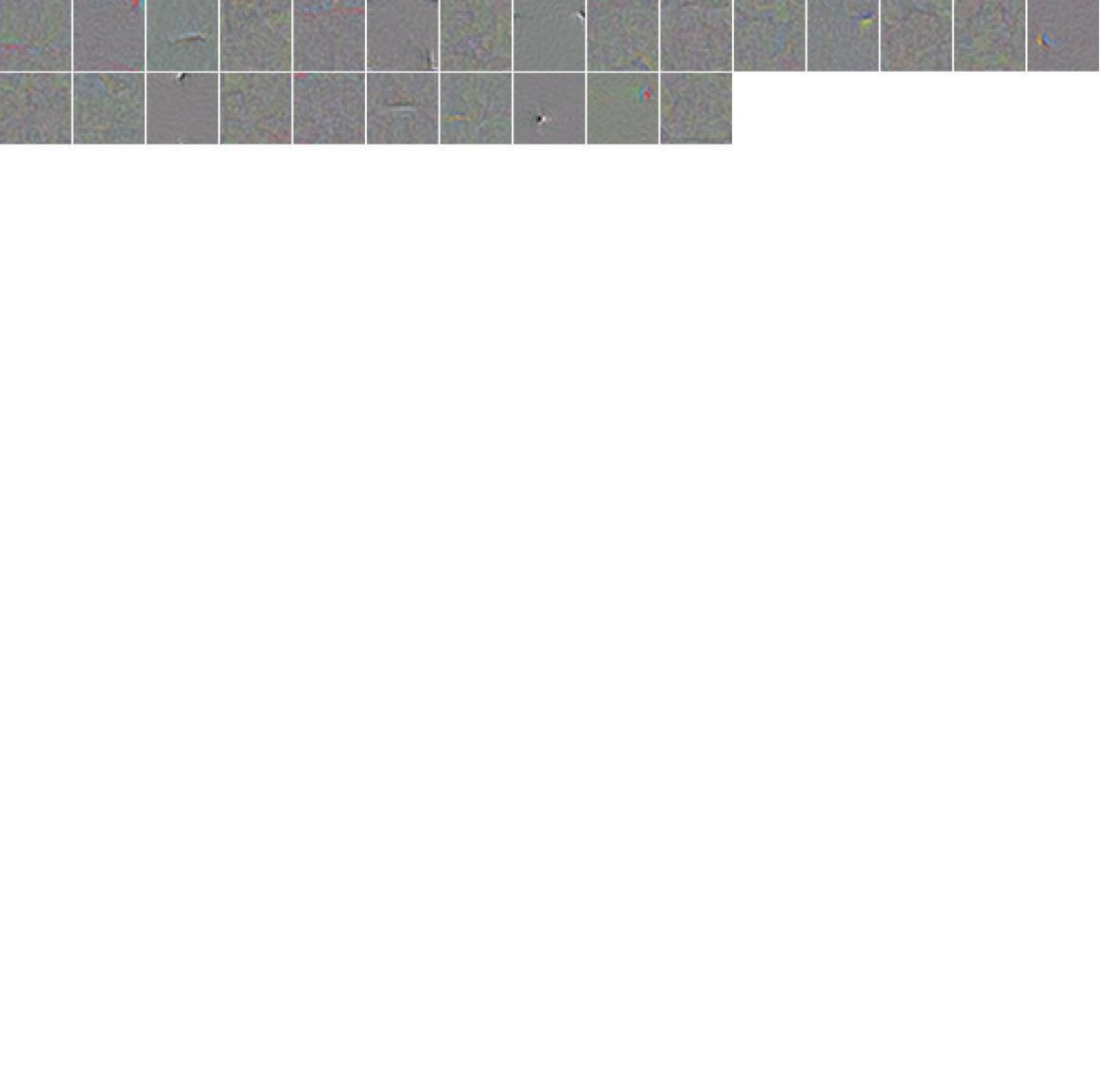}
        \includegraphics[width=0.32\linewidth]{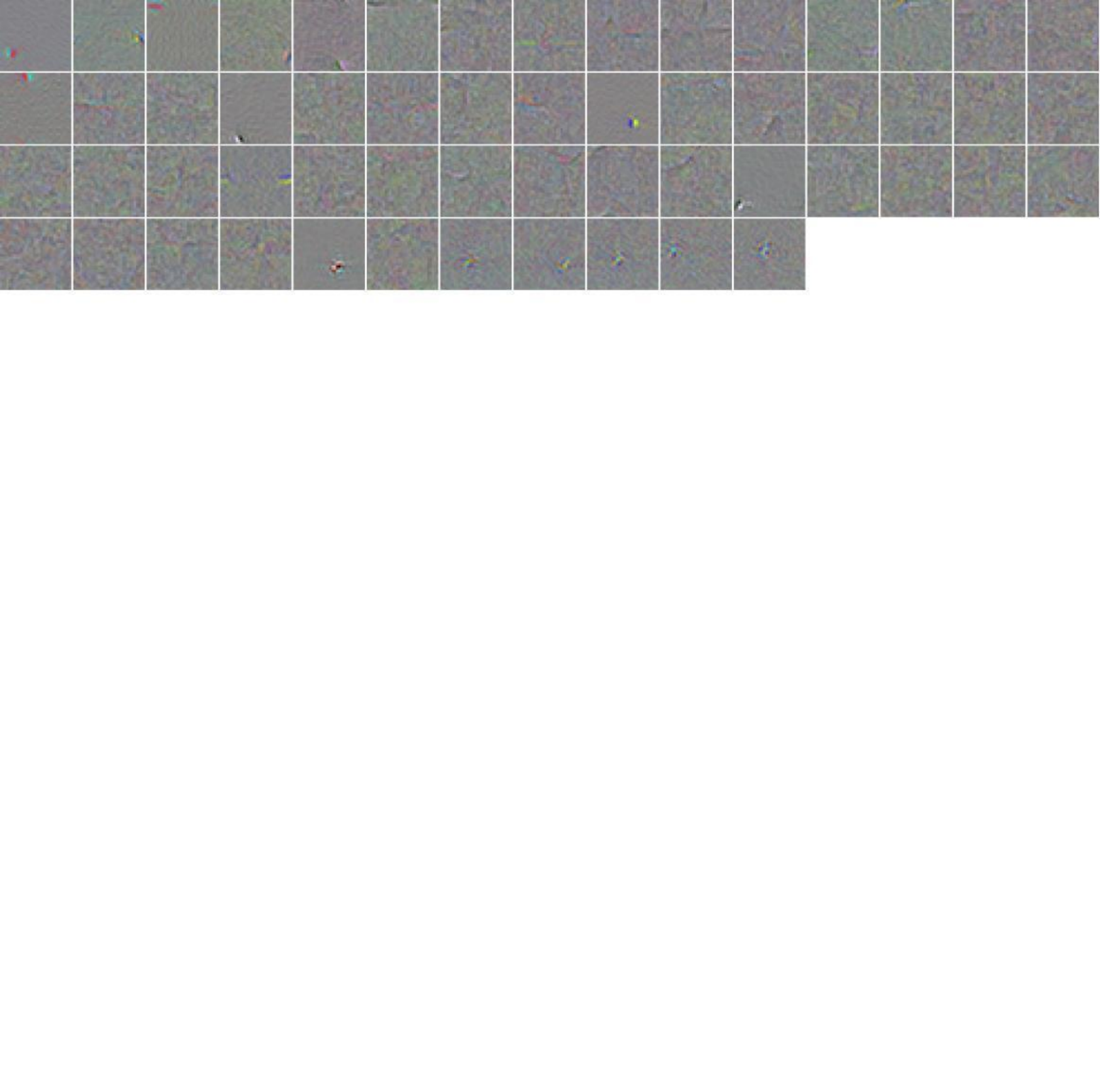}
        \includegraphics[width=0.32\linewidth]{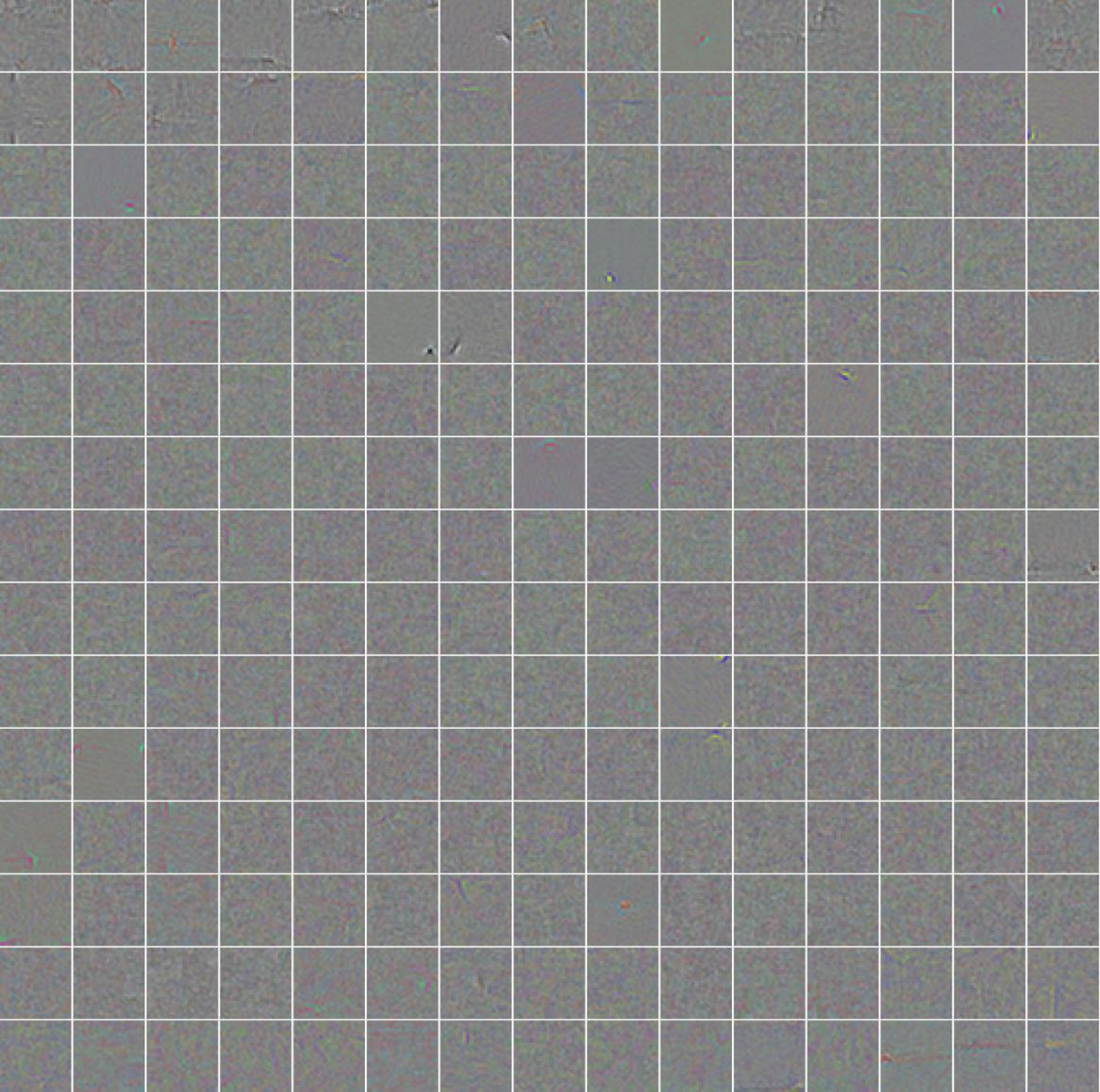}
        \caption{SAE-FNO.}
    \end{subfigure}
    \begin{subfigure}{1.0\textwidth}
        \includegraphics[width=0.32\linewidth]{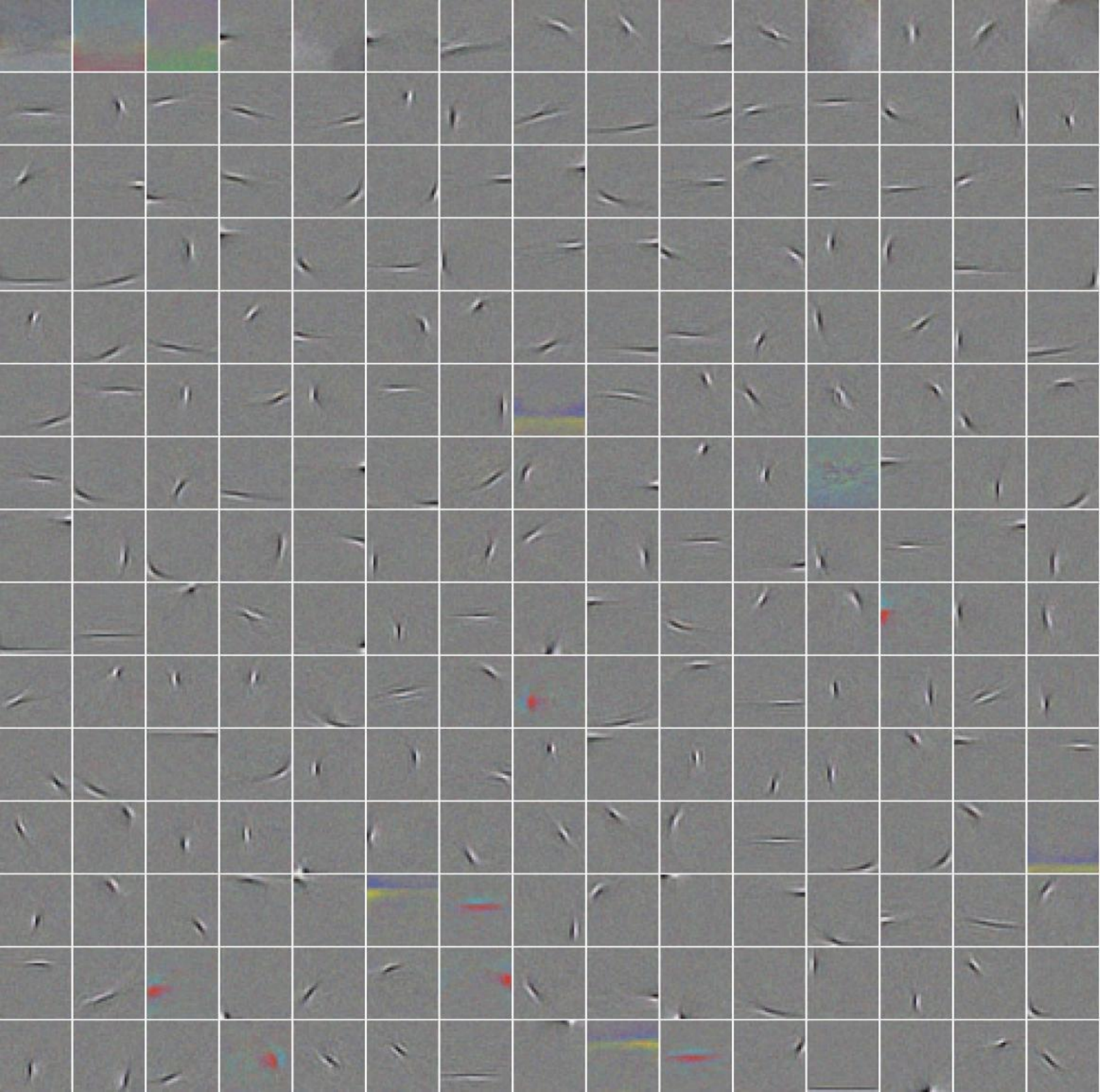}
        \includegraphics[width=0.32\linewidth]{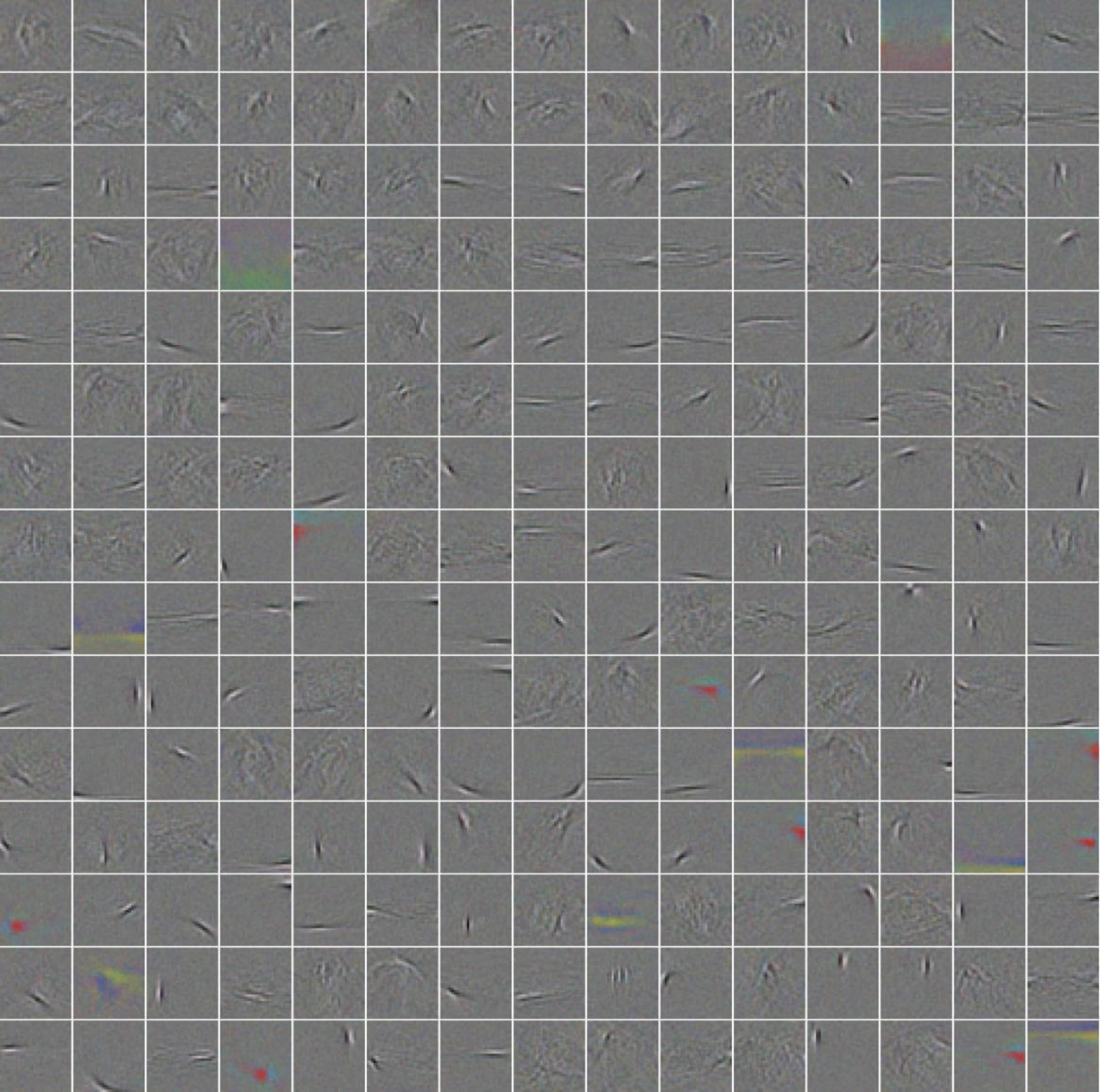}
        \includegraphics[width=0.32\linewidth]{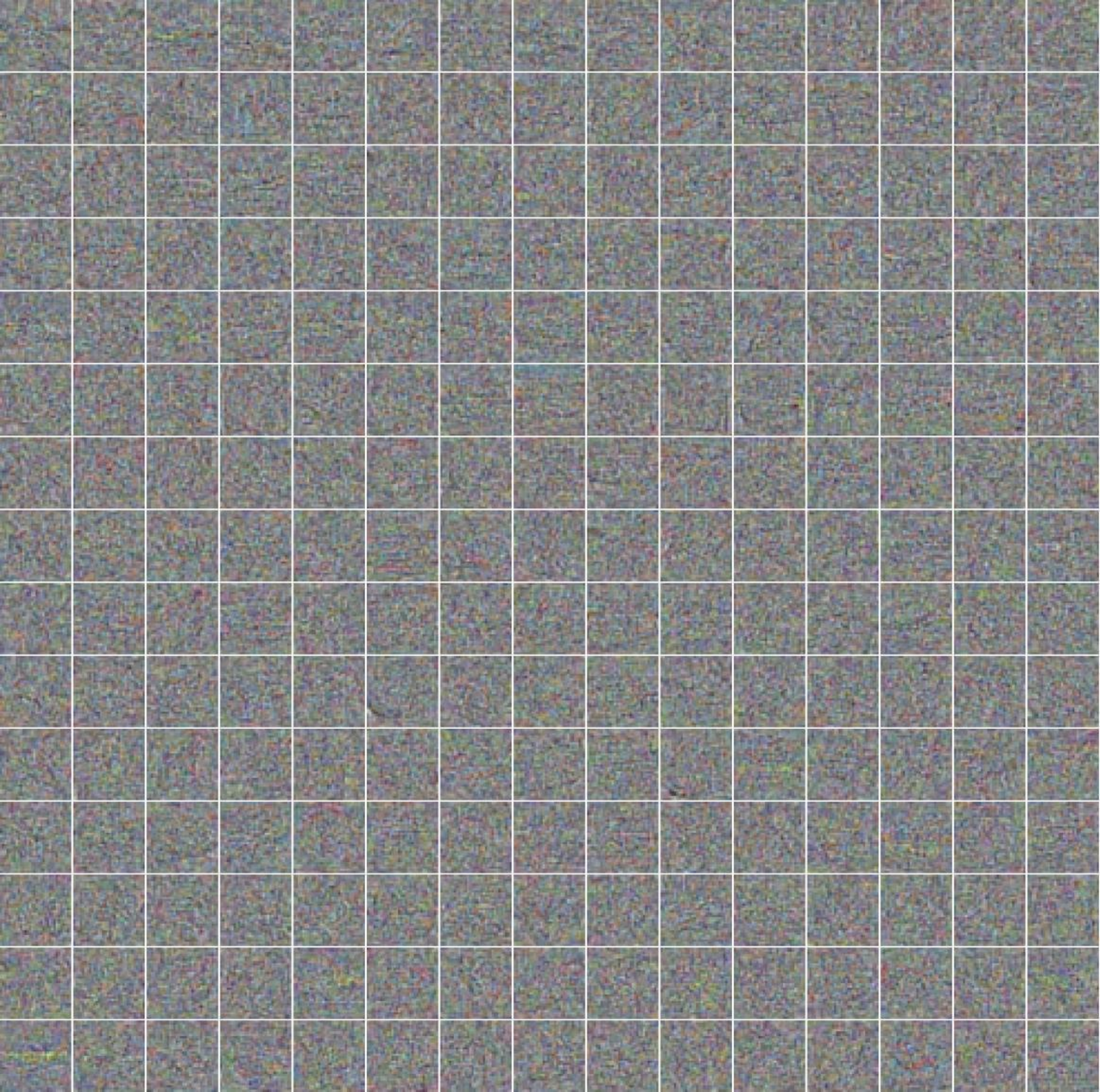}
        \caption{SAE-MLP.}
    \end{subfigure}
    \vspace{-3mm}
    \caption{\textbf{Learned concepts across sparsity levels on CIFAR $32 \times 32$ images}. Most frequently used concepts learned by (a) SAE-FNO-TopK and (b) SAE-MLP-TopK, with $p = 1000$ and $k=25, 50, 200$ from left to right. Each square represents one concept. Fewer displayed concepts indicate more efficient usage (see $R^2$ vs. used atoms in \Cref{fig:utilization}). (a) SAE-FNO learns localized frequency structures, consistent with natural image statistics~\citep{olshausen1997sparse}. SAE-FNO concept characteristics remain stable across sparsity levels, enabled by domain sparsity and adaptability to input size. As sparsity decreases, more concepts are activated. (b) SAE-MLP concepts change more substantially with sparsity, losing structure and failing to capture locality. SAE-FNO requires fewer effective concepts than SAE-MLP, indicating superior expressivity for structured, large-domain inputs.}
    \label{fig:concept-visualization-32}
    \vspace{-4mm}
\end{figure}

\newpage
\begin{figure}[H]
    \centering
    \begin{subfigure}{1.0\textwidth}
        \includegraphics[width=0.32\linewidth]{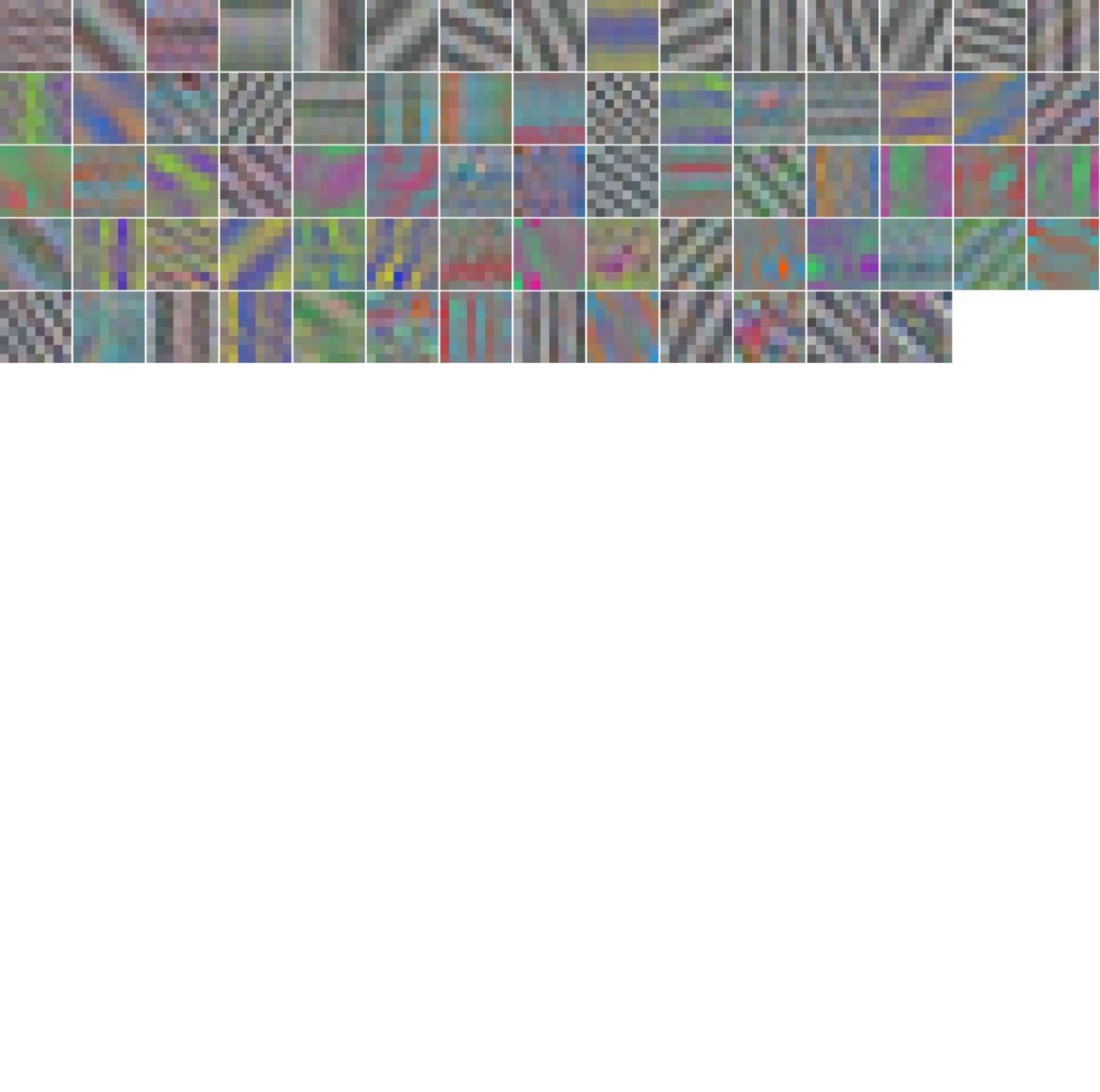}
        \includegraphics[width=0.32\linewidth]{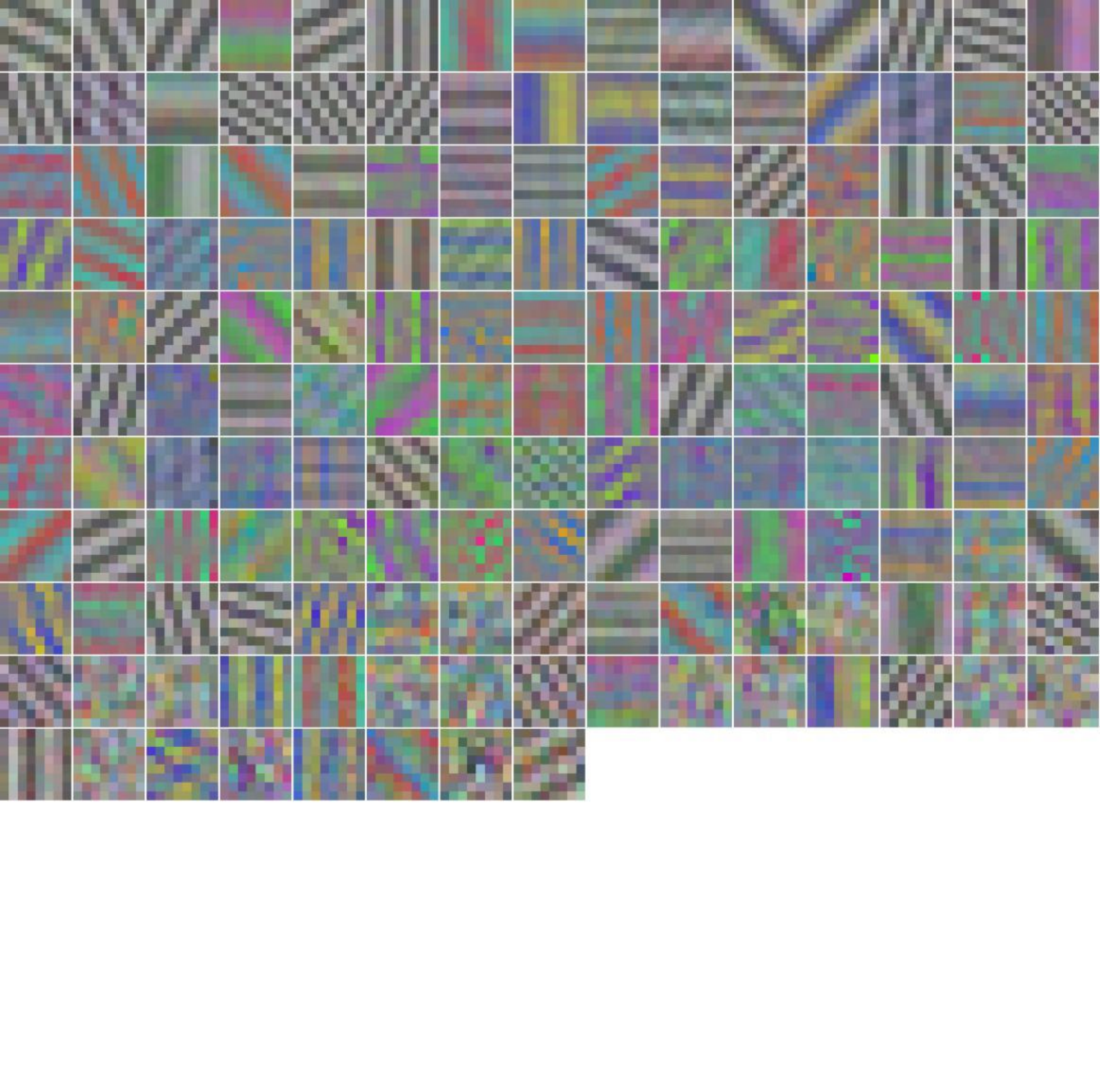}
        \includegraphics[width=0.32\linewidth]{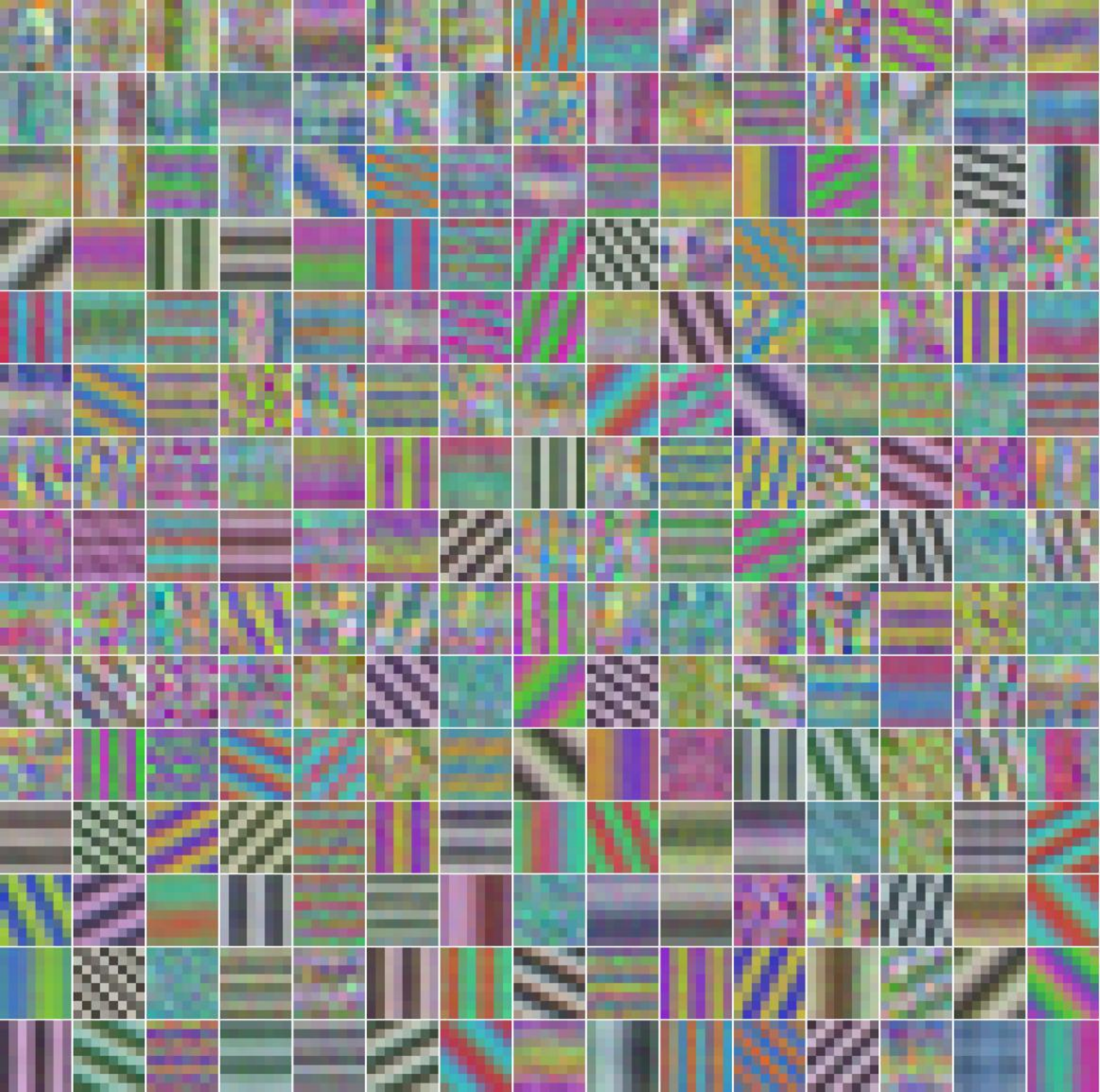}
        \caption{1-sparse.}
    \end{subfigure}
    \begin{subfigure}{1.0\textwidth}
        \includegraphics[width=0.32\linewidth]{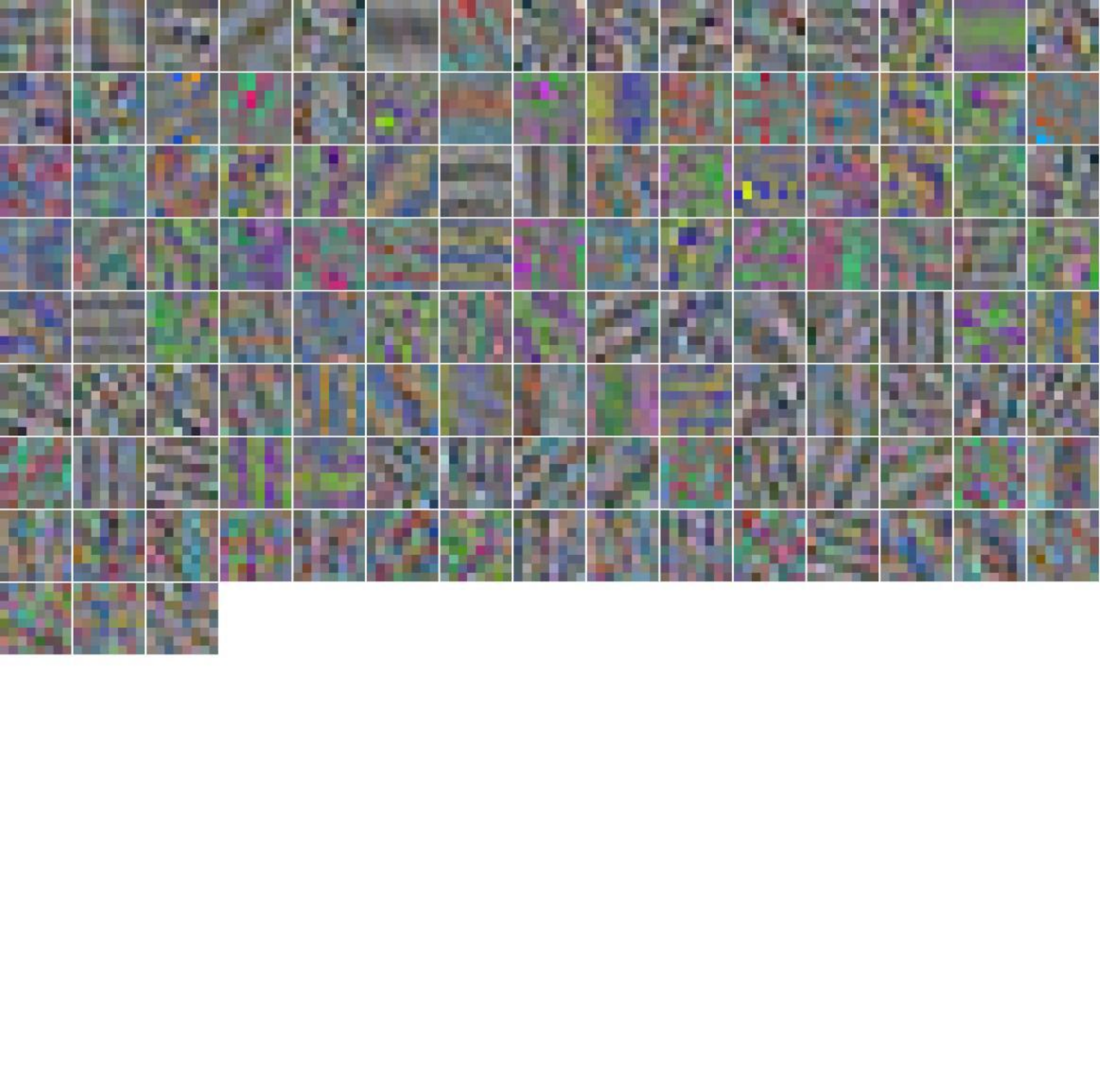}
        \includegraphics[width=0.32\linewidth]{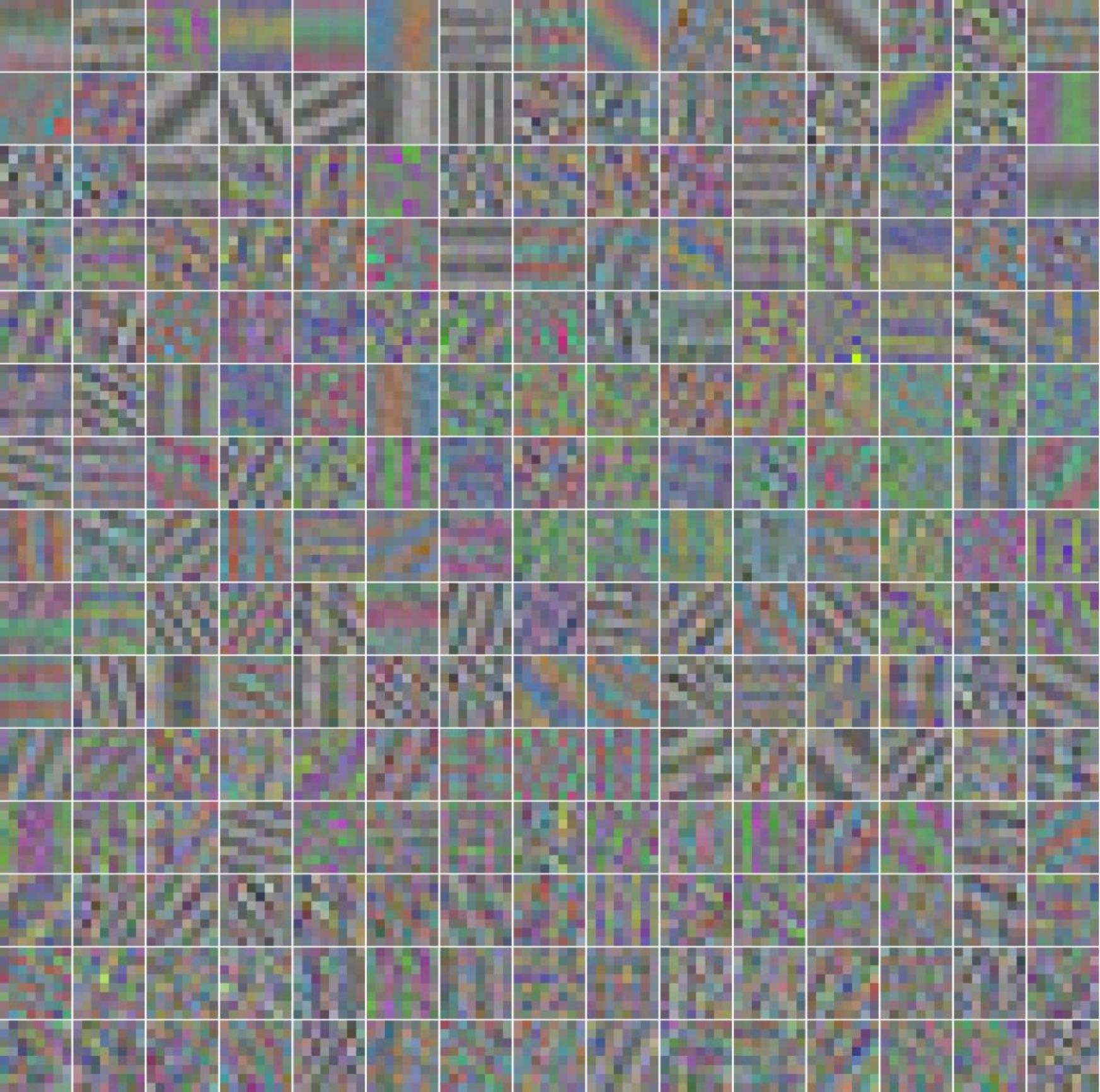}
        \includegraphics[width=0.32\linewidth]{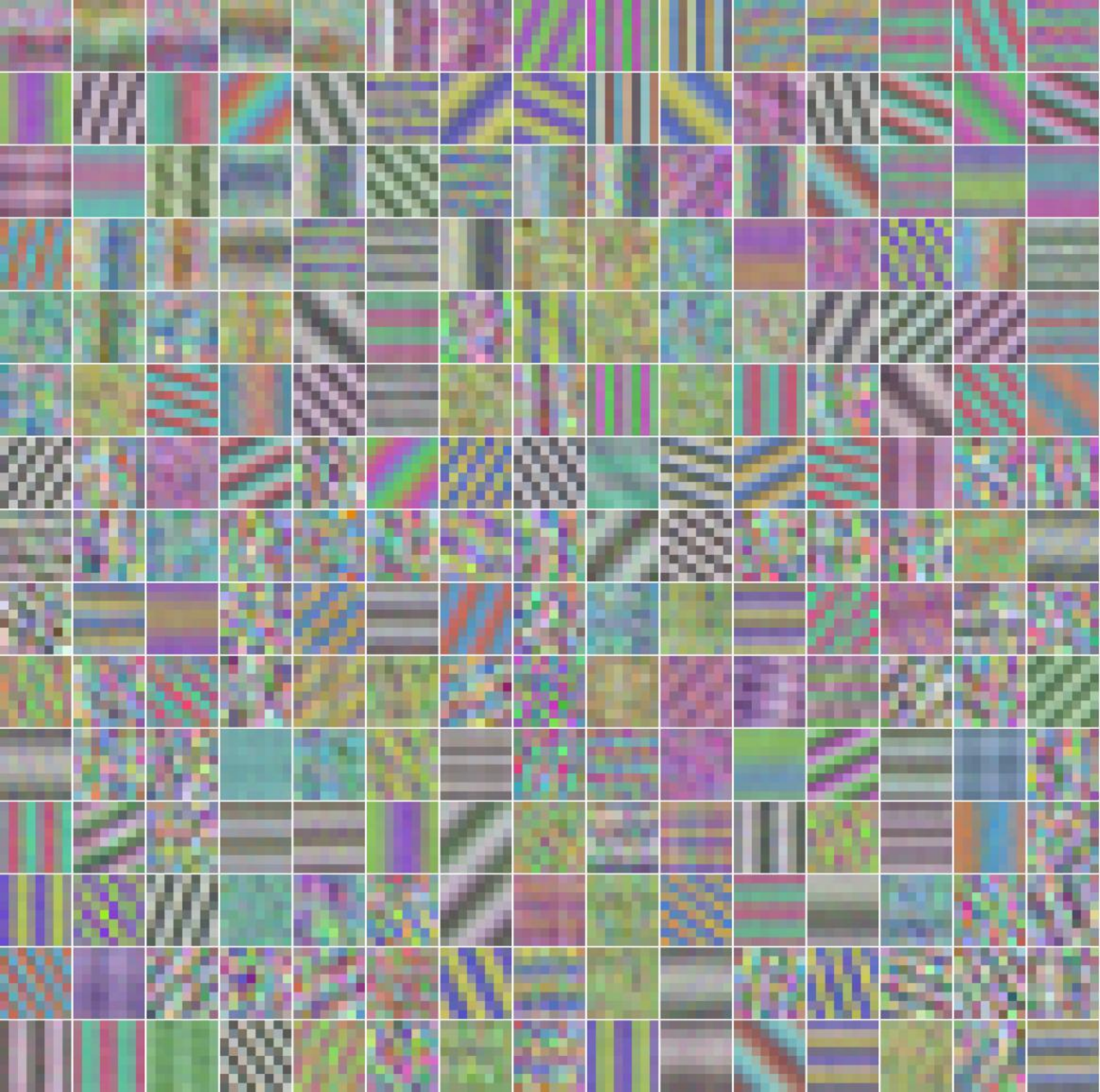}
        \caption{3-sparse.}
    \end{subfigure}
    \begin{subfigure}{1.0\textwidth}
        \includegraphics[width=0.32\linewidth]{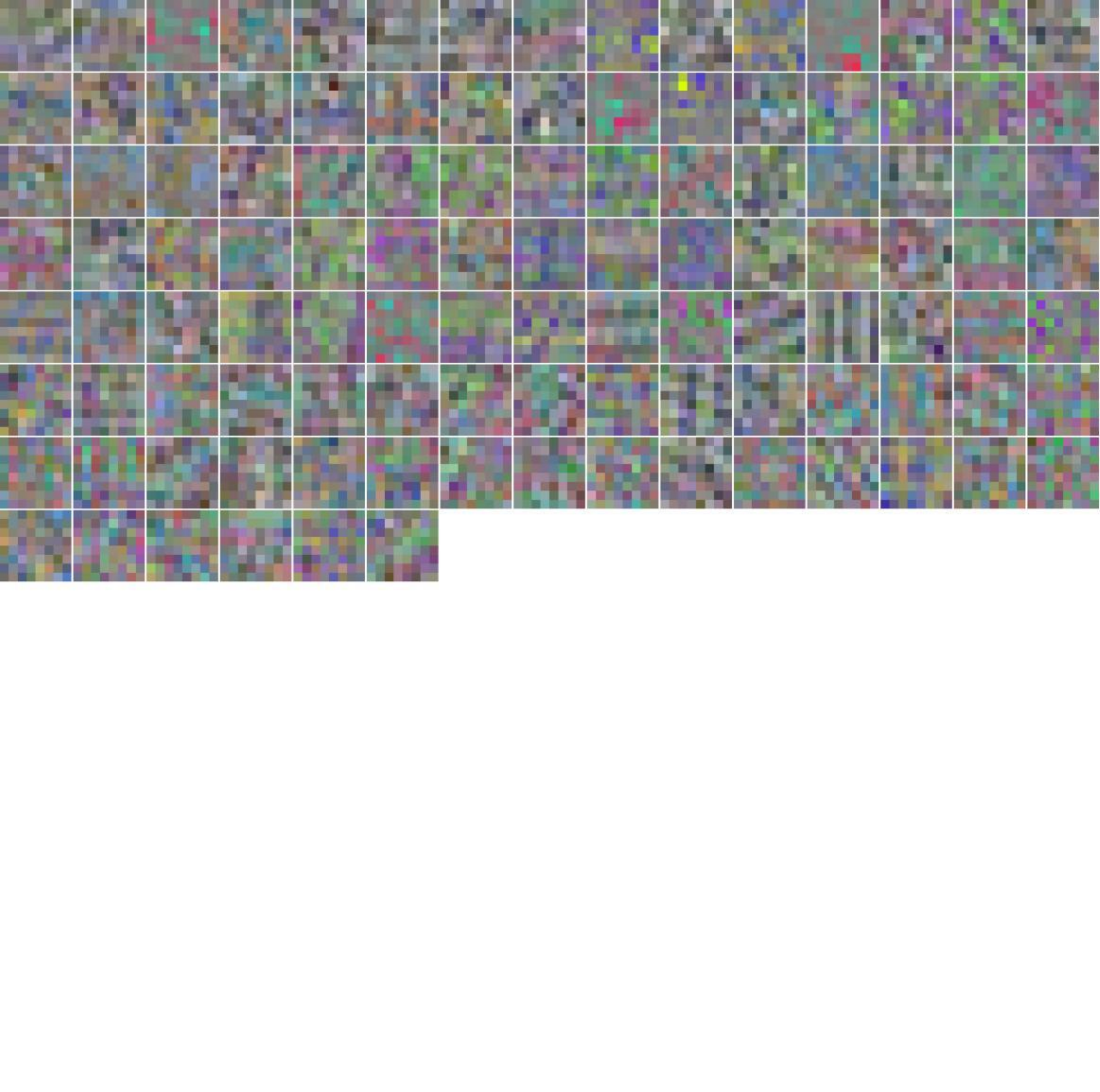}
        \includegraphics[width=0.32\linewidth]{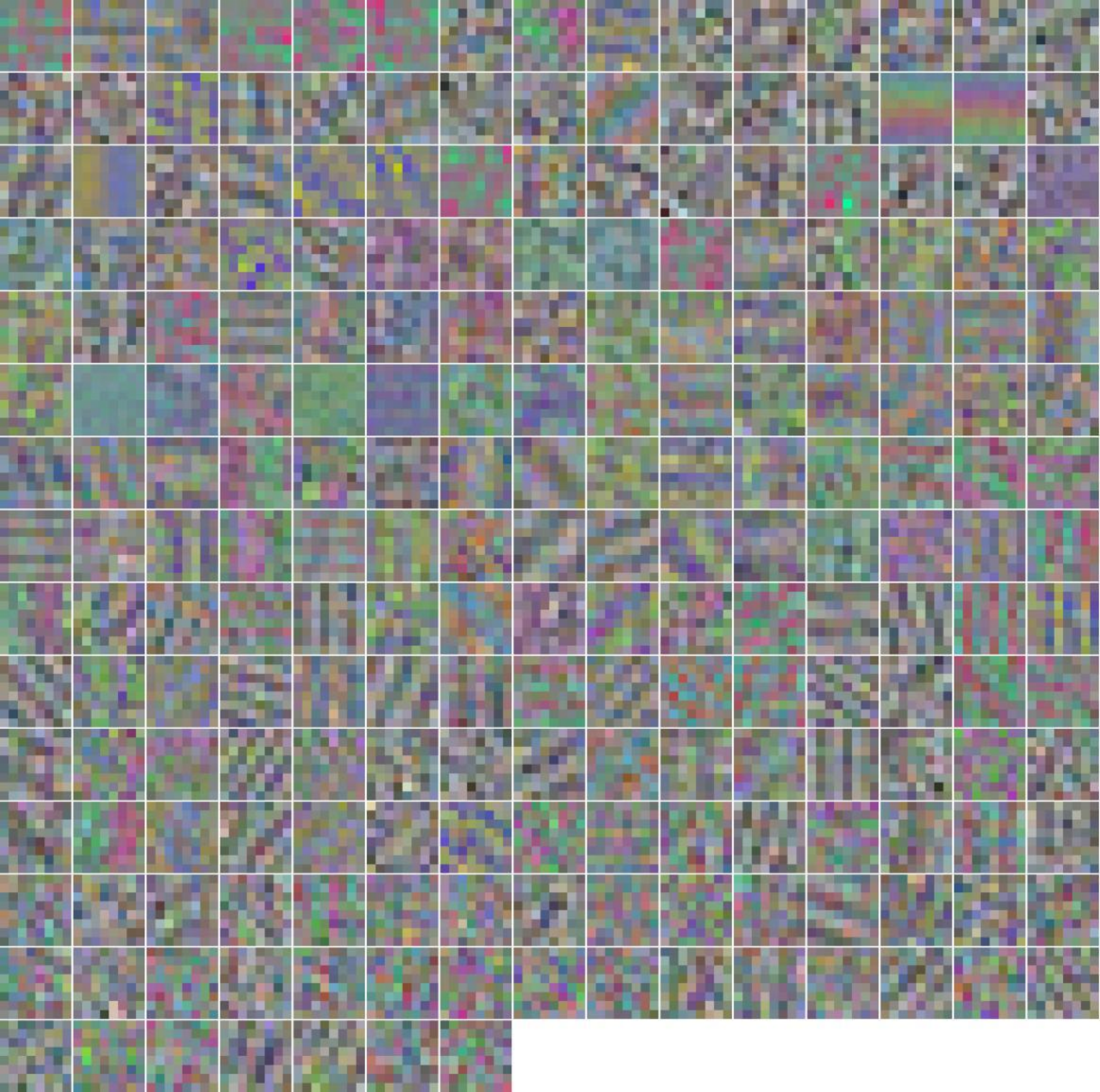}
        \includegraphics[width=0.32\linewidth]{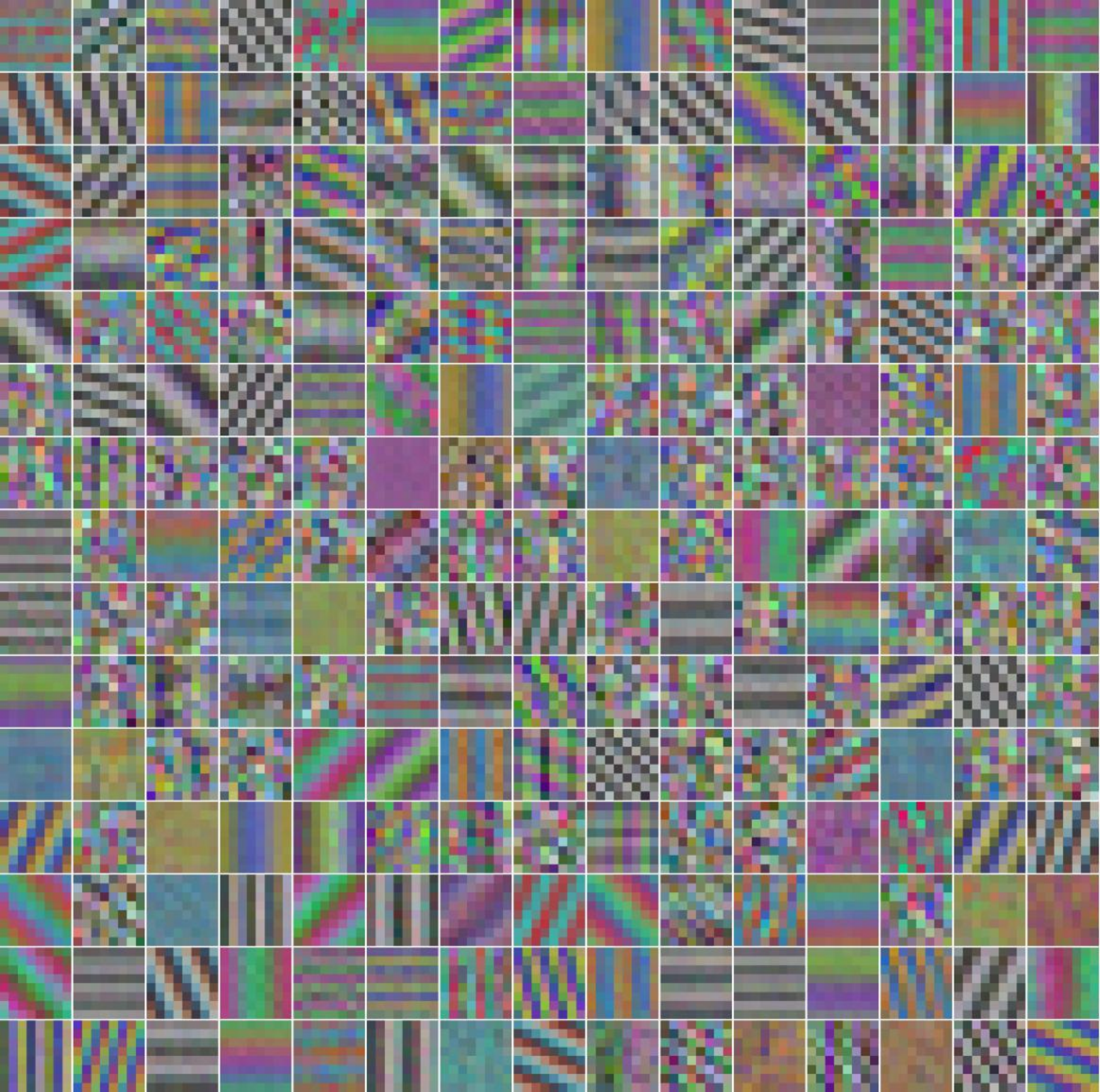}
        \caption{6-sparse.}
    \end{subfigure}
    \begin{subfigure}{1.0\textwidth}
        \includegraphics[width=0.32\linewidth]{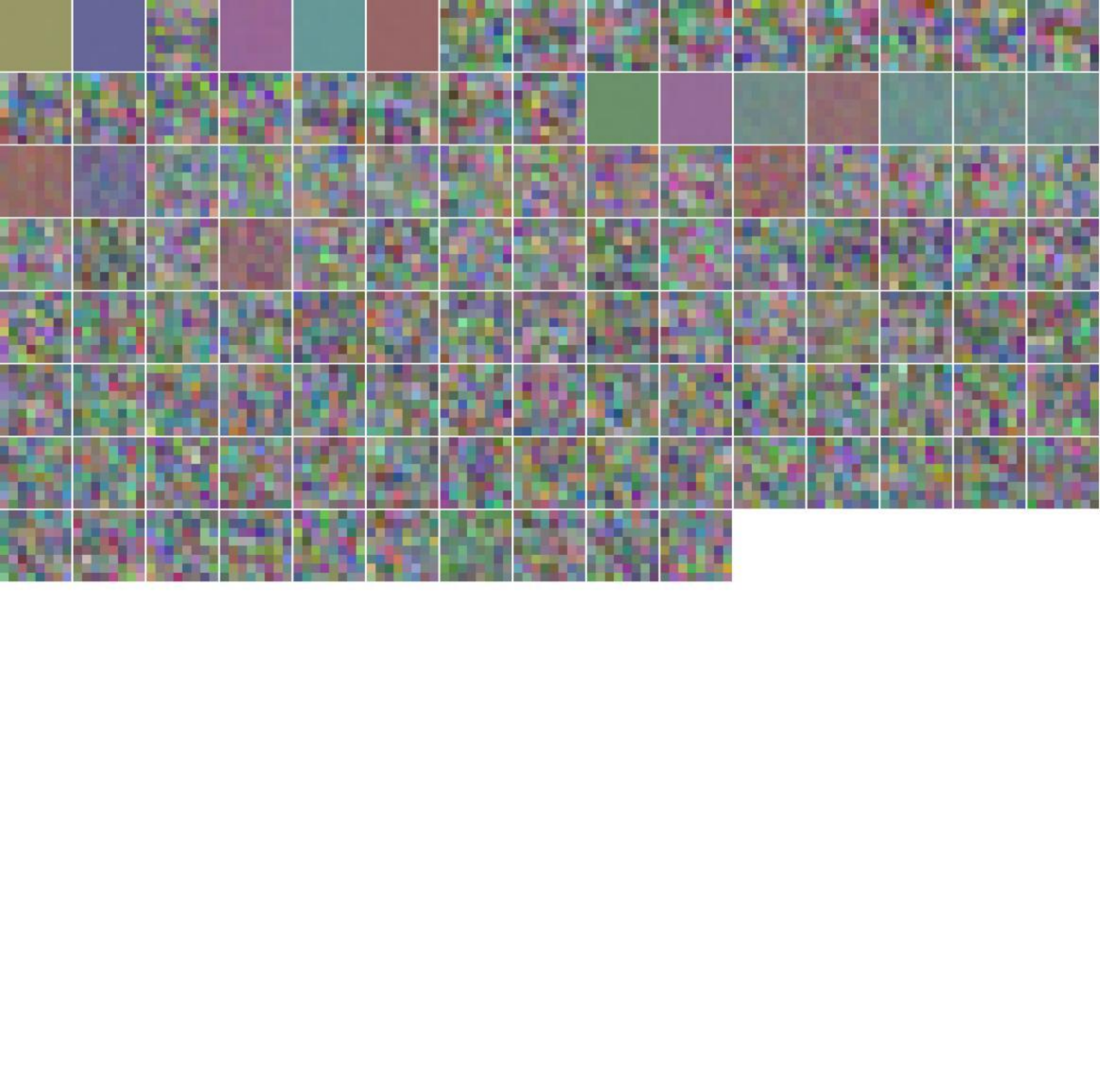}
        \includegraphics[width=0.32\linewidth]{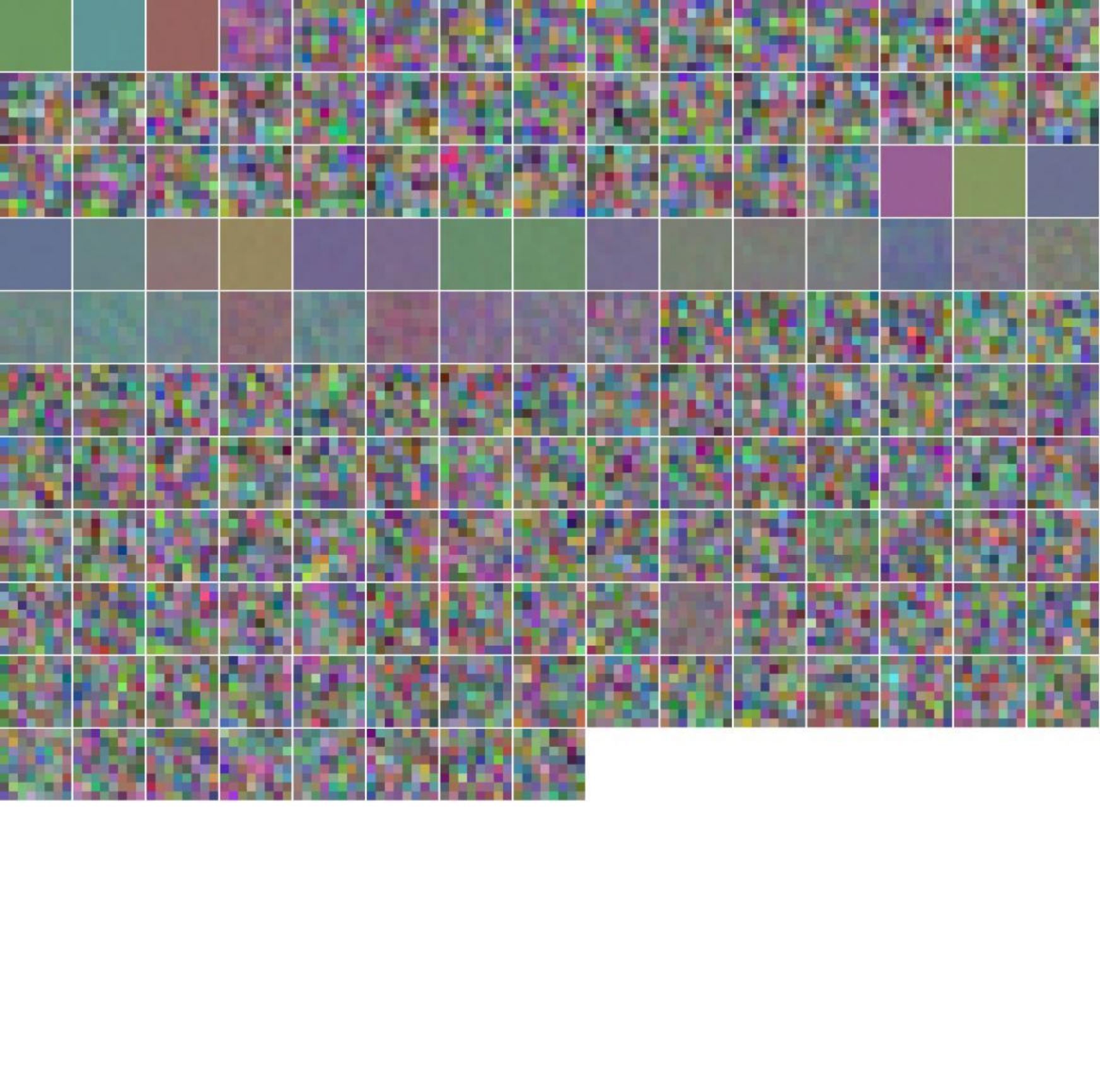}
        \includegraphics[width=0.32\linewidth]{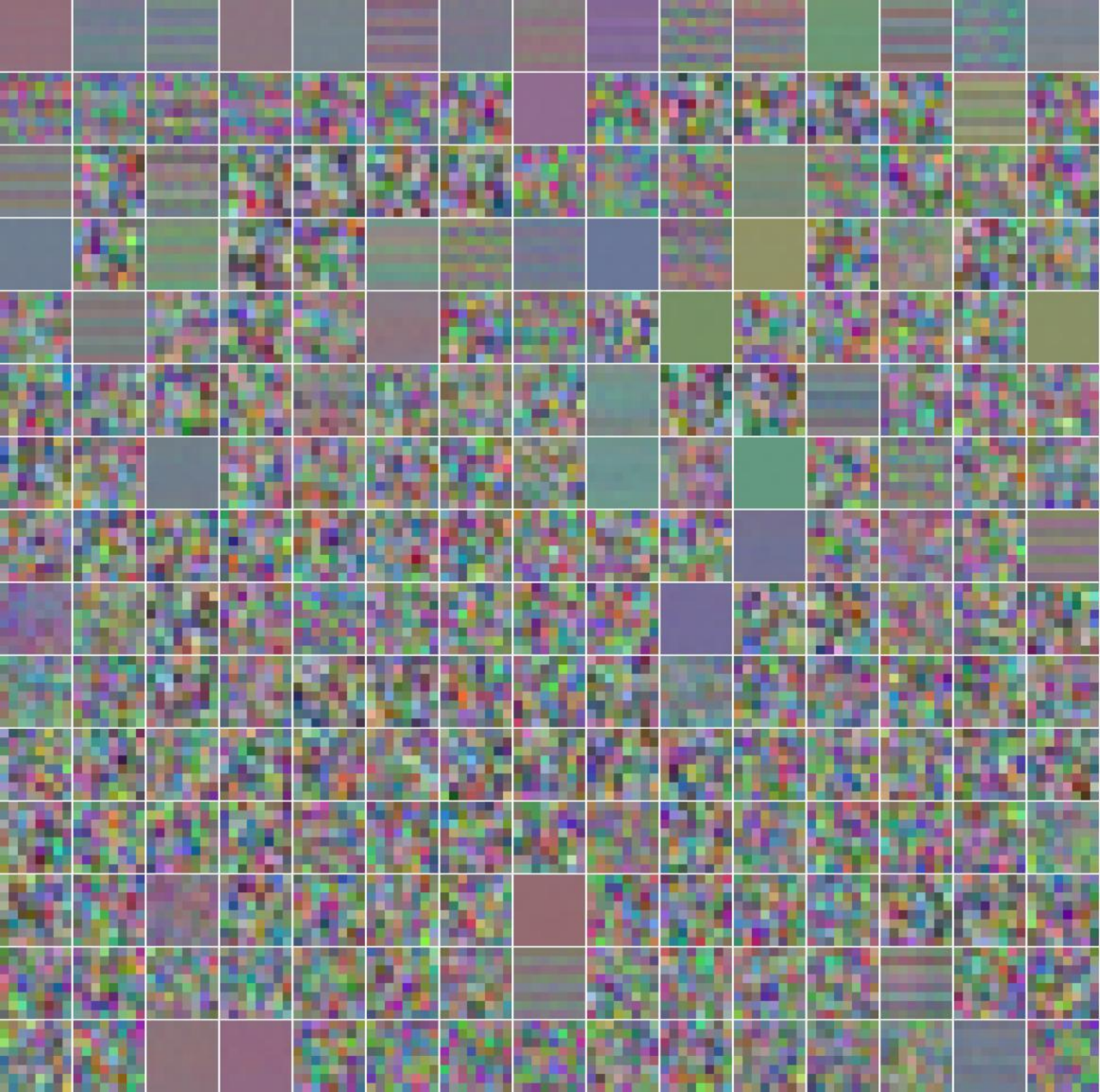}
        \caption{32-sparse.}
    \end{subfigure}
    \vspace{-3mm}
    \caption{\textbf{Learned concepts across the dataset for CIFAR $8 \times 8$ patches for SAE-FNO (modes = 6) as domain sparsity changes.} From left to right, concept sparsity is $k=25, 50, 200$. Same visualization setup as \Cref{fig:concept-visualization-32}. For concept utilization and efficiency, see~\Cref{fig:cifar-concept-utilization-fno-change-ss-mode}a.}
    \label{fig:change-strcuture-domain-cifar}
    \vspace{-4mm}
\end{figure}

\newpage
\begin{figure}[H]
    \centering
    \begin{subfigure}{1.0\textwidth}
        \includegraphics[width=0.32\linewidth]{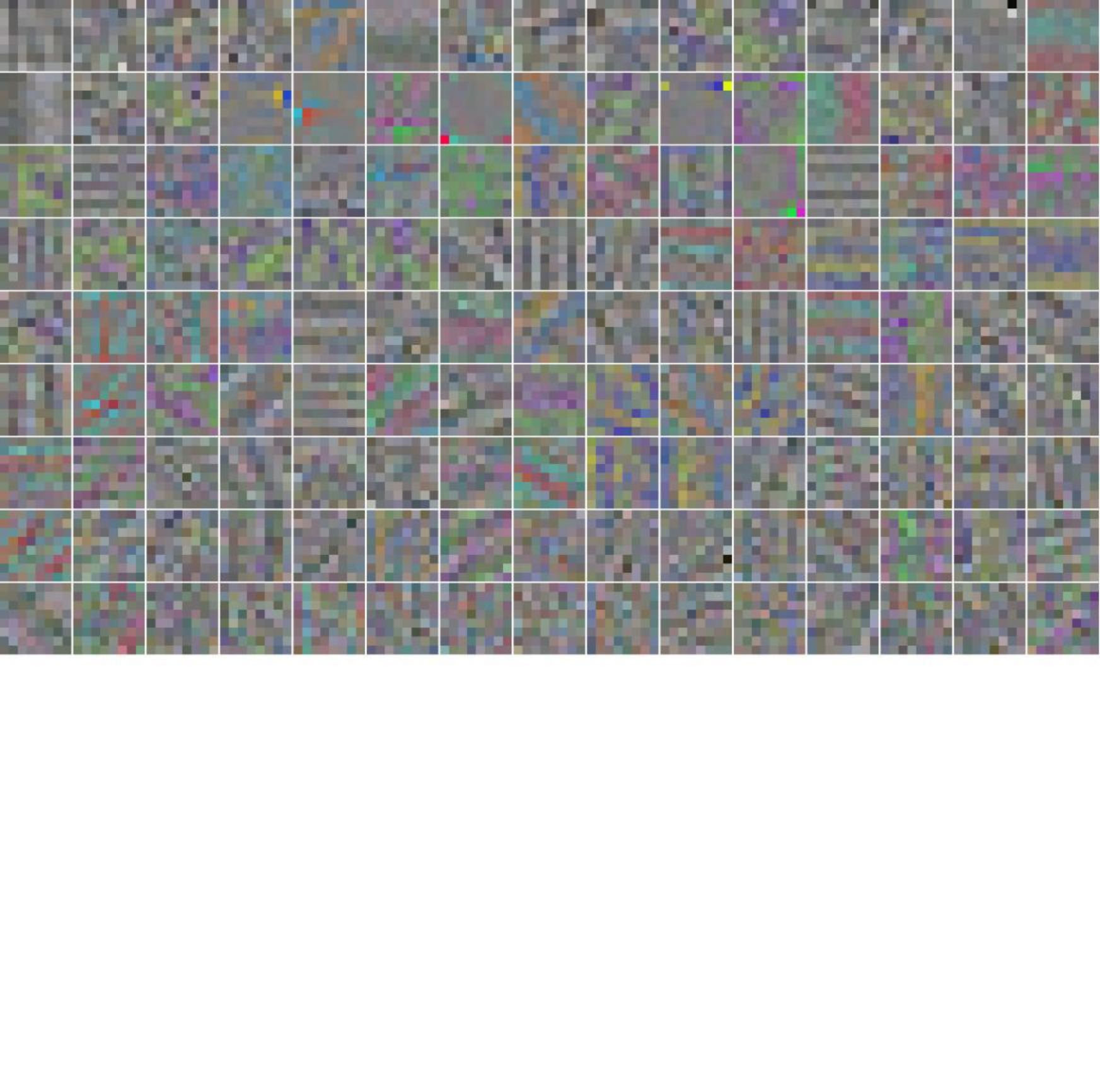}
        \includegraphics[width=0.32\linewidth]{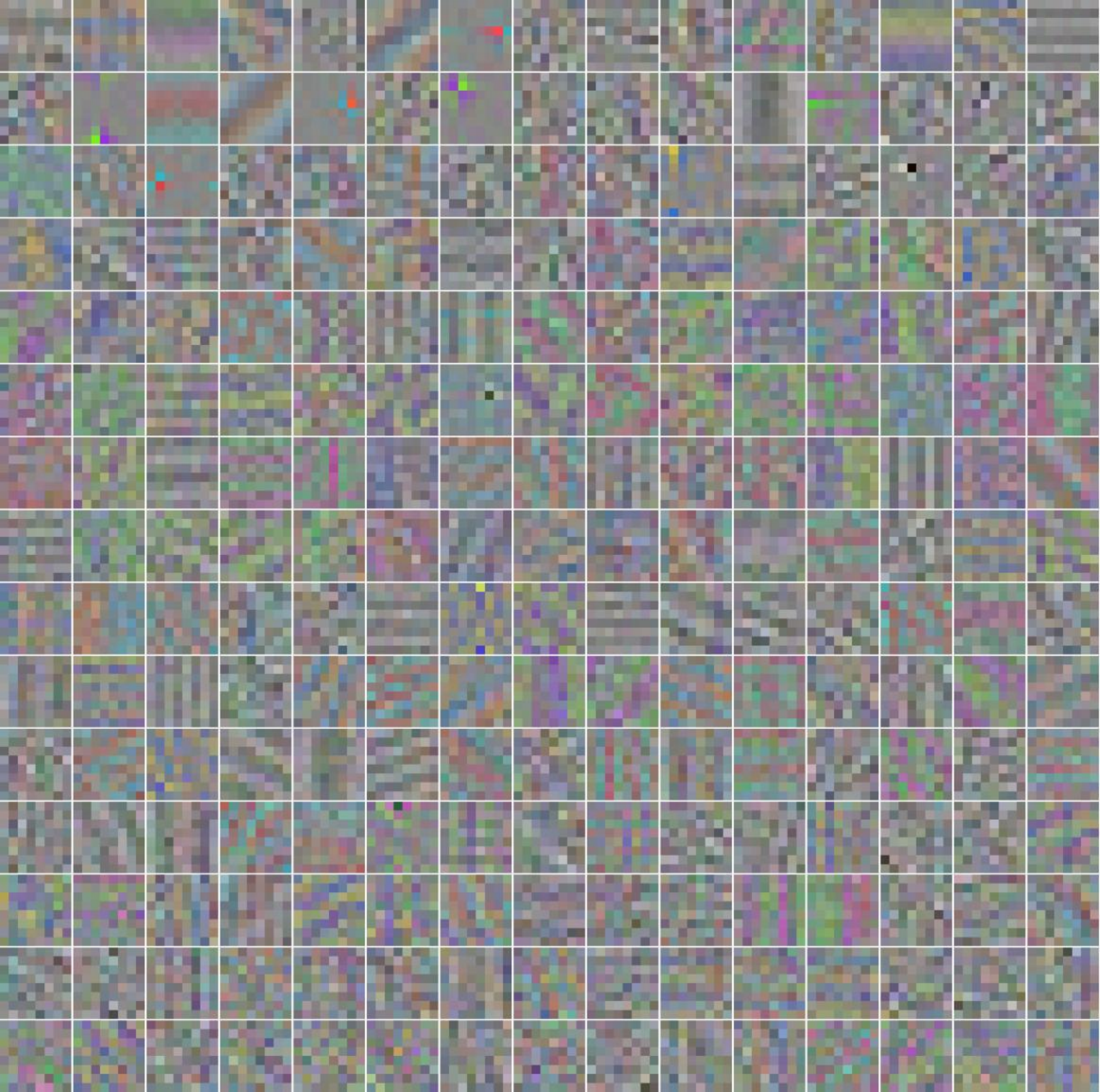}
        \includegraphics[width=0.32\linewidth]{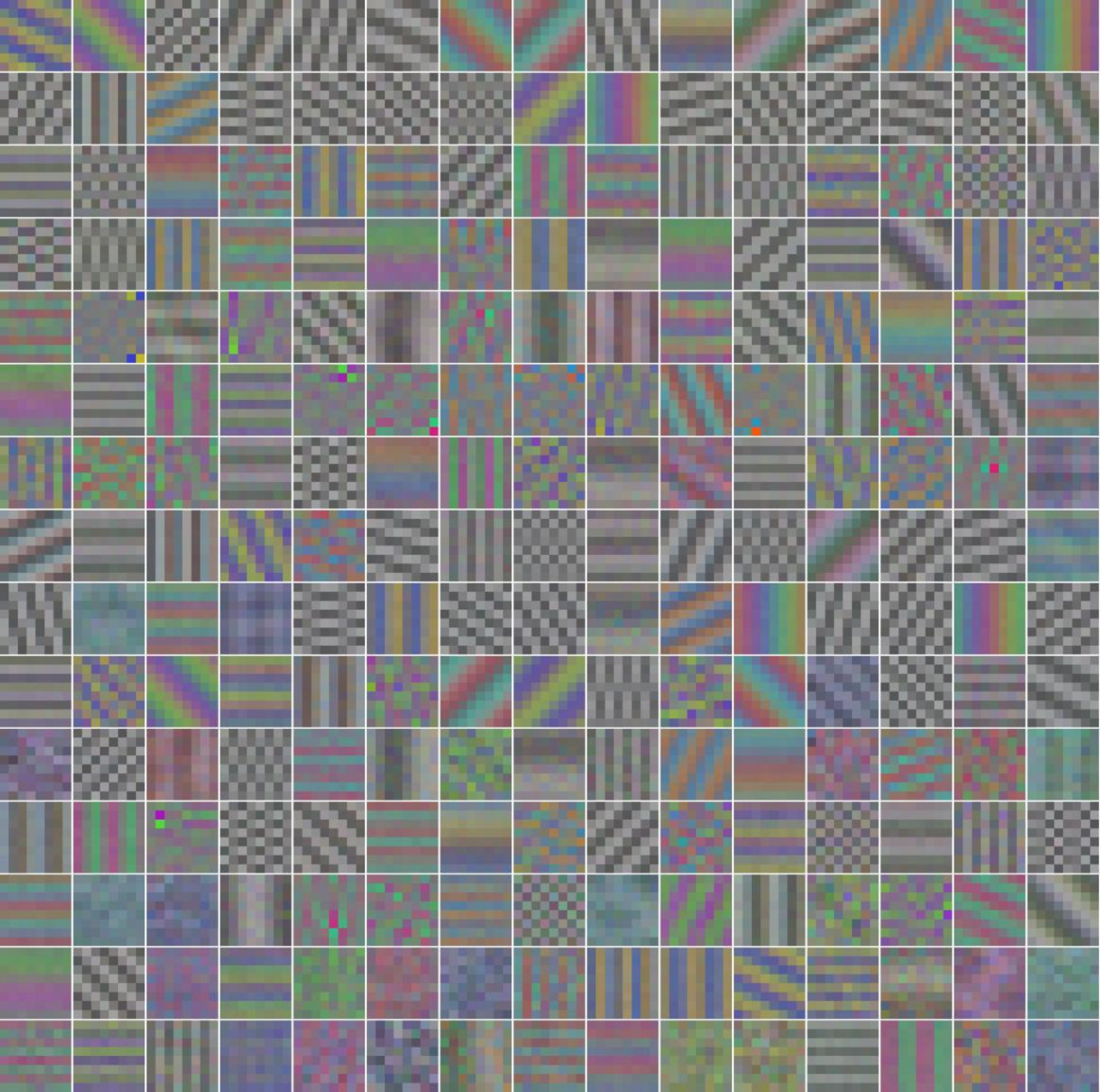}
        \caption{8 modes (full).}
    \end{subfigure}
    \begin{subfigure}{1.0\textwidth}
        \includegraphics[width=0.32\linewidth]{figures/cifar_concepts/CIFAR10_TopKSAEfno2d_p1000_k25_img8x8_m6_ch3_cs_ss0.05_lreg1e-05_20260126161655_top_concepts_grid_l0.pdf}
        \includegraphics[width=0.32\linewidth]{figures/cifar_concepts/CIFAR10_TopKSAEfno2d_p1000_k50_img8x8_m6_ch3_cs_ss0.05_lreg1e-05_20260126161928_top_concepts_grid_l0.pdf}
        \includegraphics[width=0.32\linewidth]{figures/cifar_concepts/CIFAR10_TopKSAEfno2d_p1000_k200_img8x8_m6_ch3_cs_ss0.05_lreg1e-05_20260126192711_top_concepts_grid_l0.pdf}
        \caption{6 modes.}
    \end{subfigure}
    \begin{subfigure}{1.0\textwidth}
        \includegraphics[width=0.32\linewidth]{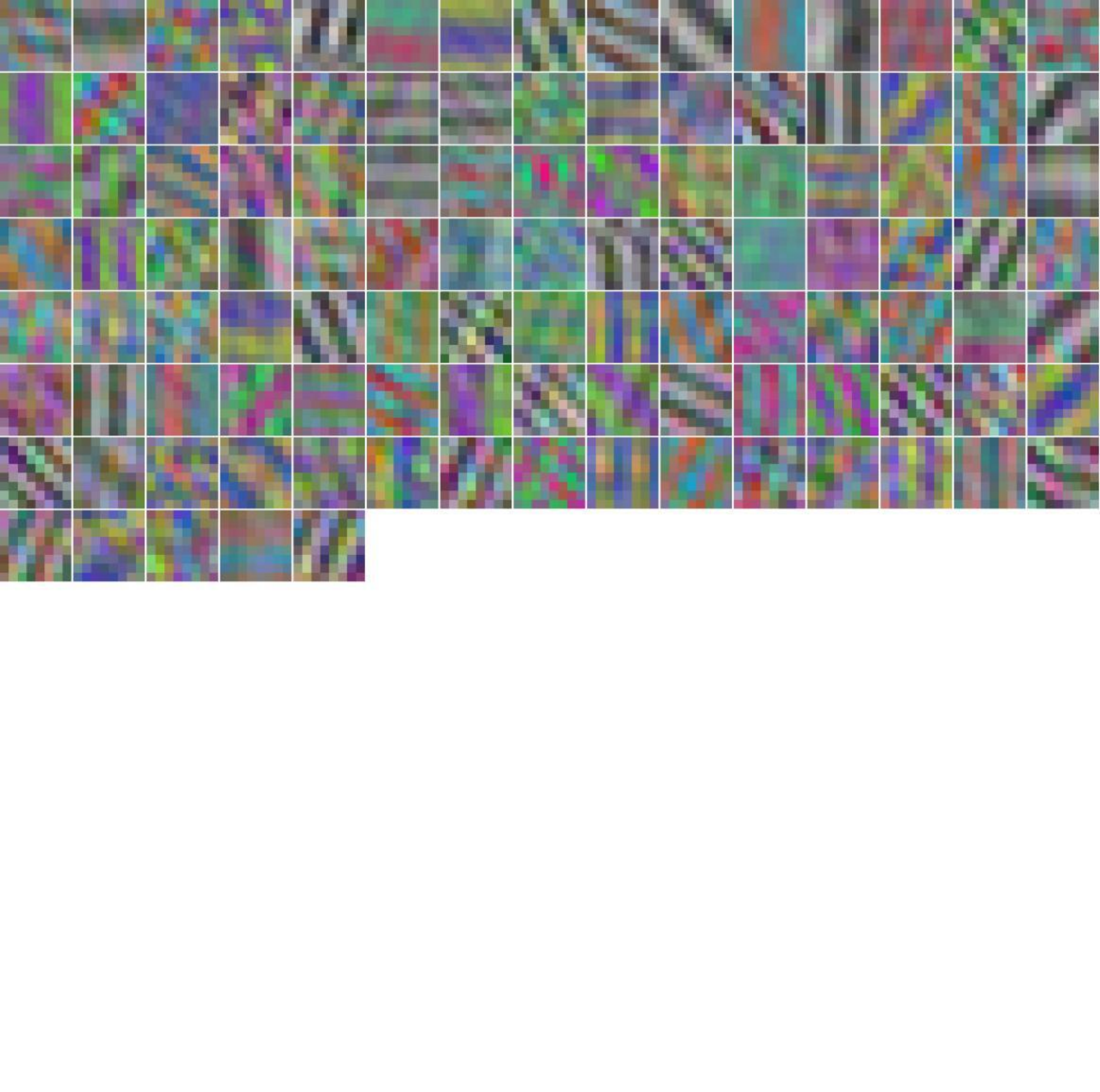}
        \includegraphics[width=0.32\linewidth]{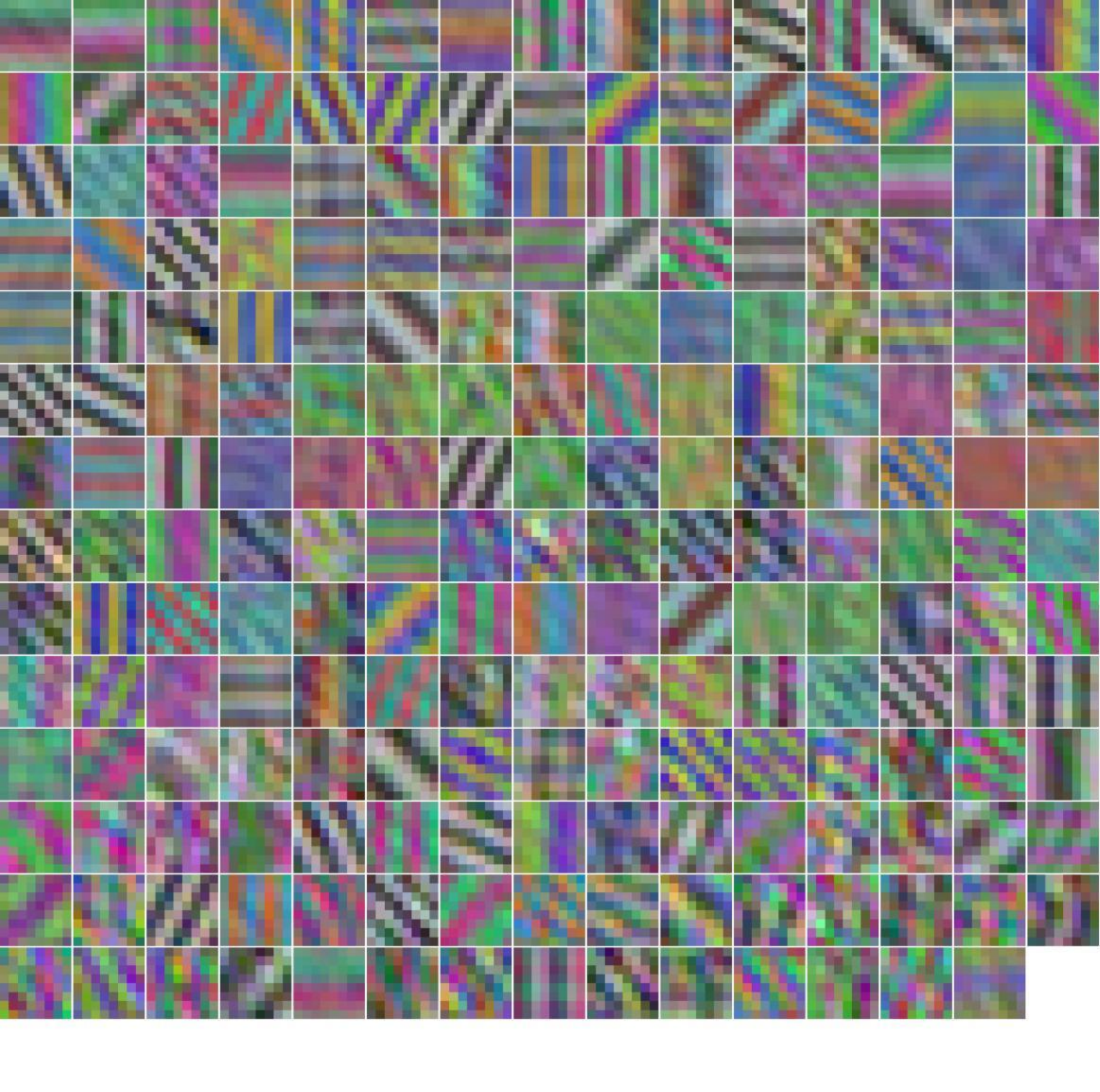}
        \includegraphics[width=0.32\linewidth]{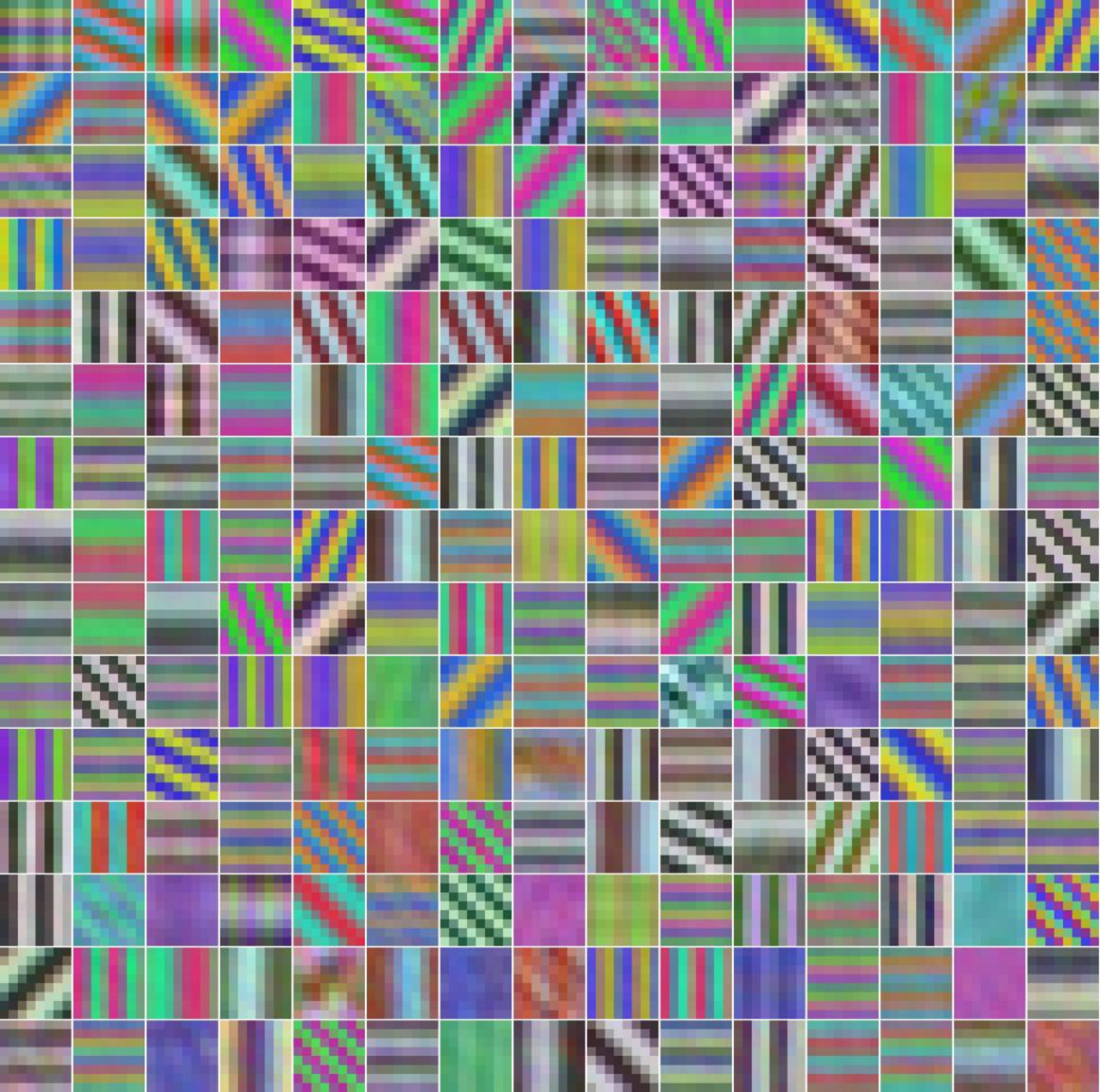}
        \caption{4 modes (half).}
    \end{subfigure}
    \vspace{-3mm}
    \caption{\textbf{Learned concepts across the dataset for CIFAR $8 \times 8$ patches for SAE-FNO (spatial sparsity = 0.05) as the number of Fourier modes changes.} From left to right, concept sparsity is $k=25, 50, 200$. Same visualization setup as \Cref{fig:concept-visualization-32}. Corresponding concept utilization and efficiency plots are shown in \Cref{fig:cifar-concept-utilization-fno-change-ss-mode}b.}
    \label{fig:change-modes-cifar}
    \vspace{-4mm}
\end{figure}


\begin{figure}[H]
    \centering
    \begin{subfigure}{1.00\textwidth}
        \includegraphics[width=0.24\linewidth]{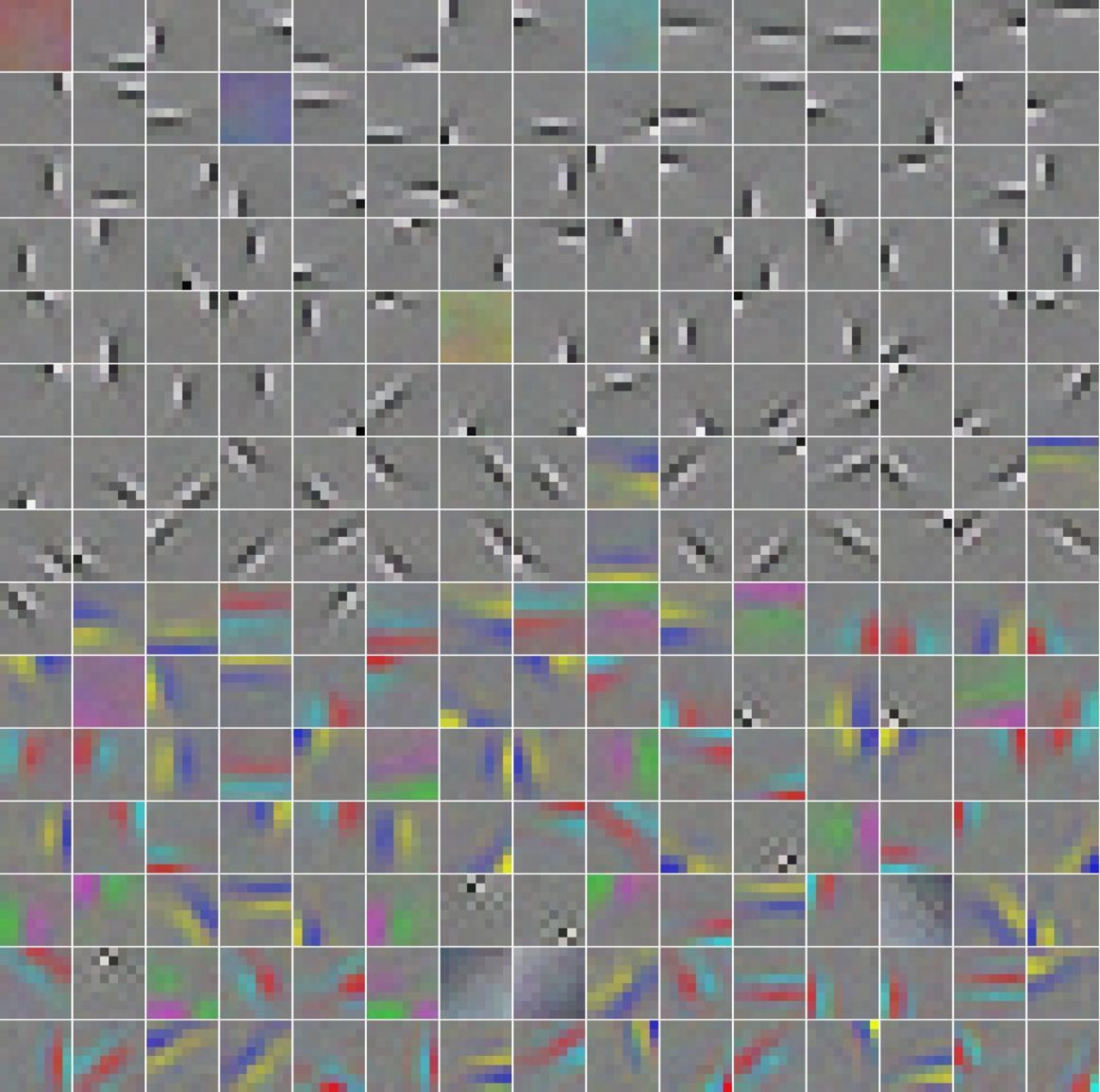}
        \includegraphics[width=0.24\linewidth]{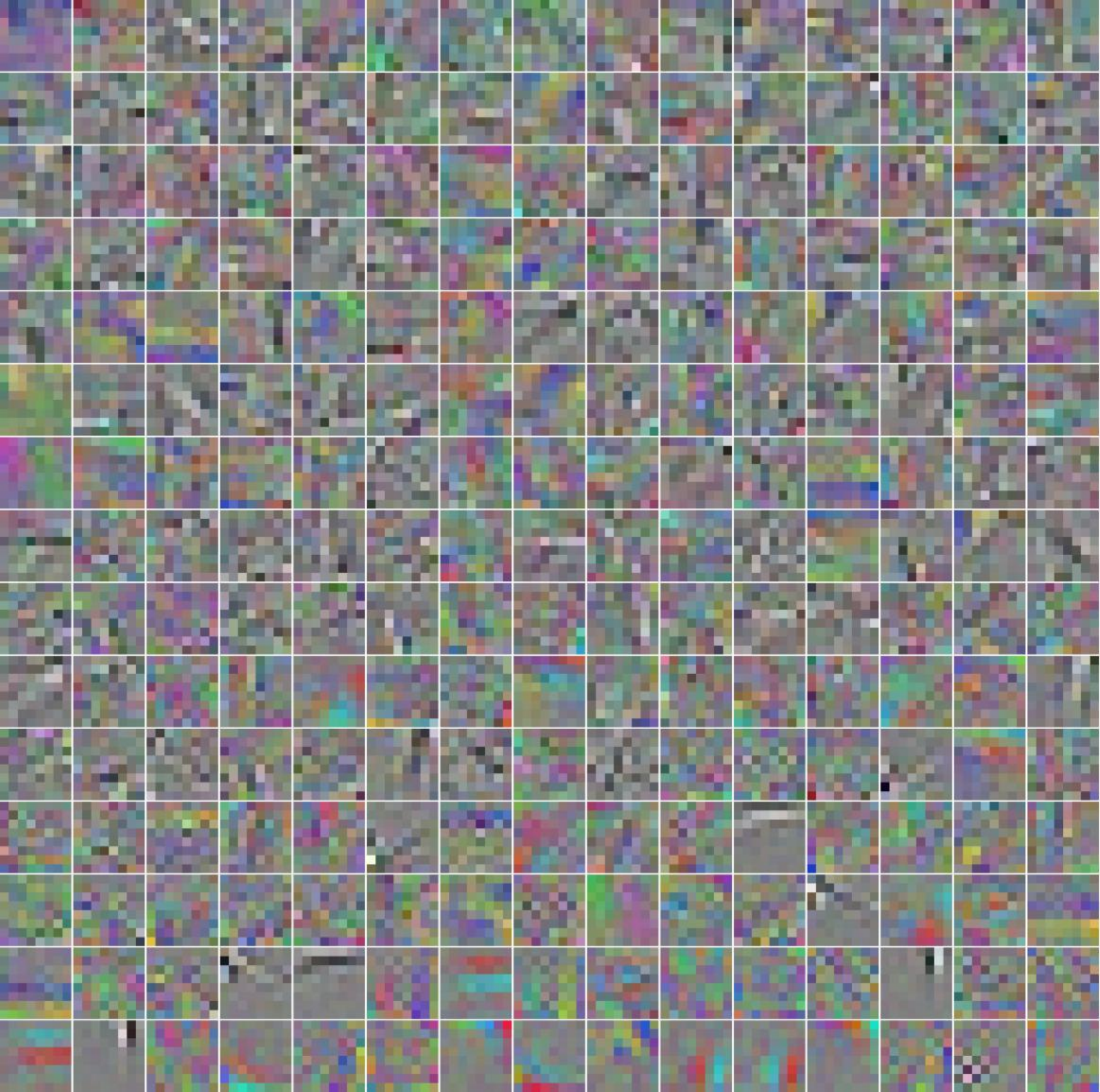}
        \includegraphics[width=0.24\linewidth]{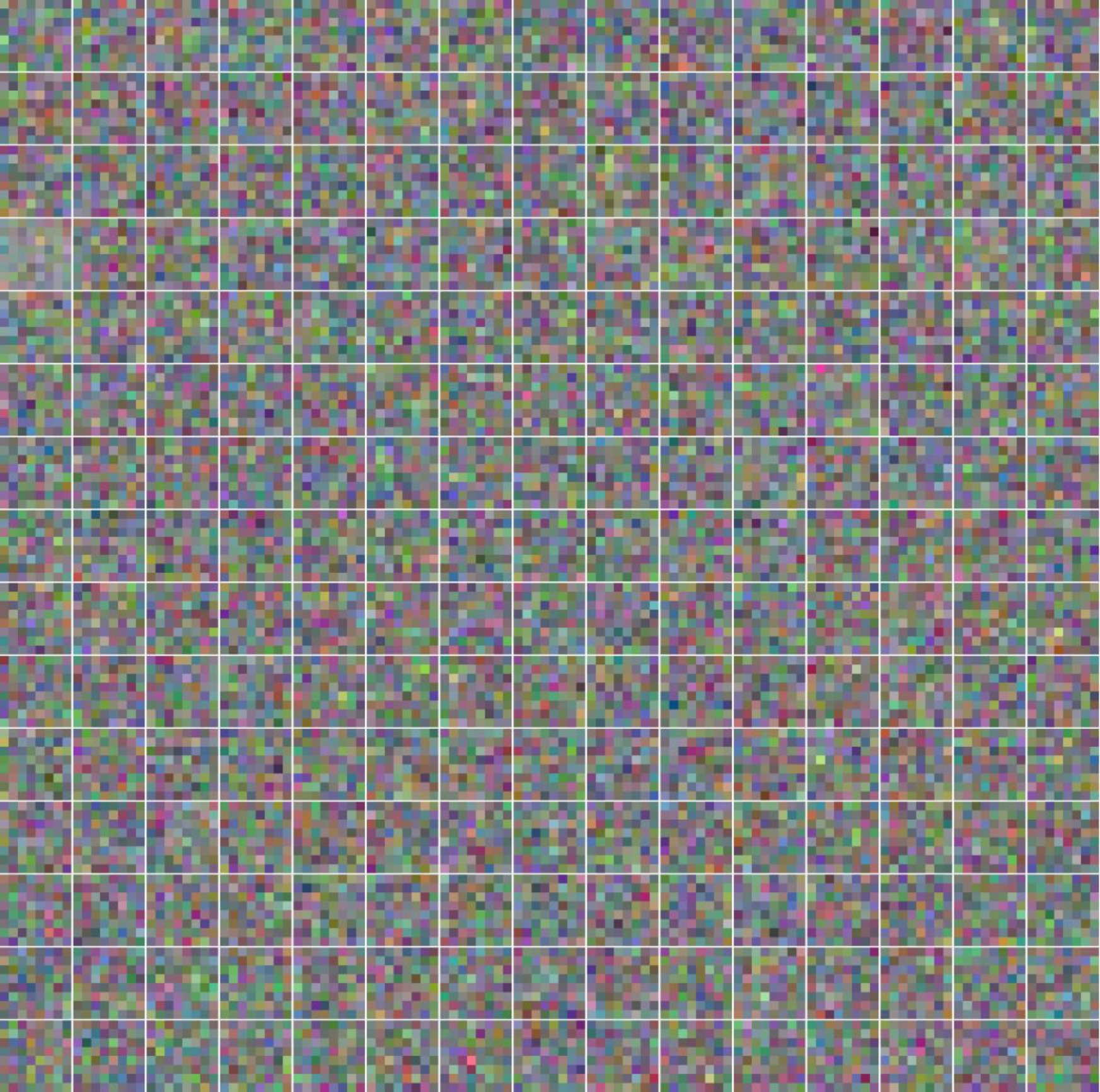}
        \includegraphics[width=0.24\linewidth]{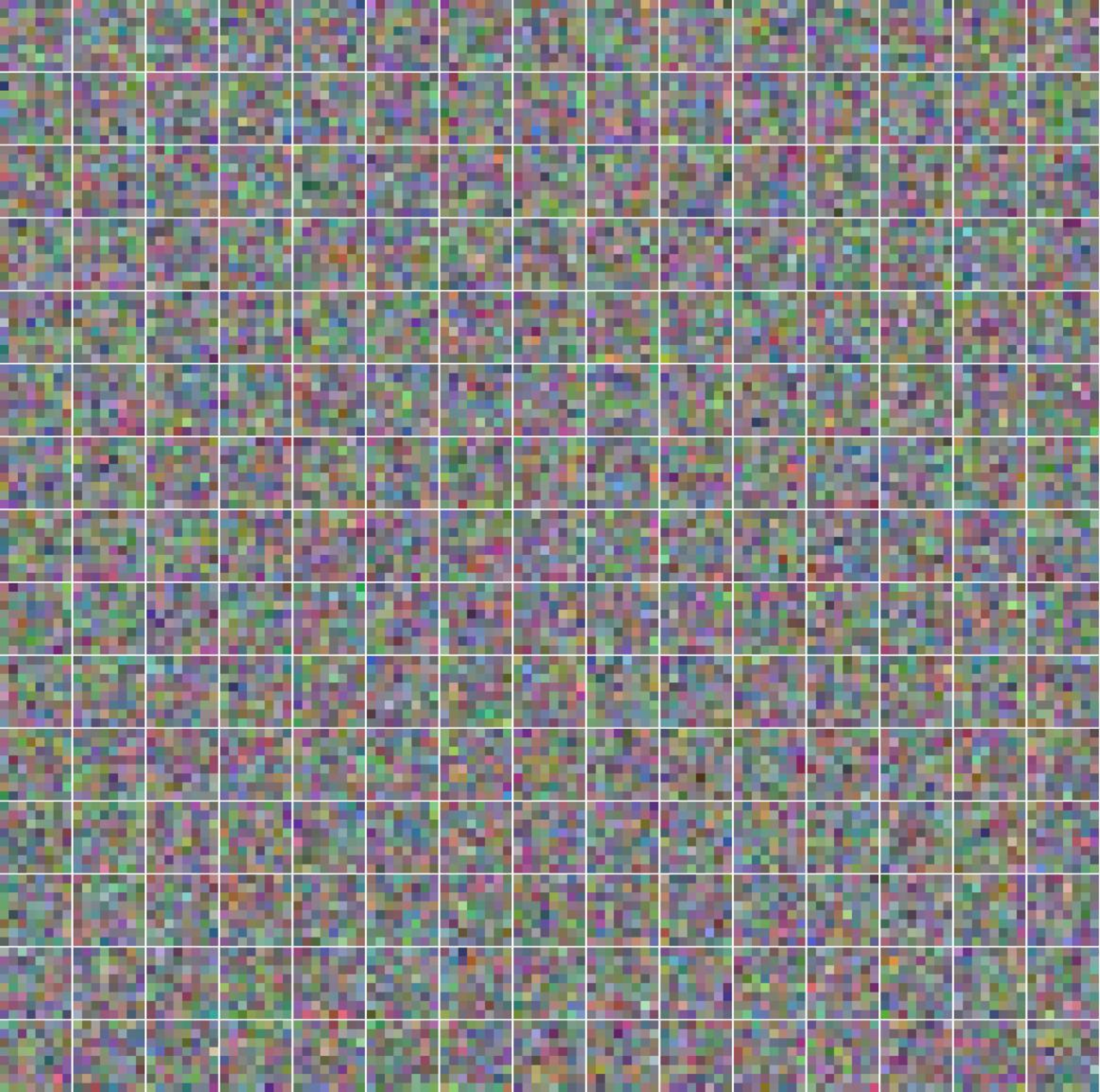}
        \caption{SAE-MLP.}
    \end{subfigure}
        \begin{subfigure}{1.00\textwidth}
        \includegraphics[width=0.24\linewidth]{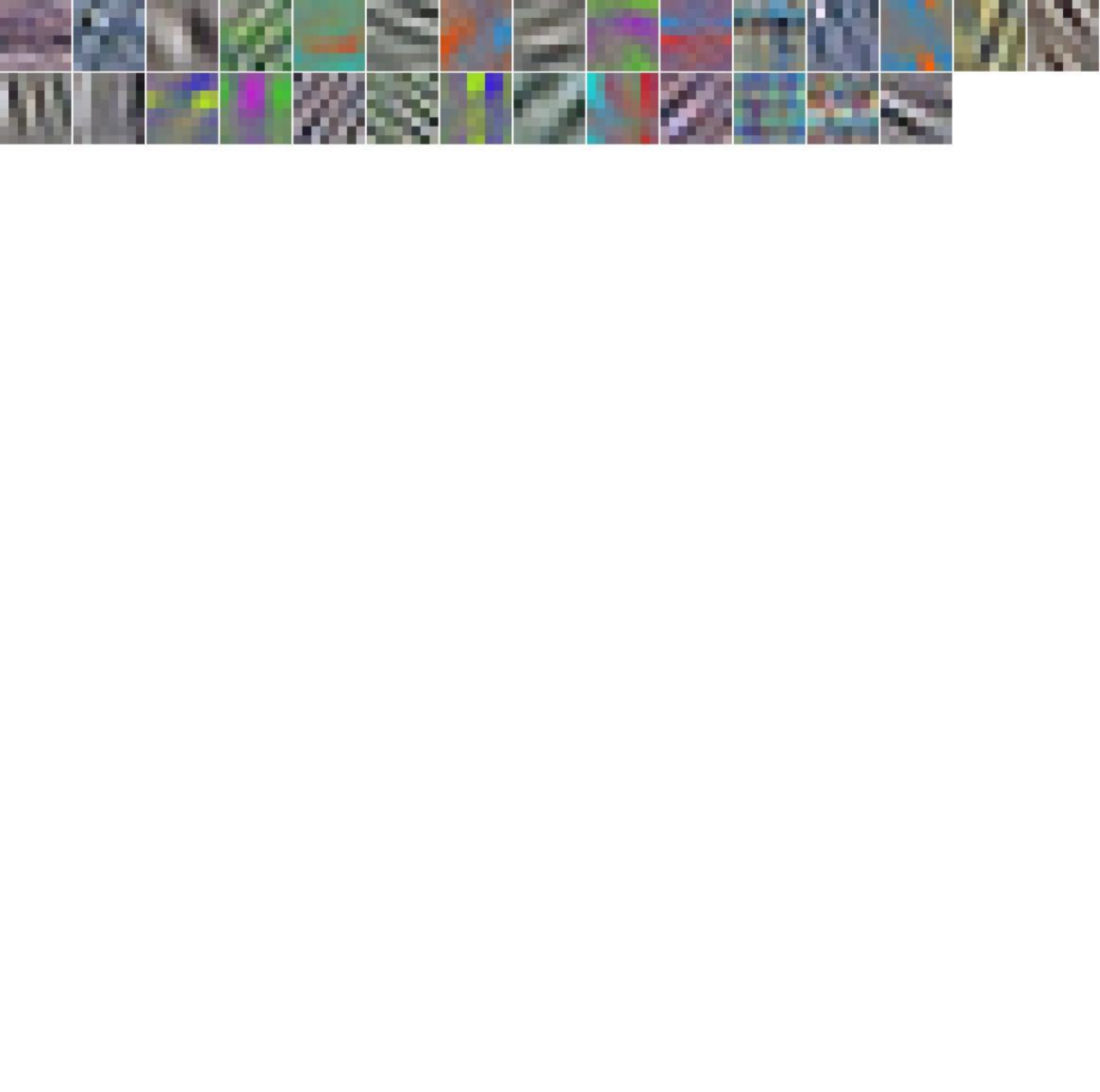}
        \includegraphics[width=0.24\linewidth]{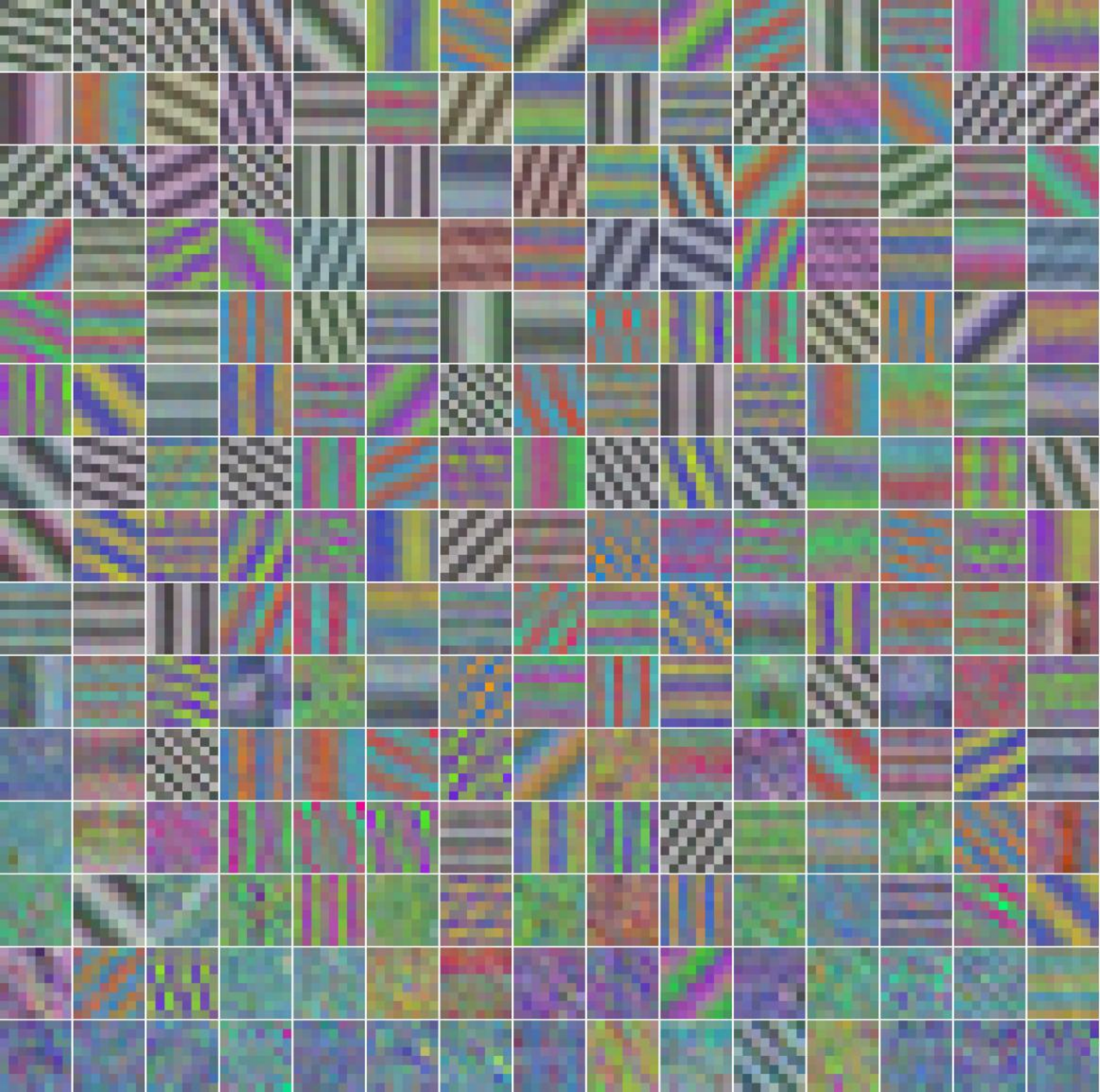}
        \includegraphics[width=0.24\linewidth]{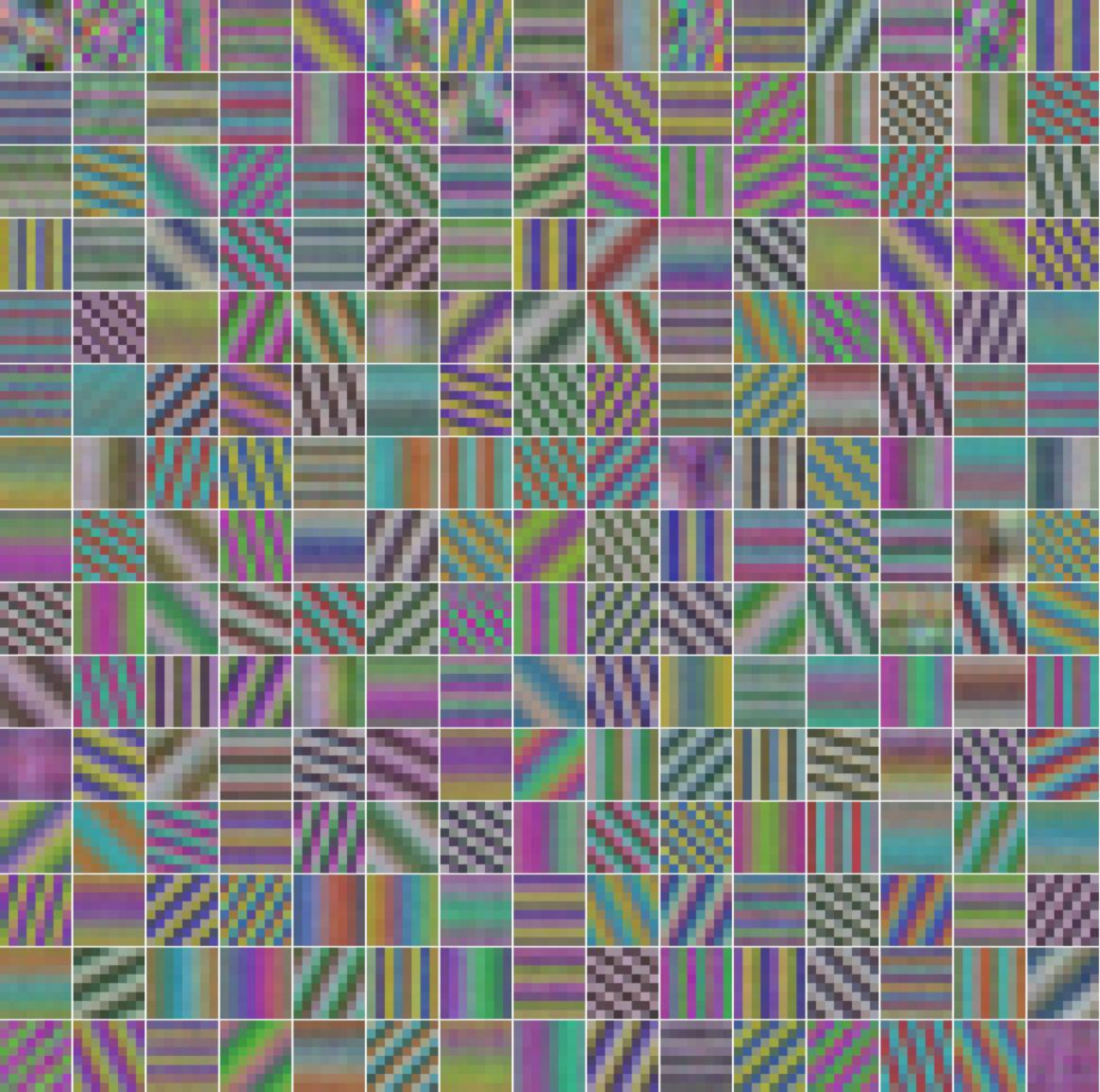}
        \includegraphics[width=0.24\linewidth]{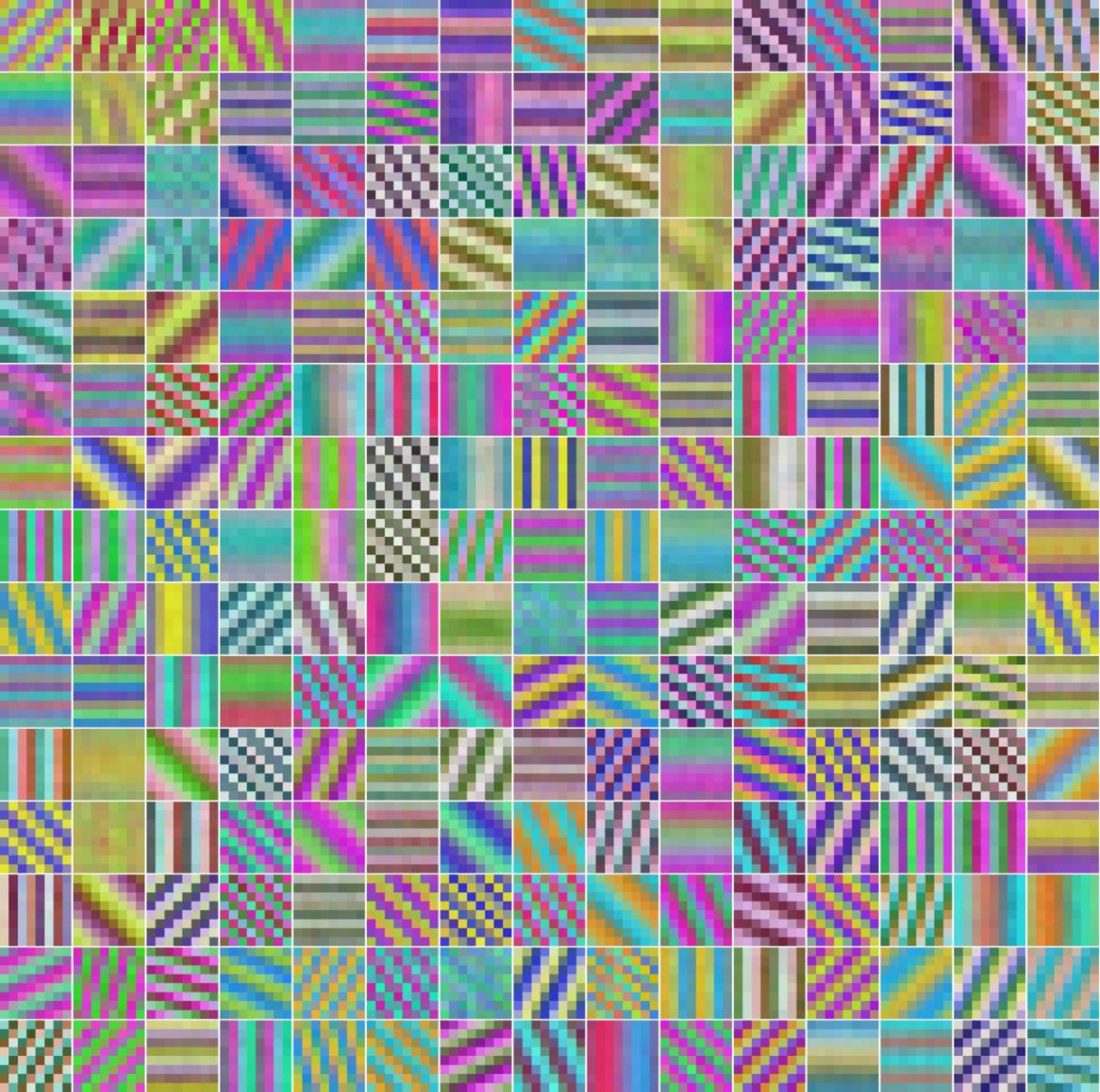}
            \caption{SAE-FNO. (modes = 6, spatially 1-sparse).}
    
    \end{subfigure}
    \caption{\textbf{SAE-FNO learns concepts with stable structural properties across concept sparsity levels, whereas SAE-MLP concepts lose structure as sparsity decreases.} Learned concepts on CIFAR-10 $8 \times 8$ patches, with $k=10, 100, 700, 1000$ (left to right). }
    \label{fig:cifar-concept-stability-mode-ss}
\end{figure}

\begin{figure}[H]
    \centering
    \begin{subfigure}{1.00\textwidth}
        \includegraphics[width=0.24\linewidth]{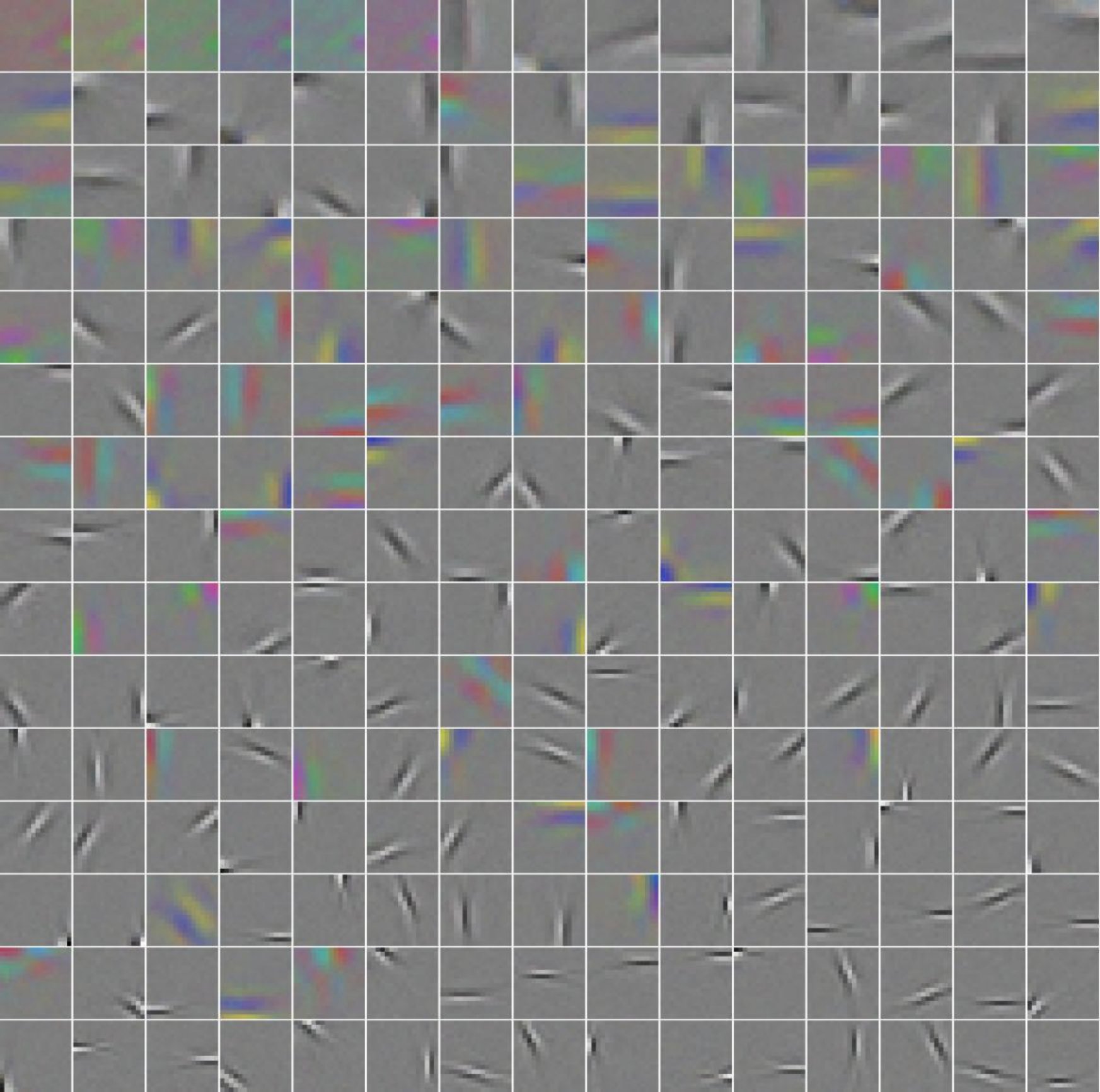}
        \includegraphics[width=0.24\linewidth]{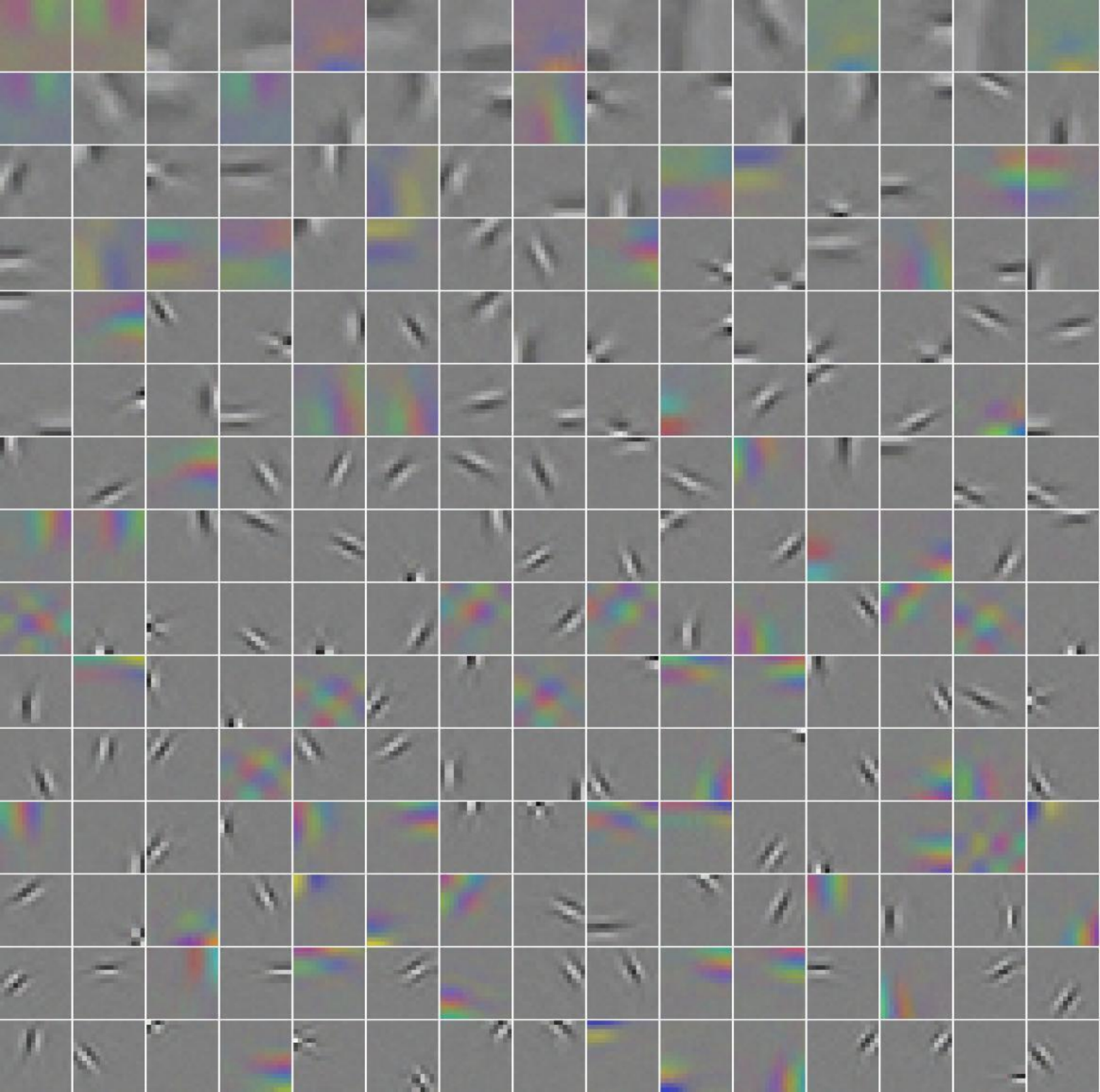}
        \includegraphics[width=0.24\linewidth]{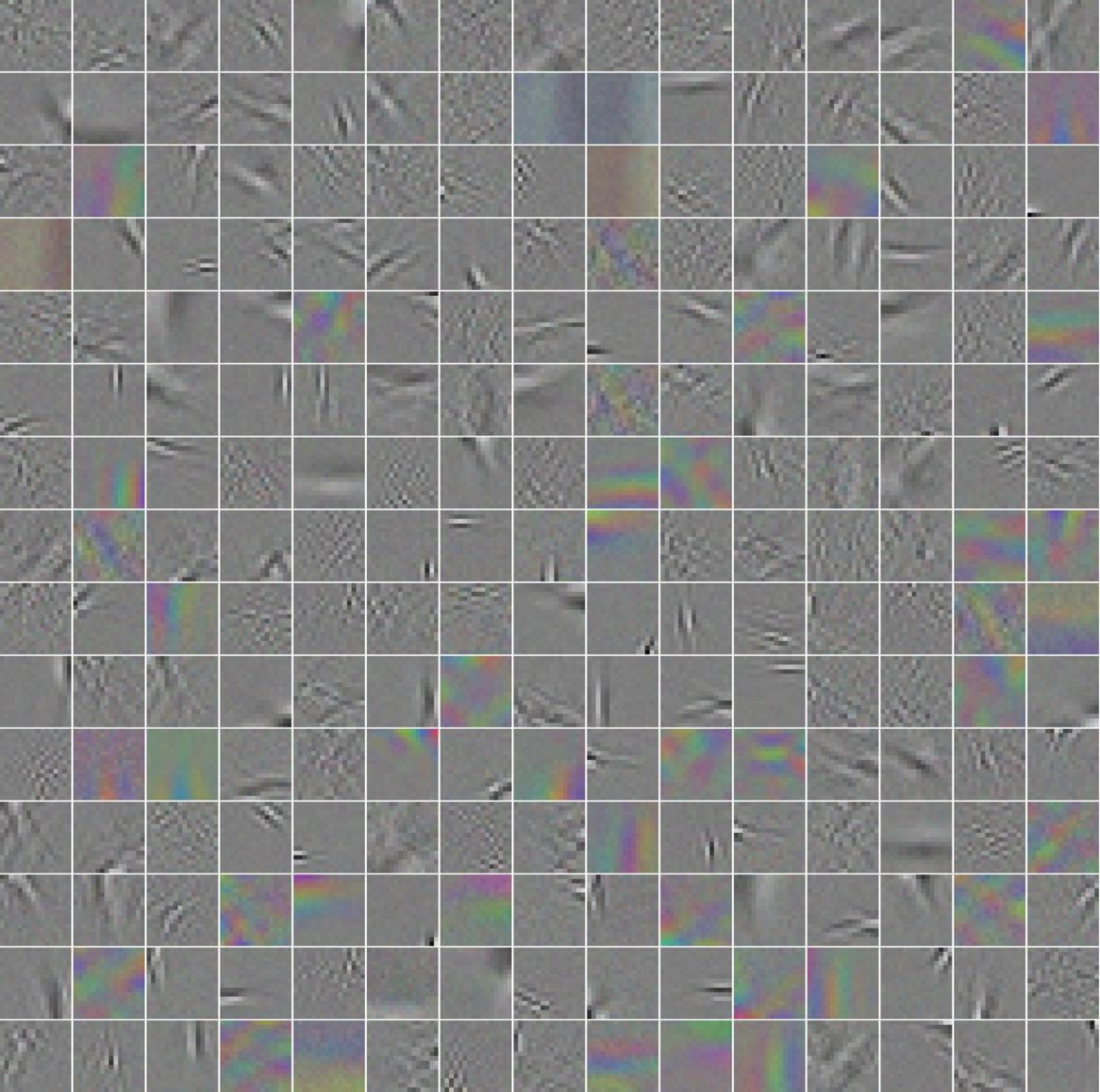}
        \includegraphics[width=0.24\linewidth]{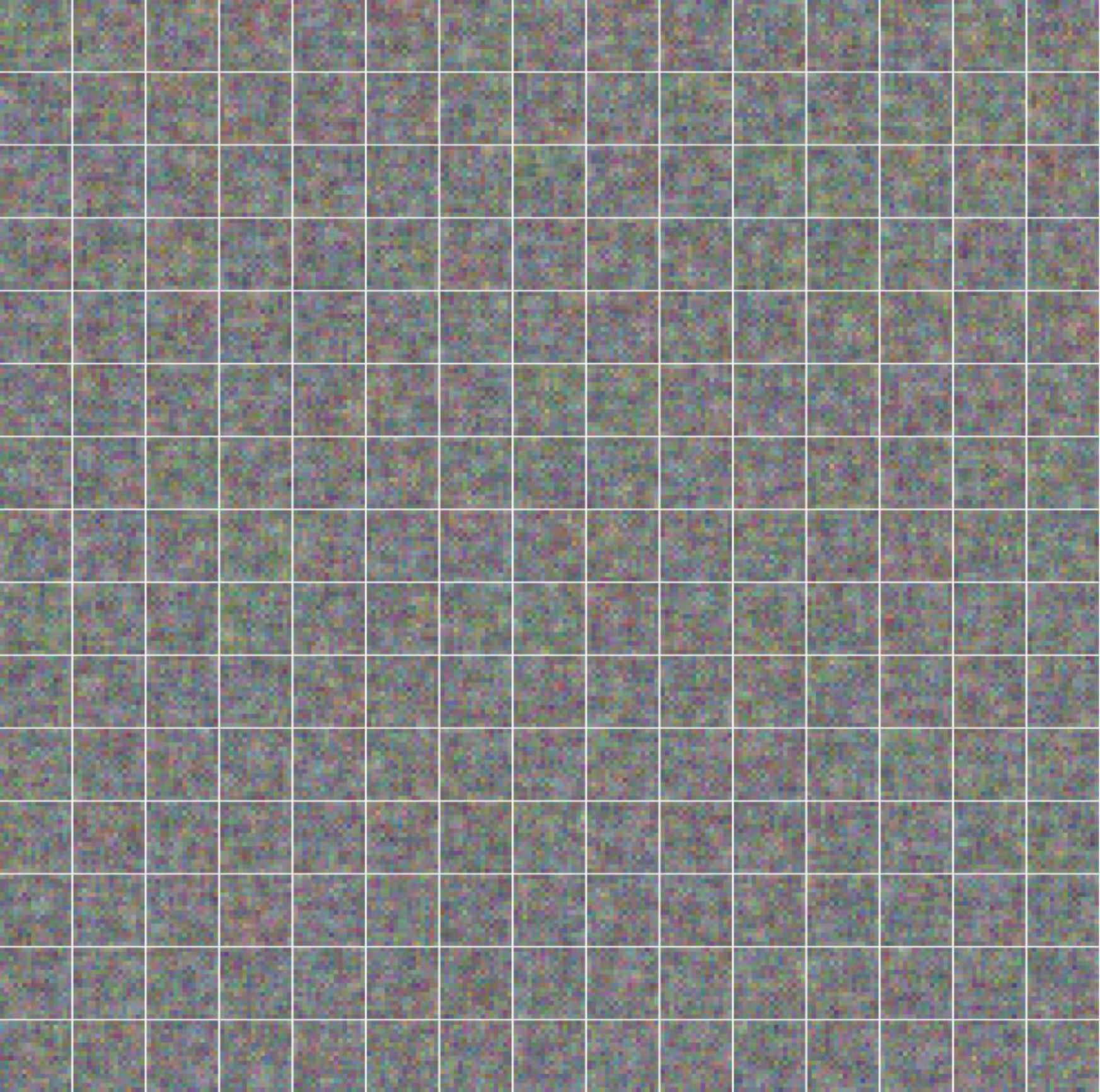}
        \label{fig:imagenet-conceptstable-mlp-topK}
        \caption{SAE-MLP-TopK.}
    \end{subfigure}
    \begin{subfigure}{1.00\textwidth}
        \includegraphics[width=0.24\linewidth]{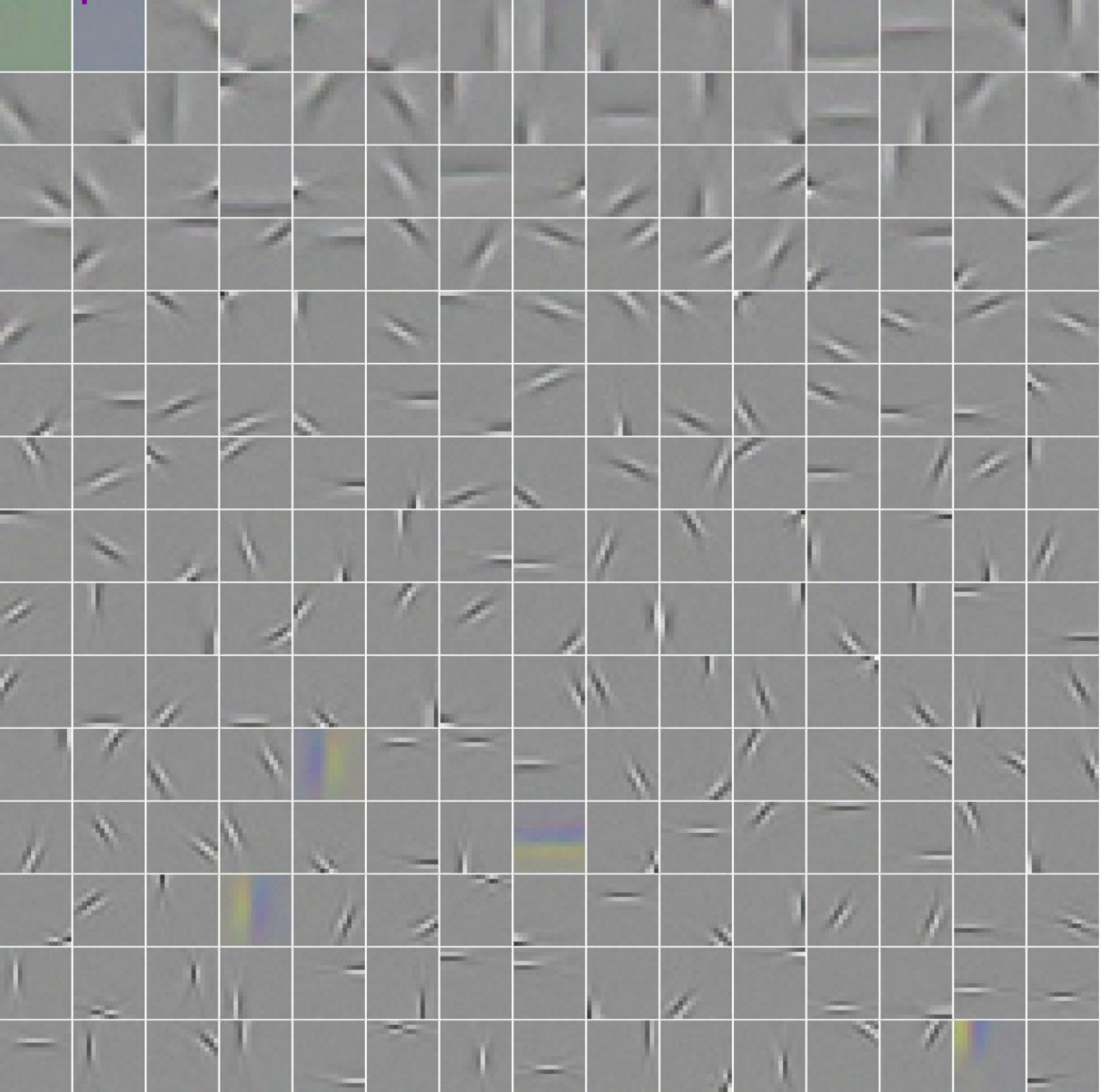}
        \includegraphics[width=0.24\linewidth]{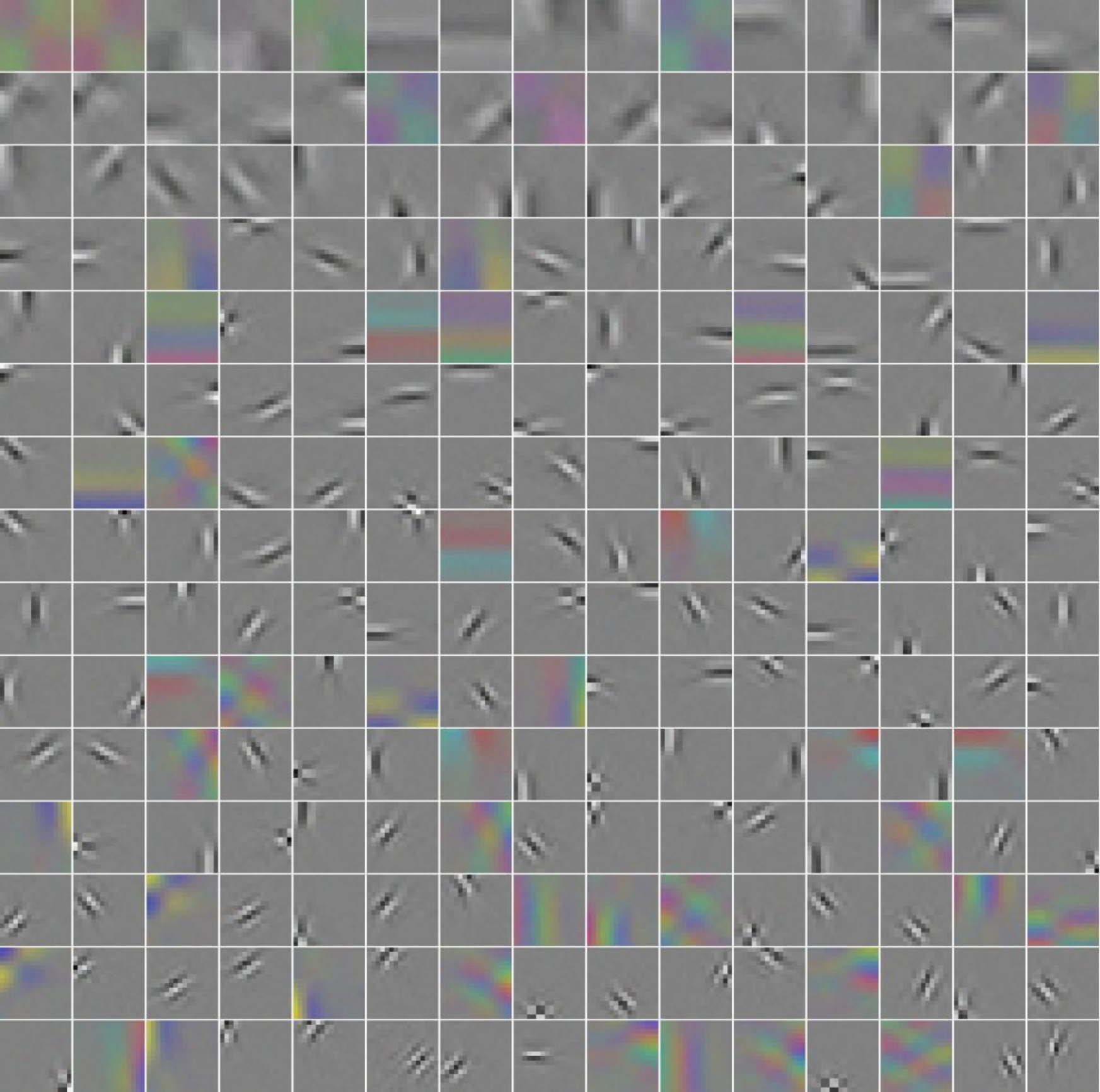}
        \includegraphics[width=0.24\linewidth]{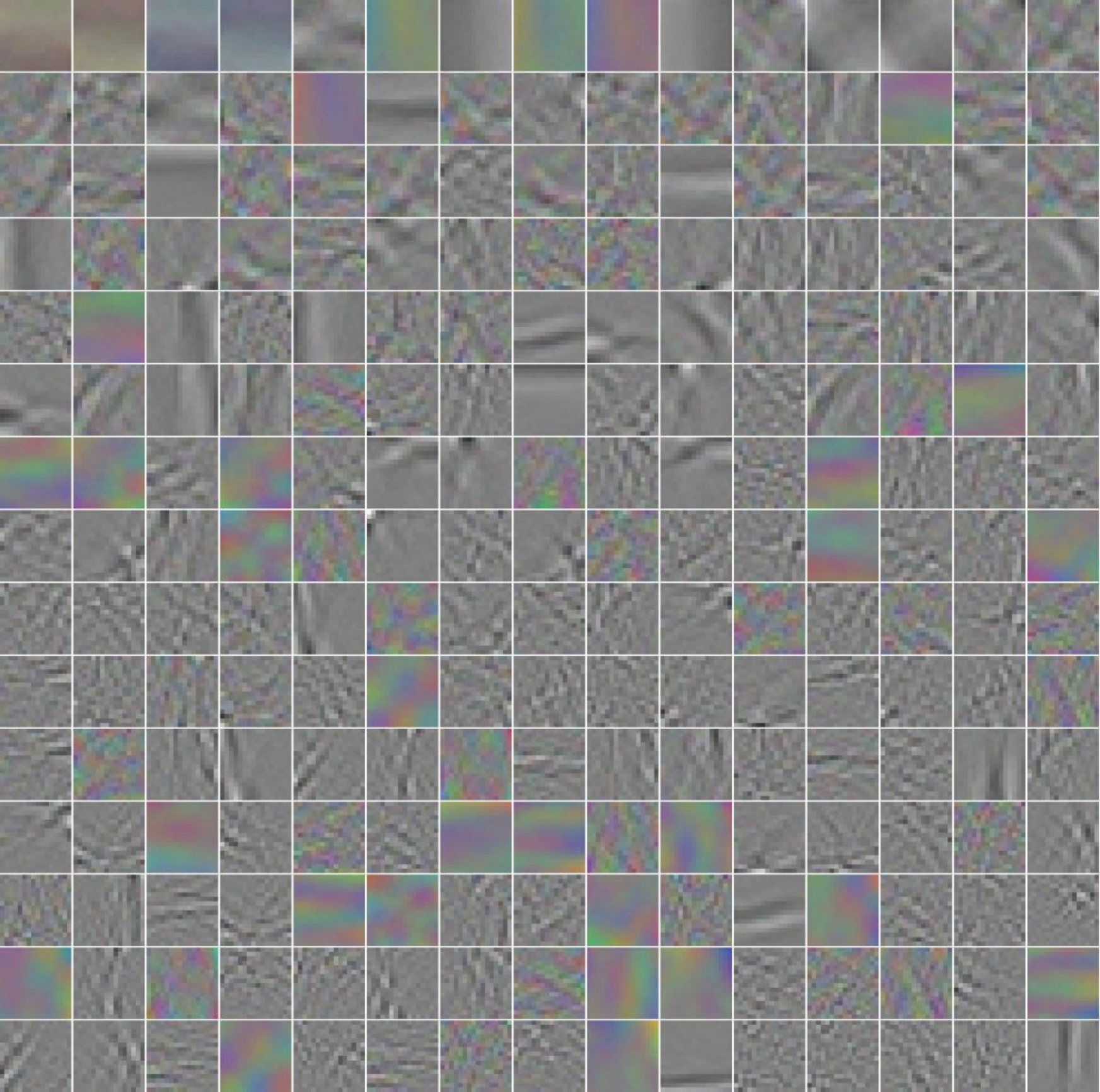}
        \includegraphics[width=0.24\linewidth]{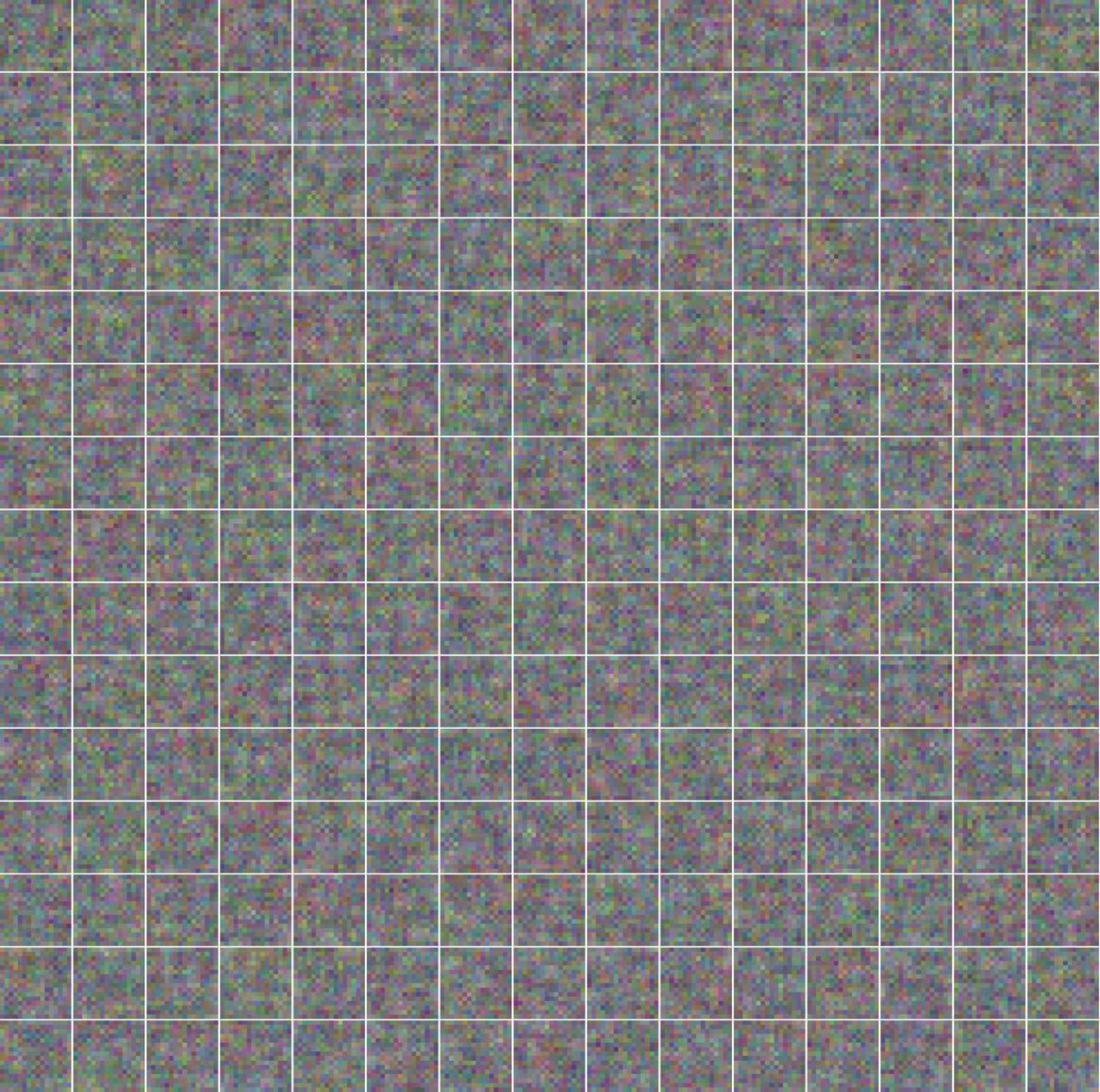}
        \label{fig:imagenet-conceptstable-mlp-batchtopk}
        \caption{SAE-MLP-BatchTopK.}
    \end{subfigure}   
        \begin{subfigure}{1.0\textwidth}
        \includegraphics[width=0.24\linewidth]{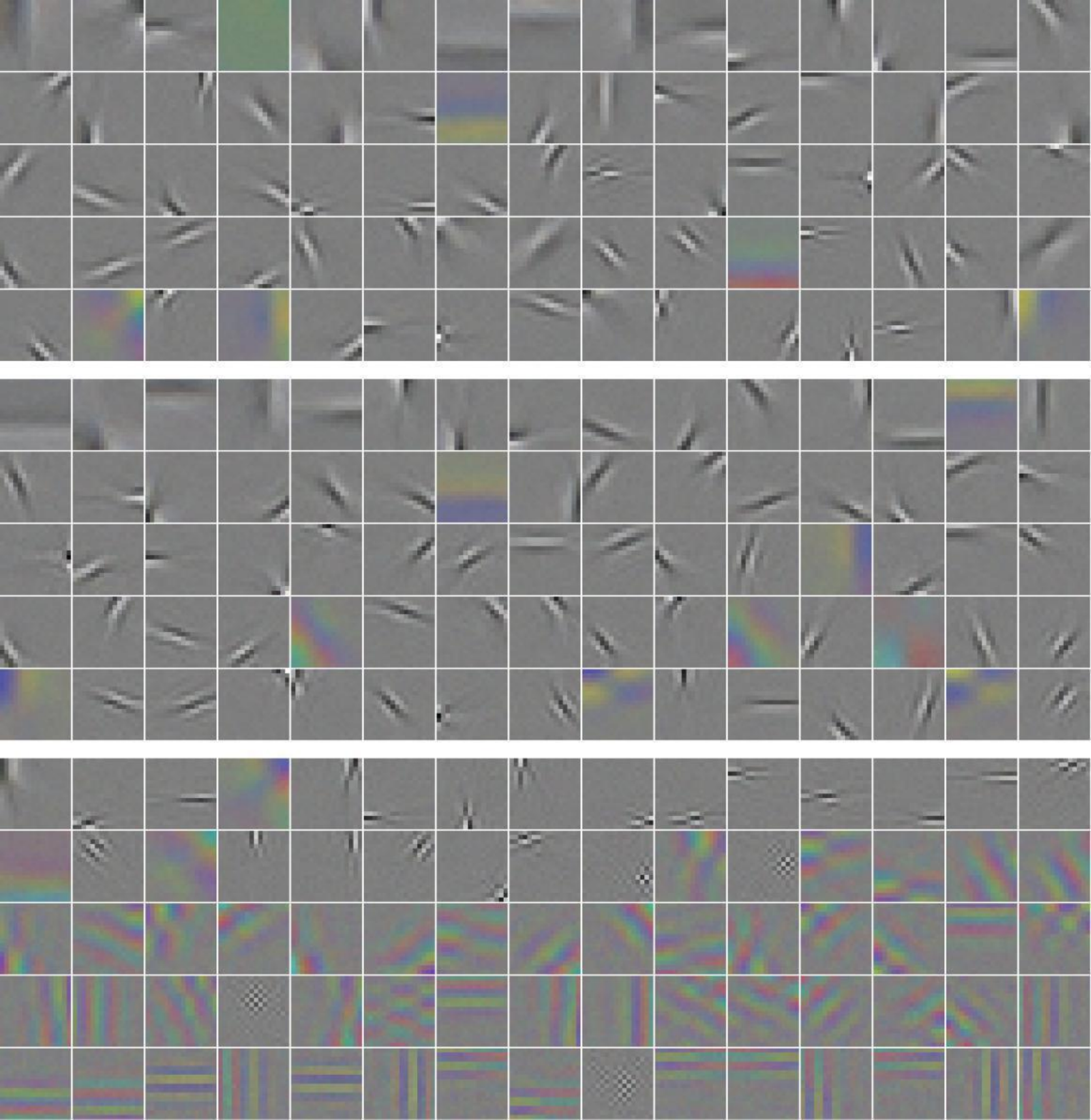}
    \includegraphics[width=0.24\linewidth]{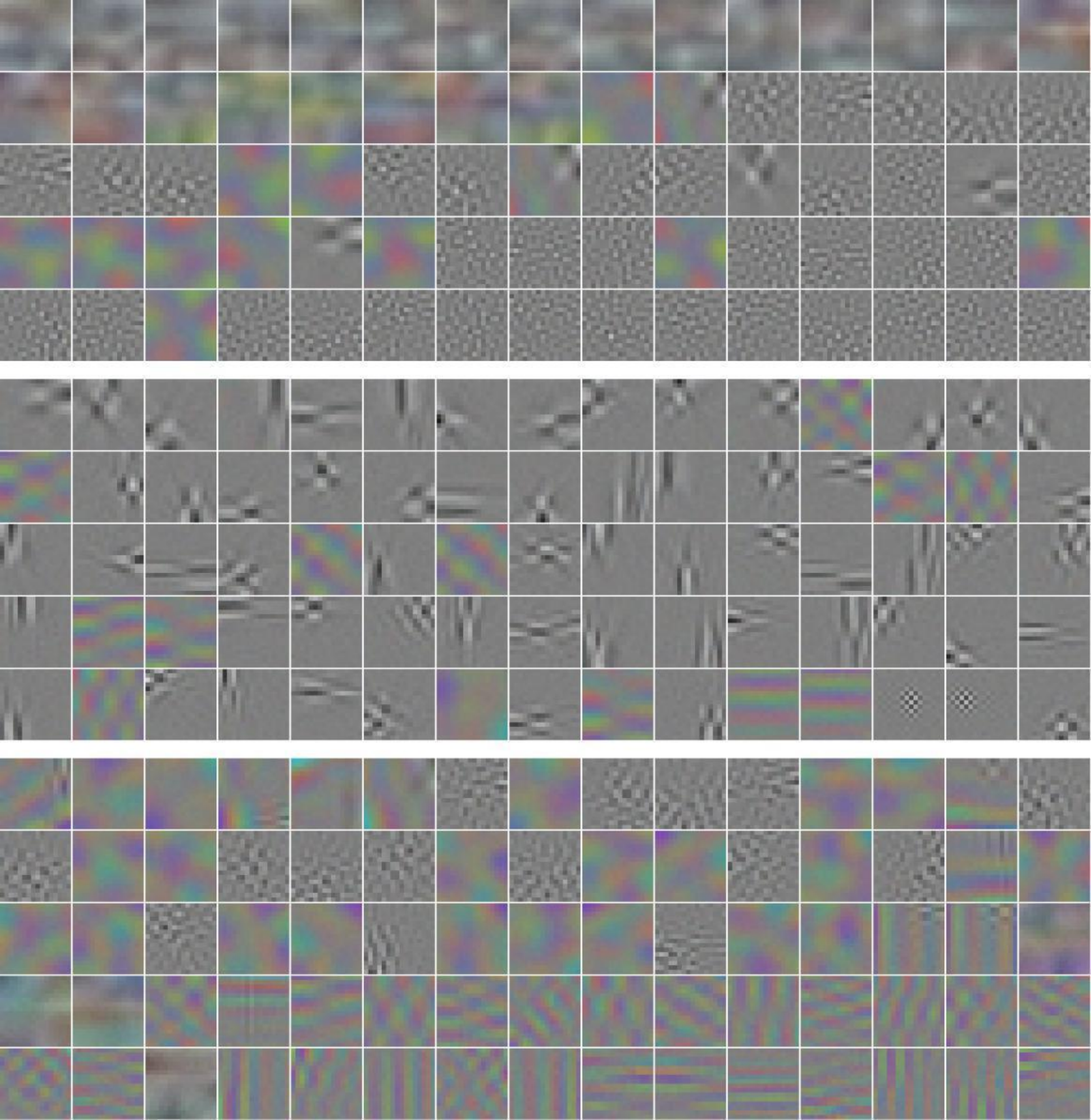}
        \includegraphics[width=0.24\linewidth]{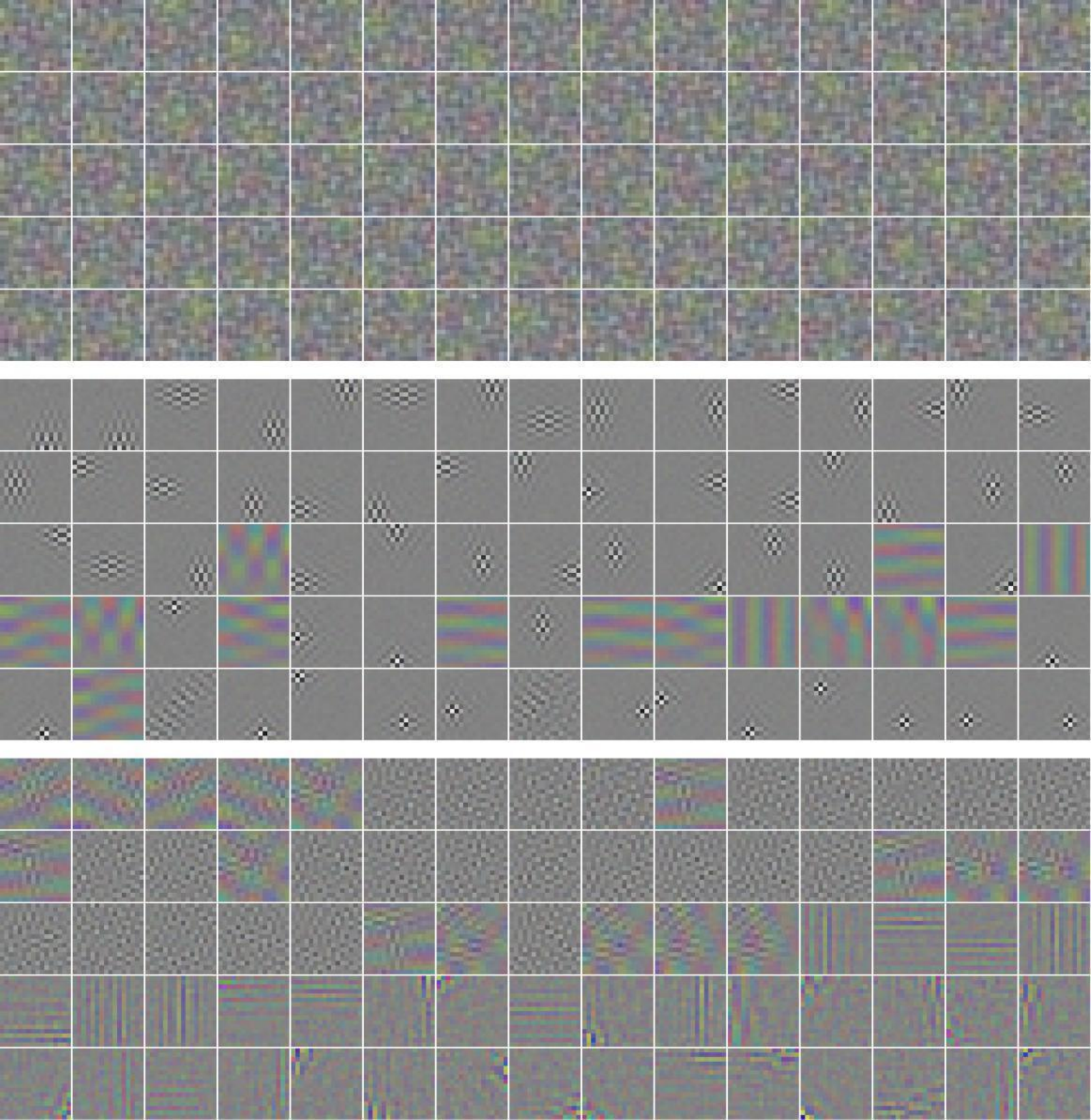}
        \includegraphics[width=0.24\linewidth]{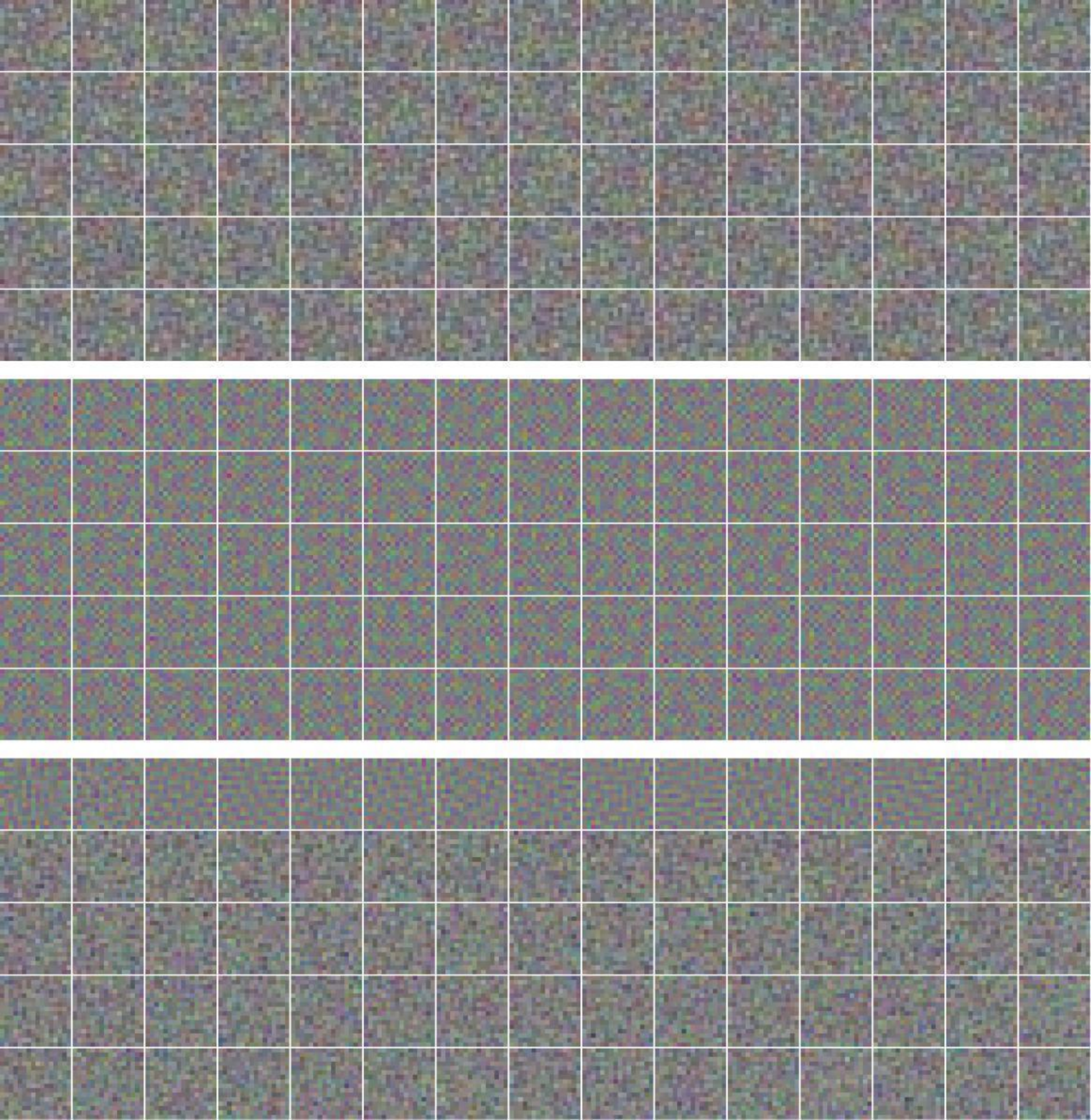}
            \label{fig:imagenet-conceptstable-mlp-matryoshka}
            \caption{Matryoshka-SAE-MLP.}
    \end{subfigure}
    \caption{\textbf{Supporting visualization demonstrating instability of SAE-MLP learned concept structure.} ImageNet patches. Left to right: $k=25, 100, 200, 500$.}
    \label{fig:imagenet-conceptstable-mlp}
\end{figure}

\begin{figure}[H]
    \centering
    \begin{subfigure}{1.00\textwidth}
        \includegraphics[width=0.24\linewidth]{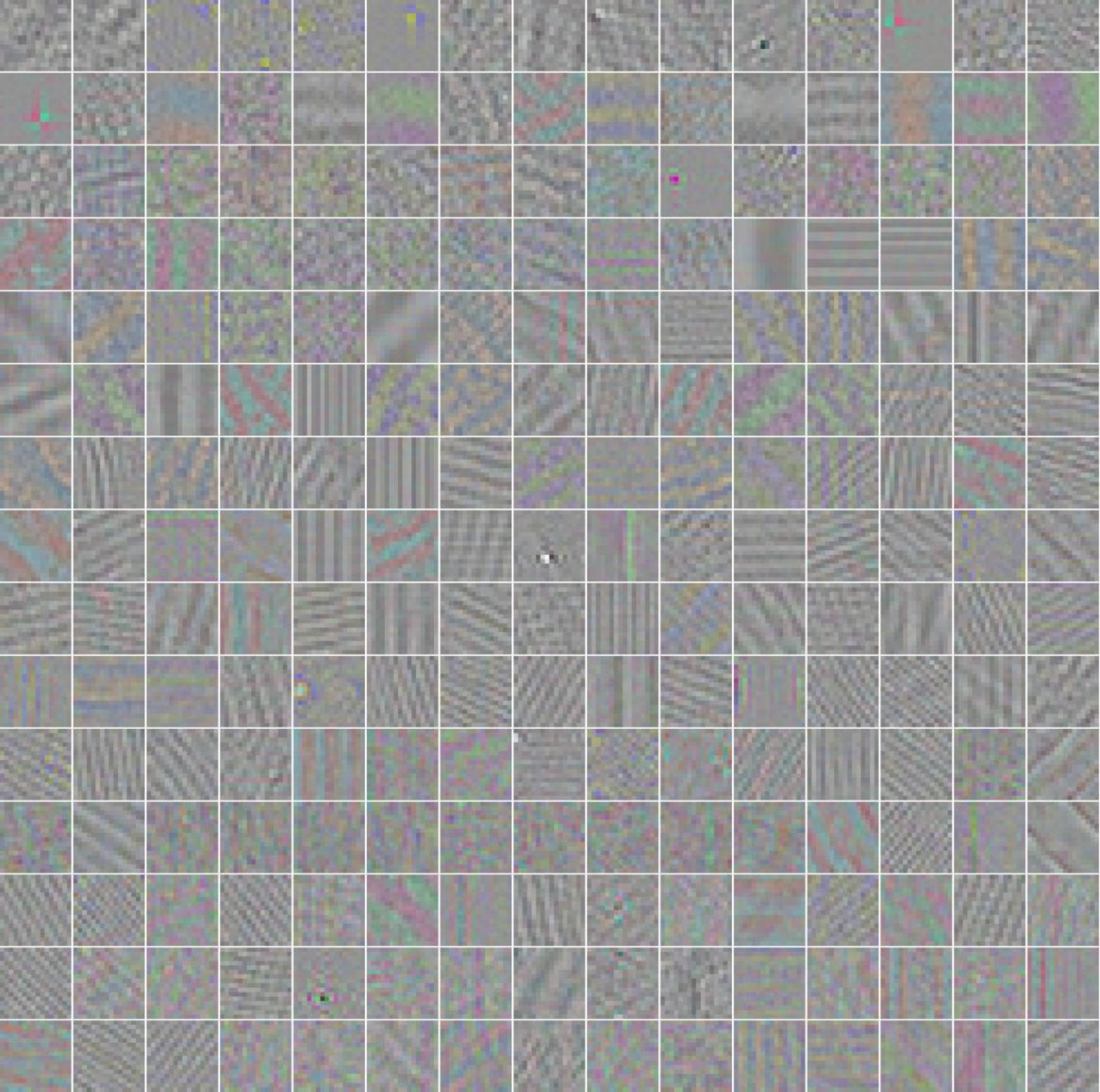}
        \includegraphics[width=0.24\linewidth]{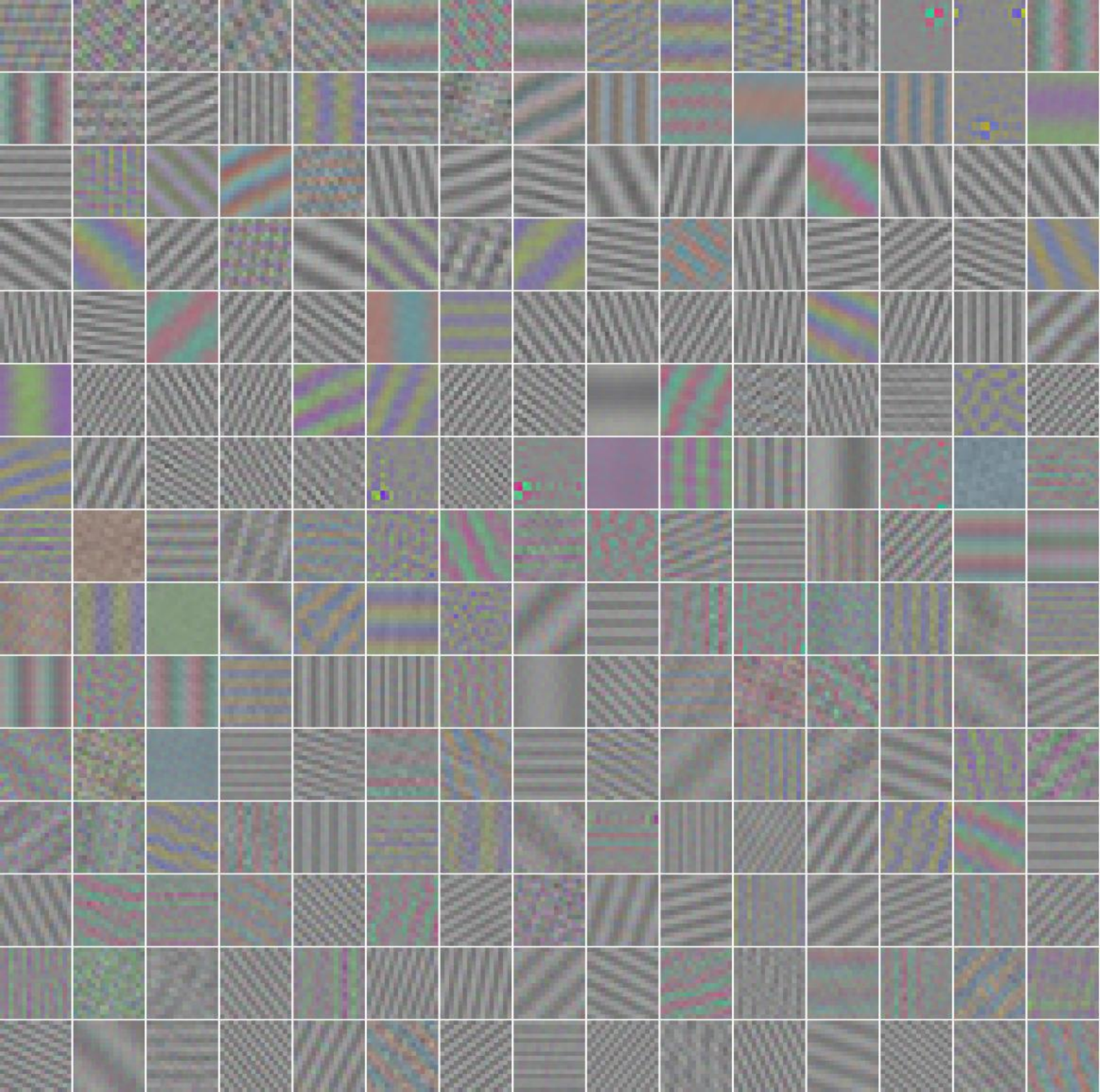}
        \includegraphics[width=0.24\linewidth]{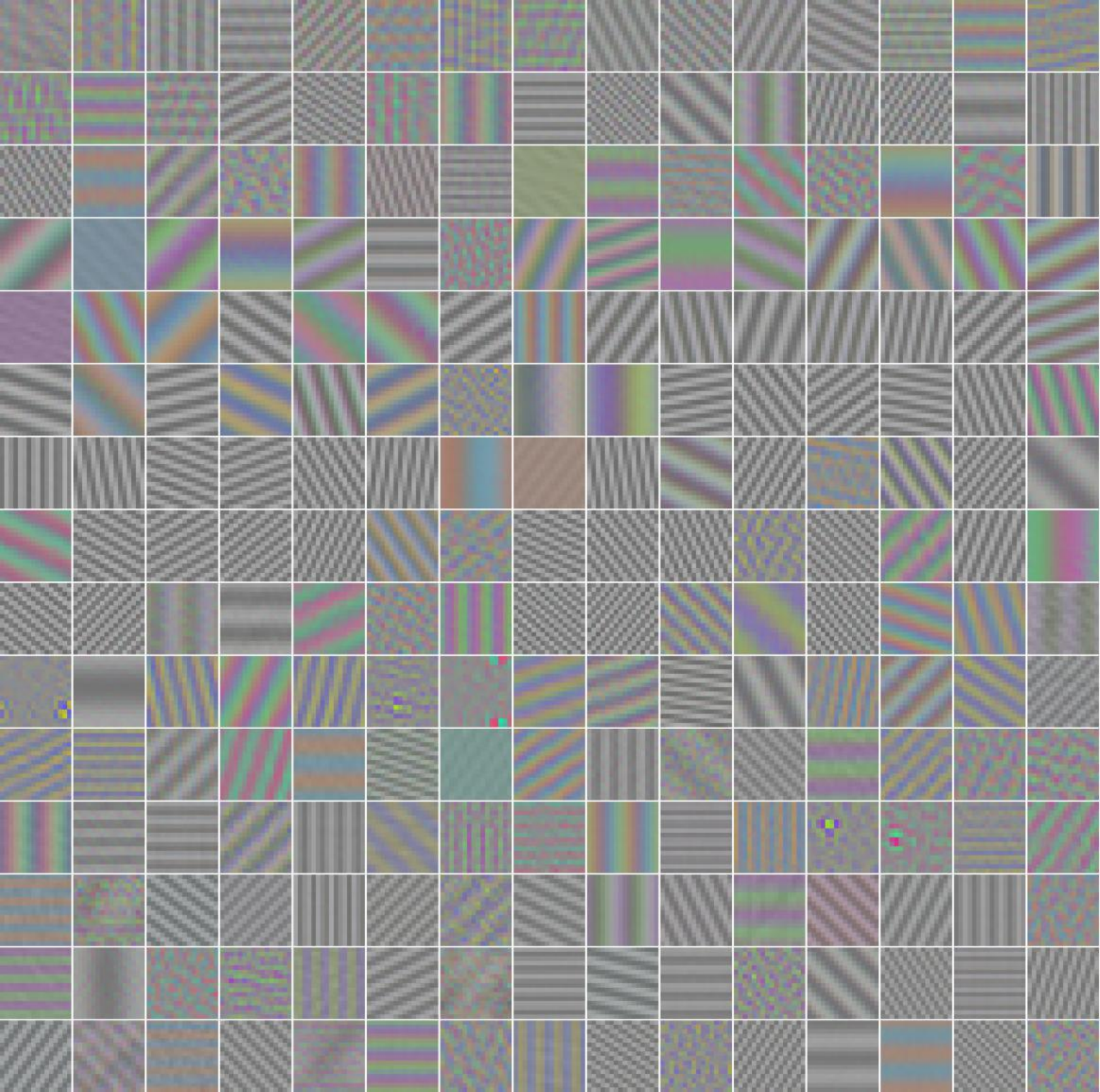}
        \includegraphics[width=0.24\linewidth]{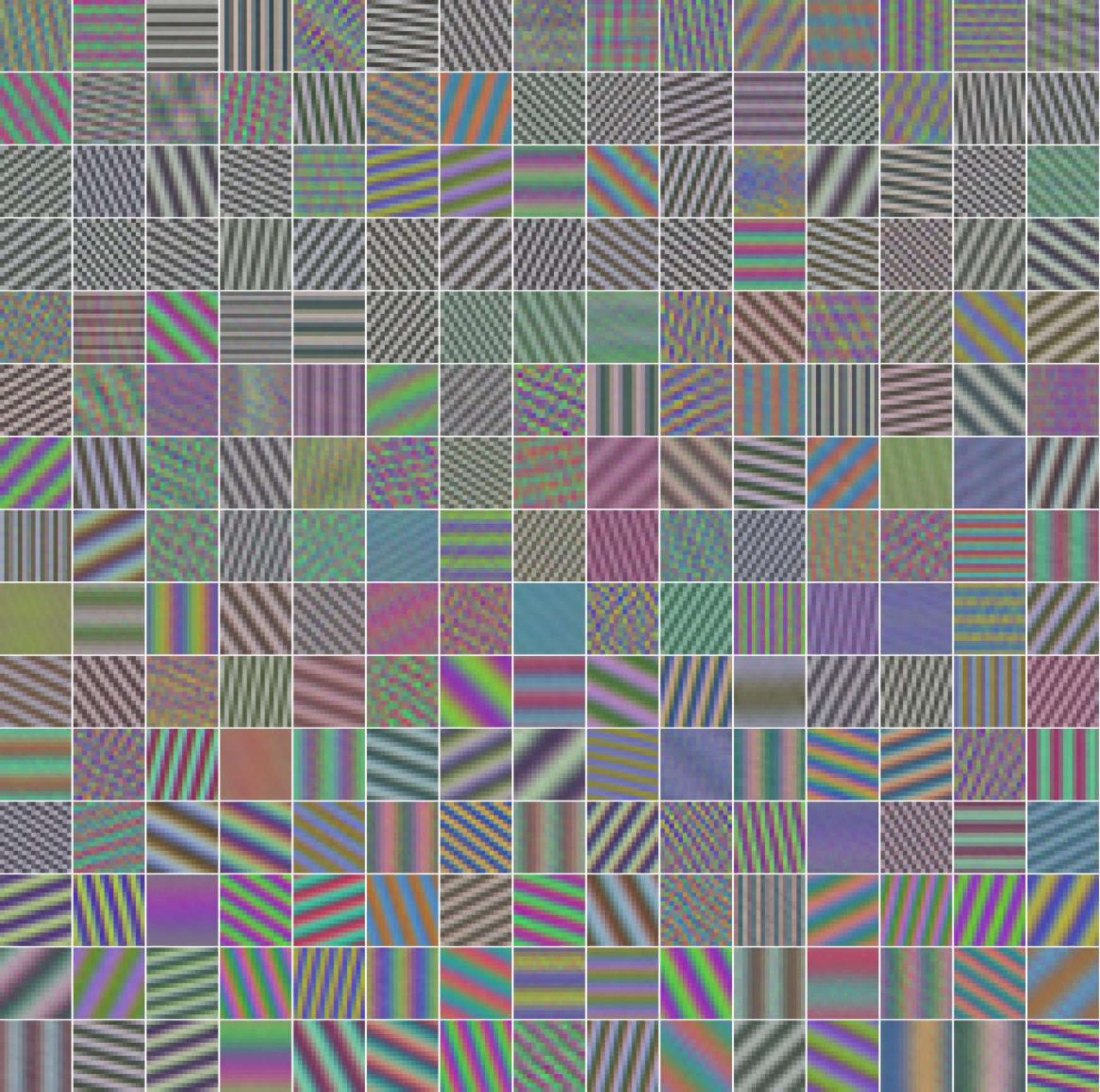}
        \caption{SAE-FNO-TopK.}
    \end{subfigure}
        \begin{subfigure}{1.00\textwidth}
        \includegraphics[width=0.24\linewidth]{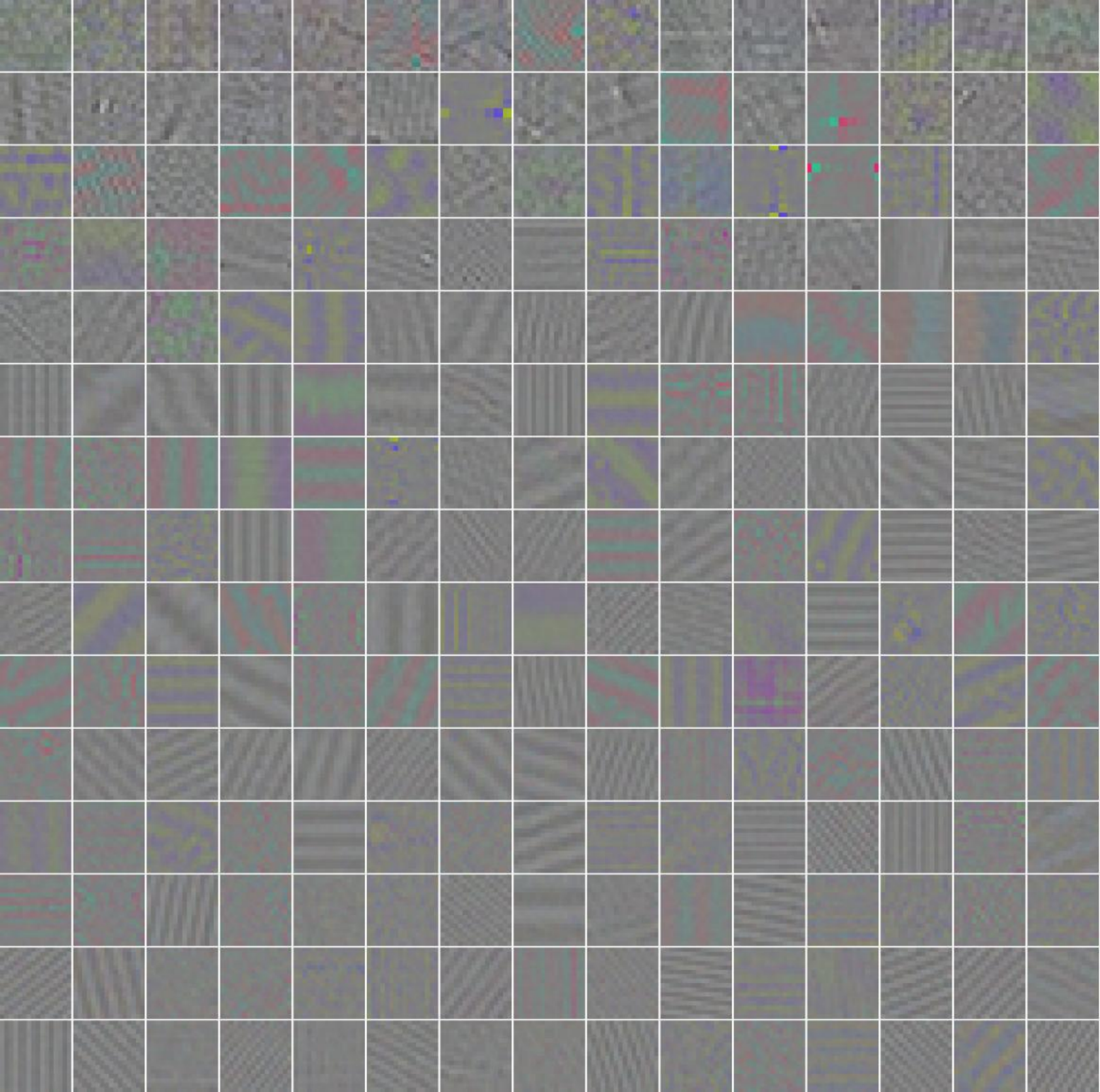}
        \includegraphics[width=0.24\linewidth]{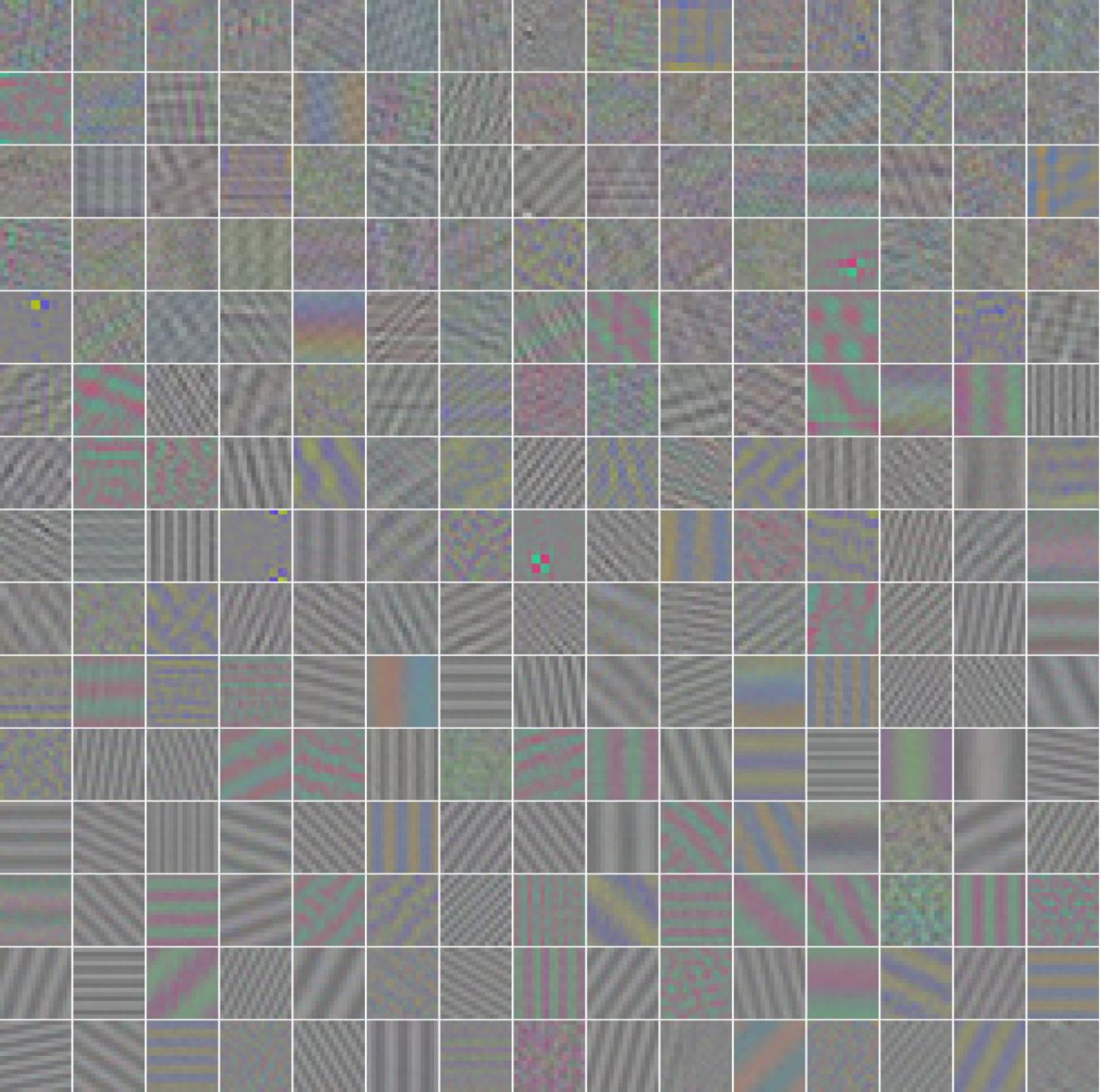}
        \includegraphics[width=0.24\linewidth]{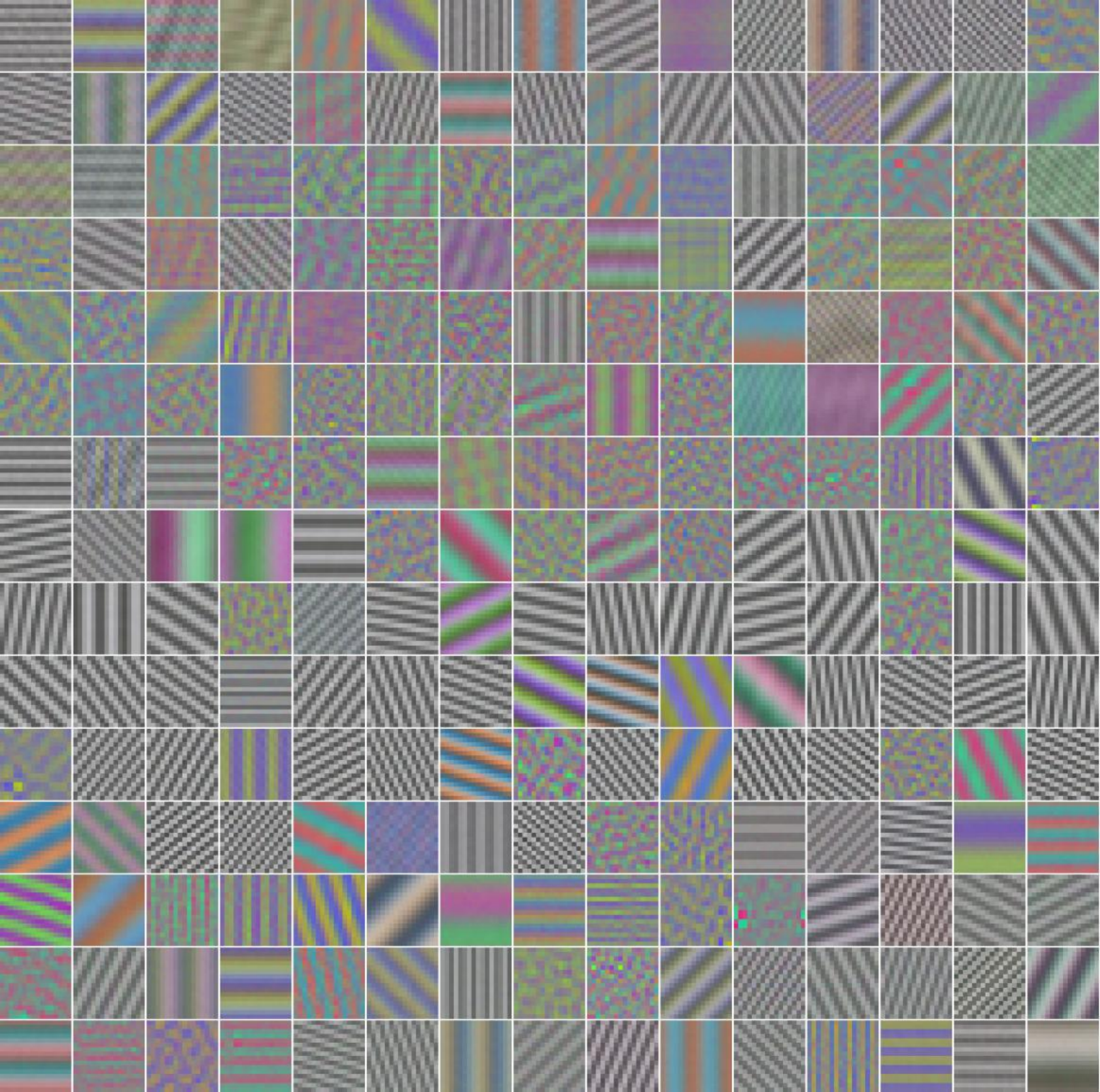}
        \includegraphics[width=0.24\linewidth]{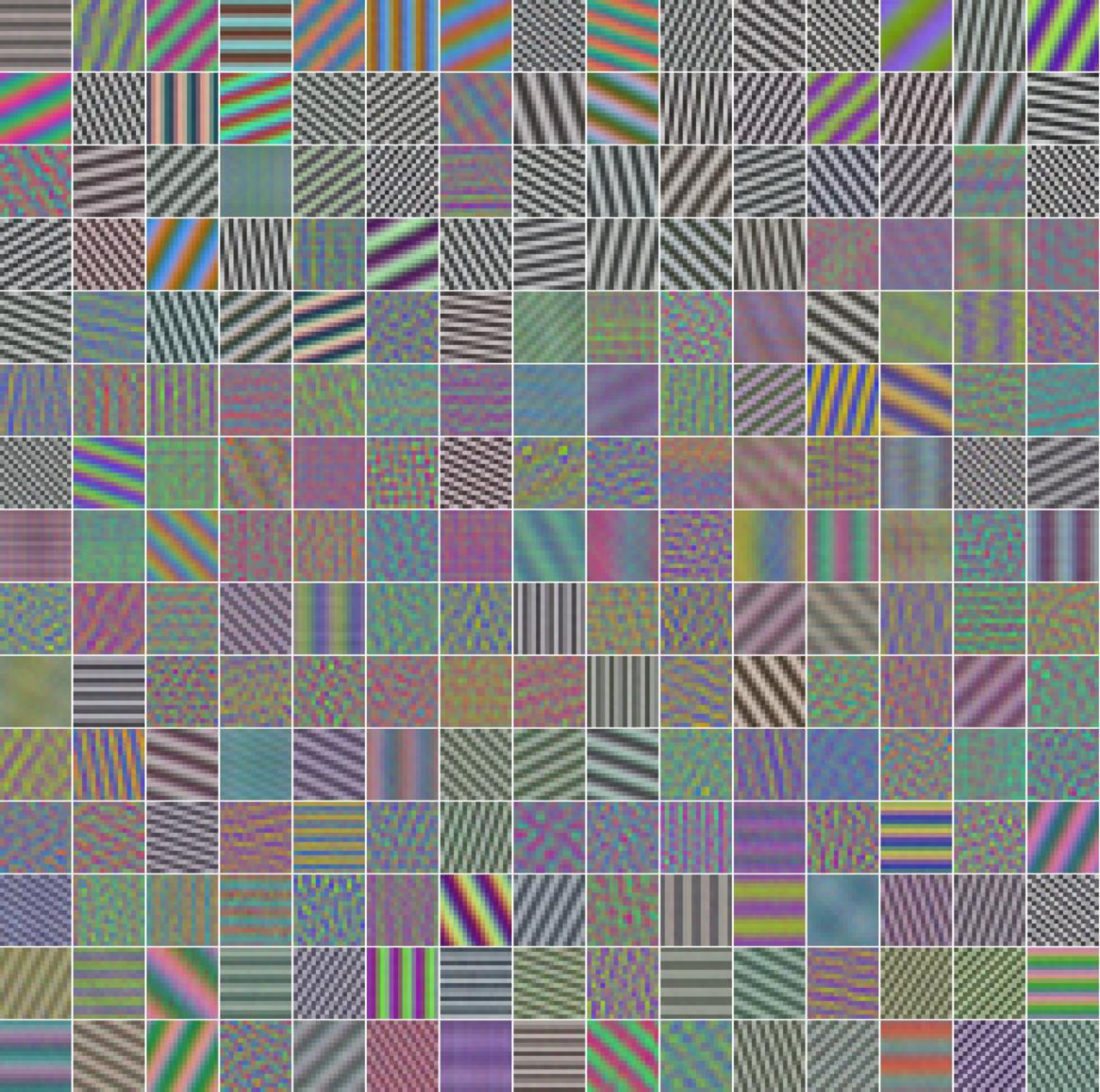}
        \caption{SAE-FNO-BatchTopK.}
    \end{subfigure}
        \begin{subfigure}{1.00\textwidth}
   \includegraphics[width=0.24\linewidth]{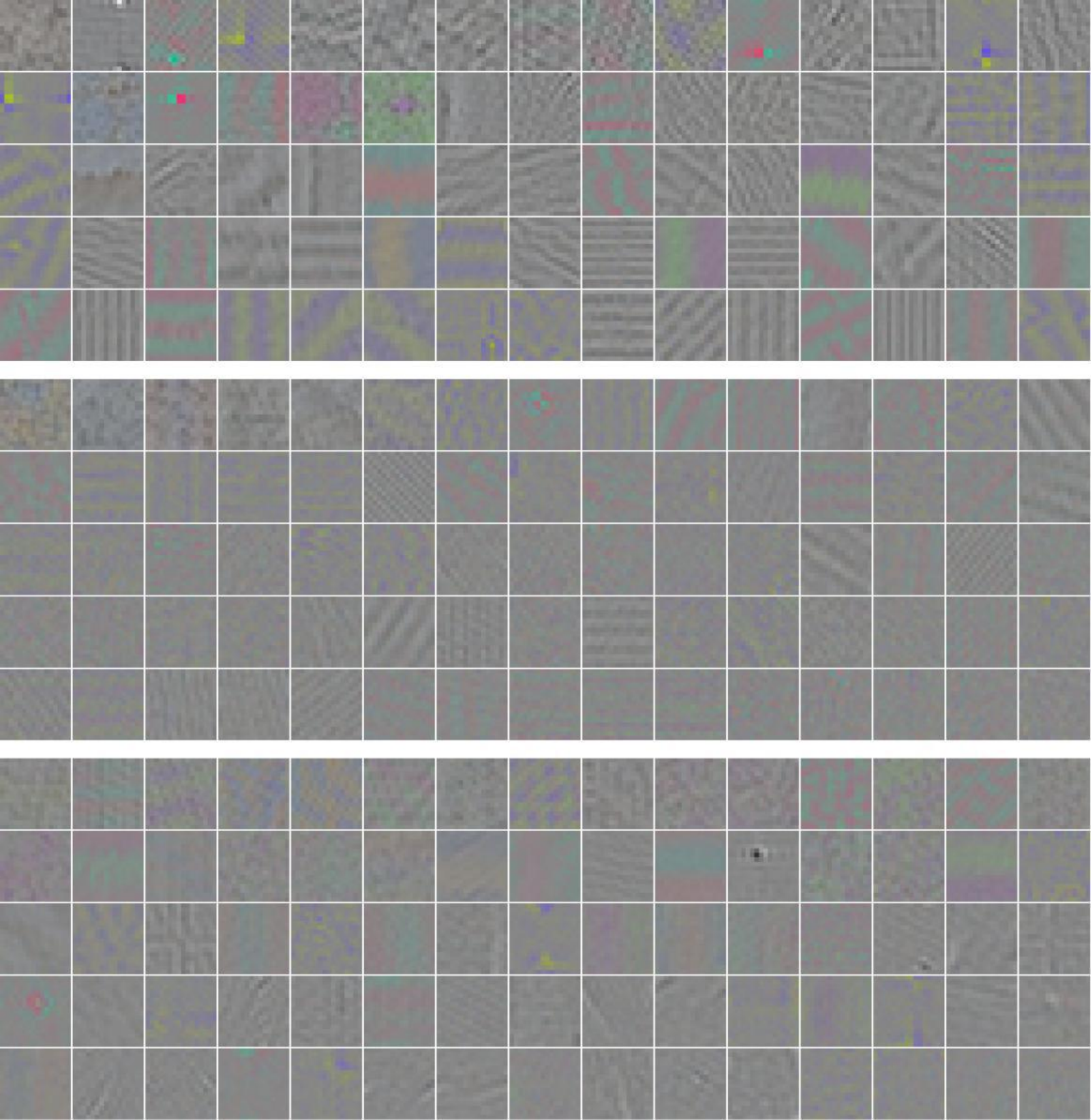}
        \includegraphics[width=0.24\linewidth]{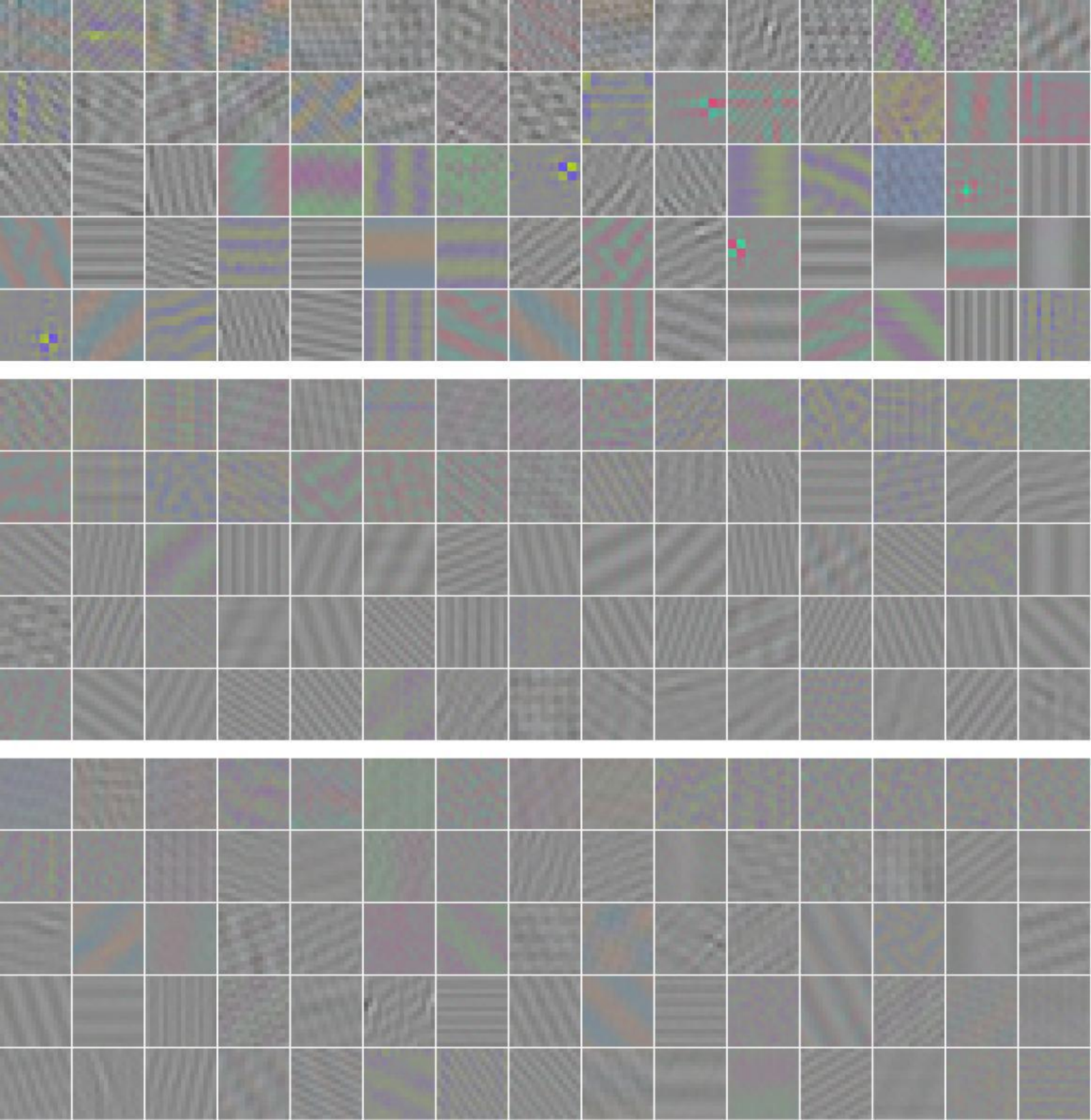}
        \includegraphics[width=0.24\linewidth]{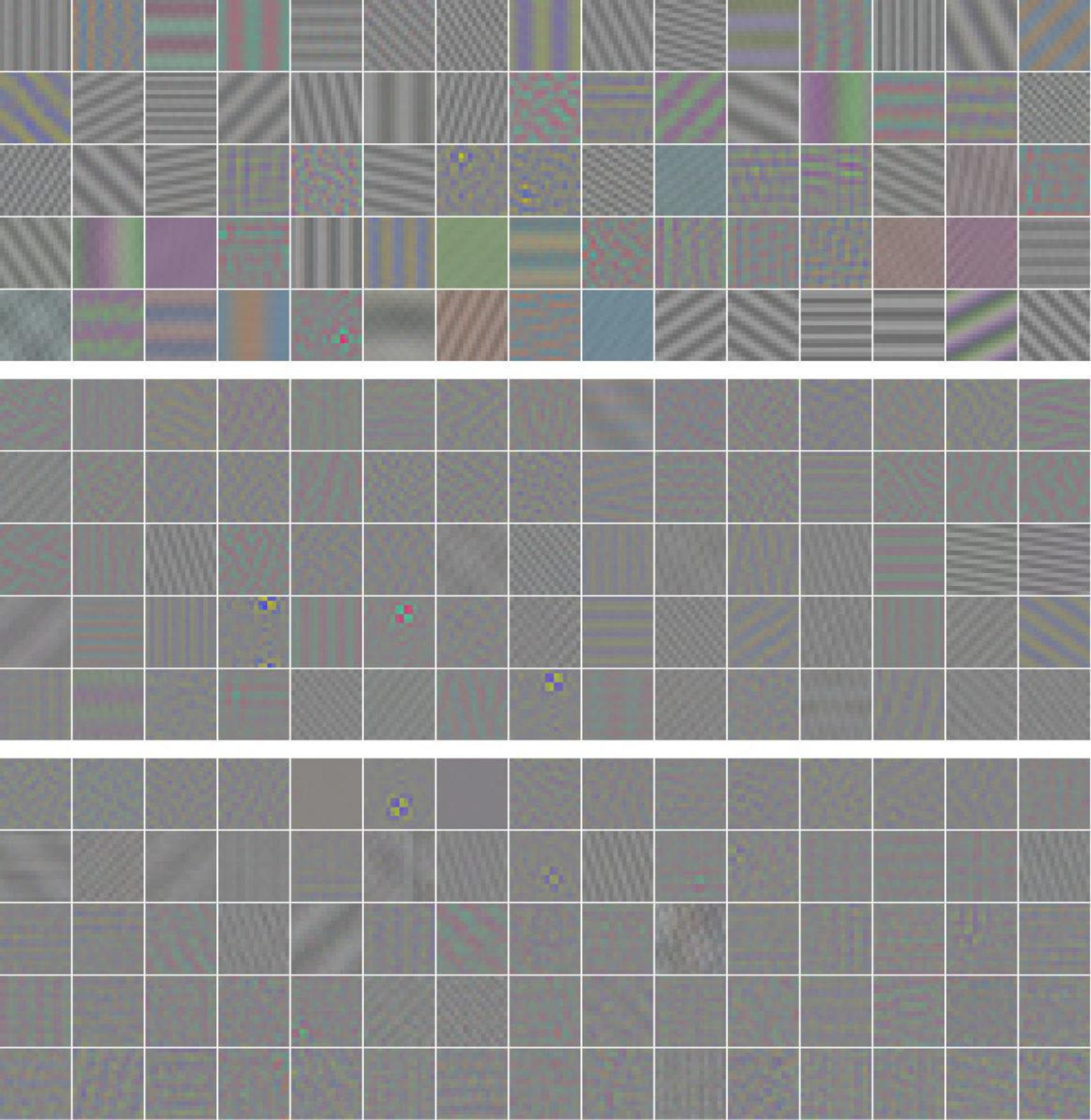}
    \includegraphics[width=0.24\linewidth]{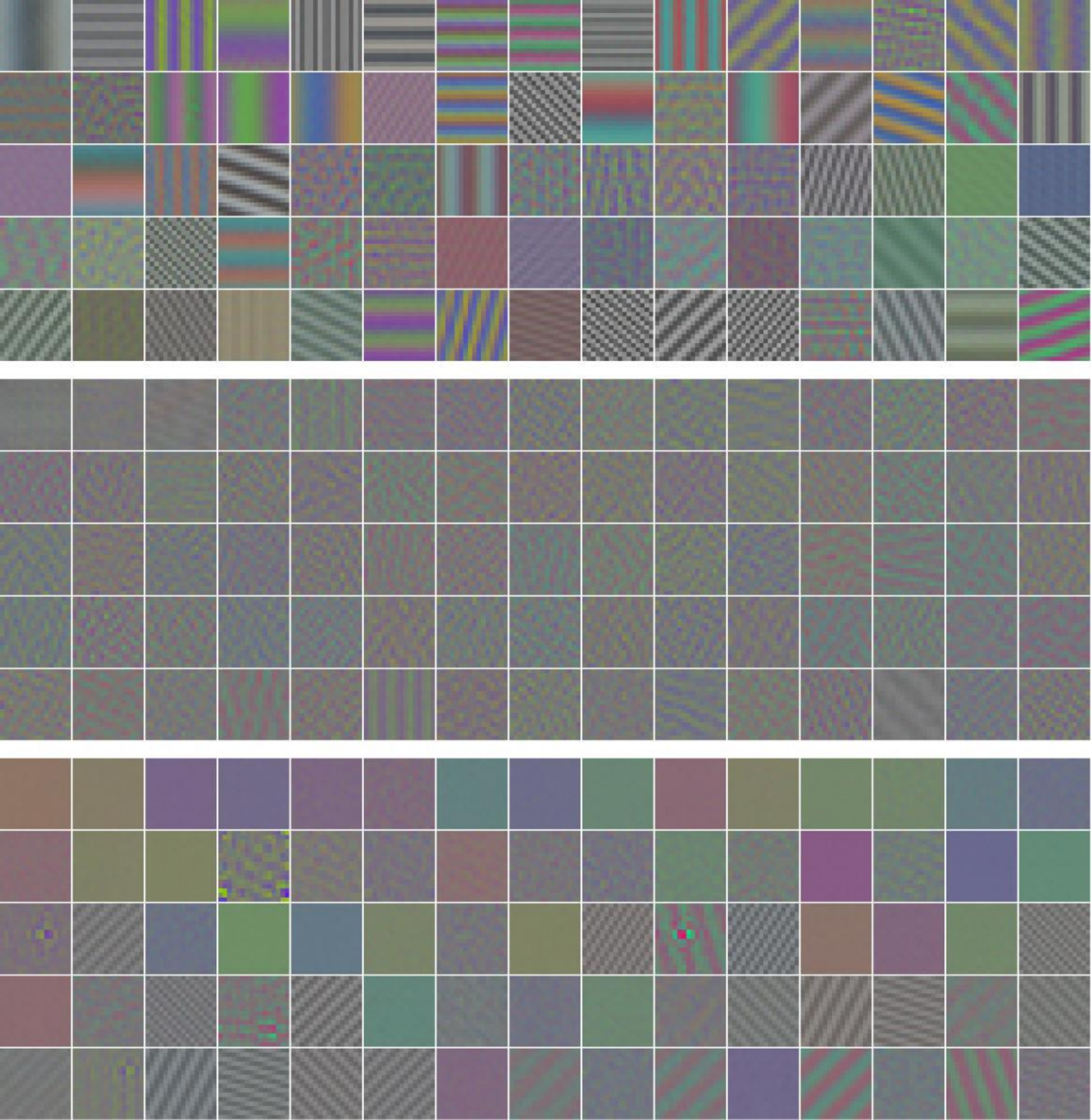}
        \caption{Matryoshka-SAE-FNO.}
    \end{subfigure}
    \caption{\textbf{Supporting visualization demonstrating stability of SAE-FNO learned concept structure.} SAE-FNO (modes = 12, spatially 5-sparse) on ImageNet patches. Left to right: $k=25, 100, 200, 500$.}
    \label{fig:imagenet-conceptstable-fno}
\end{figure}

\newpage

\begin{figure}[H]
    \centering
    \begin{subfigure}{0.49\textwidth}
        \includegraphics[width=\linewidth]{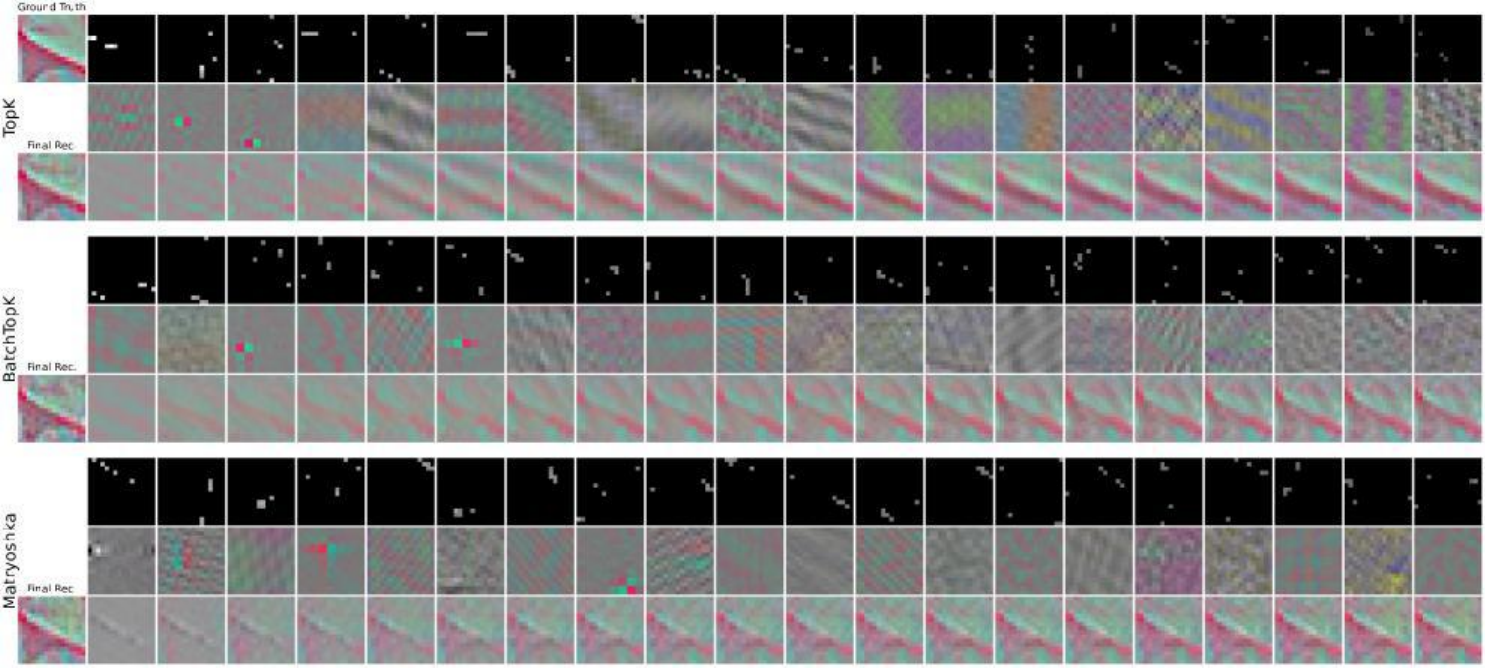}
        \caption{SAE-FNO ($k = 50$).}
    \end{subfigure}
    \hfill
    \begin{subfigure}{0.49\textwidth}
        \includegraphics[width=\linewidth]{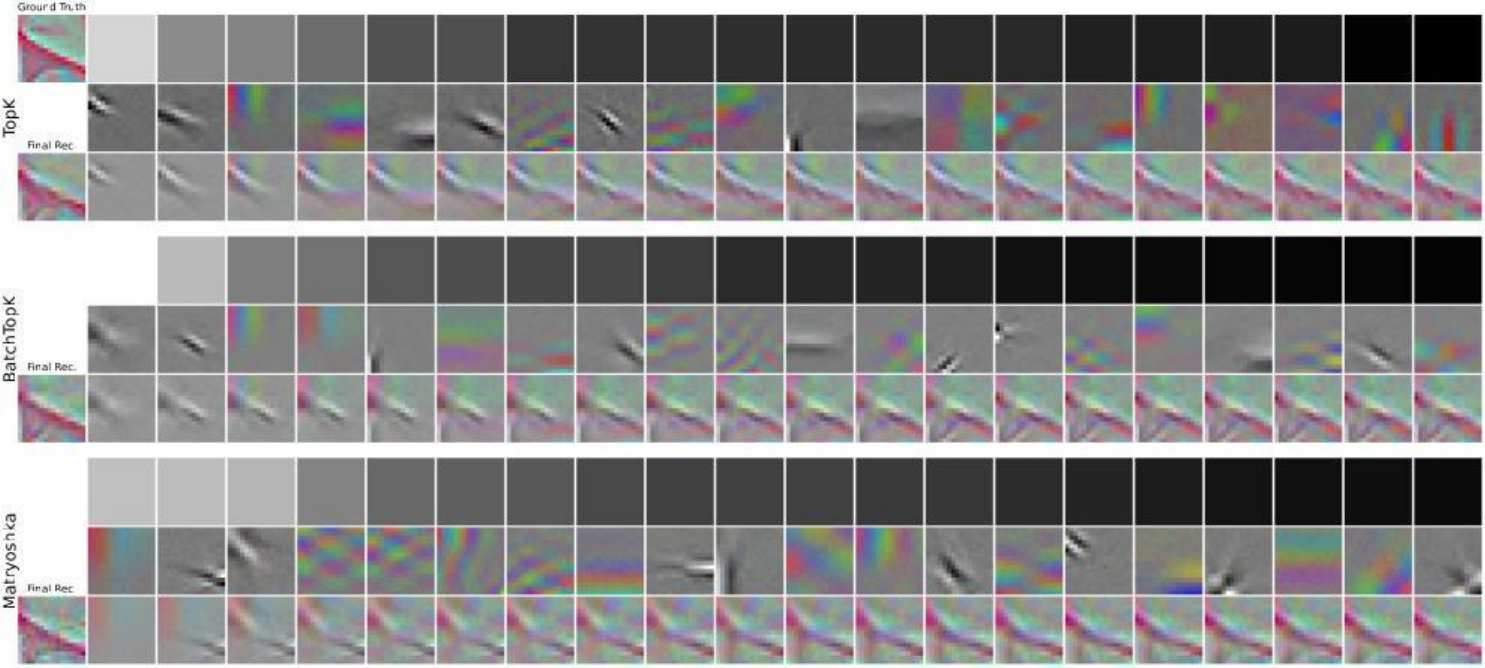}
        \caption{SAE-MLP ($k = 50$).}
    \end{subfigure}
    \hfill
    \begin{subfigure}{0.49\textwidth}
        \includegraphics[width=\linewidth]{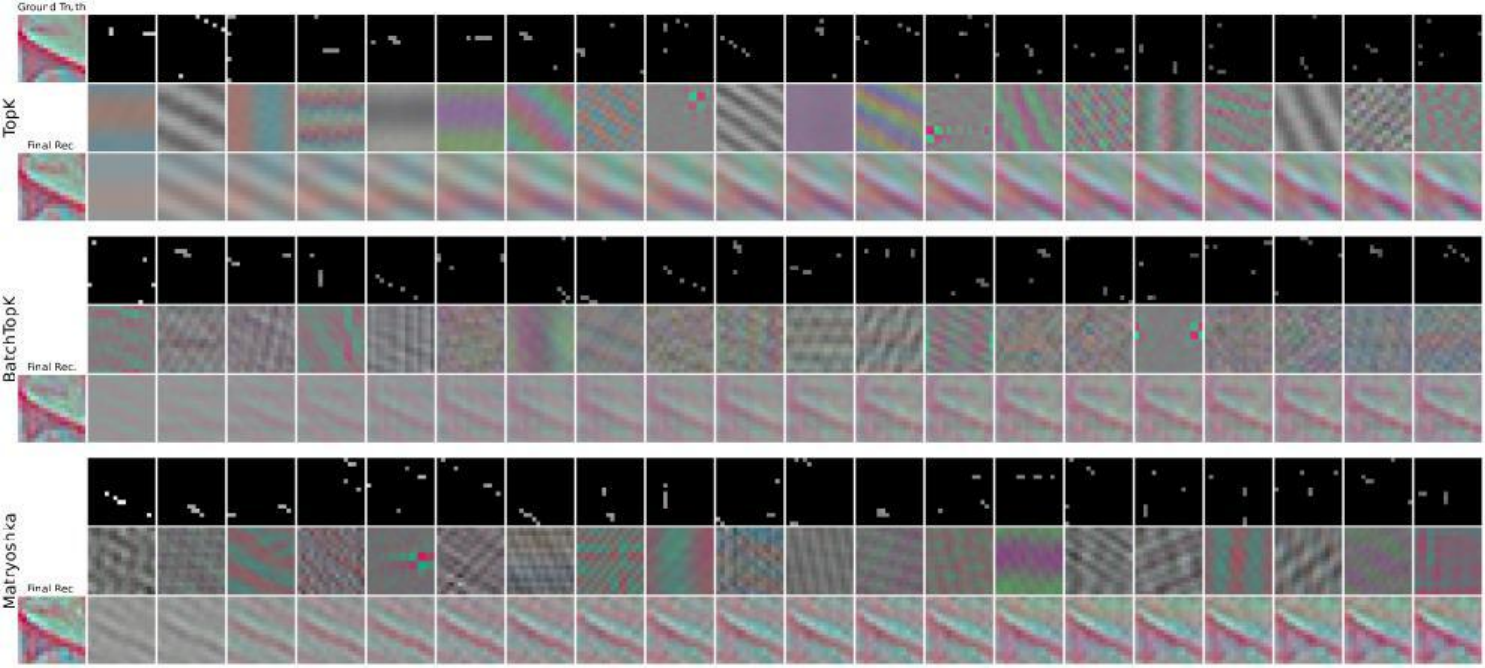}
        \caption{SAE-FNO ($k = 100$).}
    \end{subfigure}
    \hfill
    \begin{subfigure}{0.49\textwidth}
        \includegraphics[width=\linewidth]{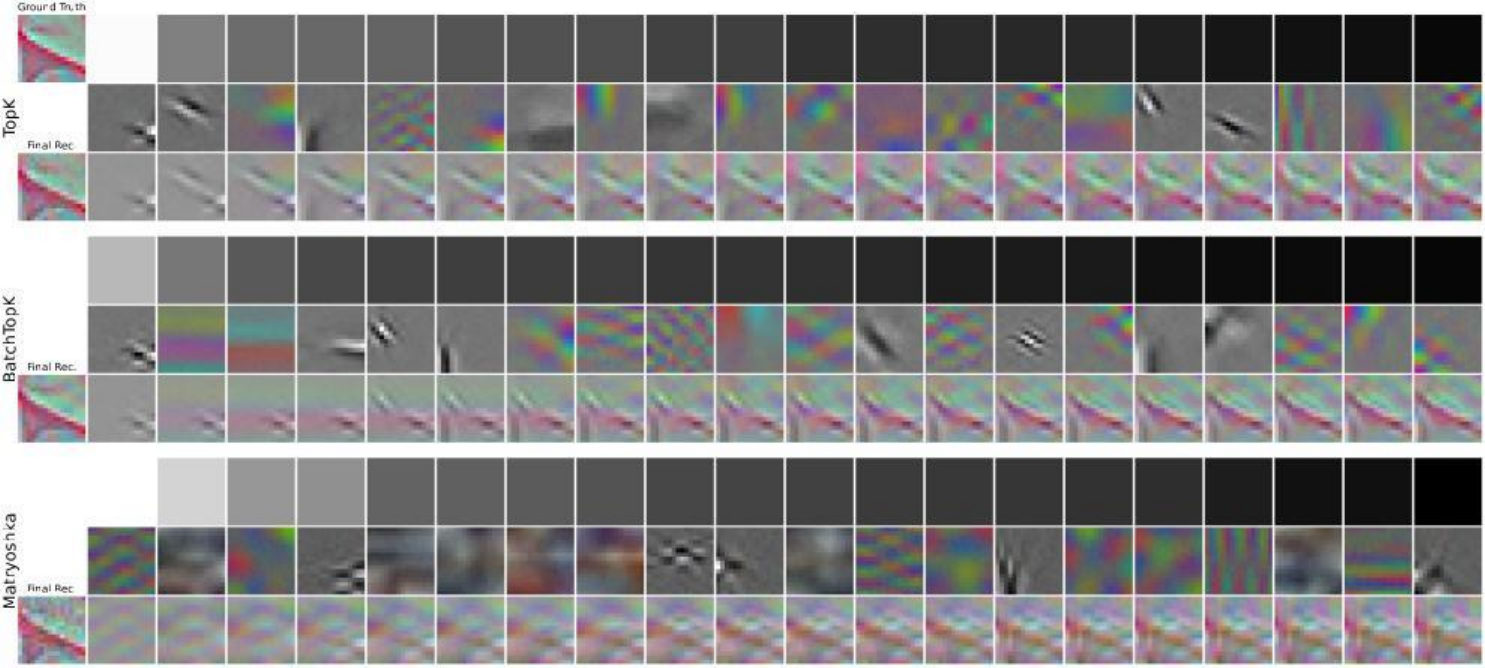}
        \caption{SAE-MLP ($k = 100$).}
    \end{subfigure}
    \hfill
    \begin{subfigure}{0.49\textwidth}
        \includegraphics[width=\linewidth]{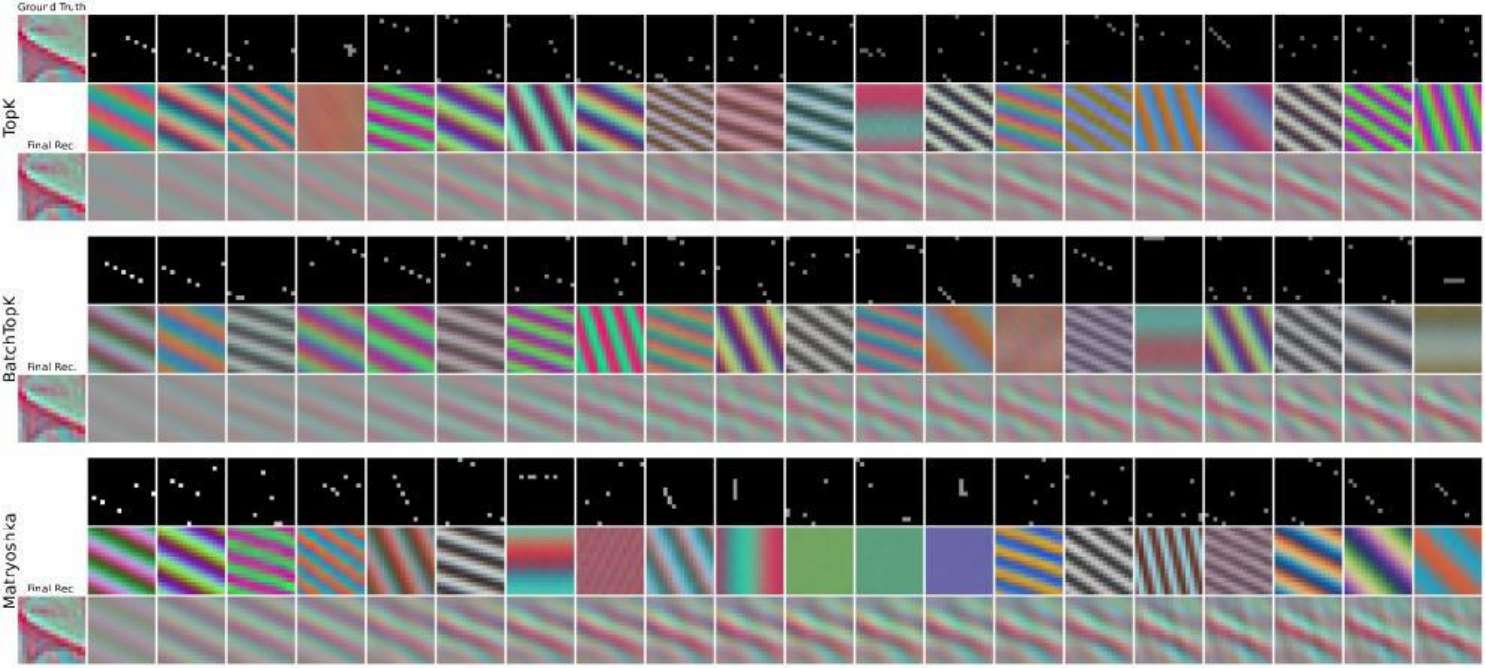}
        \caption{SAE-FNO ($k = 500$).}
    \end{subfigure}
    \hfill
    \begin{subfigure}{0.49\textwidth}
        \includegraphics[width=\linewidth]{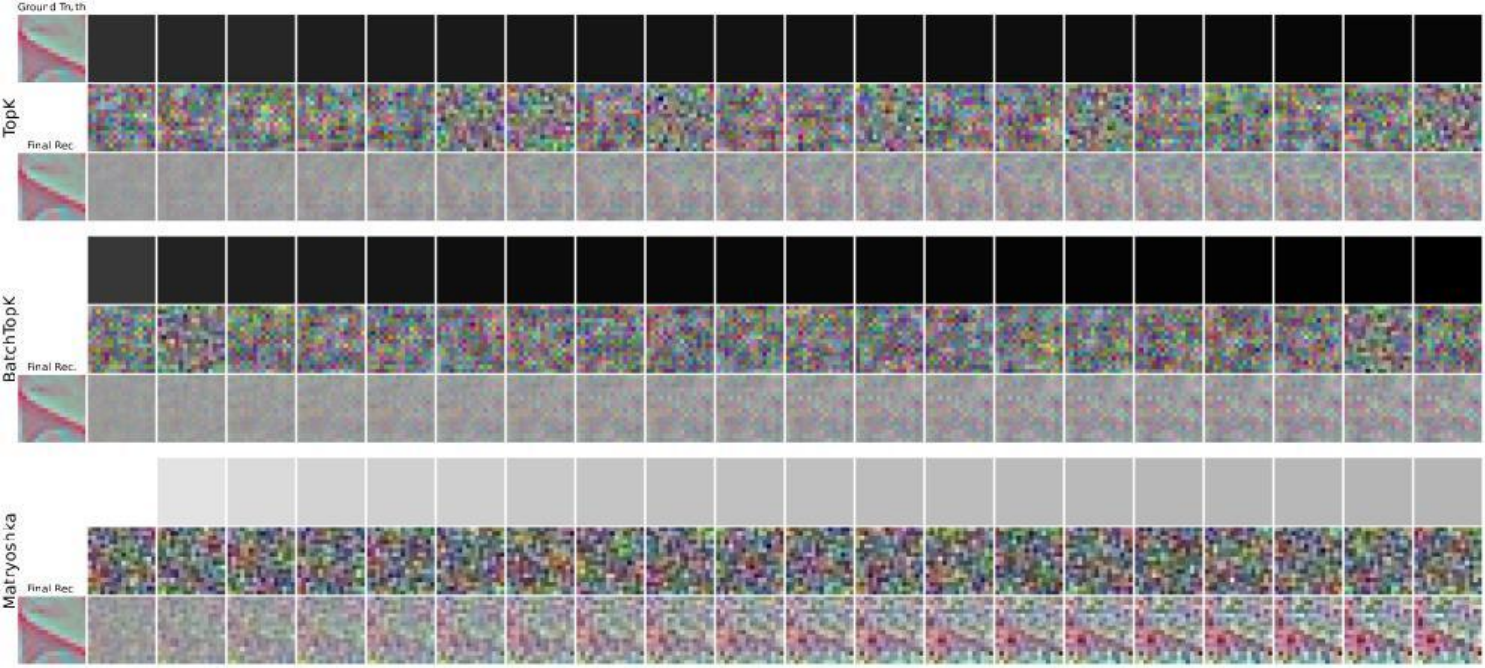}
        \caption{SAE-MLP ($k = 500$).}
    \end{subfigure}
    \caption{\textbf{Progressive reconstruction from active concepts across SAE variants.} SAE-FNO (left) and SAE-MLP (right) are compared across concept sparsity levels. Within each panel, TopK, BatchTopK, and Matryoshka results are shown from top to bottom. For each variant, the three rows (from top to bottom) display spatial sparse codes, individual learned concepts, and cumulative reconstructions obtained by successively adding the top $i$ active concepts ($i=1,\ldots,20$). The leftmost images show the ground-truth input and final reconstruction.}
    \label{fig:imagenet-rec-seq}
\end{figure}

\newpage 
\section{Appendix - Concept Utilization and Expressivity}\label{app:utilization}

\Cref{fig:imagenet-r2} uses $k=50$ on ImageNet $16 \times 16$, and SAE-FNO uses $12$ modes.

\begin{figure}[H]
    \centering
    \begin{subfigure}{0.32\textwidth}
        \includegraphics[width=\linewidth]{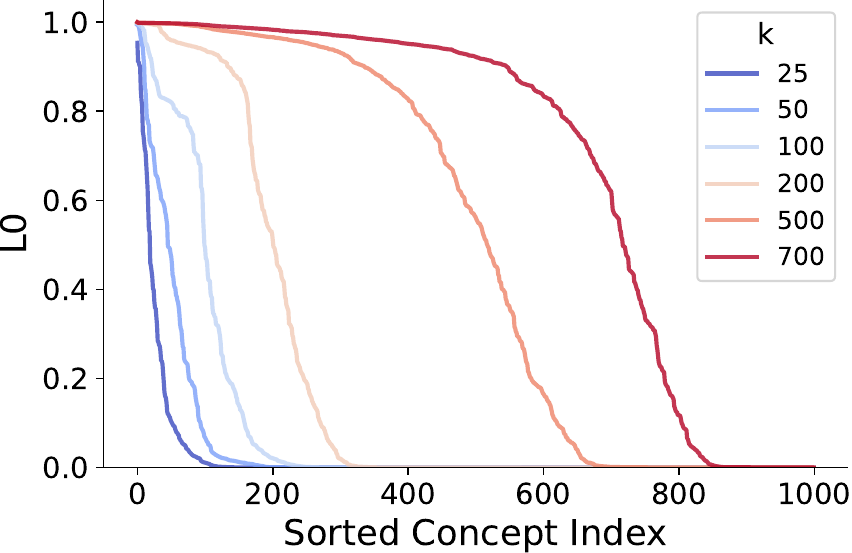}
        \caption{SAE-FNO-TopK.}
    \end{subfigure}
    \begin{subfigure}{0.32\textwidth}
        \includegraphics[width=\linewidth]{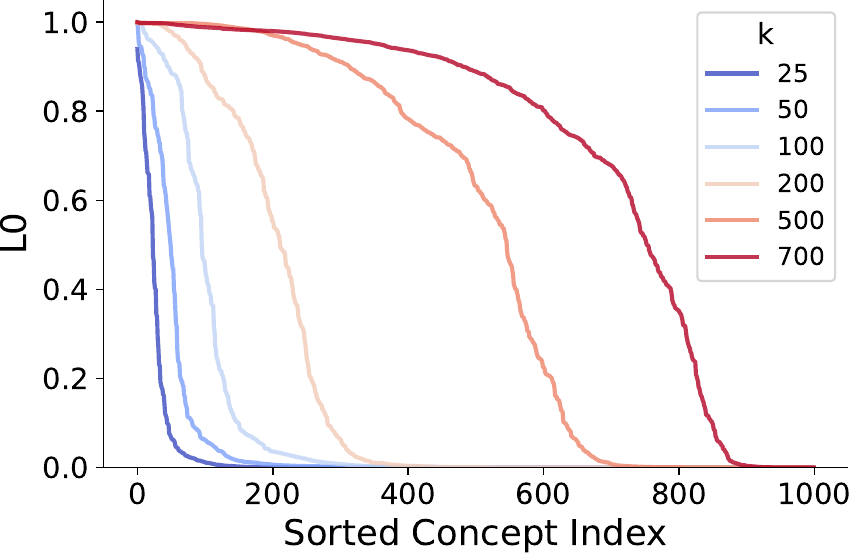}
        \caption{SAE-FNO-BatchTopK.}
    \end{subfigure}
    \begin{subfigure}{0.32\textwidth}
        \includegraphics[width=\linewidth]{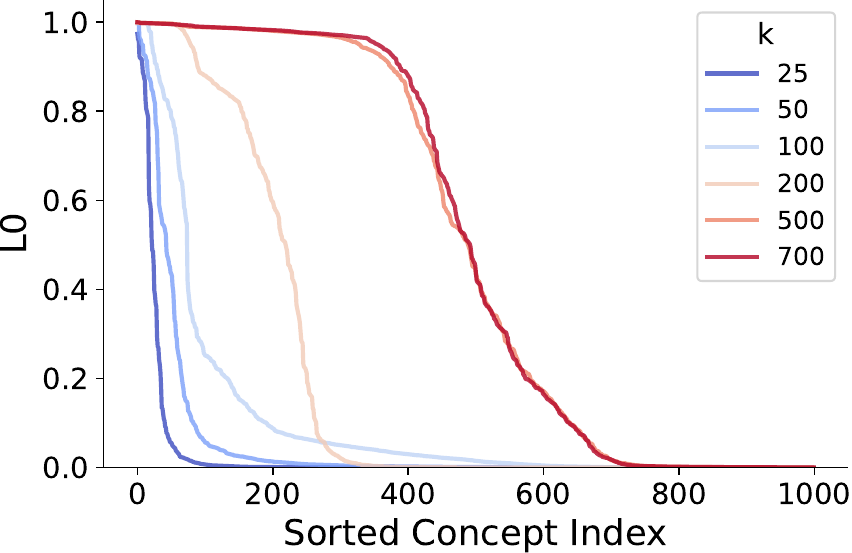}
        \caption{Matryoshka-SAE-FNO.}
    \end{subfigure}
    \\
    \begin{subfigure}{0.32\textwidth}
        \includegraphics[width=\linewidth]{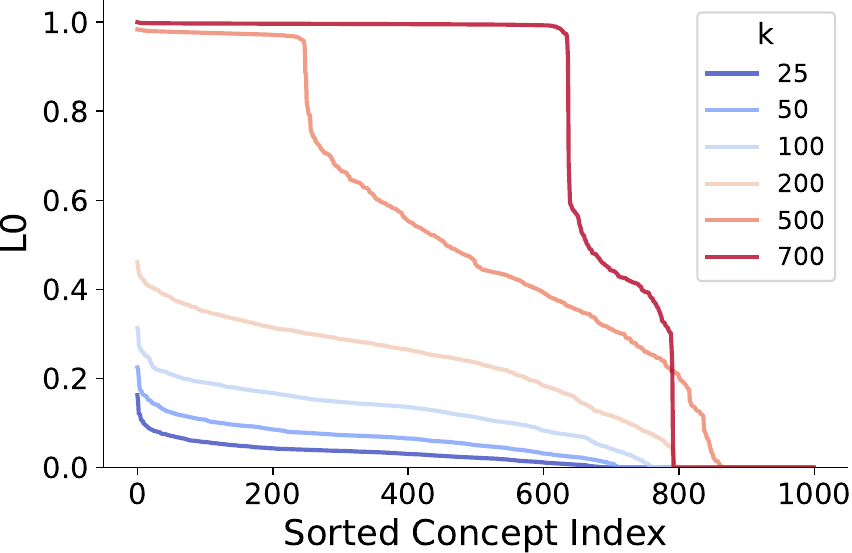}
        \caption{SAE-MLP-TopK.}
    \end{subfigure}
    \begin{subfigure}{0.32\textwidth}
        \includegraphics[width=\linewidth]{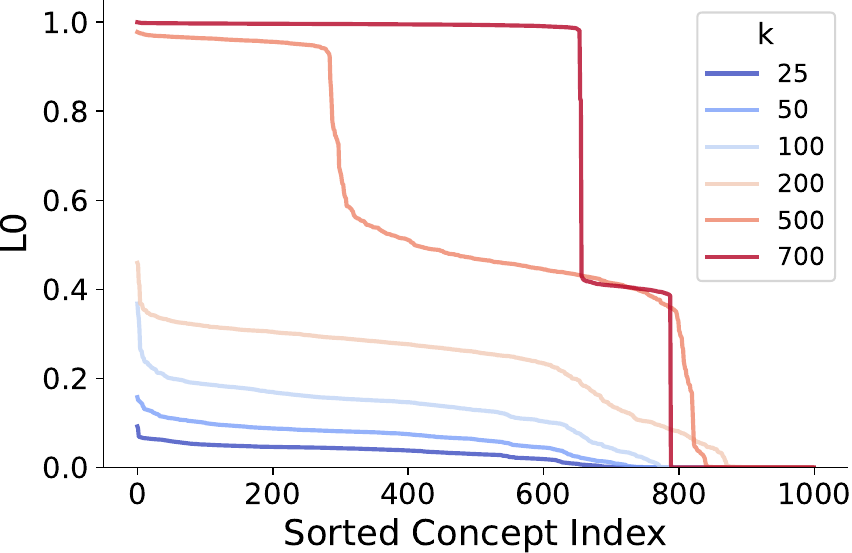}
        \caption{SAE-MLP-BatchTopK.}
    \end{subfigure}
    \begin{subfigure}{0.32\textwidth}
        \includegraphics[width=\linewidth]{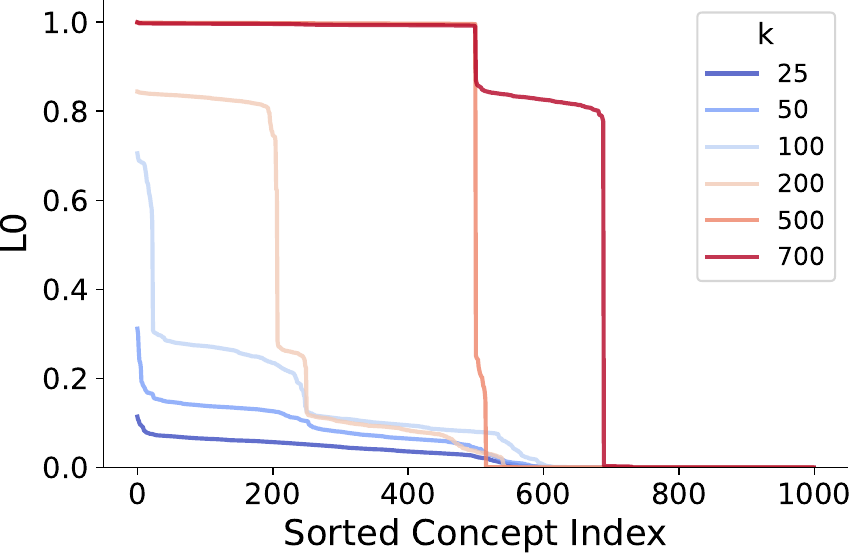}
        \caption{Matryoshka-SAE-MLP.}
    \end{subfigure}
    \caption{\textbf{Concept utilization and efficiency on ImageNet patches ($16 \times 16$).} Normalized $\ell_0$-norm of codes corresponding to each dictionary concept, sorted by decreasing utilization. The point at which the y-axis reaches zero corresponds to the effective total number of concepts used to explain the dataset. Across TopK, BatchTopK, and Matryoshka variants, SAE-FNO requires a much smaller active dictionary to represent the data compared to SAE-MLP.}
    \label{fig:imagenet-concept-utilization}
\end{figure}

\begin{figure}[H]
    \centering
    \begin{subfigure}{\textwidth}
        \includegraphics[width=0.32\linewidth]{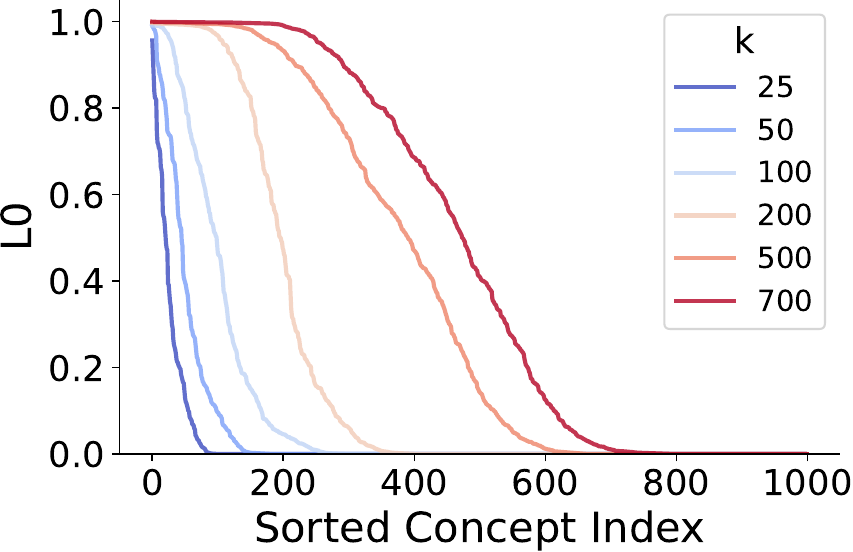}
        \includegraphics[width=0.32\linewidth]{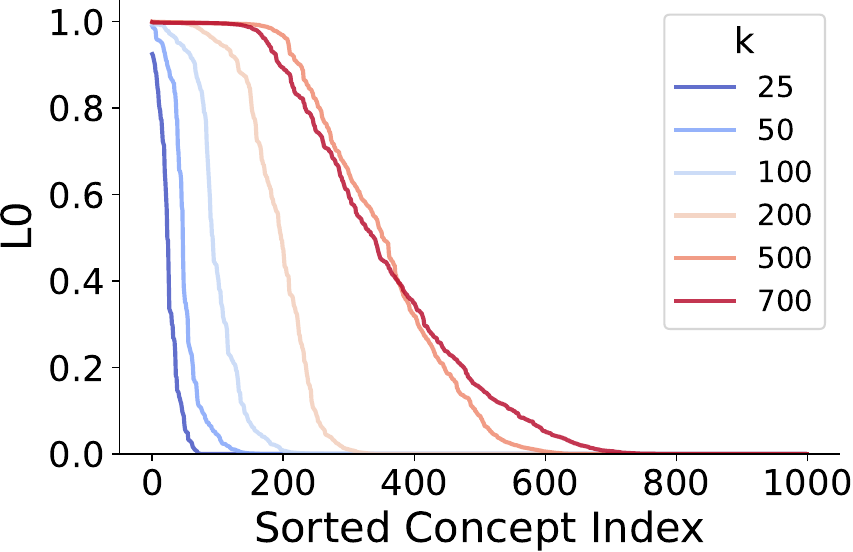}
        \includegraphics[width=0.32\linewidth]{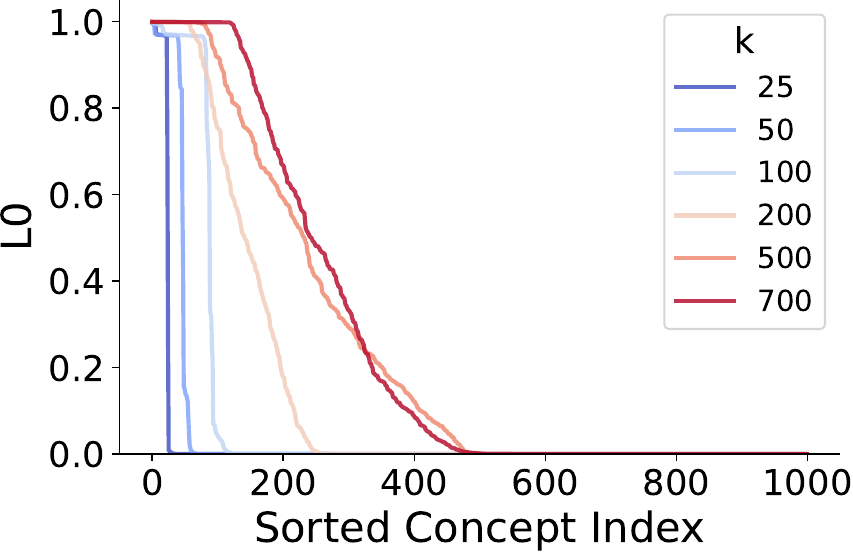}
        \caption{Left to right: spatial sparsity = 0.05, 0.1, 0.5 (modes = 6 fixed).}
    \end{subfigure}
    \\
    \begin{subfigure}{\textwidth}
        \includegraphics[width=0.32\linewidth]{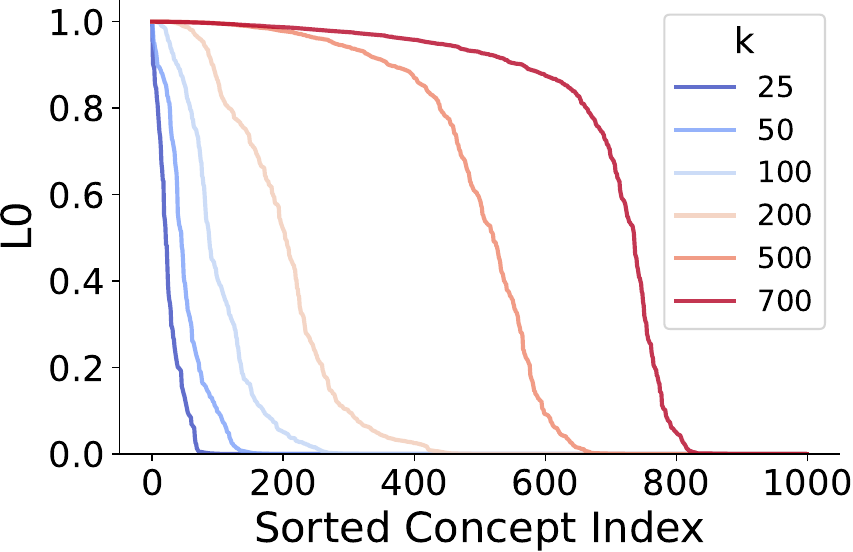}
        \includegraphics[width=0.32\linewidth]{figures/cifar_concept_stats/comparison_topksaefno_ss0p05_m6_l0_patch.pdf}
        \includegraphics[width=0.32\linewidth]{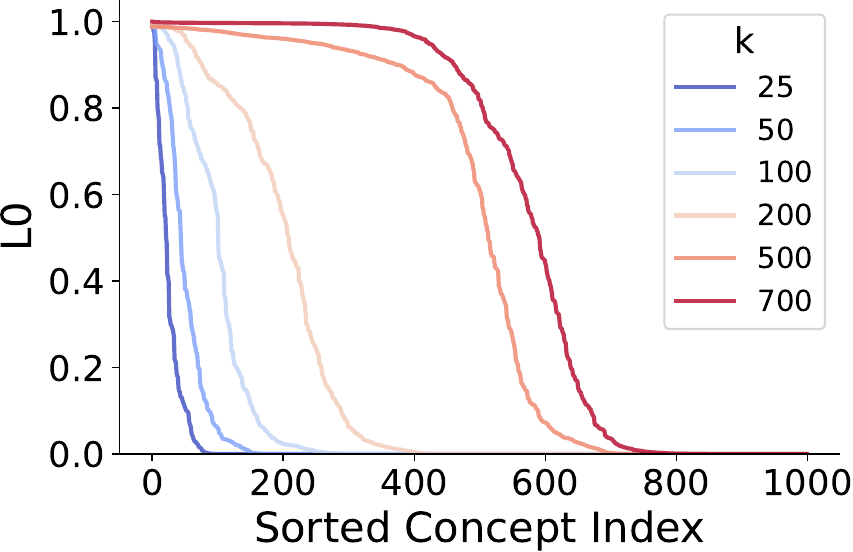}
        \caption{Left to right: modes = 8, 6, 4 (spatial sparsity = 0.05 fixed).}
    \end{subfigure}
    \caption{\textbf{Concept utilization and efficiency on CIFAR $8 \times 8$ patches for SAE-FNO-TopK across changing (a) spatial sparsity and (b) number of Fourier modes.} Same visualization setup as \Cref{fig:imagenet-concept-utilization}. Visualization of learned concepts is shown in \Cref{fig:change-strcuture-domain-cifar} (change spatial sparsity) and \Cref{fig:change-modes-cifar} (change modes).}
    \label{fig:cifar-concept-utilization-fno-change-ss-mode}
\end{figure}

\begin{figure}[H]
    \centering
    \begin{subfigure}{0.32\textwidth}
        \includegraphics[width=\linewidth]{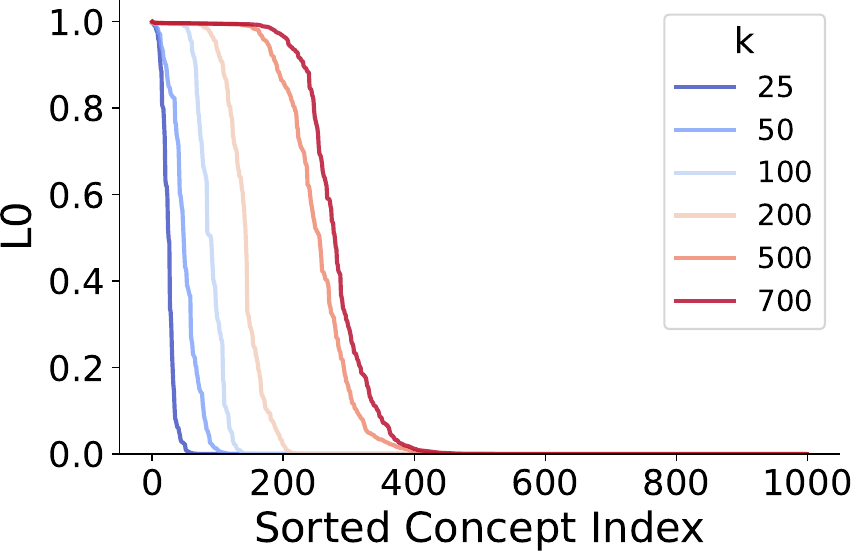}
        \caption{$10 \times 10$ Kernel.}
    \end{subfigure}
    \begin{subfigure}{0.32\textwidth}
        \includegraphics[width=\linewidth]{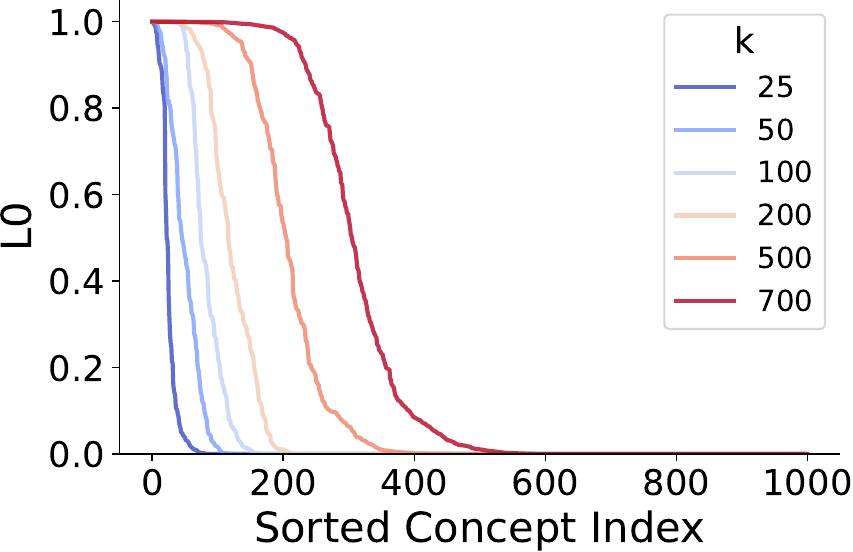}
        \caption{$15 \times 15$ Kernel.}
    \end{subfigure}
    \begin{subfigure}{0.32\textwidth}
        \includegraphics[width=\linewidth]{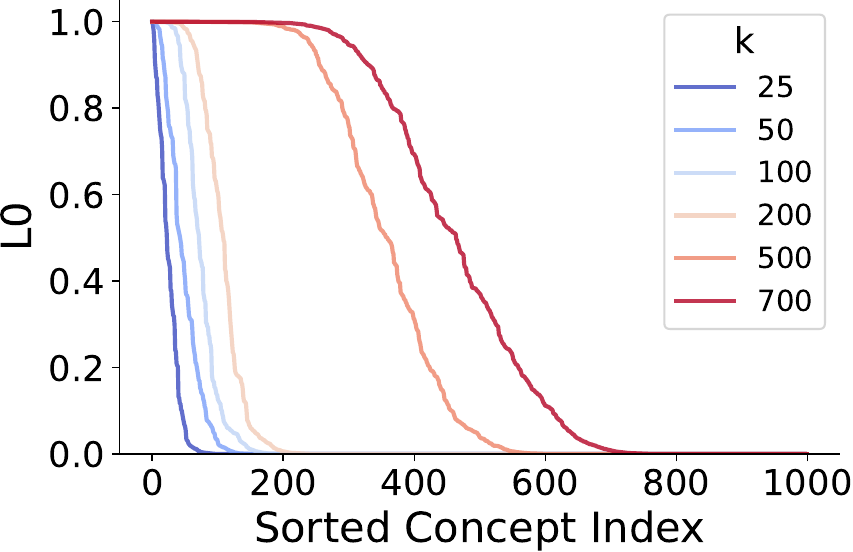}
        \caption{$20 \times 20$ Kernel.}
    \end{subfigure}
    \caption{\textbf{Concept utilization and efficiency on CIFAR $32 \times 32$ for SAE-CNN (spatially 5-sparse) across changing kernel sizes.} Same visualization setup as \Cref{fig:imagenet-concept-utilization}.}
    \label{fig:cifar-concept-utilization-cnn}
\end{figure}

\begin{figure}[H]
    \centering
    \begin{subfigure}{0.45\textwidth}
        \includegraphics[width=\linewidth]{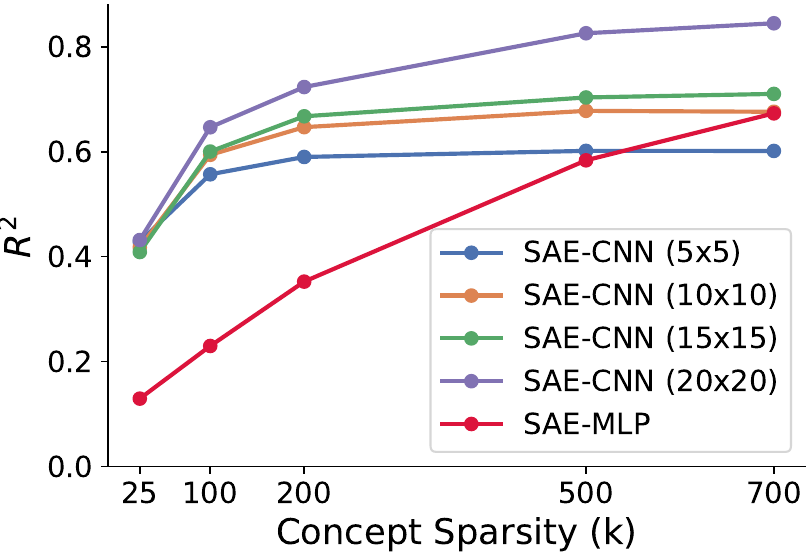}
        \caption{$R^2$ vs. $k$.}
    \end{subfigure}
    \hfill
    \begin{subfigure}{0.45\textwidth}
        \includegraphics[width=\linewidth]{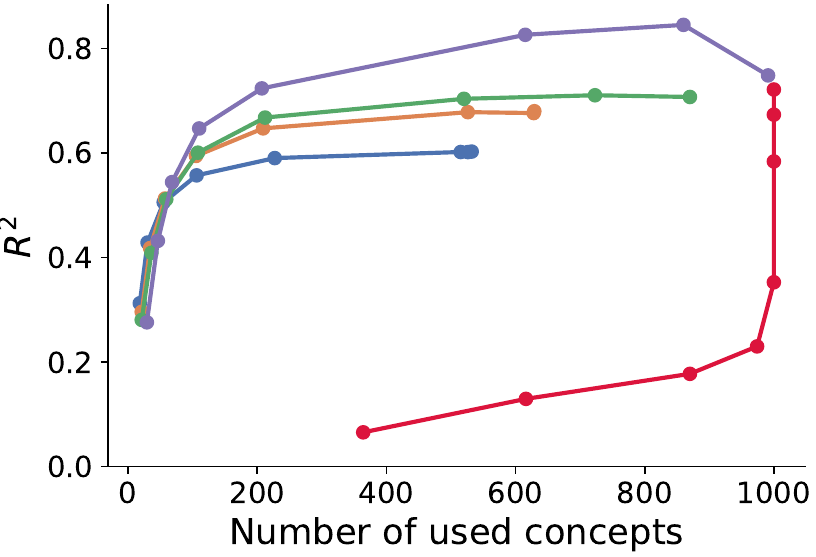}
        \caption{$R^2$ vs. used concepts}
    \end{subfigure}
    \caption{\textbf{Expressivity across concept sparsity for SAE-CNN kernel sizes.} $R^2$ on CIFAR $32 \times 32$ patches as a function of (a) concept sparsity $k$ and (b) number of used concepts. See \Cref{fig:utilization}(g-h) for SAE-FNO and SAE-MLP.}
    \label{fig:cifar-expressivity-cnn}
\end{figure}

\begin{figure}[ht]
    \centering
    \begin{subfigure}{0.45\linewidth}
    \includegraphics[width=\linewidth]{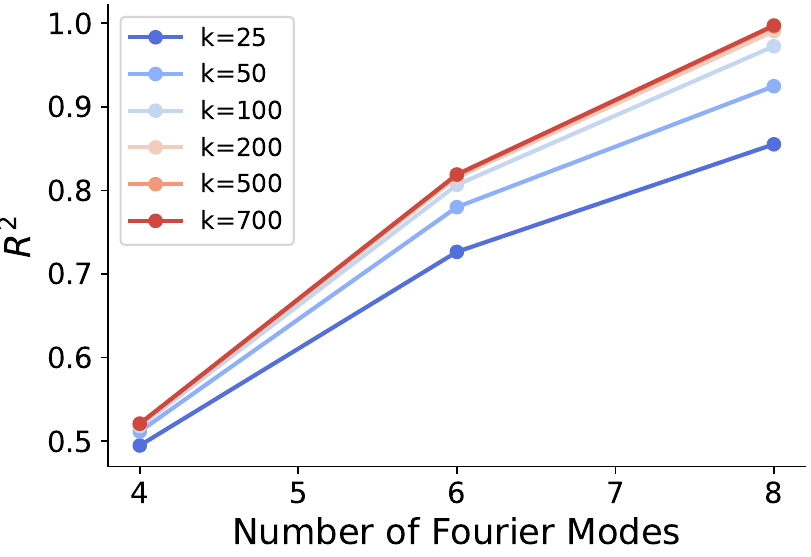}
    \caption{}
    \label{fig:fouriermode_ablation}
    \end{subfigure}
    \hfill
    \begin{subfigure}{0.45\linewidth}
    \includegraphics[width=\linewidth]{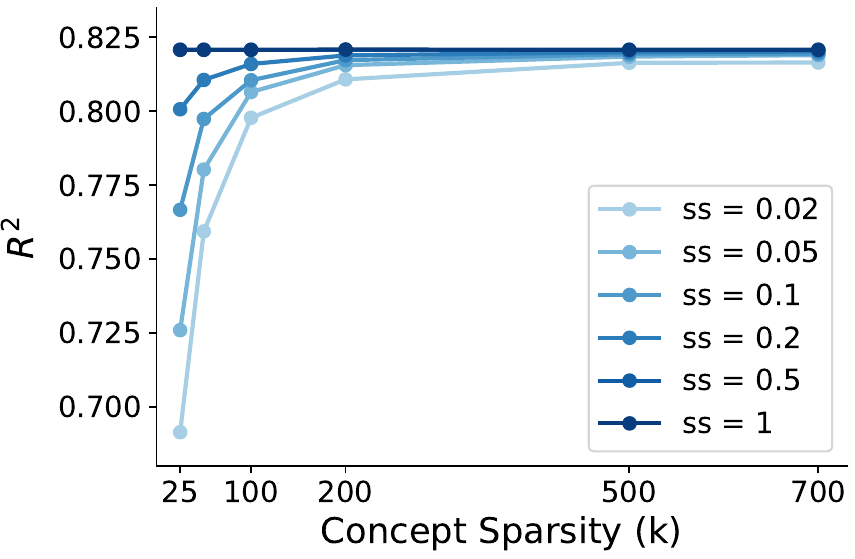}
     \label{fig:expressivity-spatial-sparsity}
     \caption{}
    \end{subfigure}
    \caption{\textbf{Expressivity of SAE-FNO as a function of Fourier modes and domain sparsity.} (a) The $R^2$ of SAE-FNO on CIFAR patches steadily increases with the number of Fourier modes. The frequency bandwidth directly controls the resolution and detail of the learned concepts. (b) SAE-FNO (6 modes) on CIFAR patches. Domain sparsity matters when $k$ is small, where raising domain sparsity compensates for fewer active concepts. For large $k$, the curves collapse and $R^2$ values stay within 0.005.}
    
\end{figure}

\newpage
\section{Appendix - Discretization Generalization}\label{app:disgeneralization}

\begin{figure}[H]
\vspace{-3mm}
    \centering 
    \begin{subfigure}{0.5\linewidth}
        \centering
        \includegraphics[width=1.0\linewidth]{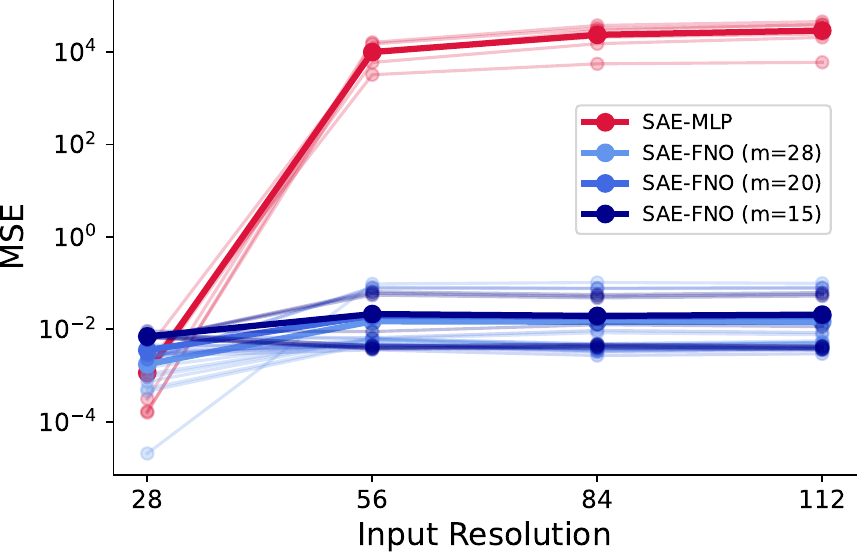}
        \vspace{-4.8mm}
    \end{subfigure}
    \caption{\textbf{SAE-FNO vs. SAE as a function of input discretization}. Models trained at $28\times 28$ resolution are evaluated at higher resolutions (upsampled images). Reconstruction MSE is shown. Transparent lines represent individual runs, and solid lines represent their average. SAE-FNOs successfully infer the underlying representations that allow them to reconstruct data across resolutions.}
    \label{fig:multires}  
\end{figure}

\begin{figure}[H]
    \centering
    \begin{subfigure}{0.47\textwidth}
        \includegraphics[width=0.48\linewidth]{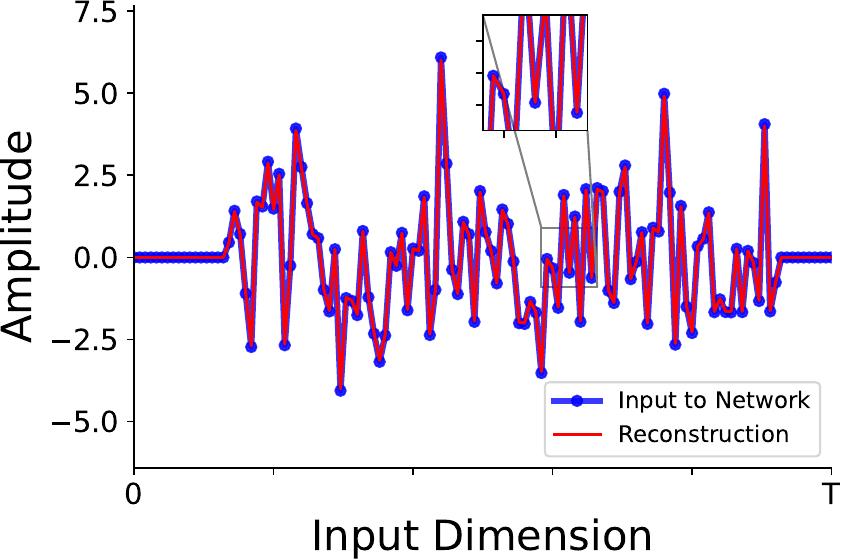}
        \hfill
        \includegraphics[width=0.47\linewidth]{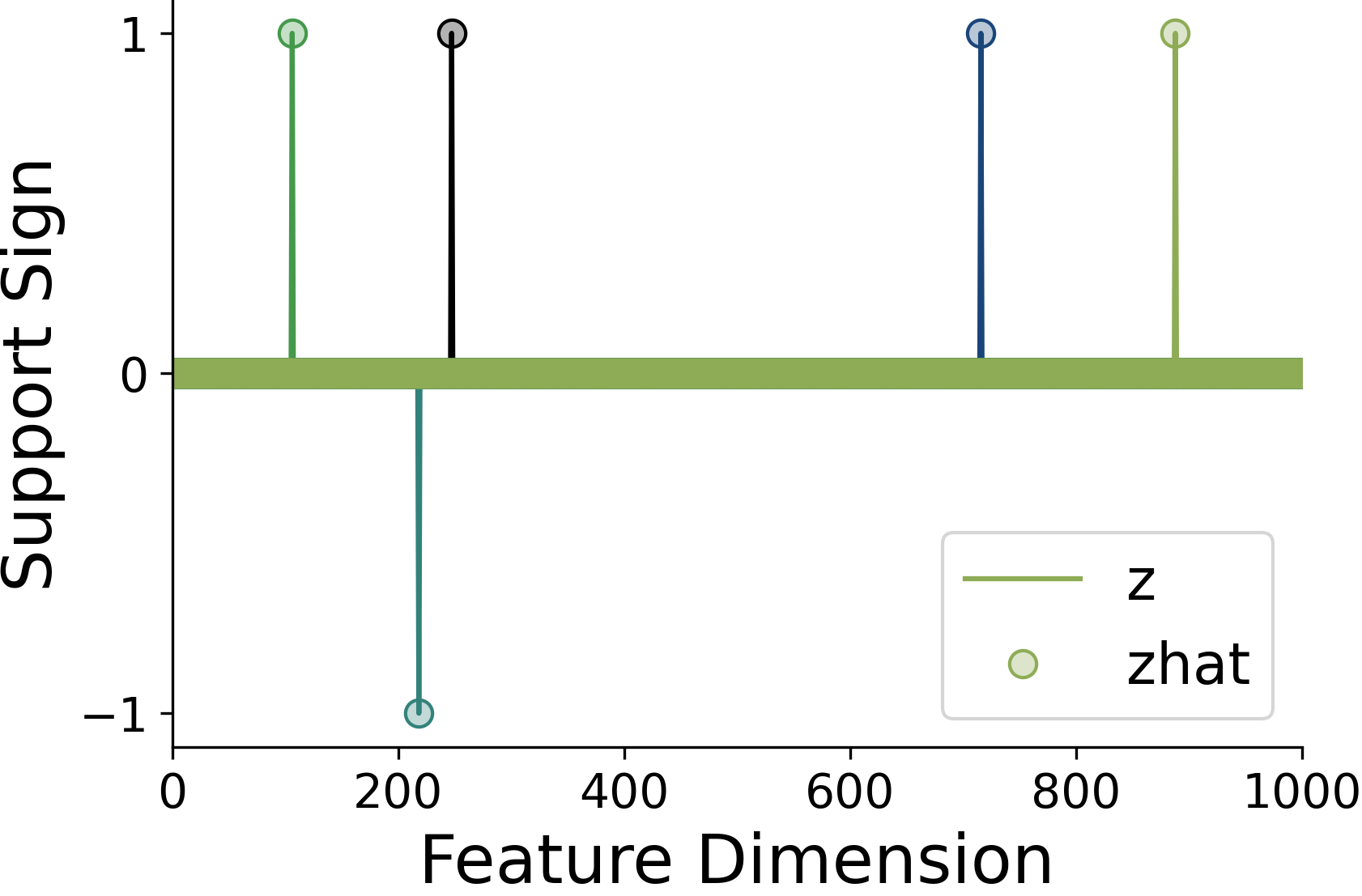}
        \caption{SAE-FNO (Original rate)}
    \end{subfigure}
    \hfill
    \begin{subfigure}{0.47\textwidth}
        \includegraphics[width=0.48\linewidth]{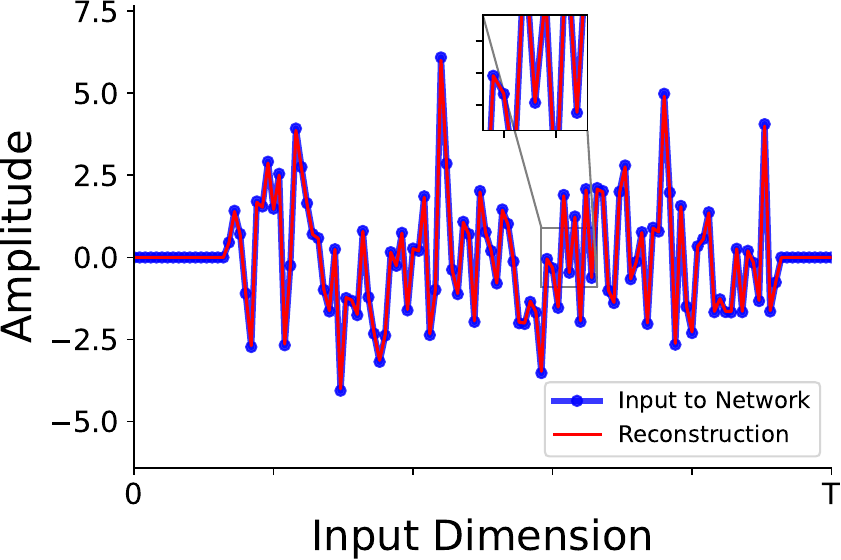}
        \hfill
        \includegraphics[width=0.48\linewidth]{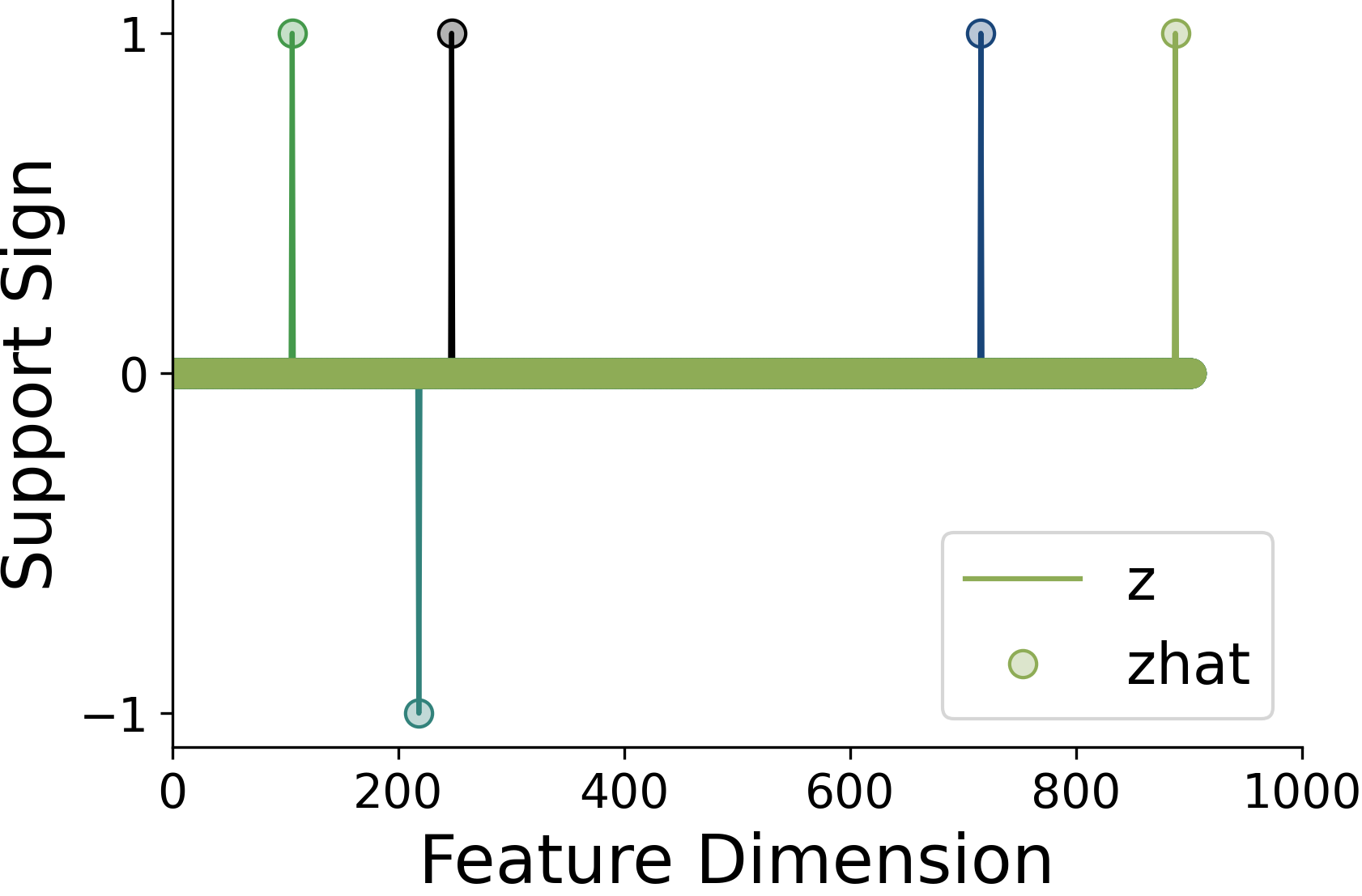}
        \caption{SAE-CNN (Original rate)}
    \end{subfigure}
    \hfill
    \begin{subfigure}{0.48\textwidth}
        \includegraphics[width=0.48\linewidth]{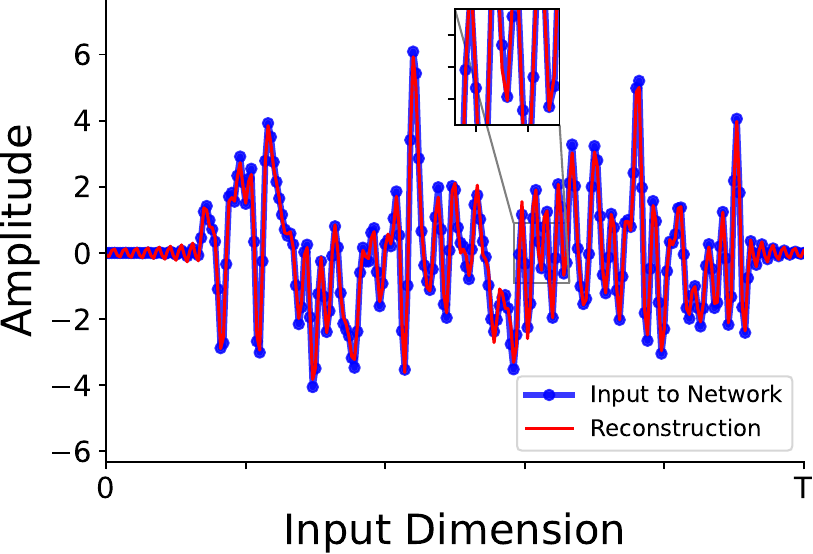}
        \hfill
        \includegraphics[width=0.48\linewidth]{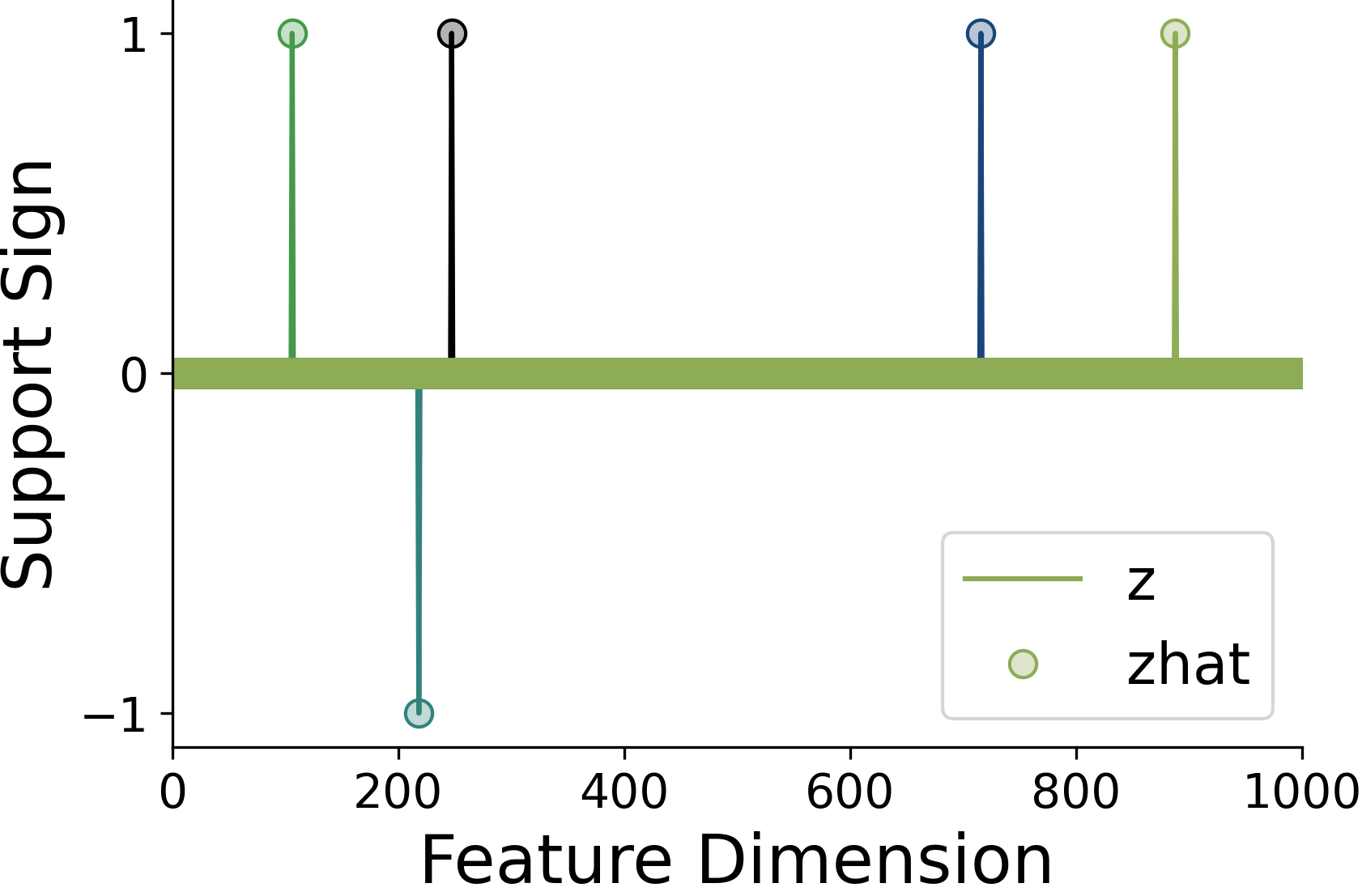}
        \caption{SAE-FNO ($2\times$ Upsampling rate)}
    \end{subfigure}
    \hfill
    \begin{subfigure}{0.48\textwidth}
        \includegraphics[width=0.48\linewidth]{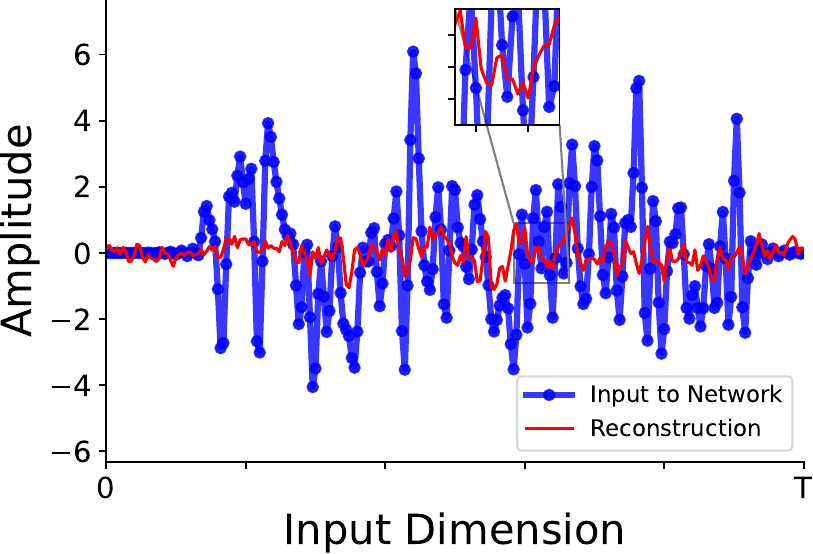}
        \hfill
        \includegraphics[width=0.48\linewidth]{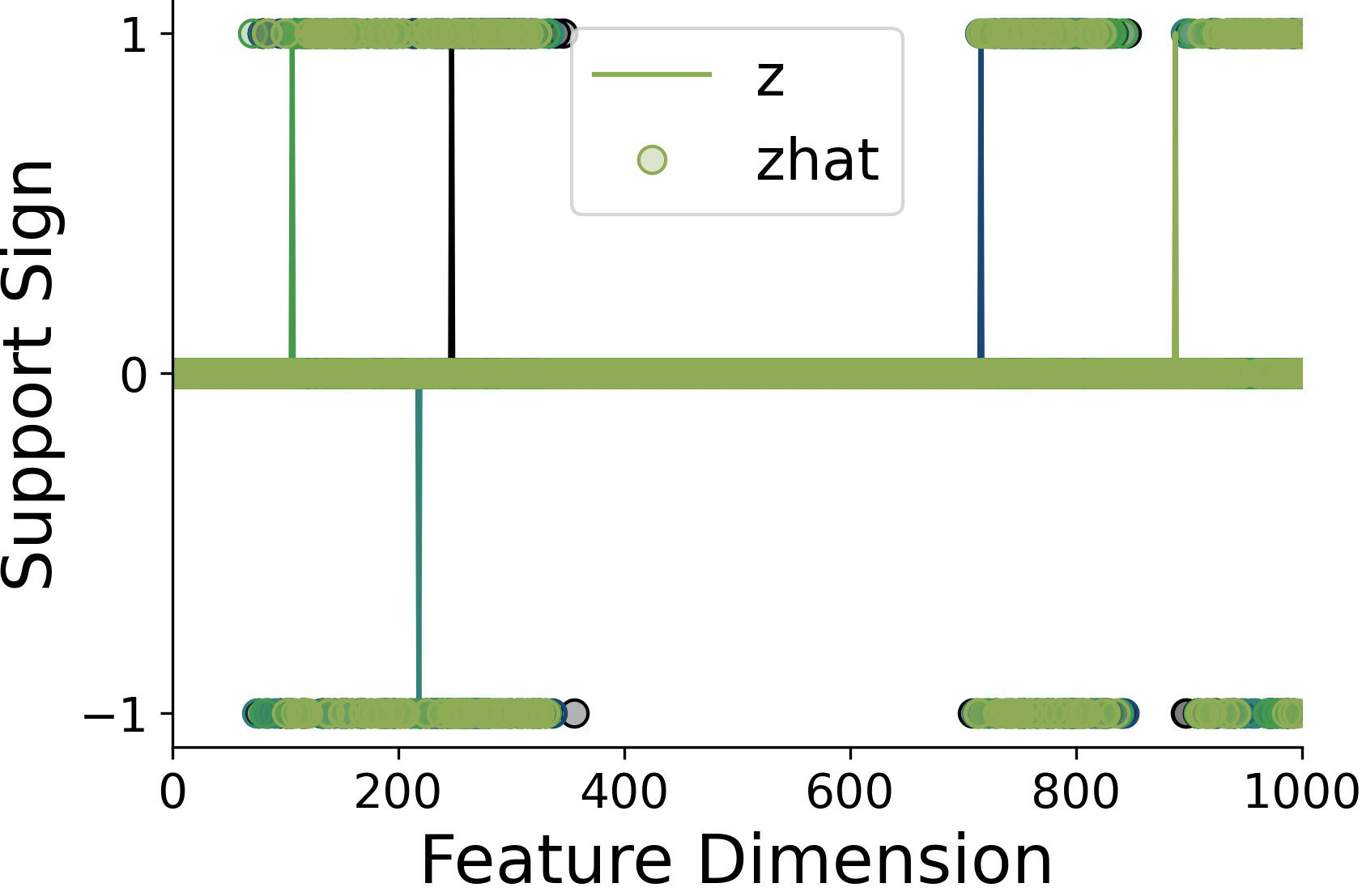}
        \caption{SAE-CNN ($2\times$ Upsampling rate)}
    \end{subfigure}
    \hfill
    \begin{subfigure}{0.48\textwidth}
        \includegraphics[width=0.48\linewidth]{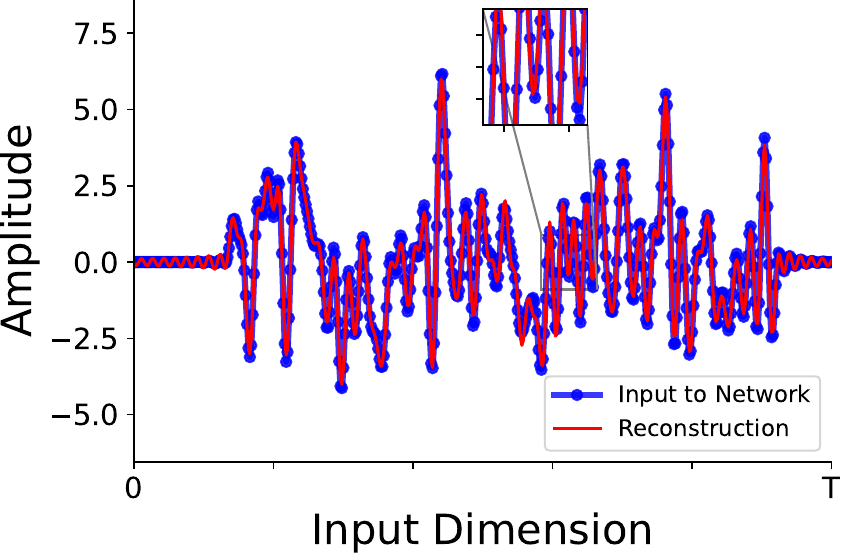}
        \hfill
        \includegraphics[width=0.48\linewidth]{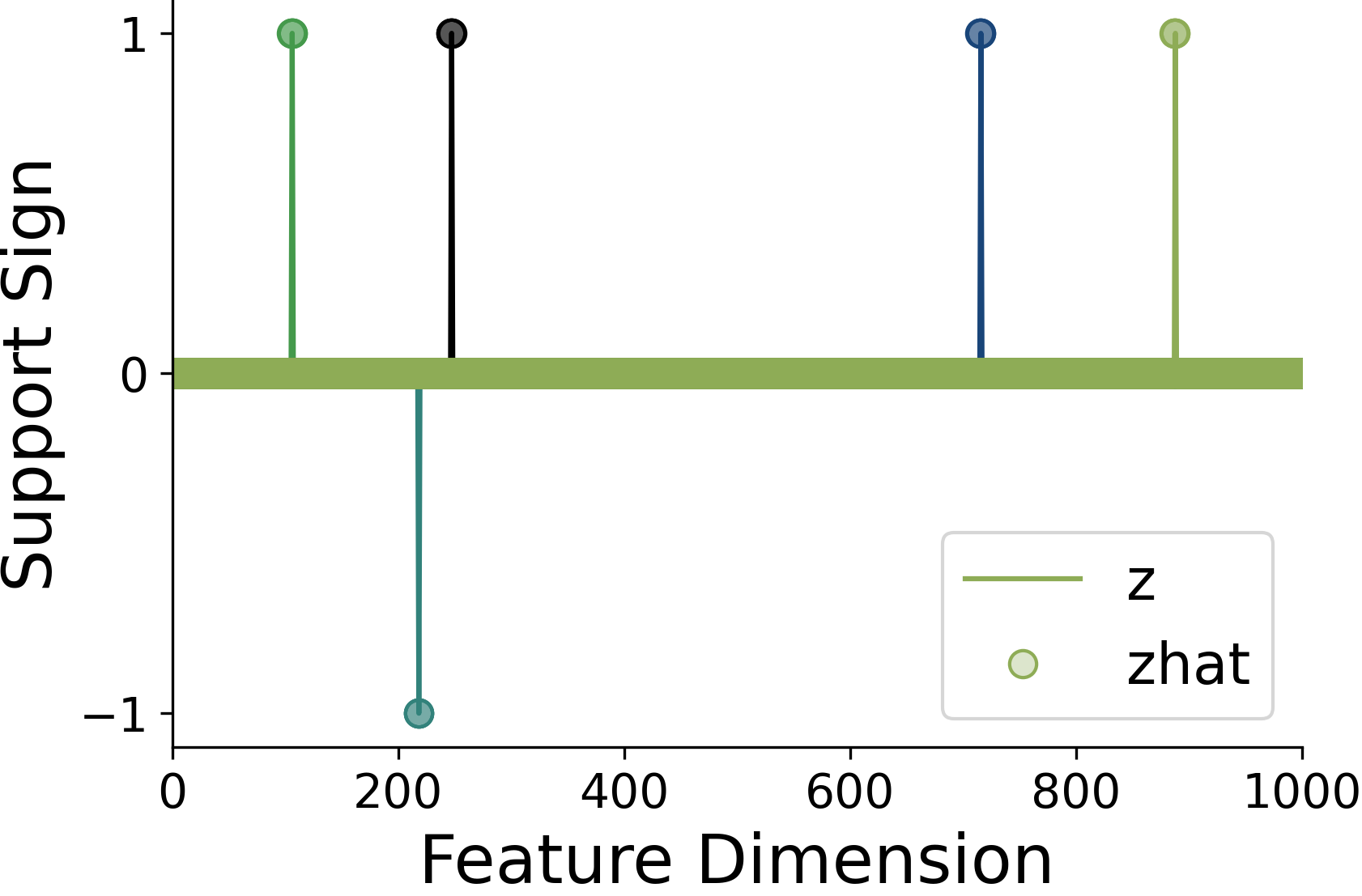}
        \caption{SAE-FNO ($4\times$ Upsampling rate)}
    \end{subfigure}
    \hfill
    \begin{subfigure}{0.48\textwidth}
        \includegraphics[width=0.48\linewidth]{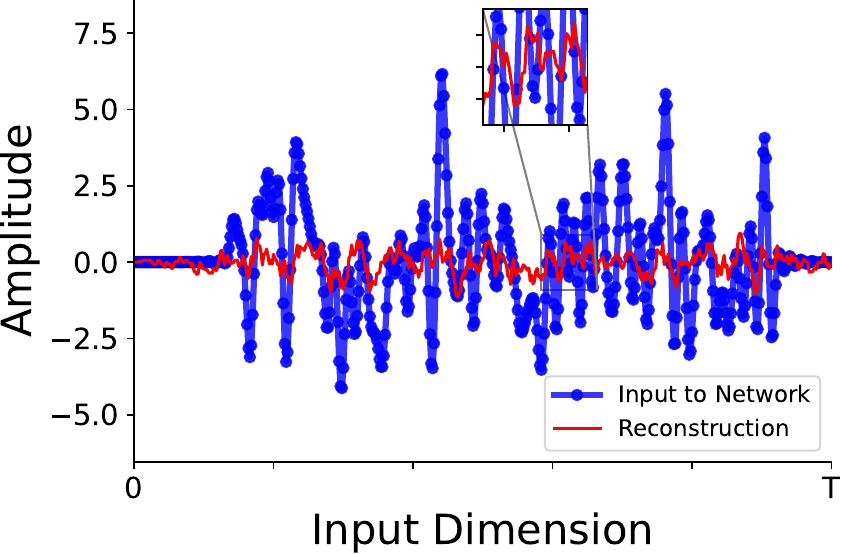}
        \hfill
        \includegraphics[width=0.48\linewidth]{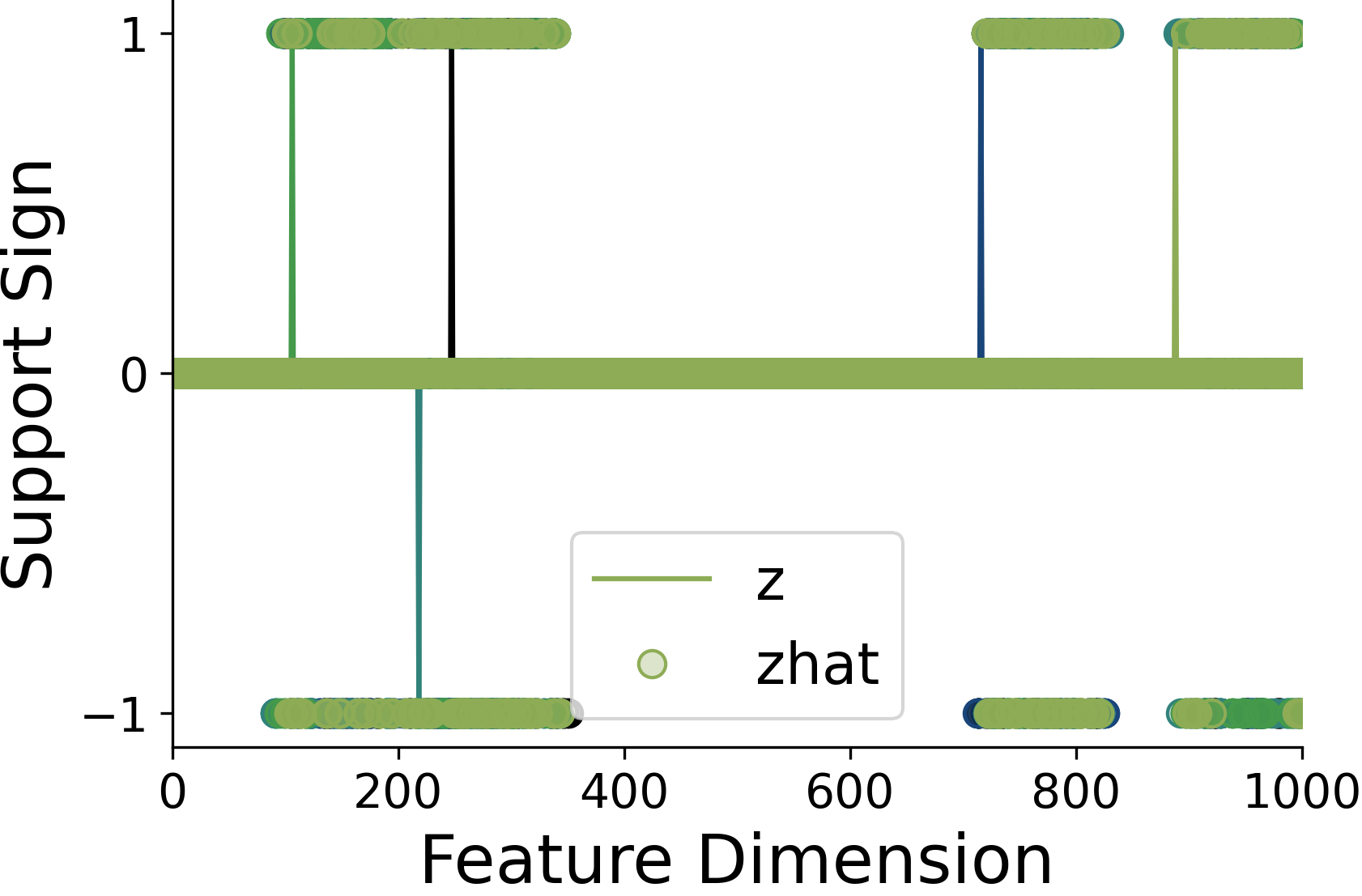}
        \caption{SAE-CNN ($4\times$ Upsampling rate)}
    \end{subfigure}
    \hfill
    
    \caption{\textbf{Discretization Generalization of SAE-FNO.} The underlying data structure has five convolutional concepts along with a $5$-sparse code (i.e., $1$-sparse representation for each concept). Each subplot shows (left) the input data and its reconstruction, and (right) the support sign (active sparse codes). We evaluate the ability of SAE-FNO (left column) and SAE-CNN (right column) to infer sparse representations and reconstruct data across varying discretization levels. Each row represents a different upsampling rate ($1 \times$ to $4 \times$ applied to the input signal). For each panel, the left plot shows the spatial-domain signal reconstruction (red) against the input to the network (blue), and the right plot shows the inference of 1-sparse code supports across 5 kernels. (a-b) Original resolution ($1\times$): Baseline performance at training resolution. Both models accurately reconstruct the signal and recover the correct support. (c-f) Higher upsampling rates ($2\times$, $4\times$): As spatial resolution increases, SAE-FNO maintains accurate reconstruction and code support inference. In contrast, SAE-CNN fails to generalize; its support inference collapses, leading to poor reconstruction. SAE-FNO's function-space parameterizations enables this generalization across discretizations.}
    \label{fig:simulated-fno-discretization-invariance}
\end{figure}

\begin{table}[H]
    \centering
    \caption{\textbf{Effect of discretization on kernel recovery.} Cosine similarity to the ground-truth
    kernel under decimation, with and without an anti-aliasing low-pass filter applied before decimation
    (setting of \Cref{fig:synthetic_c}, 5 seeds, initial cosine similarity $0.671 \pm 0.079$). Recovery degrades under decimation only in the unfiltered case. With aliasing removed, recovery does not degrade at all, identifying the cause as sampling rather than a limitation of the SAE-FNO
    parameterization.}
    \label{tab:discretization_decimation}
    \begin{tabular}{lcccc}
        \toprule
        & 1$\times$ (length 1000) & 2$\times$ (500) & 4$\times$ (250) & 5$\times$ (200) \\ \midrule
        No LPF   & $0.975 \pm 0.03$ & $0.952 \pm 0.04$ & $0.910 \pm 0.05$ & $0.845 \pm 0.07$ \\
        With LPF & $0.976 \pm 0.03$ & $0.981 \pm 0.02$ & $0.985 \pm 0.01$ & $0.987 \pm 0.01$ \\
        \bottomrule
    \end{tabular}
\end{table}

\section{Appendix - Inductive bias of SAE-FNO}

\subsection{Robustness to High-Frequency Perturbation}

\begin{figure}[H]
    \centering
    \begin{subfigure}{0.24\textwidth}
        \includegraphics[width=\linewidth]{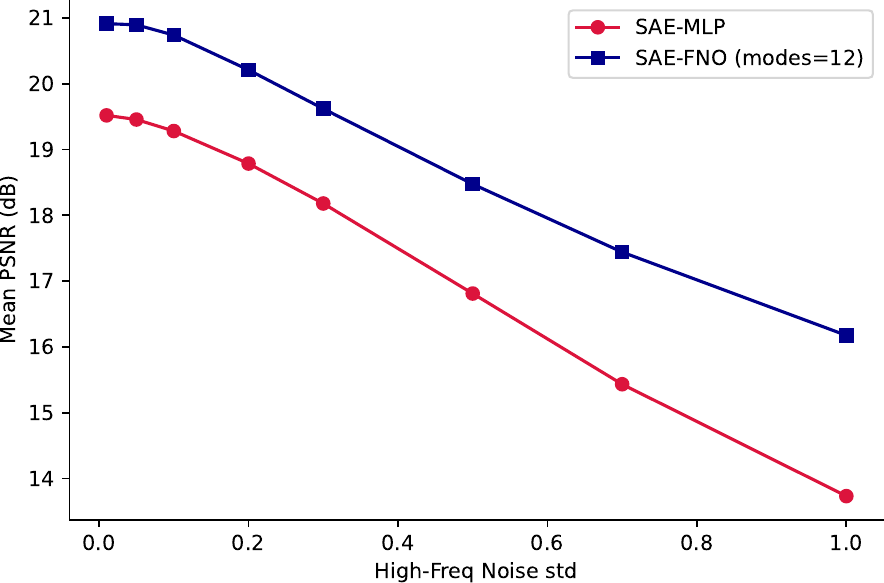}
        \caption{$k = 25$.}
    \end{subfigure}
    \begin{subfigure}{0.24\textwidth}
        \includegraphics[width=\linewidth]{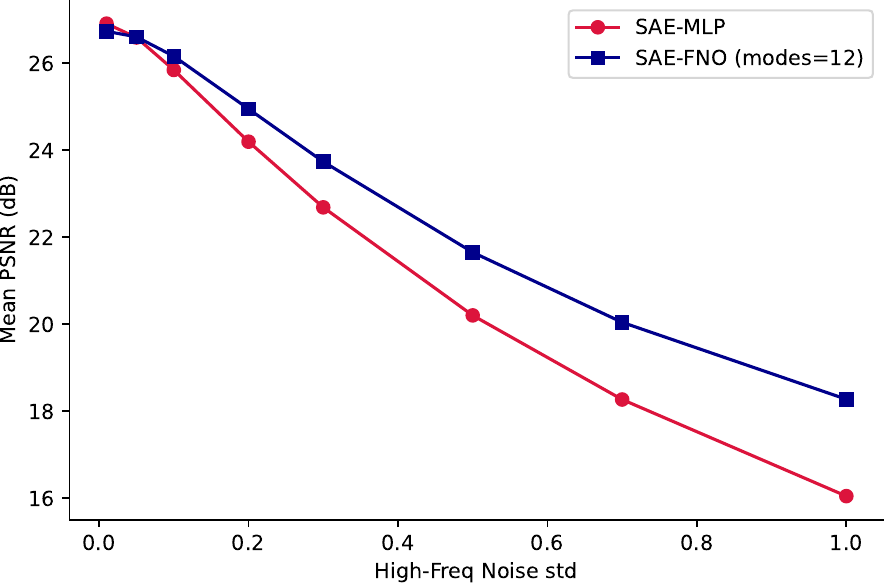}
        \caption{$k = 100$.}
    \end{subfigure}
    \begin{subfigure}{0.24\textwidth}
        \includegraphics[width=\linewidth]{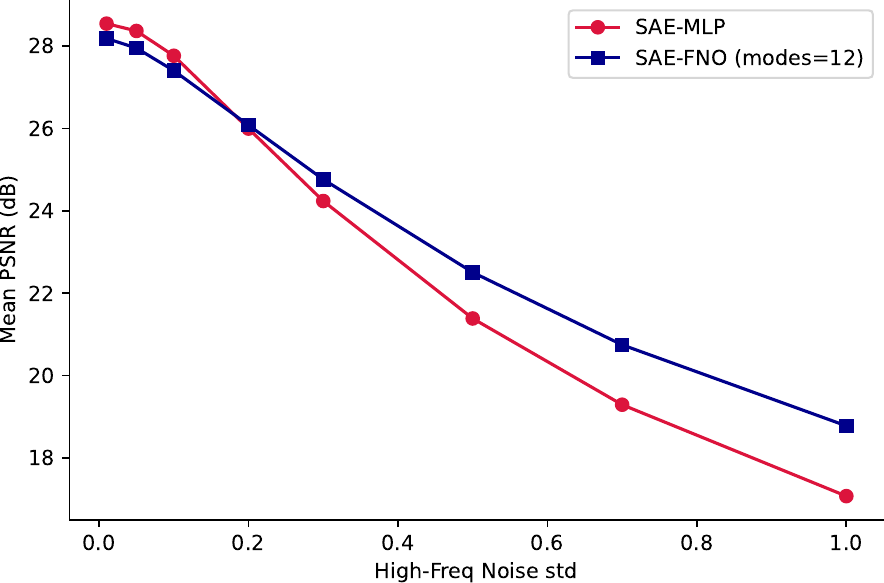}
        \caption{$k = 200$.}
    \end{subfigure}
    \begin{subfigure}{0.24\textwidth}
        \includegraphics[width=\linewidth]{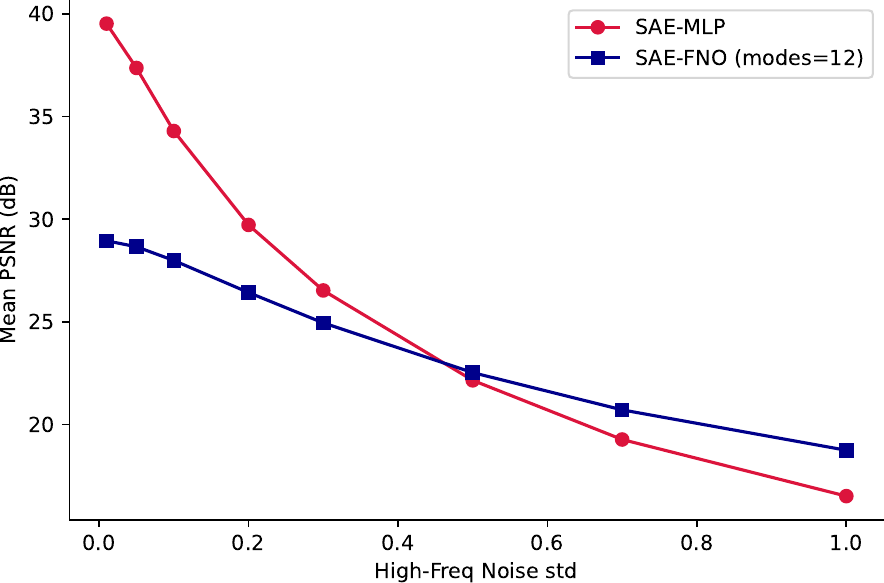}
        \caption{$k = 500$.}
    \end{subfigure}
    \caption{\textbf{High-frequency noise robustness.} PSNR as a function of high-frequency noise strength on ImageNet $16 \times 16$ patches, with $p=1000$ and varying concept sparsity. SAE-FNO shows smaller degradation under distribution shift compared to SAE-MLP, indicating greater robustness.}
    \label{fig:imagenet-highfreq-noise}
\end{figure}

\subsection{Pixel Permutation}\label{app:perm}

Under truncated modes, SAE-FNO faces a hard reconstruction ceiling that cannot be resolved by increasing concept sparsity $k$, while SAE-MLP is essentially unaffected (\Cref{tab:perm-exp1}). This ceiling is set by the fraction of spectra energy that falls inside the retained band, and it drops under permutation because the spectrum flattens. With all modes retained the ceiling disappears, leaving only an efficiency cost that scales with domain size. On $8 \times 8$ patches, data is fully represented by $k=200$ (\Cref{tab:perm-exp2-patch}), while on $32 \times 32$ images a $\sim 6\%$ relative gap in $R^2$ persists even at $k=1000$ (\Cref{tab:perm-exp2-full}). This separates two costs of violating the FRH: loss of smoothness imposes a hard ceiling under mode truncation, while loss of translation structure imposes a cost in concept budget that grows with domain size, since larger domains offer more positions at which a concept could have been reused.

\begin{table}[H]
\centering
\caption{\textbf{Reconstruction under pixel permutation.} Reporting $R^2$.}
\begin{subtable}{\linewidth}
\centering
\caption{Truncated-mode SAE-FNO (6 modes, domain sparsity 0.05), CIFAR $8 \times 8$ patches.}
\label{tab:perm-exp1}
\begin{tabular}{lcccccc}
\toprule
$k$ & MLP (clean) & FNO (clean) & MLP (perm) & FNO (perm) & MLP $\Delta$ & FNO $\Delta$ \\
\midrule
10   & 0.5892 & 0.6173 & 0.5899 & 0.4904 & $+$0.0006 & $-$0.1269 \\
25   & 0.8035 & 0.7264 & 0.8043 & 0.5756 & $+$0.0008 & $-$0.1508 \\
50   & 0.9265 & 0.7799 & 0.9267 & 0.6459 & $+$0.0002 & $-$0.1341 \\
100  & 0.9832 & 0.8065 & 0.9832 & 0.6624 & $+$0.0001 & $-$0.1441 \\
200  & 0.9988 & 0.8157 & 0.9985 & 0.6705 & $-$0.0003 & $-$0.1452 \\
500  & 0.9860 & 0.8185 & 0.9944 & 0.6725 & $+$0.0084 & $-$0.1460 \\
700  & 0.9957 & 0.8190 & 0.9948 & 0.6728 & $-$0.0009 & $-$0.1461 \\
1000 & 0.9885 & 0.8192 & 0.9873 & 0.6732 & $-$0.0012 & $-$0.1460 \\
\bottomrule
\end{tabular}
\end{subtable}
\begin{subtable}{\linewidth}
\centering
\caption{Full-mode SAE-FNO (8 modes, domain sparsity 0.05), CIFAR $8 \times 8$ patches.}
\label{tab:perm-exp2-patch}
\begin{tabular}{lcccccc}
\toprule
$k$ & MLP (clean) & FNO (clean) & MLP (perm) & FNO (perm) & MLP $\Delta$ & FNO $\Delta$ \\
\midrule
10   & 0.5892 & 0.6227 & 0.5899 & 0.5265 & $+$0.0006 & $-$0.0962 \\
25   & 0.8035 & 0.8335 & 0.8043 & 0.7687 & $+$0.0008 & $-$0.0648 \\
50   & 0.9265 & 0.9220 & 0.9267 & 0.8944 & $+$0.0002 & $-$0.0276 \\
100  & 0.9832 & 0.9681 & 0.9832 & 0.9581 & $+$0.0001 & $-$0.0100 \\
200  & 0.9988 & 0.9867 & 0.9985 & 0.9856 & $-$0.0003 & $-$0.0011 \\
500  & 0.9860 & 0.9933 & 0.9944 & 0.9927 & $+$0.0084 & $-$0.0006 \\
700  & 0.9957 & 0.9948 & 0.9948 & 0.9935 & $-$0.0009 & $-$0.0013 \\
1000 & 0.9885 & 0.9951 & 0.9873 & 0.9950 & $-$0.0012 & $-$0.0001 \\
\bottomrule
\end{tabular}
\end{subtable}
\begin{subtable}{\linewidth}
\centering
\caption{Full-mode SAE-FNO (32 modes, domain sparsity 0.005), CIFAR $32 \times 32$ full image.}
\label{tab:perm-exp2-full}
\begin{tabular}{lcccccc}
\toprule
$k$ & MLP (clean) & FNO (clean) & MLP (perm) & FNO (perm) & MLP $\Delta$ & FNO $\Delta$ \\
\midrule
10   & 0.0654 & 0.1916 & 0.0641 & 0.0740 & $-$0.0013 & $-$0.1176 \\
25   & 0.1294 & 0.3072 & 0.1311 & 0.1567 & $+$0.0017 & $-$0.1505 \\
50   & 0.1773 & 0.4084 & 0.1786 & 0.2554 & $+$0.0013 & $-$0.1531 \\
100  & 0.2297 & 0.5078 & 0.2295 & 0.3727 & $-$0.0002 & $-$0.1351 \\
200  & 0.3525 & 0.6219 & 0.3518 & 0.4901 & $-$0.0006 & $-$0.1319 \\
500  & 0.5836 & 0.7721 & 0.5825 & 0.6870 & $-$0.0012 & $-$0.0851 \\
700  & 0.6732 & 0.8184 & 0.6720 & 0.7507 & $-$0.0012 & $-$0.0677 \\
1000 & 0.7212 & 0.8544 & 0.7220 & 0.8054 & $+$0.0007 & $-$0.0490 \\
\bottomrule
\end{tabular}
\end{subtable}
\end{table}

\newpage

\subsection{Ablation on Sparsity and Parametrization}

\begin{figure}[H]
    \centering
    \begin{subfigure}{0.23\textwidth}
        \includegraphics[width=\linewidth]{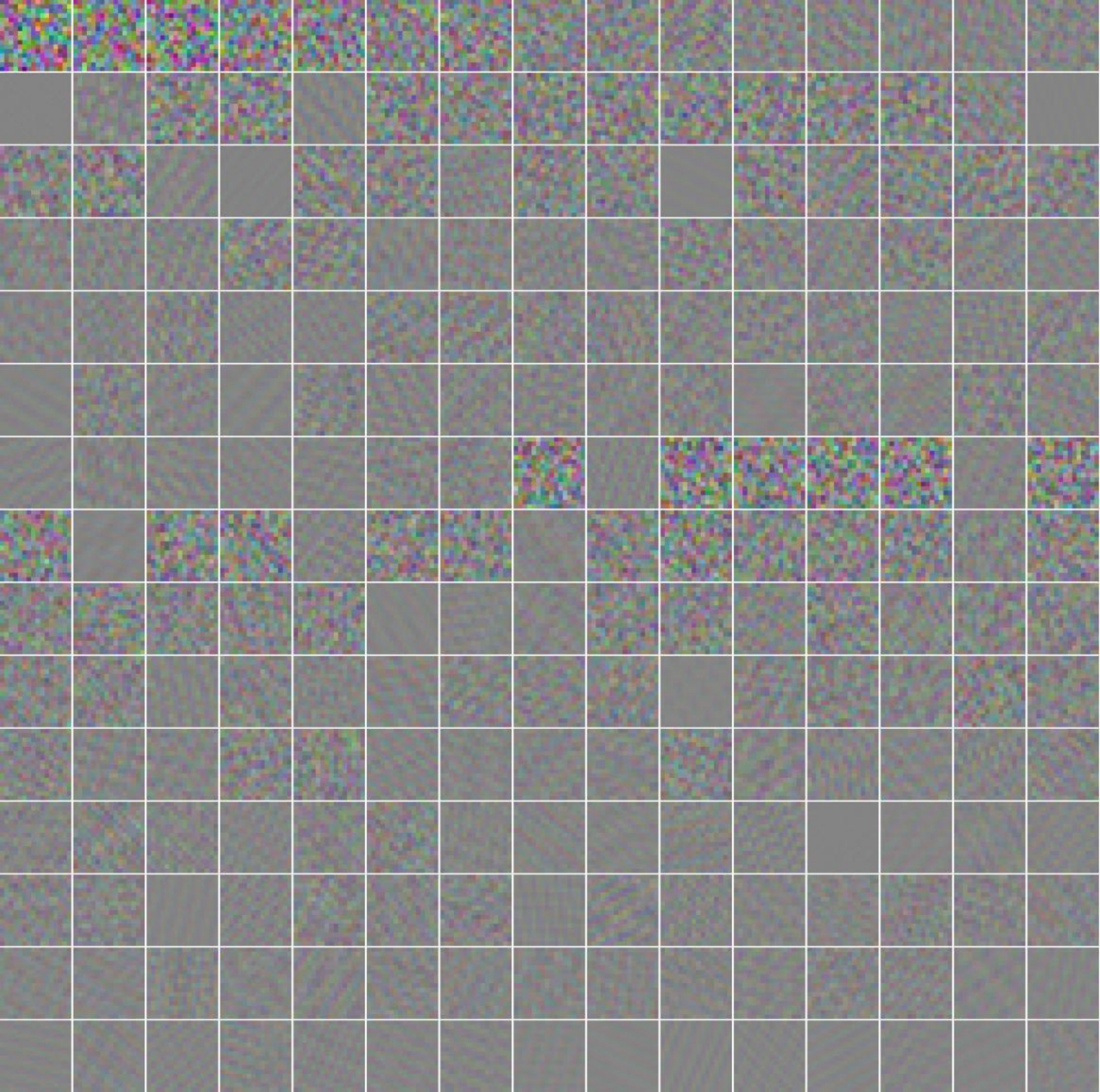}
        \caption{AE-FNO}
    \end{subfigure}
    \hfill
    \begin{subfigure}{0.72\textwidth}
        \includegraphics[width=0.32\linewidth]{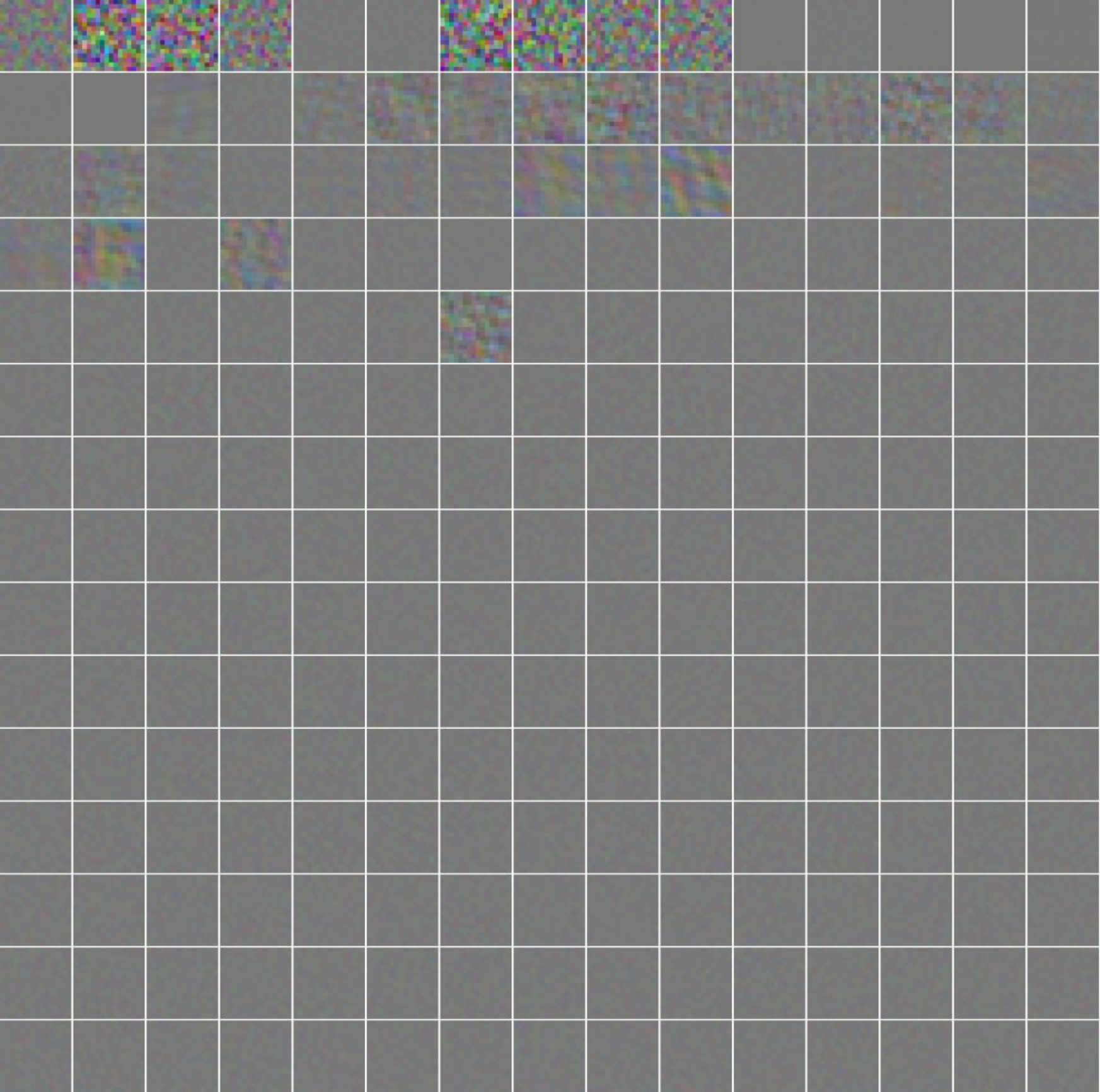}
        \includegraphics[width=0.32\linewidth]{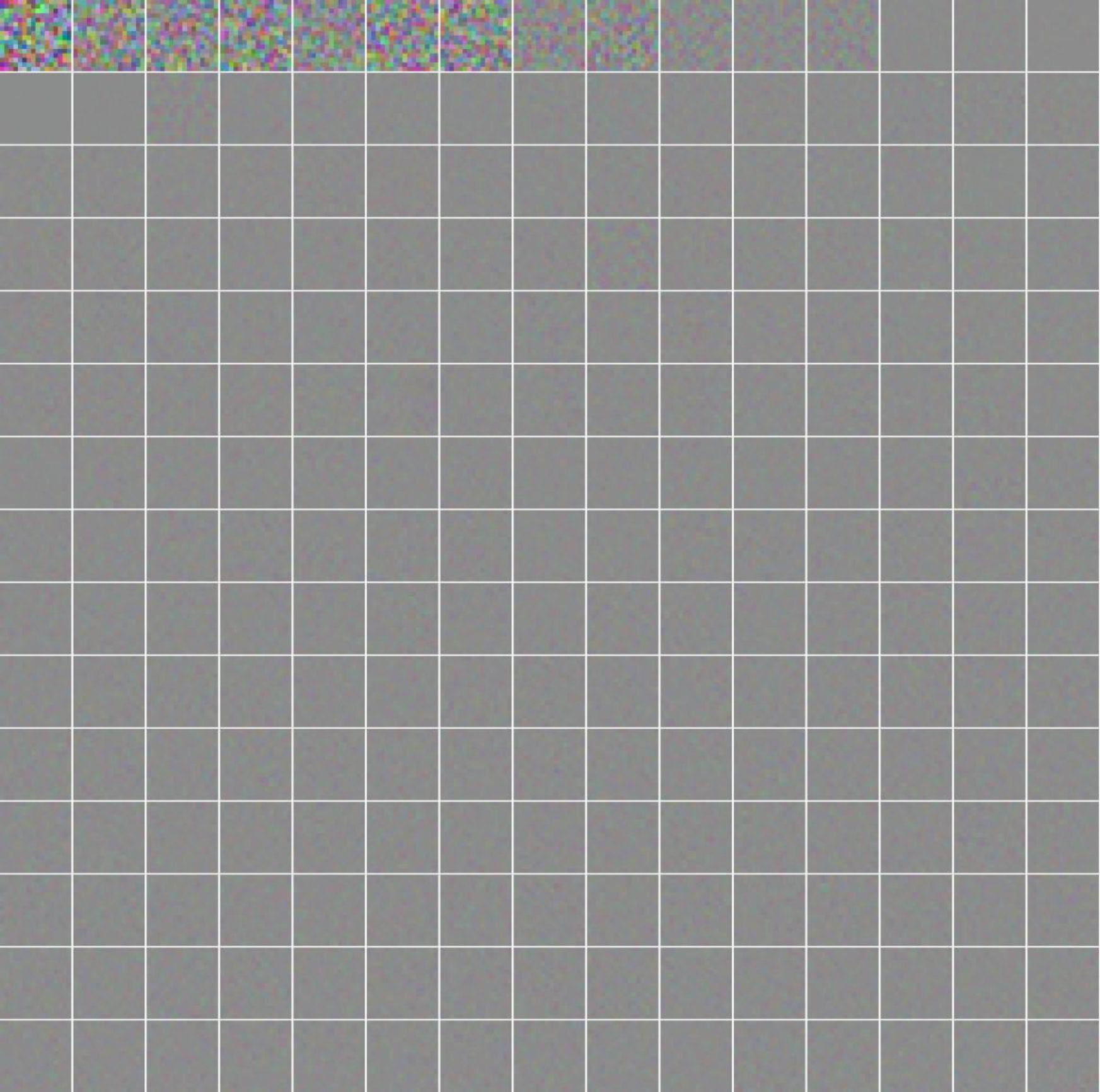}
        \includegraphics[width=0.32\linewidth]{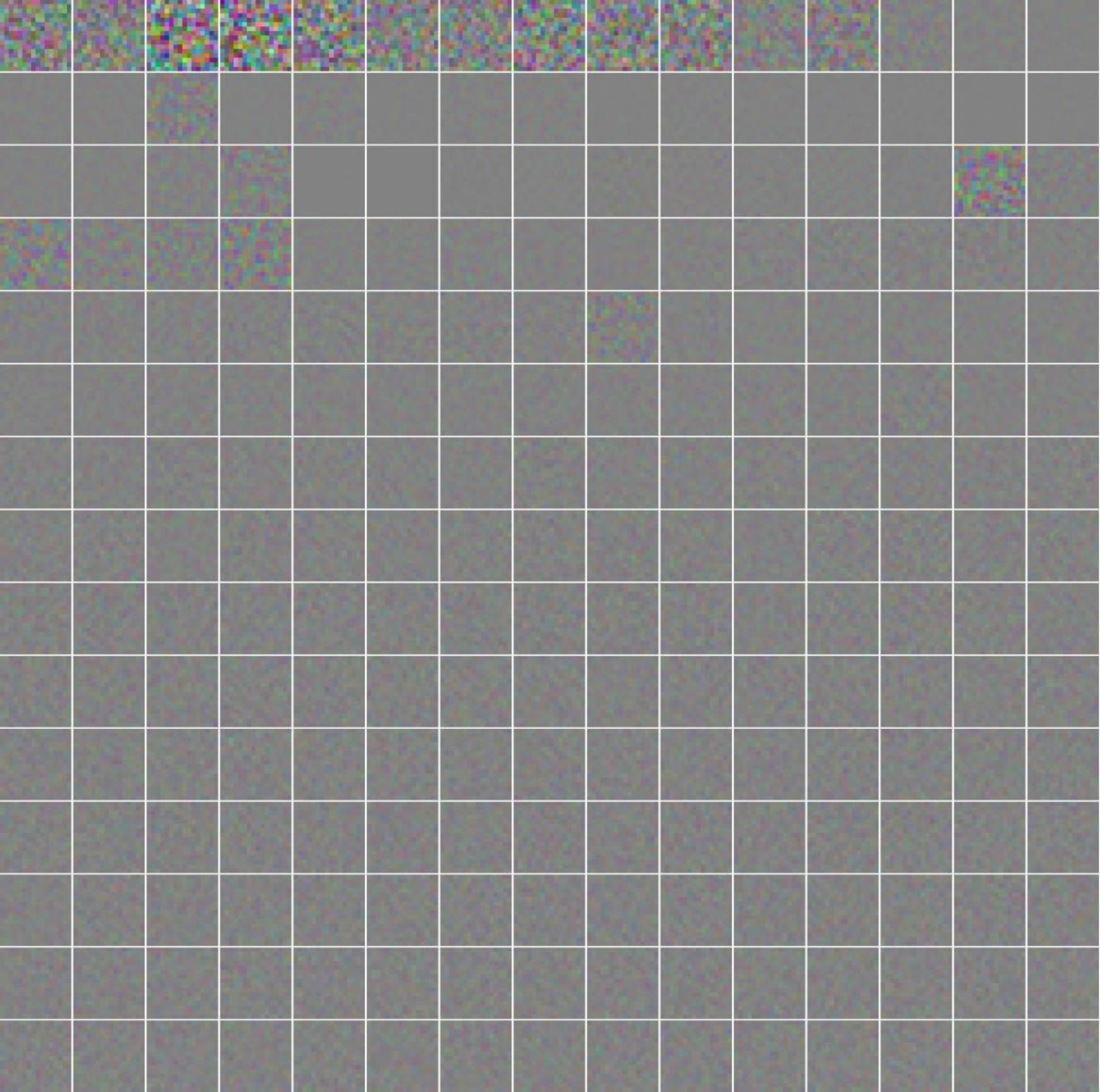}
        \caption{SAE-FNO-TopK w/o domain sparsity.}
    \end{subfigure}
    \caption{\textbf{Ablation of concept and domain sparsity.} (a) Concepts learned by AE-FNO without either sparsification mechanism. (b) Concepts learned by SAE-FNO with concept sparsity but without domain sparsity (from left to right: $k = 25, 50, 200$). Both fail to produce the localized, structured features observed with the full SAE-FNO. This confirms that the observed gains in mechanistic interpretability are specifically from the combination of functional sparse coding and domain sparsity, rather than the FNO architecture alone.}
    \label{fig:aefno-saefno-no-ss}
\end{figure}

\subsection{Frequency-oriented Concepts}\label{app:rec_freqoriented}

\begin{table}[H]
    \centering
    \caption{\textbf{Useful inductive bias of SAE-FNO in recovering frequency-oriented concepts.} Synthetic experiment recovering 10 single-frequency concepts ($k=3$, domain sparsity 2). Values represent cross-correlation recovery scores, reported as mean (standard deviation) across 5 runs.}
    \label{tab:synthetic-recovery}
    \begin{tabular}{lcccc}
        \toprule
         & Init & SAE-MLP & SAE-FNO ($m=$ full) & SAE-FNO ($m=151$) \\
        \midrule
        Clean & 0.329 {\footnotesize(0.210)} & 0.868 {\footnotesize(0.303)} & 0.863 {\footnotesize(0.324)} & 0.882 {\footnotesize(0.612)} \\
        Noisy & 0.329 {\footnotesize(0.210)} & 0.798 {\footnotesize(0.534)} & 0.853 {\footnotesize(0.627)} & 0.890 {\footnotesize(0.658)} \\
        \bottomrule
    \end{tabular}
\end{table}

\section{Appendix - Computational Cost}\label{app:compute_cost}

\begin{table}[H]
    \centering
    \caption{\textbf{Computational cost as a function of (a) data dimension and (b) model size $p$.} Table reports training and inference throughput (images/sec, L40S GPU, batch size 256, 32 workers).}
    \begin{subtable}{\linewidth}
    \centering
    \caption{As a function of data dimensionality ($k = 100, p = 1000$, FNO modes=half). Higher is better.}
    \label{tab:throughput_datadim}
    \begin{tabular}{llcc}
        \toprule
        Architecture & Data Dim & Train Throughput & Inference Throughput \\ \midrule
        SAE-MLP & $16 \times 16$ & 13,893 img/s & 9,680 img/s \\
        SAE-MLP & $32 \times 32$ & 10,551 img/s & 8,648 img/s \\
        SAE-FNO & $16 \times 16$ & 5,901 img/s & 10,169 img/s \\
        SAE-FNO & $32 \times 32$ & 1,515 img/s & 5,066 img/s \\ 
        \bottomrule
    \end{tabular}
    \end{subtable}
    \begin{subtable}{\linewidth}
    \centering
    \caption{As a function of model size $p$ ($k = 200$, $16 \times 16$ patches). Higher is better.}
    \label{tab:throughput_p}
    \begin{tabular}{llcc}
        \toprule
        Architecture & $p$ (Concepts) & Train Throughput & Inference Throughput \\ \midrule
        SAE-MLP & 1K & 17,847 img/s & 12,987 img/s \\
        SAE-FNO & 1K & 5,512 img/s & 12,763 img/s \\
        SAE-MLP & 5K & 17,744 img/s & 14,947 img/s \\
        SAE-FNO & 5K & 1,181 img/s & 4,450 img/s \\
        SAE-MLP & 10K & 11,308 img/s & 15,082 img/s \\
        SAE-FNO & 10K & 598 img/s & 2,374 img/s \\
        \bottomrule
    \end{tabular}
    \end{subtable}
    \label{tab:throughput}
\end{table}

\begin{table}[H]
    \centering
    \label{tab:compute_fno}
    \caption{\textbf{Sensitivity of computational cost to (a) the number of Fourier modes and (b) domain sparsity.} Training and inference throughput and memory (L40S GPU, batch size 256, 32 workers).}
    \begin{subtable}{\linewidth}
    \centering
    \label{tab:throughput_modes}
    \caption{As a function of the number of Fourier modes ($16 \times 16$ patches, $p = 1000$, $k = 100$, domain sparsity $0.02$).}
    \begin{tabular}{lcccc}
        \toprule
        Modes & Train Throughput & Train Memory & Inference Throughput & Inference Memory \\ \midrule
        4  & 6,411 img/s & 3.03 GB & 10,113 img/s & 1.43 GB \\
        8  & 5,901 img/s & 3.14 GB & 10,169 img/s & 1.44 GB \\
        12 & 5,254 img/s & 3.32 GB & 11,799 img/s & 1.45 GB \\
        14 & 4,943 img/s & 3.43 GB & 12,202 img/s & 1.45 GB \\
        16 & 4,839 img/s & 3.63 GB & 13,638 img/s & 1.46 GB \\
        \bottomrule
    \end{tabular}
     \end{subtable}
    \begin{subtable}{\linewidth}
    \centering
    \caption{As a function of domain sparsity (16 Fourier modes, $16 \times 16$ patches, $p = 1000$, $k = 100$).}
    \label{tab:throughput_domain_sparsity}
    \begin{tabular}{lcccc}
        \toprule
        Domain Sparsity & Train Throughput & Train Memory & Inference Throughput & Inference Memory \\ \midrule
        1     & 5,164 img/s & 3.34 GB & 12,281 img/s & 1.45 GB \\
        0.1   & 4,829 img/s & 3.63 GB & 13,986 img/s & 1.47 GB \\
        0.05  & 4,833 img/s & 3.63 GB & 13,084 img/s & 1.46 GB \\
        0.02  & 4,839 img/s & 3.63 GB & 13,638 img/s & 1.46 GB \\
        0.01  & 4,838 img/s & 3.63 GB & 14,089 img/s & 1.46 GB \\
        0.005 & 4,838 img/s & 3.63 GB & 12,849 img/s & 1.46 GB \\
        \bottomrule
    \end{tabular}
    \end{subtable}
\end{table}

\begin{table}[H]
    \centering
    \caption{\textbf{Computational cost at SAE-deployment scale.} Training and inference throughput and memory (L40S GPU, batch size 256, 32 workers) at a scale representative of current SAE deployments ($m = 2048$, matching Gemma-2-2B activation dimensionality~\cite{gemmateam2024gemmaopenmodelsbased}), across dictionary sizes $p$ and modes. Data is synthetic. SAE-FNO remains feasible at this scale, with throughput and memory scaling predictably in mode count.}
    \label{tab:throughput_deployment_scale}
    \begin{tabular}{llcccccc}
        \toprule
        Model & Modes & $p$ & Train (samp/s) & Train Mem (GB) & Infer (samp/s) & Infer Mem (GB) \\ \midrule
        MLP & --   & 10K & 3,667 & 0.48  & 21,609 & 0.32  \\
        FNO & 64   & 10K & 286   & 8.58  & 556    & 8.59  \\
        FNO & 2048 & 10K & 200   & 9.03  & 439    & 9.18  \\
        MLP & --   & 16K & 2,228 & 0.76  & 20,945 & 0.51  \\
        FNO & 64   & 16K & 181   & 13.72 & 353    & 13.73 \\
        FNO & 2048 & 16K & 126   & 14.43 & 279    & 14.67 \\
        MLP & --   & 24K & 1,540 & 1.12  & 17,055 & 0.76  \\
        FNO & 64   & 24K & 121   & 20.57 & 236    & 20.58 \\
        FNO & 2048 & 24K & 84    & 21.63 & 186    & 22.00 \\
        \bottomrule
    \end{tabular}
\end{table}